%% file: ML_BR_optimal.tex
\documentclass[3p]{elsarticle}
\usepackage{lineno}

\usepackage{hyperref}
\usepackage{caption}
\usepackage{float}
\usepackage{subfigure}
\usepackage{mathptmx}
\usepackage{arydshln}
\usepackage{pifont}
\usepackage{stmaryrd}
\usepackage{mathrsfs}
\usepackage{mathtools}
\usepackage{amssymb}
\usepackage{amsmath}
\usepackage{amsthm}
\usepackage{bbm}
\usepackage{bm}
\usepackage{comment}
\usepackage[ruled,linesnumbered]{algorithm2e}
\usepackage{multirow}
\usepackage{nicefrac}
\usepackage{tikz}
\usepackage{pgfplots}
\usetikzlibrary{cd}
\pgfplotsset{compat=newest}
\tikzset{
	treenode/.style = {shape=rectangle, rounded corners,
		draw, align=center,
		top color=white, bottom color=blue!20},
	root/.style     = {treenode, font=\Large, bottom color=red!30},
	env/.style      = {treenode, font=\ttfamily\normalsize},
	dummy/.style    = {circle,draw}
}
\hypersetup{
    colorlinks=true,
    linkcolor=blue,
    filecolor=magenta,      
    urlcolor=cyan,
    citecolor=blue
}
\usepackage[super]{nth}

\def\num#1{\numx#1}\def\numx#1e#2{{#1}\mathrm{e}{#2}}

\DeclareMathAlphabet{\mathpzc}{OT1}{pzc}{m}{it}

\newcommand{\spacelabel}{\mathscr{Y}}
\newcommand{\vect}[1]{{\bm{#1}}}

\newcommand{\hamone}[1]{\mathbbm{1}_{y_{#1} \neq y^1_{#1}}}
\newcommand{\hamtwo}[1]{\mathbbm{1}_{y_{#1} \neq y^2_{#1}}}
\newcommand{\hams}[1]{\sum_{i=1}^{m}\mathbbm{1}_{y_i \neq #1_i}}

\newcommand{\newinstance}{\vect{x}}
\newcommand{\setn}[1]{\llbracket #1 \rrbracket}

\makeatletter
\newcommand\HUGE{\@setfontsize\Huge{30}{60}}
\makeatother

\newtheorem{theorem}{Theorem}
\newtheorem{proposition}[theorem]{Proposition}
\newtheorem{corollary}[theorem]{Corollary}
\newtheorem{lemma}{Lemma}

\newtheorem{example}{Example}
\newtheorem{remark}{\bf Remark}
\newtheorem{definition}{Definition}
\theoremstyle{plain}
\usepackage{titlesec}
\titleformat{\paragraph}[runin]
	{\normalfont\bfseries}{\theparagraph}{1em}{}[.]
\newcommand{\xmark}{\ding{55}}%
\input{Notations.tex}

\begin{document}
\title{Skeptical binary inferences in multi-label problems with sets of probabilities\tnoteref{mytitlenote}}
\tnotetext[mytitlenote]{This paper significantly extends the published paper in 12th International Symposium on Imprecise Probabilities: Theories and Applications (ISIPTA).}

\author[]{Yonatan Carlos Carranza Alarc\'on}
\ead{ycarranza.alarcon@gmail.com}

\author[utcaddress]{S\'ebastien Destercke} 
\ead{sebastien.destercke@hds.utc.fr}

\address[utcaddress]{UMR CNRS 7253 Heudiasyc, Sorbonne universit\'es, Universit\'{e} de technologie de Compi\`egne CS 60319 - 60203 Compi\`egne cedex, France }

\begin{abstract}
In this paper, we consider the problem of making distributionally robust, skeptical inferences for the multi-label problem, or more generally for Boolean vectors. By distributionally robust, we mean that we consider a set of possible probability distributions, and by skeptical we understand that we consider as valid only those inferences that are true for every distribution within this set. Such inferences will provide partial predictions whenever the considered set is sufficiently big. We study in particular the Hamming loss case, a common loss function in multi-label problems, showing how skeptical inferences can be made in this setting. Our experimental results are organised in three sections; (1) the first one indicates the gain computational obtained from our theoretical results by using synthetical data sets, (2) the second one indicates that our approaches produce relevant cautiousness on those hard-to-predict instances where its precise counterpart fails, and (3) the last one demonstrates experimentally how our approach copes with imperfect information (generated by a downsampling procedure) better than the partial abstention~\cite{nguyen2019} and the rejection rules.

\end{abstract}
\begin{keyword}
    multilabel \sep cautious predictions \sep binary vectors \sep imprecise probabilities \sep downsampling 
\end{keyword}

\maketitle


\section{Introduction}\label{sec:introduction}
In contrast to multi-class problems where each instance is associated to one label, multi-label classification consists in associating an instance to a subset of relevant labels from a set of possible labels. Such problems can arise in different research fields, such as the classification of proteins in bioinformatics~\cite{tsoumakas2007multi}, text classification in information retrieval~\cite{furnkranz2008multilabel}, object
recognition in computer vision~\cite{boutell2004learning}, and so on.

Considering all possible subsets of relevant labels as possible predictions make the estimation and decision steps of a learning problem significantly more difficult: partial observations are more likely to occur, especially when the number of labels increase, and the output space over which the probability needs to be estimated grows exponentially with the number of labels.  This means that in some applications where guaranteeing the robustness and reliability of predictions is of particular importance, one may consider being cautious about such predictions, by predicting a set of possible answers rather than a single one when uncertainties are too high. In the literature, such strategies can be called partial rejection rules~\cite{pillai2013multi}, partial abstention~\cite{nguyen2019} or indeterminate classification~\cite{destercke2014multilabel,antonucci2017multilabel}. 

The main goal of this paper \footnote{This paper extends a shorter conference version~\cite{alarcon2021ipsita} by providing additional examples, details and experiments. It also includes significant original content; notably, the study of the imprecise binary relevance with other decision criteria in \ref{app:other_dec_crit}, an investigation of how our skeptical approach can address the problem with noisy labels in Section~\ref{sec:expnoisedatasets}, and finally, a first comparative investigation of how the partial abstention, rejection threshold and our skeptical approach  behave to downsampling in Section~\ref{sec:downsampling}.} is to exclusively study the problem of making such set-valued predictions by performing skeptic inferences when our uncertainty is described by a set of probabilities (the more uncertainties, the bigger the set). By skeptic inference, we understand the logical procedure that consists, in the presence of multiple models, to accept only those inferences that are true for every possible model. Such approaches are different from thresholding approaches~\cite{nguyen2019,pillai2013multi}, and are closer in spirit to distributionally robust approaches, even if these later typically consider precise, minimax inferences, that are cautious yet not skeptic~\cite{hu2018does,chen2018robust}. We also make no assumptions about the considered set of probabilities, thus departing from usual distributionally robust approaches, that typically consider precise predictions, or from existing works dealing with sets of probabilities and multi-label problems~\cite{antonucci2017multilabel}, that considered specific probability sets and zero/one loss function (which is seldom used in multi-label problems). More precisely, and before detailing the content of our paper, our main contributions are as follows:
\begin{enumerate}
	\item proposing theoretical procedures to reduce the complexity time of the skeptical inference step when we consider any  set of joint probability distributions and the Hamming loss;
	\item proposing to generalize the classical binary relevance by using imprecise marginal distributions;
	\item proposing additional implications on different type of decision criteria when we consider imprecise marginal distributions and the Hamming loss;
	\item providing experiments to compare our exact procedure against the existing outer-approximation;
	\item providing experiments to how our skeptical approach can address the problem with missing and noisy labels;
	\item providing a first comparative investigation of how the partial abstention, rejection and our skeptical approach can react under downsampling.
\end{enumerate}
 
Section \ref{sec:prel} introduces the notations we will use for the multi-label setting, and give the necessary reminders about skeptic inferences made with sets of probabilities. Once this is done, we provide in Section~\ref{sec:theoretical_res} novel theoretical results concerning the Hamming loss and the maximality decision criterion, those results ending in an inference procedure that has an almost linear time complexity with respect to the size of the output space. We also investigate under which conditions the previous results using marginal probability bound (or  under the assumption of strong independence on output space) become exact. Besides, under the latter conditions,  the set-valued predictions induced by the E-admissible decision criterion match perfectly with that of the maximality criterion. 

In the \ref{app:other_dec_crit}, we provide some additional implications (and a graph with its relations) regarding different types of decision criteria which can be considered as less and more conservatives in terms of skeptical decisions. In other words, the ones which can infer the largest or smallest set-valued of predictions than the other well-known ones, such as the E-admissible and the maximality.  

In Section~\ref{sec:expe}, we perform a set experiments on simulated data sets to study; (1) the exactness of the existing outer-approximation~\cite{destercke2014multilabel} against our exact optimal procedure, when we consider sets of joint probability distributions, and (2) the real computation time trend of the naive against the exact optimal procedure.

Finally, in Section~\ref{sec:expebinarybr}, we perform an additional set of experiments on real data sets, which is divided in two sections; the former aims to study the behaviour of the skeptical inference against its precise counterpart; (1) under an assumption of strong independence and (2) on missing and noisy labels, and the latter aims to compare the partial abstention, the partial rejection rules and our skeptical approach under a downsampling setup. 

Proofs of technical results as well as complementary experimental results, from this paper, are provided in \ref{app:supproof}, \ref{app:supresults} and \ref{app:suppresampling}, respectively.

\section{Preliminaries}
\label{sec:prel}

In this section, we introduce the necessary background to deal with our problem. 

\subsection{Multi-label problem}

In multi-label problem, given a subset $\Omega=\{\omega_1,\ldots,\omega_m\}$, one assumes that to each instance $\vect{x}$ of an input space $\mathcal{X}=\mathbb{R}^d$ is associated a subset $\Lambda \subseteq \Omega$ of relevant labels. In practice, we will identify such subsets with the space of Boolean vectors $\mathcal{Y}=\{0,1\}^m$, denoting a vector $\vect{y}=(y_1,\ldots,y_m)$ and having $y_i=1$ if $\omega_i \in \Lambda$, $0$ else. 

We assume that observations are i.i.d. samples of a distribution $p:\mathcal{X} \times \mathcal{Y} \to [0,1]$, and denote $p_\vect{x}(\bm{y}):= p(\vect{y}|\vect{x})$ the conditional probability of $\vect{y}$ given $\vect{x}$. We denote by $\mathbf{Y}= (Y_1, \dots, Y_m)$ the random binary vector over $\mathcal{Y}$. Given a subset $\mathcal{I} \subseteq \{1,\ldots,m\}$ of indices, we  denote by $\mathcal{Y}_\mathcal{I}$ the space of binary vectors over those indices, by $Y_{\mathcal{I}}$ and $Y_{-\mathcal{I}}$ will the marginals of $\mathbf{Y}$ over these indices and over the complementary indices $\{1,\ldots,m\} \setminus \mathcal{I}$, respectively. In particular, $Y_{\{i\}}$ will denote the marginal random variable over the $i$th label. Similarly, we will denote by $\vect{y}_{\mathcal{I}}$ the values of a vector restricted to elements indexed in $\mathcal{I}$, and by $\vect{b}_{\mathcal{I}}$ a particular assignment over these elements. The associated marginal probability will be
$$P_x(\vect{b}_{\mathcal{I}})=\sum_{\substack{\vect{y} \in \mathcal{Y}, \vect{y}_{\mathcal{I}}=\vect{b}_{\mathcal{I}} }} p_\vect{x}(\vect{y}).$$
We will also consider the complement of a given vector or assignment over a subset of indices. These will be denoted by $\bm{\overline{\vect{y}}}_{\mathcal{I}}$ and $\overline{\vect{b}}_{\mathcal{I}}$, respectively. 

Given two vectors $\vect{y}^1$ and $\vect{y}^2$, we will denote by $\mathcal{I}_{\vect{y}^1\neq \vect{y}^2}:=\{i \in \{1,\ldots,m\}: y^1_i\neq y^2_i \}$ the set of indices over which two vectors are different, and similarly by $\mathcal{I}_{\vect{y}^1=\vect{y}^2}:=\{i \in \{1,\ldots,m\}: y^1_i= y^2_i \}$ the sets of indices for which they will be equal.


\begin{example}
Consider the probabilistic tree developed in Figure~\ref{fig:egprecisetr} defined over $\mathcal{Y}=\{0,1\}^2$ describing a full joint distribution over two labels. In such trees, the probability of any vector $\vect{y}$ is simply the product of the probabilities along its path. We also have that the partial vector $(\cdot,1)$ has probability
$$P((\cdot,1))=P((0,1))+P((1,1))=0.5 \cdot 0.2 + 0.5 \cdot 0.7= 0.45.$$
\begin{figure}[H]
	\centering
	\input{eg_precise_tree}
	\caption{Probabilistic binary tree of two labels}
	\label{fig:egprecisetr}
\end{figure}
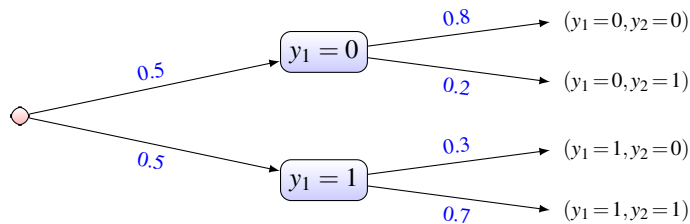	
In the sequel of this paper, we will use such trees to illustrate our results, replacing the precise probabilities on the branches by intervals. An example will be provided later. The resulting set of probabilities over $\mathcal{Y}$ will then simply be the set of all joint probabilities obtained by taking precise values within those intervals. 
\end{example}

As in this paper we are interested in making set-valued predictions for the multi-label problems, we will use the notation $\mathbb{Y} \subseteq \mathcal{Y}$ for generic subsets of $\mathcal{Y}$. We will use the notation $\mathfrak{Y}=\{0,1,*\}^m$ for the specific subsets induced by partially specified binary vectors $\bm{\mathfrak{y}} \in \mathfrak{Y}$, where a symbol $*$ stands for a label on which we abstain. Denoting by $\mathcal{I}^*$ the indices of such labels, we will also slightly abuse the notation $\bm{\mathfrak{y}}$ and $\mathfrak{Y}$ to also denote the corresponding family of subsets over $\mathcal{Y}$, i.e.,
$$\bm{\mathfrak{y}}:=\{\bm{y} \in \mathcal{Y} : \forall i \not\in \mathcal{I}^*, y_i=\mathfrak{y}_i\}. $$
Such subsets are indeed often used to make partial multi-label predictions, and we will refer to them on multiple occasions, calling them partial vectors. However, using only subsets within $\mathfrak{Y}$ may be insufficient if one wants to express complex partial predictions. For instance, in the case where $m=2$, the partial prediction $\mathbb{Y}=\{(0,1),(1,0)\}$ cannot be expressed as an element of $\mathfrak{Y}$, as approximating $\mathbb{Y}$ with an element of $\mathfrak{Y}$ would result in $\mathbb{Y}$, and not the initial subset.

\subsection{Skeptic inferences with distribution sets}
\paragraph{Basic representation} In this paper, We assume that our uncertainty is described by a convex set of probabilities $\credal$, a.k.a. a \textit{credal set}~\cite{levi1980a}, defined over $\mathcal{Y}$. Such sets can arise in various ways: as a native result of the learning method~\cite{antonucci2017multilabel,corani2015credal}; as the result of an agnostic\footnote{With respect to the missingness process.} estimation in presence of imprecise data~\cite{plass2019reliable}; or as a neighbourhood taken over an initial estimated distribution $\hat{p}$, such as in distributionally robust approaches~\cite{chen2018robust}.


\paragraph{Skeptic inference and decision} Once our uncertainty is described by a credal set $\credal$, the next step in the learning process is to deliver an optimal prediction, given a loss function $\ell:\mathcal{Y} \times \mathcal{Y} \to \reals$ where $\ell(\hat{\vect{y}},\vect{y})$ is the loss incurred by predicting $\hat{\vect{y}}$ when $\vect{y}$ is the ground-truth.

When the estimate $\hat{p}$ is precise, this is classically done by picking the prediction minimizing the expected loss, i.e.
\begin{equation}\label{eq:minexpe}\hat{\bm{y}}^{\hat{p}}_\ell=\arg\min_{\bm{y}' \in \mathcal{Y}} \expe_{\hat{p}} \left(\ell(\bm{y}',\cdot) \right)= \arg\min_{\bm{y}' \in \mathcal{Y}} \sum_{\bm{y} \in \mathcal{Y}} \hat{p}(\bm{y}) \ell(\bm{y}',\bm{y}) \end{equation}
or, equivalently, by picking the maximal elements of the linear ordering $\succeq_{\ell}^{\hat{p}}$ where $\bm{y}'' \succeq_{\ell}^{\hat{p}} \bm{y}'$ if  
\begin{align}\label{eq:compaprec} 
\expe_{\hat{P}} \left(\ell(\bm{y}',\cdot) - \ell(\bm{y}'',\cdot) \right) &\!\!=\!\!  \sum_{\bm{y} \in \mathcal{Y}} \hat{p}(\bm{y}) \left( \ell(\bm{y}',\bm{y}) - \ell(\bm{y}'',\bm{y}) \right) & \\
& = \expe_{\hat{P}} \left(\ell(\bm{y}',\cdot)\right) - \expe_{\hat{P}} \left(\ell(\bm{y}'',\cdot) \right) \!\! \geq 0, &  \nonumber  
\end{align}
Since $\succeq_{\ell}^{\hat{p}}$ is a complete pre-order, picking any of the possibly indifferent maximal elements will be equivalent with respect to expected loss minimization.


When considering a set $\credal$ as our uncertainty representation, there are many ways~\cite{Troffaes07} to extend Equation~\eqref{eq:compaprec}. In this paper, we will consider the two main decision rules that may return more than one decision in case of insufficient information: E-admissibility and maximality. These rules follow a skeptical strategy, in the sense that the returned set of predictions is guaranteed to contain the optimal prediction, whatever the true distribution within $\credal$. 

\begin{definition}\label{def:E-adm}E-admissibility returns the set of predictions that are optimal for at least one probability within the set $\credal$. In other words, the E-admissibility rule returns the prediction set
\begin{equation}\label{eq:E-adm-set}
\hat{\mathbb{Y}}^E_{\ell,\credal}=\condset{\vect{y} \in \mathcal{Y}}{\exists P \in \credal \textrm{ s.t. } \vect{y}=\hat{\vect{y}}^{P}_\ell}.
\end{equation} 
\end{definition}

\begin{definition}\label{def:maxi} Maximality consists in returning the maximal, non-dominated elements of the partial order $\succ_{\ell}^{\credal}$ such that $\vect{y} \succ_{\ell}^{\credal} \vect{y}'$ if 
\begin{equation}\label{eq:compaIP} \lexpe\left(\ell(\vect{y}',\cdot) - \ell(\vect{y},\cdot) \right):= \inf_{P \in \credal} \expe_{P}\left(\ell(\vect{y}',\cdot) - \ell(\vect{y},\cdot) \right)  > 0,\end{equation}
that is if exchanging  $y'$ for $y$ is guaranteed to give a positive expected loss. The maximality rule returns the prediction set
\begin{equation}\label{eq:maximalset}
\hat{\mathbb{Y}}^{M}_{\ell,\credal}=\condset{\vect{y} \in \mathcal{Y}}{\not \exists \vect{y}' \in \mathcal{Y} \textrm{ s.t. } \vect{y}' \succ_{\ell}^{\credal} \vect{y} }.
\end{equation} 
\end{definition} 
Since $\succ_{\ell,\credal}$ is in general a partial order, $\hat{\mathbb{Y}}^{M}_{\ell,\credal}$ may result in a set of multiple, incomparable elements. Clearly, the more imprecise is $\credal$, the larger are the sets $\hat{\mathbb{Y}}^{E}_{\ell,\credal}$ and $\hat{\mathbb{Y}}^{M}_{\ell,\credal}$. It can also be proven~\cite{Troffaes07} that  $\hat{\mathbb{Y}}^{E}_{\ell,\credal} \subseteq \hat{\mathbb{Y}}^{M}_{\ell,\credal}$, making E-admissibility a less conservative skeptical decision rule. Yet, the set $\hat{\mathbb{Y}}^{E}_{\ell,\credal}$ is typically harder to compute than $\hat{\mathbb{Y}}^{M}_{\ell,\credal}$, which can be used as a reasonable approximation. The following example illustrates what we just described.

\begin{example}
Let us consider a multi-label problem composed of two labels $Y_1$ and $Y_2$, in other words, composed of a Boolean space $\spacelabel=\{00, 01, 10, 11\}$, and a set of four estimated probability distributions $\hat{P}_Y^*$ of the credal set $\hat{\credal}_Y$ in Table~\ref{tbl:exaadmissiblemax}(a), and a classical zero-one loss function $\losszero$.
\begin{table}[!ht]
\centering
\subfigure[Four estimated distributions $\hat{P}\in\hat{\credal}_Y$.]{
\setlength{\tabcolsep}{8pt}
\renewcommand{\arraystretch}{1.08} 
\begin{tabular}{cc||c:c:c:c}
\hline $y_1$ & $y_2$ & $\hat{P}^1$ & $\hat{P}^2$  & $\hat{P}^3$ & $\hat{P}^4$ \\ \hline
	$0$  & $0$ & $\bf 0.4$ & $\bf 0.4$ & $0.3$ & $0.1$\\
	$0$  & $1$ & $0.3$ & $0.3$ & $0.1$ & $0.2$\\
	$1$  & $0$ & $0.1$ & $0.0$ & $0.1$ & $\bf 0.4$\\
	$1$  & $1$ & $0.2$ & $0.3$ & $\bf 0.5$ & $0.3$\\\hline
\end{tabular}
}\quad
\subfigure[Infimum expectation of the maximality criterion]{ 
\begin{tabular}{c|c:c:c:c}
\hline $\vect{y}$ & $00$ & $01$ & $10$ & $11$ \\\hline
$\lexpe_{\hat{\credal}_{Y}}\left[\losszero(\vect{y}, \cdot)-\losszero(00, \cdot)\right]$ & $\cdot$ & $-0.1$ & $-0.3$ & $-0.2$ \\
$\lexpe_{\hat{\credal}_{Y}}\left[\losszero(\vect{y}, \cdot)-\losszero(01, \cdot)\right]$ & $-0.2$ & $\cdot$ & $-0.2$ & $-0.4$  \\
$\lexpe_{\hat{\credal}_{Y}}\left[\losszero(\vect{y}, \cdot)-\losszero(10, \cdot)\right]$ & $-0.4$ & $-0.3$ & $\cdot$ & $-0.4$ \\
$\lexpe_{\hat{\credal}_{Y}}\left[\losszero(\vect{y}, \cdot)-\losszero(11, \cdot)\right]$ & $-0.2$ & $-0.1$ & $-0.1$ & $\cdot$\\\hline
\end{tabular}
}
\caption{Comparaison between Maximality and E-admissibility criterions.}
\label{tbl:exaadmissiblemax}
\end{table}

Under the zero-one loss function $\losszero$, we can apply Equation~\eqref{eq:compaIP} to obtain the infimum expectations of each pairwise comparison of values of the Boolean space $\spacelabel$ (so 12 comparisons) in Table ~\ref{tbl:exaadmissiblemax}(b), and likewise applying Equation~\eqref{eq:E-adm-set} we obtain the maximum probability per probability distribution $\hat{P}^*$ (in bold in Table~\ref{tbl:exaadmissiblemax}(a)). Therefore, the set-valued of predictions of the E-admissibility and the Maximality criterions are:
\begin{align*}
	\hat{\mathbb{Y}}^E_{\losszero,\credal}=\{00, 10, 11\} \quad\text{and}\quad
	\hat{\mathbb{Y}}^{M}_{\losszero,\credal} =\{00, 01, 10, 11\}.
\end{align*}
Note that $\hat{\mathbb{Y}}^{M}_{\losszero,\credal}$ is more conservative than $\hat{\mathbb{Y}}^E_{\losszero,\credal}$ since the last one does not consider the solution $\{01\}$.
\end{example}

Computing $\hat{\mathbb{Y}}^{E}_{\ell,\credal}$ and $\hat{\mathbb{Y}}^{M}_{\ell,\credal}$ can be a computationally demanding task, making the prediction step critical when considering combinatorial spaces, such as the one considered in this paper. For instance, obtaining $\hat{\mathbb{Y}}^{M}_{\ell,\credal}$ may require at worst to perform $(|\mathcal{Y}|)(|\mathcal{Y}|-1)/2$ comparisons, where $|\mathcal{Y}|=2^m$, ending up with a complexity of $\mathcal{O}(2^{2m})$ that quickly becomes untractable even for small values of $m$. 

\begin{example}\label{exm:expe_inf_illu}
Figure~\ref{fig:egprecisetrloss} illustrates the computation of an expected loss in the case of a probabilistic tree and the zero/one loss function ($\ell(\vect{y}',\vect{y})=1$ if $\vect{y} \neq \vect{y}'$, 0 else), when comparing the two items $\vect{y}'=(0,1)$ and $\vect{y}''=(1,0)$. Global expectation is reached by computing local expectations recursively at each node, starting from the leaves of the tree to get at the root. In this case, we have that $(1,0) \succ^p_{\ell_{0/1}} (0,1)$, since the expectation of the difference $\ell_{0/1} \left((0,1),\cdot\right) - \ell_{0/1} \left((1,0),\cdot\right)$  is positive.  
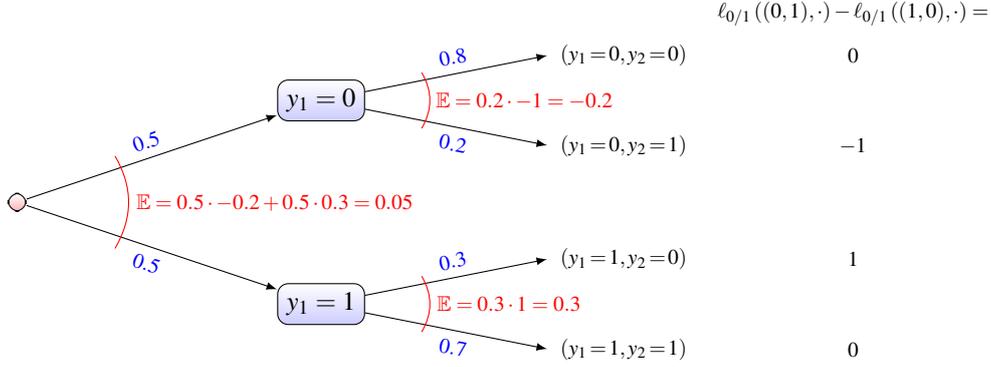
\begin{figure}[!ht]
	\centering
	\input{eg_prob_loss_tree.tex}
	\caption{Probabilistic tree and expected loss}
	\label{fig:egprecisetrloss}
\end{figure}

Figure~\ref{fig:egIprecisetrloss} pictures an imprecise probabilistic tree for the same situation, where probabilities in each branch are replaced by intervals (that in the binary case are sufficient to represent any convex set). The computation of the corresponding lower expectation is done in the same way as in the precise case, starting from the leaves and picking the right interval bounds to obtain lower values of the local expectations. In the example, we still have that $(1,0) \succ^\credal_{\ell_{0/1}} (0,1)$, as the final lower expectation is positive.
\begin{figure}[!ht]
	\centering
	\input{eg_Iprob_loss_tree.tex}
	\caption{Imprecise probabilistic tree and lower expected loss}
	\label{fig:egIprecisetrloss}
\end{figure}
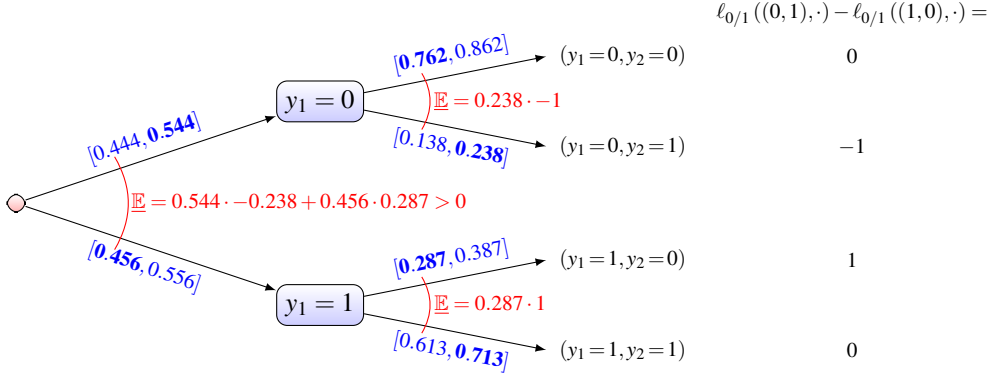
\end{example}

All of this means that simply enumerating elements of $\mathcal{Y}$ is not practically possible, and other strategies need to be adopted. In what follows, we show that in the case of Hamming loss, one of the most common loss used in multi-label and binary problems, we can use an efficient algorithmic procedure to perform skeptic inferences. This is done both for general sets $\credal$ and for specific sets induced from binary relevance models. We also show that some previous results giving rough outer-approximations of skeptic inferences in the general case turn out to be exact for such binary relevance models. 

\section{Skeptic inference for the Hamming loss}
\label{sec:theoretical_res}

The hamming loss, that we will denote $\ell_H$, is a commonly used loss in multi-label problems. It simply amounts to compute the Hamming distance between the ground truth $\vect{y}$ and a prediction $\hat{\vect{y}}$, that is
\begin{equation}\label{eq:Hloss}
\ell_H(\hat{\vect{y}},\vect{y})=\sum_{i=1}^m \indicator{\hat{y}_i \neq y_i }
=|\mathcal{I}_{\hat{\vect{y}}\neq \vect{y}}|\end{equation}
where $\indicator{A}$ denotes the indicator function of event $A$. Note that in contrast with the subset loss $\ell_{0/1}(\hat{\vect{y}},\vect{y})=\indicator{\hat{\vect{y}}\neq \vect{y}}$, the Hamming loss differentiate the situations where only some mistakes are made from the ones where a lot of them are made (being maximum when $\hat{\vect{y}}=\overline{\vect{y}}$ is the complement of $\vect{y}$). 

In the case of precise probabilities, it is also useful to recall that the optimal prediction for the Hamming loss~\cite{dembczynski2012label}, i.e. the vector $\hat{\vect{y}}_{\ell_H,p}$ satisfying Equation~\eqref{eq:minexpe} is
\begin{equation}\label{eq:optiHamPrec}\hat{y}_{i,\ell_H,p}=\begin{cases} 1 & \text{ if } p(Y_{\{i\}}=1) \geq \frac{1}{2} \\ 0 & \text{ else. }\end{cases}\end{equation}

When considering a set $\credal$ of distribution, one is immediately tempted to adopt as partial prediction the partial vector $\hat{\bm{\mathfrak{y}}}_{\ell_H,\credal} \in \mathfrak{Y} $ such that
\begin{equation}\label{eq:approxHamIP}
	\hat{\mathfrak{y}}_{i,\ell_H,\credal}=
	\begin{cases} 
		1 & \text{ if } \underline{P}(Y_{\{i\}}=1) > \frac{1}{2} \\ 
		0 & \text{ if } \underline{P}(Y_{\{i\}}=0) > \frac{1}{2} \\ 
		* & \text{ if } \frac{1}{2} \in [\underline{P}(Y_{\{i\}}=1),\overline{P}(Y_{\{i\}}=1)].
	\end{cases}
\end{equation}
It has however been proven that $\hat{\vect{\mathfrak{y}}}_{\ell_H,\credal}$ is in general an outer-approximation of $\hat{\mathbb{Y}}^{E}_{\ell,\credal}$ and $\hat{\mathbb{Y}}^{M}_{\ell,\credal}$ (i.e., $\hat{\mathbb{Y}}^{E}_{\ell,\credal}, \hat{\mathbb{Y}}^{M}_{\ell,\credal} \subseteq \hat{\bm{\mathfrak{y}}}_{\ell_H,\credal} $), thus providing a quick heuristic to get an approximate answer~\cite{destercke2014multilabel}. 

The following sub-sections study the problem of providing exact skeptic inferences; firstly for any possible probability set $\mathcal{P}$, and then for the specific case where $\mathcal{P}$ is built from marginal models on each label, that corresponds to binary relevance approaches in multi-label learning.

\subsection{General case}\label{sec:generalcase}

In this section, we demonstrate that for the Hamming loss, we can use inference procedures that are much more efficient than an exhaustive, naive enumeration. Let us first simplify the expression of the expected value.
\begin{lemma}\label{prop:expcond}
    In the case of Hamming loss and given $\vect{y}^1,\vect{y}^2$, we have
    	\begin{equation}\label{eq:expcond}
			\mathbb{E}\left[ \ell_H(\vect{y}^2, \cdot) - \ell_{H}(\vect{y}^1, \cdot) \right]  =
		 \sum_{i=1}^m  P(Y_i=y^1_i) - P(Y_i=y^2_i)  
\end{equation}
\end{lemma}

If we consider a set of indices $\mathcal{I}_{\vect{y}^1 = \vect{y}^2}$ on which the Equation~\eqref{eq:expcond} is cancelled, it can be rewritten 
\begin{equation}\label{eq:hamsimpl}
\sum_{i \in \mathcal{I}_{\vect{y}^1\neq \vect{y}^2}}  P(Y_i=y^1_i) - P(Y_i=y^2_i).
\end{equation}
The next proposition shows that this expression can be leveraged to perform the maximality check of Equation~\eqref{eq:compaIP} on a limited number of vectors. 

\begin{proposition}\label{prop:newdecision}
    For a given set $\mathcal{I}$ of indices, let us consider an assignment $\vect{a}_{\mathcal{I}}$ and its complement $\overline{\vect{a}}_\mathcal{I}$. Then, for any two vectors $\vect{y}^1,\vect{y}^2$ such that $\vect{y}^1_\mathcal{I}=\vect{a}_{\mathcal{I}}$, $\vect{y}^2_\mathcal{I}=\overline{\vect{a}}_{\mathcal{I}}$ and $\vect{y}^1_{-\mathcal{I}}=\vect{y}^2_{-\mathcal{I}}$, we have
\begin{equation}\label{eq:maximalitycriterion}
 \vect{y}^1 \succ_M \vect{y}^2 \iff \inf_{P \in \credal}  
 \sum_{i \in \mathcal{I}}  P(Y_i=a_i) > \frac{|\mathcal{I}|}{2} 	
\end{equation}
\end{proposition}

In the remaining of the paper, given a partial assignment $\vect{b}_{\mathcal{I}}$ over a subset of indices $\mathcal{I}$, we will define the partial Hamming loss between $\vect{b}_{\mathcal{I}}$ and an observation $\vect{y}$ as
\begin{equation}\label{eq:partial_ham_loss}
	\ell_H^*(\vect{b}_\mathcal{I}, \vect{y})=\sum_{i \in \mathcal{I}} \indicator{b_i \neq y_i}.
\end{equation}
When $\mathcal{I}=\{1,\ldots,m\}$, we simply retrieve the usual Hamming loss. The next proposition shows that the condition of Proposition~\ref{prop:newdecision} actually comes down to minimize the expected partial Hamming loss. 

\begin{proposition}\label{prop:Ham_partial_loss}
    For a given set $\mathcal{I}$ of indices, let us consider an assignment $\vect{a}_{\mathcal{I}}$ and its complement $\overline{\vect{a}}_\mathcal{I}$. We have
\begin{equation}
	\inf_{P \in \credal}  \sum_{i \in \mathcal{I}}  P(Y_i=a_i) = 
	\underline{\mathbb{E}}[\ell_H^*(\overline{\vect{a}}_\mathcal{I},\cdot)] 
\end{equation}
\end{proposition}

\begin{algorithm}[!ht]
 \caption{Maximal solutions under Hamming loss and general set}
 \label{alg:HamMaxim}
 \KwData{$\mathscr{P}$ (convex set of distributions)}
 \KwResult{$\hat{\mathbb{Y}}^{M}_{\ell_H,\credal}$ (set of undominated  solutions)}
	    $S=\mathcal{Y}$\;
	    \For{i in 1:m}
	    {$\mathcal{Z}_i = \{ \mathcal{I} : \mathcal{I} \subseteq \{1,\ldots,m\}, |\mathcal{I}|=i \}$ \tcp*{\small Index sets of size $i$}
	     	\ForAll{$z \in \mathcal{Z}_i$}
	     	{
	     		\ForAll{$\bm a_z \in \mathcal{Y}_z$ \tcp*{\small Binary vectors over indices in $z$}} 
	     			{\lIf{$\inf_{P \in \credal}  \sum_{j \in z}   P(Y_j=a_j) > \frac{i}{2}$}
	     			{$S = S \setminus \{\vect{y} \in \mathcal{Y} : \vect{y}_z=\overline{\vect{a}}_z\}$}
	     		}
	        }	
	     }
	\end{algorithm}

This allows us to use Algorithm~\ref{alg:HamMaxim} to find  $\hat{\mathbb{Y}}^{M}_{\ell_H,\credal}$. The following result provides the time complexity of the algorithm. 

\begin{proposition}\label{prop:algcomplexity}
Algorithm~\ref{alg:HamMaxim} has to perform $3^m - 1$ computations, and its complexity is in $\mathcal{O}(3^m)$
\end{proposition}

Proposition~\ref{prop:algcomplexity} tells us that, in the case of Hamming loss, finding $\hat{\mathbb{Y}}^{M}_{\ell,\credal}$ can be done almost linearly with respect to the size of $\mathcal{Y}$. This is to be compared to a naive enumeration, that requires $(2^m)(2^m-1)$ computations. Figure~\ref{fig:compaalgonaive} plots the two curves as a function of the number $m$ of labels, demonstrating that our result allows a significant gain in computations. Furthermore, using the same experimental setup of Section~\ref{sec:expe}, in which we shall study the differences between $\hat{\mathbb{Y}}^{M}_{\ell,\credal}$ and the crude approximation of Equation~\eqref{eq:approxHamIP}, we show in Figure~\ref{fig:compaalgonaivepractical} the empirical evolution of the average computation time on $10^4$ comparaisons in milliseconds (and in $log_{10}$ scale) of the naive version against the exact optimal procedure proposed in Algorithm~\ref{alg:HamMaxim}, where small imprecision and high imprecision refer to $\epsilon=0.05$ and $\epsilon=0.45$, respectively (see Section~\ref{sec:expe} for more details).

\begin{figure}[!ht]
    \centering
    \subfigure[\sc Theorical]{
       \label{fig:compaalgonaive}
	   \begin{tikzpicture}[scale=0.6]
		\begin{axis}[
			xlabel=$m$,ylabel=\# of Evaluation needed ($log_{10}$),
			                xmin=0, xmax=35, 
			                legend style={at={(0.5,-0.15)},
	      anchor=north,legend columns=-1},]
		\addplot[mark=*,black,domain=0:35] {ln(2^(2*x)-2^x)*(1/ln(10))};
		\addplot[mark=square*,red,domain=0:35] {ln(3^x-1)*(1/ln(10))};
		\legend{Naive,Algorithm~1};
		\end{axis}%
	   \end{tikzpicture}%
	}\qquad%
	\subfigure[\sc Empirical Experimentation]{
		\includegraphics[scale=0.425]{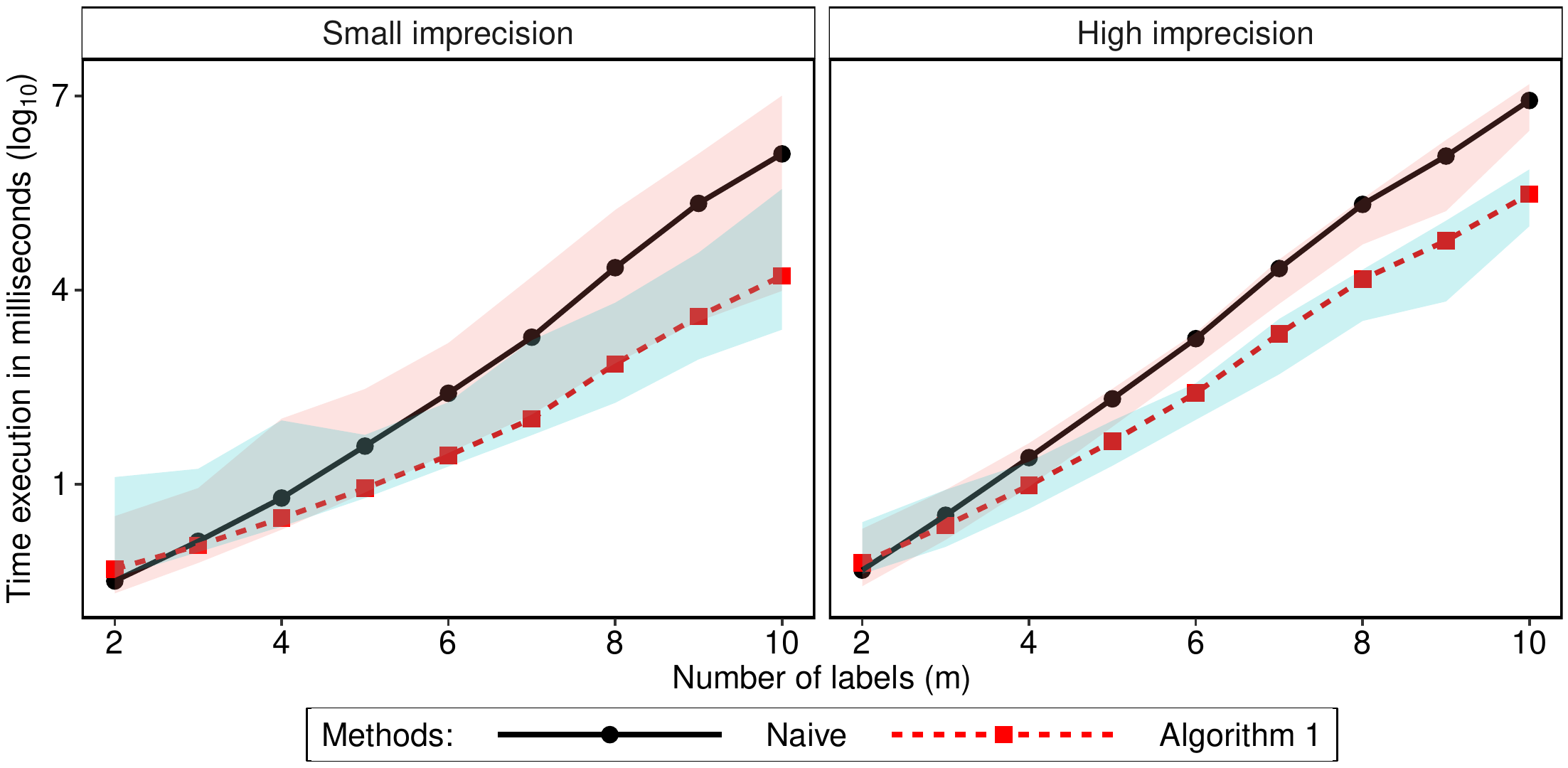}
		\label{fig:compaalgonaivepractical}
	}	
    \caption{Comparison of Algorithm~\ref{alg:HamMaxim} with naive enumeration (log-ordinate scale).}
    
\end{figure}
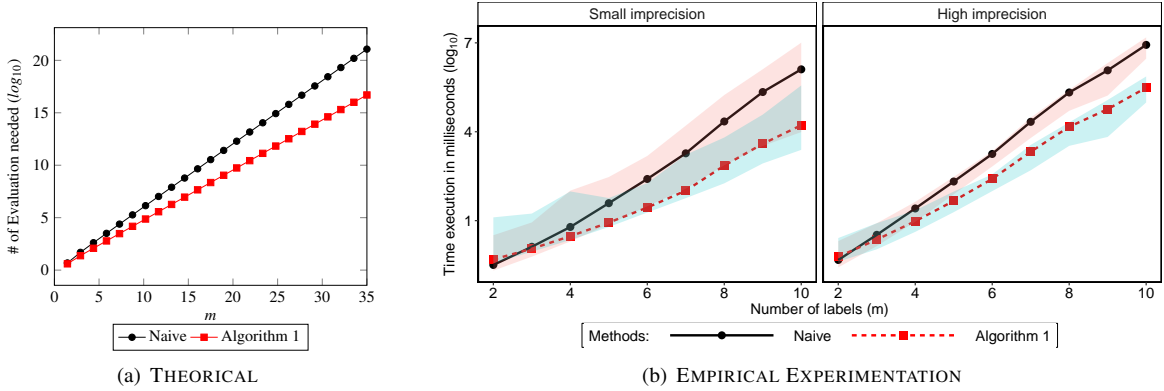

It should strongly be noted that the time complexity obtained in the Proposition~\ref{alg:HamMaxim} can still be reduced using two different strategies: (1) removing dominated elements already verified (cf. \cite[algo.~16.4]{augustin2014introduction}), and (2) using the precise prediction as a solution of the set of solutions E-admissible (c.f \cite[appx. A]{alarcon2021imprecise}). However, in the worst-case scenario, in which all elements are non-dominated, the time complexity remains the same.

As said before, the set $\hat{\mathbb{Y}}^{M}_{\ell_H,\credal}$ will in general not be exactly described by a partial vector within $\mathfrak{Y}$, as shows the next example. 

\begin{example}\label{exa:notpartialY}
	Consider again the tree provided in Figure~\ref{fig:egIprecisetrloss}. The result of applying Algorithm~\ref{alg:HamMaxim} provides the following results:
\begin{align*}
		\underline{\mathbb{E}}\left[ \ell_H((1, *),\cdot) \right]=0.444 >0.5 
			\implies (0, *) \not\succ_M (1, *),\\
		\underline{\mathbb{E}}\left[ \ell_H((0, *),\cdot) \right]=0.456 >0.5 
			\implies (1, *) \not\succ_M (0, *),\\
		\underline{\mathbb{E}}\left[ \ell_H((*, 1),\cdot) \right]=0.498 > 0.5 
			\implies (*, 0) \not\succ_M (*, 1),\\
		\underline{\mathbb{E}}\left[ \ell_H((*, 0),\cdot) \right]=0.354 > 0.5 
			\implies (*, 1) \not\succ_M (*, 0),\\
		\underline{\mathbb{E}}\left[ \ell_H((1, 1),\cdot) \right]=0.942 > 1.0
			\implies (0, 0) \not\succ_M (1, 1),\\
		\underline{\mathbb{E}}\left[ \ell_H((1, 0),\cdot) \right]=0.846 > 1.0
			\implies (0, 1) \not\succ_M (1, 0),\\
		\underline{\mathbb{E}}\left[ \ell_H((0, 1),\cdot) \right]=1.001 > 1.0
			\implies \bf (1, 0) \succ_{\emph{M}} (0, 1),\\
		\underline{\mathbb{E}}\left[ \ell_H((0, 0),\cdot) \right]=0.810 > 1.0
			\implies (1, 1) \not\succ_M (0, 0),
\end{align*}
	where for two partial vectors $\vect{y}^1,\vect{y}^2$ such that $\mathcal{I}_{\vect{y}^1}^*=\mathcal{I}_{\vect{y}^2}^*$, we use the short-hand notation $\vect{y}^1 \succ_M \vect{y}^2$ to say that the dominance relation given by Definition~\ref{def:maxi} holds for any fixed replacement of the abstained labels. 
	
	About this example, we can first note that only $3^2-1=8$ comparisons are performed (in accord with Proposition~\ref{prop:algcomplexity}). Secondly, also note that the final solution which is the set
	\begin{align*}
		\hat{\mathbb{Y}}^{M}_{\ell_H,\credal} = \{(1, 0), (0, 0), (1, 1)\}
	\end{align*} does not belong to $\mathfrak{Y}$.
\end{example}

\begin{remark}
Note that if for some partial vector $\vect{y}$ Proposition~\ref{prop:newdecision} holds, the preference also holds for any completion of such a vector. More precisely, if we denote by $\overline{\mathcal{I}}^*$ the indices of non-abstained labels, and $\mathbf{a}_{\mathcal{I}}$ an assignment over indices $\mathcal{I} \subseteq \mathcal{I}^*$ (where $\mathcal{I}^*=\{1, \dots,m\} \setminus \overline{\mathcal{I}}^*$), one can deduce from $\vect{y} \succ_M \overline{\vect{y}}$ that $(\vect{y}_{\overline{\mathcal{I}}^*},\vect{a}_{\mathcal{I}}) \succ_M (\overline{\vect{y}}_{\overline{\mathcal{I}}^*},\vect{a}_{\mathcal{I}})$. For instance if $(0, *, *) \succ_M (1, *, *)$, then we can additionally deduce $(0, *, 0) \succ_M (1, *, 0)$. 
\end{remark}

\begin{remark}\label{rem:imp_prec_loss}
A key finding of the results of this section, illustrated by Example~\ref{exa:notpartialY}, is that when considering sets of distributions and skeptic inferences, it is not sufficient to consider marginal probabilities in order to get optimal, exact predictions. This contrasts heavily with the case of precise distributions, in which having only the marginal information allows to get optimal predictions for a number of loss functions, including the Hamming loss, but also precision\verb+@+k, micro- and macro-F measure, as well as others~\cite{kotlowski2016surrogate,koyejo2015consistent}.
\end{remark}

\begin{remark}
Despite our best efforts and except for some few tweaks, we were not really able to significantly lower the complexity of finding $\hat{\mathbb{Y}}^{M}_{\ell_H,\credal}$, i.e., to go from the still exponential complexity of Proposition~\ref{prop:algcomplexity} to a polynomial one in the number of labels. This contrasts with the precise case, where one can use the marginal information to obtain the result in a polynomial time in the number of labels. The fact that we cannot rely on the marginal bounds of $P(Y_i)$ suggest us that reaching such a polynomial complexity for exact inferences over generic credal sets $\credal$ may not be doable. 
\end{remark}

\subsection{Binary relevance and partial vectors}
\label{sec:Binary_rev}
The previous section looked at the very general case where the set $\credal$ is completely arbitrary and proposed some efficient inference methods for this case. In this section, we are interested in conditions imposed upon $\credal$ that guarantee the sets $\hat{\mathbb{Y}}^{M}_{\ell_H,\credal}$ and $\hat{\mathbb{Y}}^{E}_{\ell_H,\credal}$ to be partial vectors, that is to belong to $\mathfrak{Y}$. In particular, we show that this is the case when considering models that generalize binary relevance notions by using imprecise marginals with an assumption of independence. The interest of studying such models is that they constitute the basic models when it comes to multi-label problems. 

In this section, we therefore consider that the joint probability $p$ over $\mathcal{Y}$ and its imprecise extension are built in the following way: we have some information on the marginal probability $p_{i} \in [0,1]$ of $y_i$ being positive, and define the probability of a vector $\vect{y}$ as
\begin{equation}\label{eq:BRvectprob}
p(\vect{y})=\prod_{\{i| y_i=1\}} p_{i} \prod_{\{i| y_i=0\}} (1-p_{i}).
\end{equation}
Without loss of generality, the imprecise version then amounts to consider that the information we have is an interval $[\underline{p}_i,\overline{p}_i]$, as every convex set of probabilities on a binary space (here, $\{0,1\}$) is an interval. We then consider that a probability set $\credal_{BR}$ over $\mathcal{Y}$ amounts to consider the robust version of Equation~\eqref{eq:BRvectprob}, that is
\begin{equation}\label{eq:BRvectIPprob}
p(\vect{y}) \in \left\{\prod_{\{i| y_i=1\}} p_{i} \prod_{\{i| y_i=0\}} (1-p_{i})| p_i \in  [\underline{p}_i,\overline{p}_i]\right\}.
\end{equation}
In this specific case, we can show that $\hat{\mathbb{Y}}^{E}_{\ell_H,\credal}$ can be exactly described by a partial vector. 

\begin{proposition}
    \label{prop:HammingPartialBR} Given a probability set $\credal_{BR}$ and the Hamming loss, the set $\hat{\mathbb{Y}}^{E}_{\ell_H,\credal_{BR}} \in \mathfrak{Y}$ 
\end{proposition}
  
Proposition~\ref{prop:HammingPartialBR} shows that in the specific yet important case of binary relevance models, $\hat{\mathbb{Y}}^{E}_{\ell_H,\credal}$ can be computed efficiently and easily presented to users. In particular, Equation~\eqref{eq:approxHamIP} is in this case exact, and can be used to compute $\hat{\mathbb{Y}}^{E}_{\ell_H,\credal_{BR}}$.  We will denote by $\hat{\vect{\mathfrak{y}}}_{\ell_H,\credal_{BR}}$ the partial vector corresponding to $\hat{\mathbb{Y}}^{E}_{\ell_H,\credal_{BR}}$, as the next proposition show that  it is also an exact estimation of $\hat{\mathbb{Y}}^{M}_{\ell_H,\credal_{BR}}$.

\begin{proposition}\label{prop:EadmEqualMax}Given a probability set $\credal_{BR}$ and the Hamming loss, we have $$\hat{\mathbb{Y}}^{E}_{\ell_H,\credal_{BR}}=\hat{\mathbb{Y}}^{M}_{\ell_H,\credal_{BR}}.$$
\end{proposition}

\begin{remark}
As the optimal prediction for the 0/1 or subset loss $\ell_{0/1}$ in the precise case is the same as  Equation~\eqref{eq:optiHamPrec}, Proposition~\ref{prop:HammingPartialBR} is also true for this loss, as well as Proposition~\ref{prop:EadmEqualMax}.
\end{remark}

In~\ref{app:other_dec_crit}, we provide a couple of complementary results with respect to the other decision criteria, that we do not consider in the main text as they are usually either very conservative (i.e., interval dominance) or not skeptic (i.e., minimax and minimin), in the sense that their inferences are always precisely valued, not matter how big $\credal$ is. 

In what follows, we will describe two independent experiments that empirically highlight the interest of using sets of probabilities rather than a single distribution. The first one is dedicated to compare our exact solution against the existing outer-approximation using simulated data sets, whereas the second one is dedicated; (1) to show our approach's effectiveness when it is confronted with missing and noisy data sets, and (2) to compare our approach with other existing skeptical approaches.

\section{Experiments of Skeptic inferences with General Case} \label{sec:expe}

In this section, we aim to perform an empirical experiment\footnote{Implemented in Python, see~\url{https://github.com/sdestercke/classifip}.} to evaluate the quality of our exact procedure of making skeptical inferences described in Section~\ref{sec:generalcase} against that of its counterpart, i.e outer-approximation described in~\cite{destercke2014multilabel}. First, in Section~\ref{sec:inferencemultilabel}, we formalize the procedure we used in Example~\ref{exm:expe_inf_illu} to compute the lower expectation in binary tree structures used for instance to verify the equation of Proposition~\ref{prop:newdecision}. We then investigate, through simulation, the difference between exact inferences and the outer-approximation on different number of labels.
\subsection{Inference with imprecise multi-label classifier}\label{sec:inferencemultilabel}
As we saw in Proposition \ref{prop:newdecision} and Algorithm~\ref{alg:HamMaxim}, estimating $\hat{\mathbb{Y}}^{M}_{\ell_H,\credal}$, given an observed instance $\newinstance$, implies the calculation of the infimum expectation $\underline{\mathbb{E}}_{\mathbf{Y}|\mathbf{X}=\newinstance}\left[\ell_H(\cdot, \overline{\vect{a}}_\mathcal{I}) \right]$ given an assignment $\vect{a}_\mathcal{I}$. One possibility to compute it is to write it as an iterated conditional expectation over the chain of labels, i.e.,
\begin{align}\label{eq:glob_lower_ex}
	\underline{\mathbb{E}}_{\mathbf{Y}|\mathbf{X}=\newinstance}\left[ \ell_H(\cdot, \overline{\vect{a}}_\mathcal{I}) \right] 
	= \inf_{P\in\mathcal{P}}
	\mathbb{E}_{Y_1}\left[
		\mathbb{E}_{Y_2}\left[
		\dots
		\mathbb{E}_{Y_m}\left[ \ell_H(\cdot, \overline{\vect{a}}_\mathcal{I}) \middle| \mathbf{X}=\newinstance, Y_{\mathcal{I}_{\setn{m-1}}}=\vect{y}_{\mathcal{I}_{\setn{m-1}}}\right] 
		\dots\right]
	 	\middle| \mathbf{X}=\newinstance\right],
\end{align}
where $\setn{j}\!=\!\{1, 2, \dots, j\!-\!1, j\}$ is a set of previous indices and $Y_{\mathcal{I}_{\setn{m-1}}}=\{Y_1, \dots, Y_{m-1}\}$ is a random binary vector. 
While such an expectation has to be computed globablly, it has been shown by Hermans and De Cooman~\cite{hermans2009imprecise} that in the specific case of tree structures, it can be computed recursively \footnote{A backward recursive efficient algorithm was implemented by Gen et al in \cite{yang2014nested}.} using the law of iterated lower expectations\footnote{In general, there is only an inequality between Equations~\eqref{eq:glob_lower_ex} and~\eqref{eq:recinfexp}}
\begin{align}\label{eq:recinfexp}
	 \underline{\mathbb{E}}_{\mathbf{Y}|\mathbf{X}=\newinstance}\left[ \ell_H(\cdot, \overline{\vect{a}}_\mathcal{I}) \right] =\underline{\mathbb{E}}_{Y_1}\left[
	 	\underline{\mathbb{E}}_{Y_2}\left[
	 	\dots
	 	\underline{\mathbb{E}}_{Y_m}\left[ \ell_H(\cdot, \overline{\vect{a}}_\mathcal{I}) \middle| \mathbf{X}=\newinstance, Y_{\mathcal{I}_{\setn{m-1}}}=\vect{y}_{\mathcal{I}_{\setn{m-1}}}\right] 
	 	\dots\right]
	 	\middle| \mathbf{X}=\newinstance \right].
\end{align}
Equation~\eqref{eq:recinfexp} allows one to compute global infimum expectation using local models and backward recursion, i.e., we first compute the local lower expectation starting from the leaves of the tree and proceed iteratively (for further details see \cite{yang2014nested}). Example~\ref{exa:recinfexp} provides us an illustration of this procedure. 
 
\begin{example}\label{exa:recinfexp}
	Let us consider a multi-label problem with two labels $\{Y_1, Y_2\}$  with the credal set $\credal$ over $\mathcal{Y}$ defined by the tree pictured in Figure~\ref{fig:eginfexpectation}. Consider and $\vect{y}^1=(\cdot, 1)$ and $\vect{y}^2=(\cdot, 0)$ two binary vectors which have the same value of the label $Y_1$. According to the Proposition~\ref{prop:newdecision}, the assignment of these vectors is $a_{\{2\}} = (1)$ (and its complement $\overline{a}_{\{2\}} = (0)$). In order to verify whether $\vect{y}^1$ dominates $\vect{y}^2$ (in the sense of the maximality criterion), we have to check whether
	\begin{align}
		 \underline{\mathbb{E}}\left[ \ell^*_H(\overline{\vect{a}}_\mathcal{I},\cdot) \right] > 0.5,
	\end{align}
	where the cost vector of the partial Hamming loss is $ =(0, 1, 1, 0)$ as can be verified in the Figure~\ref{fig:eginfexpectation}.
	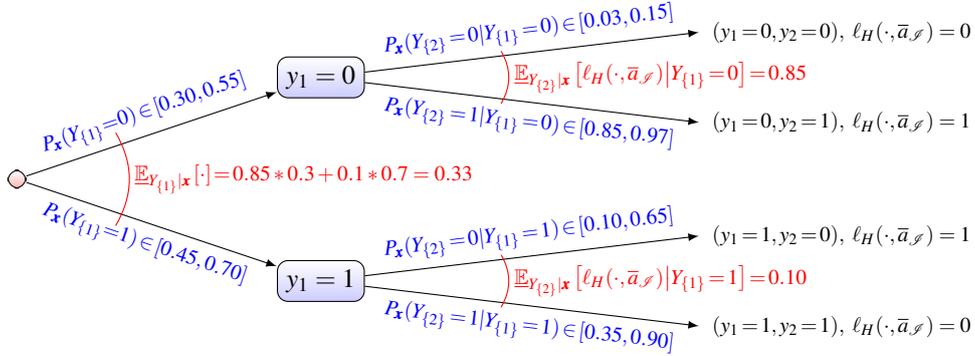
\begin{figure}[H]
		\centering
		\input{eg_infimum_expectation}
		\caption{Example of computing the infimum expectation.}
		\label{fig:eginfexpectation}
	\end{figure}	
	Thus, applying recursively Equation~\eqref{eq:recinfexp}, we obtain an infimum expectation $\underline{\mathbb{E}}_{\mathbf{Y}|\mathbf{X}=\newinstance}\left[ \ell_H(\cdot, \overline{\vect{a}}_\mathcal{I}) \right]=0.33$. As it is lower than $0.5$, we cannot conclude that $\vect{y}^1 \succ_M \vect{y}^2 $.
\end{example}

Finally, let us note that computing marginals $\underline{P}(Y_{\{i\}}=0)$ and $\underline{P}(Y_{\{i\}}=1)$ used in Equation~\eqref{eq:approxHamIP} is equally easy, as it amounts to compute the lower expectation of the indicator functions $\indicator{y_{i}=0}$ and $\indicator{y_{i}=1}$, respectively. 

\subsection{Exact vs approximate skeptic inference}\label{sec:expesimulation}
In this section, we want to assess how good is the outer-approximation proposed in~\cite{destercke2014multilabel} (and given by Equation~\eqref{eq:approxHamIP}), by comparing it to an exact estimation of the set $\hat{\mathbb{Y}}^{M}_{\ell_H,\credal_{BR}}$. Such an estimate is essential to know in which situation Equation~\eqref{eq:approxHamIP} is likely to give a too conservative outer-approximation, and in which cases it can safely be used.

To perform this study, we simulate credal sets $\credal$ over $\mathcal{Y}$ by generating binary trees in the following way: we choose an $\epsilon\in[0, 0.5]$, and for a label $Y_{i}$ and a path $y_1,\ldots,y_{i-1}$, we generate a random $\theta \sim \mathcal{U}([0, 1])$ to obtain the interval
\begin{align*}
	&\underline{P}_\newinstance(Y_{\{i\}} = 1|y_1,\ldots,y_{i-1}) = \max(0, \theta - \epsilon)=1-\overline{P}_\newinstance(Y_{\{i\}} = 0|y_1,\ldots,y_{i-1}), \\
	&\overline{P}_\newinstance(Y_{\{i\}} = 1|y_1,\ldots,y_{i-1}) = \min (\theta + \epsilon, 1)=1-\underline{P}_\newinstance(Y_{\{i\}} = 0|y_1,\ldots,y_{i-1}),
\end{align*}
where $\mathcal{U}([0, 1])$ is a uniform distribution and $\epsilon$ is a parameter representing the imprecision level of our interval. The value of parameter $\epsilon$ impact directly on the width of the probability interval, and therefore, the precision of the obtained prediction. The tree in Figure~\ref{fig:eginfexpectation} is of this kind. 


In order to ensure the truthfulness and completeness of the comparison of two skeptic inferences, we evaluate them on five different samples of $2000$ binary trees, each sample having a fixed $\epsilon$ (i.e., $10^3$ instances). For each instance, we evaluate the quality of the outer-approximation by computing the number of added elements in the corresponding set of binary vectors, i.e.,
\begin{align}\label{eq:distcard}
d_{(\hat{\vect{\mathfrak{y}}}, \hat{\mathbb{Y}})}^{\epsilon}= |\hat{\vect{\mathfrak{y}}}_{\ell_H, \mathscr{P}}| - |\hat{\mathbb{Y}}_{\ell_H, \mathscr{P}}^M|.
\end{align}

As we have that $\hat{\vect{\mathfrak{y}}}_{\ell_H, \mathscr{P}}\supseteq\hat{\mathbb{Y}}_{\ell_H, \mathscr{P}}^M$, Equation \ref{eq:distcard} will never be negative. Also, since different number of labels will induce different upper bounds for Equation~\eqref{eq:distcard}, we uniformize the results across different numbers by partitioning the results in four bins:
\begin{align*}
	q_0 &= \#\left\{(\hat{\vect{\mathfrak{y}}}, \hat{\mathbb{Y}})_i^{(2000)} ~\middle|~ d_{(\hat{\vect{\mathfrak{y}}}, \hat{\mathbb{Y}})_i}^{\epsilon} = 0\right\},\\
	q_{\leq 0.25} &= \#\left\{(\hat{\vect{\mathfrak{y}}}, \hat{\mathbb{Y}})_i^{(2000)} ~\middle|~  0 < d_{(\hat{\vect{\mathfrak{y}}}, \hat{\mathbb{Y}})_i}^{\epsilon}  \leq  2^{|\Omega|}/4 \right\},\\
	q_{\leq 0.5} &= \#\left\{(\hat{\vect{\mathfrak{y}}}, \hat{\mathbb{Y}})_i^{(2000)} ~\middle|~ 2^{|\Omega|}/4 < d_{(\hat{\vect{\mathfrak{y}}}, \hat{\mathbb{Y}})_i}^{\epsilon} \leq  2^{|\Omega|}/2\right\},\\
	q_{\leq 1} &= \#\left\{(\hat{\vect{\mathfrak{y}}}, \hat{\mathbb{Y}})_i^{(2000)} ~\middle|~ 2^{|\Omega|}/2 < d_{(\hat{\vect{\mathfrak{y}}}, \hat{\mathbb{Y}})_i}^{\epsilon} \leq 2^{\Omega}\right\}.
\end{align*}

Finally, we perform the computer simulations on a discretization of the parameter $\epsilon\in\{0.05, 0.15, \dots, 0.45\}$. Thus, the results obtained, in percentage and with confidence interval (of the five repetitions), for each $\epsilon$ value and partitions $q_*$ are shown in the Table~\ref{tab:ressimul}, besides we omitted the results of $\epsilon=0.45$ since it always yields $q_0=100\%$ for all labels.

\setlength{\tabcolsep}{8pt} 
\renewcommand{\arraystretch}{1.3} 
\newcommand{\labelstwoTofive}{\small
\begin{tabular}{c|c|cccc}
	\multirow{2}{*}{\#label} & \multirow{2}{*}{$\epsilon$} & \multicolumn{4}{c}{$d_{\hat{\vect{\mathfrak{y}}}, \hat{\mathbb{Y}}}^{\epsilon}$}\\[.5em]\cline{3-6}
	& &	 $q_0$ & $q_{\leq 0.25}$ & $q_{\leq 0.5}$ & $q_{\leq 1}$ \\
	\hline
		\multirow{4}{*}{2} 
	& $0.05$ & $\bf 100.0 \pm 0.00\%$ & $0.00 \pm 0.00\%$ & $0.00 \pm 0.00\%$ & $0.00 \pm 0.00\%$\\
	& $0.15$ & $\bf 98.93 \pm 0.11\%$ & $0.00 \pm 0.00\%$ & $\bf 1.07 \pm 0.11\%$ & $0.00 \pm 0.00\%$\\
	& $0.25$ & $\bf 98.98 \pm 0.18\%$ & $0.00 \pm 0.00\%$ & $\bf 1.02 \pm 0.18\%$ & $0.00 \pm 0.00\%$\\
	& $0.35$ & $\bf 100.0 \pm 0.00\%$ & $0.00 \pm 0.00\%$ & $0.00 \pm 0.00\%$ & $0.00 \pm 0.00\%$\\\hline
		\multirow{4}{*}{3} 
	& $0.05$ & $\bf 99.04 \pm 0.06\%$ & $\bf 0.66 \pm 0.07\%$ & $\bf 0.30 \pm 0.09\%$ & $0.00 \pm 0.00\%$\\
	& $0.15$ & $\bf 98.17 \pm 0.27\%$ & $\bf 1.45 \pm 0.29\%$ & $\bf 0.38 \pm 0.10\%$ & $0.00 \pm 0.00\%$\\
	& $0.25$ & $\bf 98.11 \pm 0.17\%$ & $\bf 0.46 \pm 0.08\%$ & $\bf 1.43 \pm 0.17\%$ & $0.00 \pm 0.00\%$\\
	& $0.35$ & $\bf 99.82 \pm 0.04\%$ & $0.00 \pm 0.00\%$ & $\bf 0.18 \pm 0.04\%$ & $0.00 \pm 0.00\%$\\\hline
		\multirow{4}{*}{4} 
	& $0.05$ & $\bf 97.05 \pm 0.25\%$ & $\bf 2.95 \pm 0.25\%$ & $0.00 \pm 0.00\%$ & $0.00 \pm 0.00\%$\\
	& $0.15$ & $\bf 95.85 \pm 0.38\%$ & $\bf 2.97 \pm 0.24\%$ & $\bf 1.17 \pm 0.17\%$ & $\bf 0.01\pm 0.02\%$ \\
	& $0.25$ & $\bf 99.02 \pm 0.17\%$ & $\bf 0.08 \pm 0.05\%$ & $\bf 0.90 \pm 0.18\%$ & $0.00 \pm 0.00\%$\\
	& $0.35$ & $\bf 100.0 \pm 0.00\%$ & $0.00 \pm 0.00\%$ & $0.00 \pm 0.00\%$ & $0.00 \pm 0.00\%$\\\hline
		\multirow{4}{*}{5} 
	& $0.05$ & $\bf 90.94 \pm 0.65\%$ & $\bf 8.02 \pm 0.51\%$ & $\bf 1.04 \pm 0.23\%$ & $0.00 \pm 0.00\%$\\
	& $0.15$ & $\bf 92.79 \pm 0.18\%$ & $\bf 4.53 \pm 0.42\%$ & $\bf 2.01 \pm 0.37\%$ & $\bf 0.67 \pm 0.21\%$\\
	& $0.25$ & $\bf 97.92 \pm 0.05\%$ & $\bf 1.05 \pm 0.20\%$ & $\bf 0.73 \pm 0.15\%$ & $\bf 0.30 \pm 0.09\%$\\
	& $0.35$ & $\bf 100.0 \pm 0.00\%$ & $0.00 \pm 0.00\%$ & $0.00 \pm 0.00\%$ & $0.00 \pm 0.00\%$\\\hline
		\multirow{4}{*}{6} 
	& $0.05$ & $\bf 90.26 \pm 0.44\%$ & $\bf 9.74 \pm 0.44\%$ & $0.00 \pm 0.00\%$ & $0.00 \pm 0.00\%$\\
	& $0.15$ & $\bf 91.44 \pm 0.63\%$ & $\bf 4.75 \pm 0.35\%$ & $\bf 2.79 \pm 0.19\%$ & $\bf 1.02 \pm 0.23\%$\\
	& $0.25$ & $\bf 97.98 \pm 0.18\%$ & $\bf 1.28 \pm 0.06\%$ & $\bf 0.71 \pm 0.12\%$ & $\bf 0.03 \pm 0.02\%$\\
	& $0.35$ & $\bf 100.0 \pm 0.00\%$ & $0.00 \pm 0.00\%$ & $0.00 \pm 0.00\%$  & $0.00 \pm 0.00\%$\\\hline
\end{tabular}%
}
\newcommand{\labelssixTonine}{\small
\begin{tabular}{c|c|cccc}
	\multirow{2}{*}{\#label} & \multirow{2}{*}{$\epsilon$} & \multicolumn{4}{c}{$d_{\hat{\vect{\mathfrak{y}}}, \hat{\mathbb{Y}}}^{\epsilon}$}\\[.5em]\cline{3-6}
	& &	 $q_0$ & $q_{\leq 0.25}$ & $q_{\leq 0.5}$ & $q_{\leq 1}$ \\
	\hline
		\multirow{4}{*}{7} 
	& $0.05$ & $\bf 85.39 \pm 0.53\%$ & $\bf 13.91 \pm 0.47\%$ & $\bf 0.70 \pm 0.08\%$ & $0.00 \pm 0.00\%$\\
	& $0.15$ & $\bf 92.99 \pm 0.61\%$ & $\bf 6.62 \pm 0.58\%$ & $\bf 0.36 \pm 0.08\%$ & $\bf 0.03 \pm 0.02\%$\\
	& $0.25$ & $\bf 98.60 \pm 0.15\%$ & $\bf 0.37 \pm 0.07\%$ & $\bf 1.03 \pm 0.13\%$ & $0.00 \pm 0.00\%$\\
	& $0.35$ & $\bf 100.0 \pm 0.00\%$ & $0.00 \pm 0.00\%$ & $0.00 \pm 0.00\%$ & $0.00 \pm 0.00\%$\\\hline
		\multirow{4}{*}{8} 
	& $0.05$ & $\bf 78.61 \pm 1.35\%$ & $\bf 19.9 \pm 1.31\%$ & $\bf 1.49 \pm 0.25\%$ & $0.00 \pm 0.00\%$\\
	& $0.15$ & $\bf 91.66 \pm 0.33\%$ & $\bf 5.97 \pm 0.31\%$ & $\bf 1.78 \pm 0.15\%$ & $\bf 0.59 \pm 0.14\%$ \\
	& $0.25$ & $\bf 97.70 \pm 0.21\%$ & $\bf 1.66 \pm 0.21\%$ & $\bf 0.64 \pm 0.20\%$ & $0.00 \pm 0.00\%$\\
	& $0.35$ & $\bf 99.67 \pm 0.04\%$ & $0.00 \pm 0.00\%$ & $0.33 \pm 0.04\%$ & $0.00 \pm 0.00\%$\\\hline
		\multirow{4}{*}{9} 
	& $0.05$ & $\bf 76.25 \pm 0.60\%$ & $\bf 22.49 \pm 0.57\%$ & $\bf 1.26 \pm 0.16\%$ & $0.00 \pm 0.00\%$\\
	& $0.15$ & $\bf 91.11 \pm 0.76\%$ & $\bf 4.94 \pm 0.58\%$ & $\bf 3.30 \pm 0.33\%$ & $\bf 0.65 \pm 0.19\%$\\
	& $0.25$ & $\bf 99.46 \pm 0.08\%$ & $0.00 \pm 0.00\%$ & $\bf 0.54 \pm 0.08\%$ & $0.00 \pm 0.00\%$\\
	& $0.35$ & $\bf 99.85 \pm 0.09\%$ & $0.00 \pm 0.00\%$ & $\bf 0.15 \pm 0.09\%$ & $0.00 \pm 0.00\%$\\\hline
		\multirow{4}{*}{10} 
	& $0.05$ & $\bf 74.28 \pm 0.92\%$ & $\bf 25.03 \pm 0.96\%$ & $\bf 0.69 \pm 0.07\%$ & $0.00 \pm 0.00\%$\\
	& $0.15$ & $\bf 93.43 \pm 0.32\%$ & $\bf 4.44 \pm 0.34\%$ & $\bf 1.38 \pm 0.33\%$ & $\bf 0.75 \pm 0.25\%$\\
	& $0.25$ & $\bf 98.50 \pm 0.15\%$ & $0.00 \pm 0.00\%$ & $\bf 1.50 \pm 0.15\%$ & $0.00 \pm 0.00\%$\\
	& $0.35$ & $\bf 100.0 \pm 0.00\%$ & $0.00 \pm 0.00\%$ & $0.00 \pm 0.00\%$ & $0.00 \pm 0.00\%$\\\hline
	\multirow{4}{*}{11} 
	& $0.05$ & $\bf 73.63 \pm 0.60\%$ & $\bf 24.99 \pm 0.66\%$ & $\bf 1.38 \pm 0.13\%$ & $0.00 \pm 0.00\%$\\
	& $0.15$ & $\bf 93.72 \pm 0.64\%$ & $\bf 4.20 \pm 0.55\%$ & $\bf 2.08 \pm 0.56\%$ & $0.00 \pm 0.00\%$\\
	& $0.25$ & $\bf 97.20 \pm 0.20\%$ & $\bf 2.80 \pm 0.20\%$ & $0.00 \pm 0.00\%$ & $0.00 \pm 0.00\%$\\
	& $0.35$ & $\bf 100.0 \pm 0.00\%$ & $0.00 \pm 0.00\%$ & $0.00 \pm 0.00\%$ & $0.00 \pm 0.00\%$\\\hline
\end{tabular}%
}
\begin{table}[!ht]
	\centering
	\resizebox{\linewidth}{!}{%
	  \subfigure{\labelstwoTofive}%
	  \quad
  	  \subfigure{\labelssixTonine}%
	}%
	\caption{Average partitions amounts $q_*$ (\%) with confidence interval.}\label{tab:ressimul}
\end{table}
\setlength{\tabcolsep}{6pt} 
\renewcommand{\arraystretch}{1} 

We can summarise the main findings of those simulations as follows:
\begin{itemize}
    \item globally, $\hat{\vect{\mathfrak{y}}}_{\ell_H, \mathscr{P}}$ provides a quite accurate approximation of the true set, as it is exact (i.e., in $q_0$) most of the time;
    \item the quality of $\hat{\vect{\mathfrak{y}}}_{\ell_H, \mathscr{P}}$ decreases as the number of labels increases, making it unfit for applications involving a very high number of labels such as extreme multi-label~\cite{jain2016extreme};
	\item the quality of $\hat{\vect{\mathfrak{y}}}_{\ell_H, \mathscr{P}}$ seems to be the worst for moderate imprecision, probably because a high imprecision will tend to provide more vacuous (i.e., empty vectors) predictions;
	\item there are a few cases where $\hat{\vect{\mathfrak{y}}}_{\ell_H, \mathscr{P}}$ provides bad (i.e., are in $q_{\leq 0.5}$) to really bad approximation (i.e., are in $q_{\leq 1}$). This indicates that having exact inference methods may be helpful to identify those cases. 
\end{itemize}

All these last findings are confirmed in Figure~\ref{fig:evolutiontwoimprecision} that displays the evolutions of the partitions $q_0$ (left) and $q_{\leq 0.25}$(right). 
\begin{figure}[!th]
	\centering
	\resizebox{\linewidth}{!}{%
		\subfigure{
		   \includegraphics{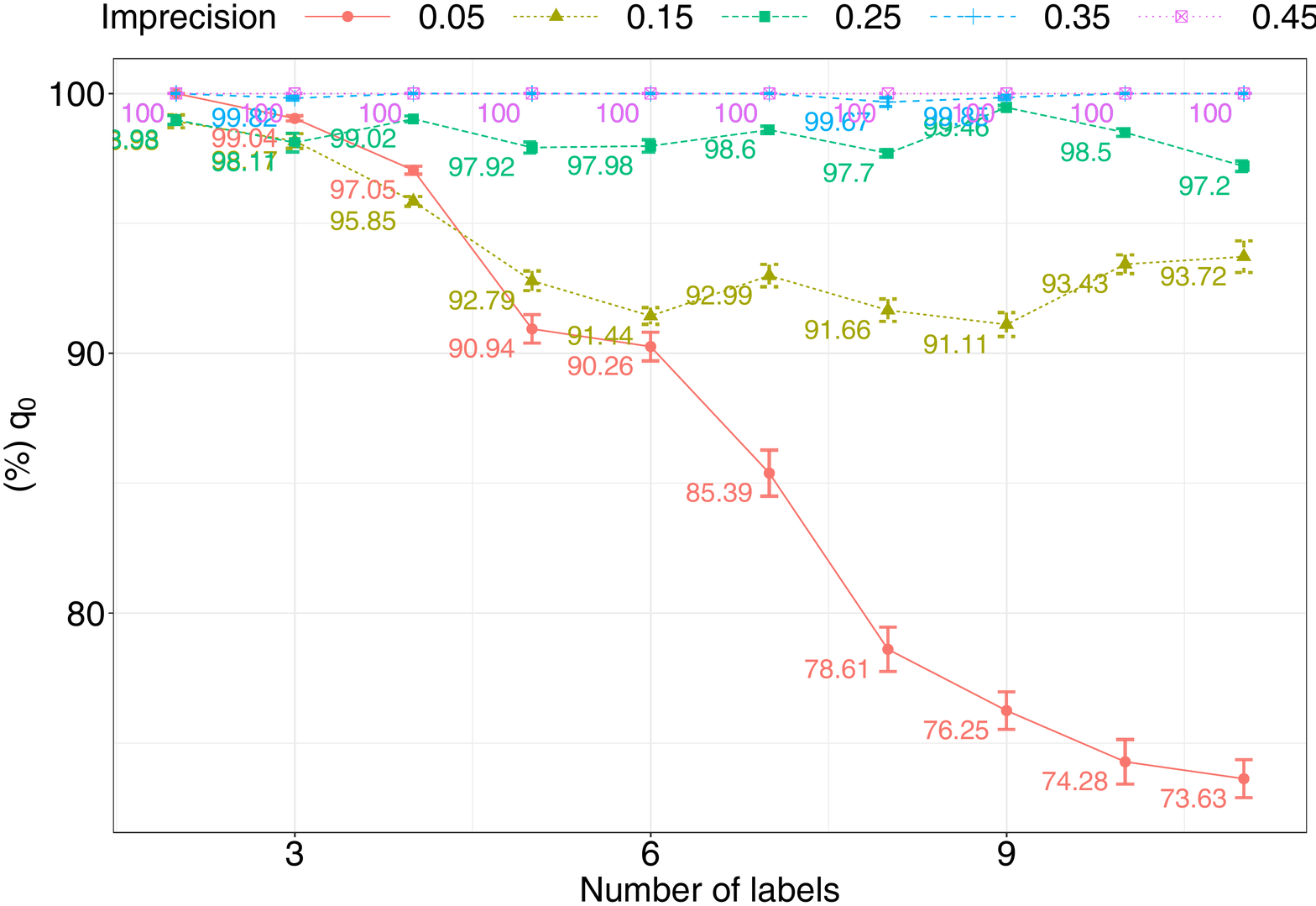}
		}
		\subfigure{
		   \includegraphics{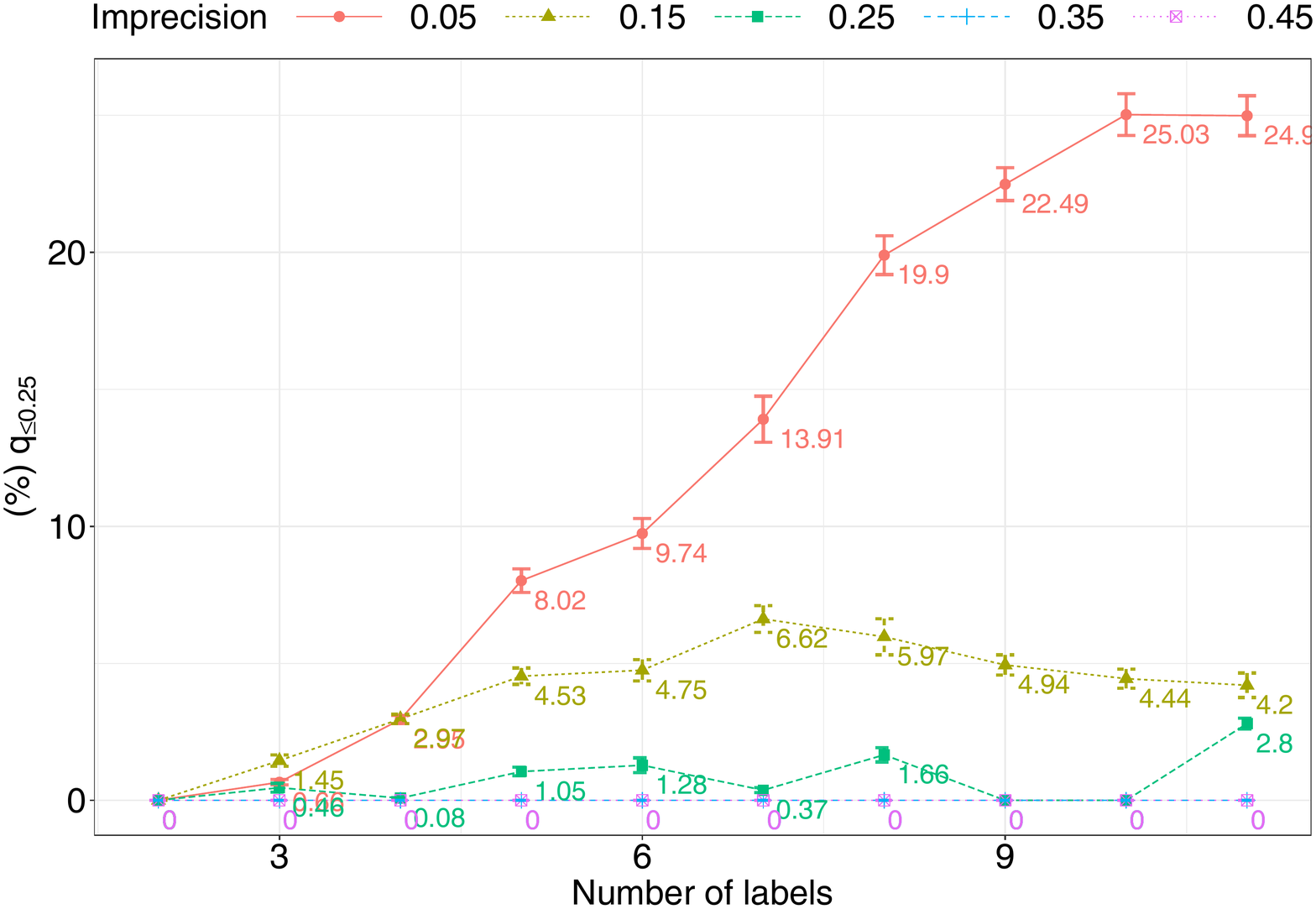}
		}
	}
	\caption{Evolution of average partitions amounts $q_*$ (\%) (with confidence interval) of the partition $q_0$ (left) and $q_{\leq 0.25}$(right)}%
	\label{fig:evolutiontwoimprecision}%
\end{figure}


In what follows, we perform additional experimental studies on real data sets to check how skeptic inferences for multi-label problems behave: (1) in presence of noisy or missing labels and (2) in comparison to the other existing skeptic strategies, namely the partial abstention and the rejection rules.

\section{Experiments of Skeptic inferences with Binary relevance}
\label{sec:expebinarybr}
In Section~\ref{sec:expnoisedatasets} and \ref{sec:downsampling}, we perform a set of additional experiments on real-world data sets\footnote{Implemented in Python, see~\url{https://github.com/sdestercke/classifip}.} to investigate the usefulness of using skeptical inferences in multi-label problems, rather than precisely-valued inferences, or partial abstentions~\cite{nguyen2019}, or rejection rules, when uncertainties are higher due to imperfect information. The former aims, under the assumption of {\bf strong independence} on labels, to verify how skeptical and precisely-valued inferences cope with two different settings: missing and noisy labels. While the latter aims to compare how the partial abstention, the reject option, and our approach behave in terms of skeptic predictability when the amount of training data considerably decreases.

\subsection{Experiments with noisy and missing data sets}
\label{sec:expnoisedatasets}

In this subsection, we perform a set of experiments to investigate the usefulness of using skeptic inferences in multi-label problems. In particular, we investigate what happens when some labels are noisy or missing. 
To that end, we use a set of standard real-word data sets from the MULAN repository\footnote{\url{http://mulan.sourceforge.net/datasets.html}} (c.f. Table~\ref{tab:datasets}), following a $10\!\times\!10$ cross-validation procedure to fit the model.

\begin{table}[!ht]
	\centering
	\resizebox{!}{!}{%
	\begin{tabular}{ccccccc}
		 Data set & Domain &\#Features & \#Labels & \#Instances & \#Cardinality & \#Density \\
		\hline
		emotions & music & 72 & 6 & 593 & 1.90 & 0.31 \\
	 	scene & image & 294 & 6 & 2407 & 1.07 & 0.18 \\
		yeast & biology & 103 & 14 & 2417 & 4.23 & 0.30
	\end{tabular}}
	\caption{Multi-label data sets summary}\label{tab:datasets}
\end{table}

\paragraph{\bf Evaluation} As we perform set-valued predictions, usual measures used in multi-label problems cannot be adopted here. We thus consider it appropriate to use an incorrectness measure (IC), coupled with a completeness (CP) measure~\cite[\S 4.1]{destercke2014multilabel}, defined as follows
\begin{align}
	IC(\hat{\mathbb{Y}}, \bm y) &= \frac{1}{|Q|} \sum_{\hat{y}_i \in Q} \mathbbm{1}_{(\hat{y}_i\neq y_i)}, \label{eq:incorrectness} \\
	CP(\hat{\mathbb{Y}}, \bm y) &= \frac{|Q|}{m},
		\label{eq:completeness}
\end{align}
where $Q$ denotes the set of predicted label such that $\hat{y}_i=1$ or $\hat{y}_i=0$ (in other words any abstained predicted label $\hat{y}_i=*$ is not in $Q$). When predicting complete vectors, then $CP=1$ and $IC$ equals the Hamming loss~\eqref{eq:Hloss}, and when predicting the empty vector, i.e. all labels equals to $\hat{y}_i=*$, then $CP=0$ and by convention $IC=0$. Since those measures are adapted to partial vectors, we will use a simple binary relevance strategy in the experiments.

\paragraph{Naive Credal classifier} To obtain marginal probability intervals over each label, i.e.$[\overline{p}_i, \underline{p}_i], \forall i\in \setn{m}$, we use an imprecise classifier called the naïve credal classifier (NCC)\footnote{Bearing in mind that it can be replaced by any other (credal) imprecise classifier, see~\cite[\S 10]{augustin2014introduction} or \cite{alarcon2021imprecise}}~\cite{zaffalon2002naive}, which extends the classical naive Bayes classifier (NBC). We refer to \ref{app:nccmethod} for further technical details on how to obtain these probability intervals, and will only recall here that the imprecision of this classifier (i.e. the width  of probability intervals) is regulated by  the value of the hyper-parameter $s\in\mathbb{R}$, with the imprecision being higher as $s$ increases (for $s=0.0$, we retrieve basic empirical frequencies estimate, in other words, the NBC model). 

In this paper, we restrict the values of the hyper-parameter of the imprecision to $s\in\{0.0, 0.5, 1.5, 2.5, 3.5, 4.5, 5.5\}$. Our purpose here is not to find the ``optimal'' value of $s$, but to show the effectiveness of injecting imprecision (i.e., to provide robust and skeptical inferences). As the NCC requires discrete features, when those were continuous we simply discretized in $z$ equal-width intervals, with two levels of discretization $z\!=\!5$ and $z\!=\!6$.

\paragraph{Missing labels} To simulate missingness, we proceed uniformly to pick at random a percentage of missing labels, with five different levels of missingness: $\{0, 20,  40, 60, 80\}$. Missing values are removed from the training data. Table~\ref{tab:noisemissing} illustrates a data set data with missing values (or partially labelled instances).
\begin{table}[ht]
	\centering
	\begin{tabular}{ccccc|ccc|ccc|ccc}
		\multicolumn{5}{c|}{{\bf Features}} & \multicolumn{3}{|c|}{{\bf Missing}} & \multicolumn{3}{|c}{{\bf Noise-Reversing}} & \multicolumn{3}{|c}{{\bf Noise-Flipping}} \\
	\hline 
	$X_1$ & $X_2$ & $X_3$ & $X_4$ & $X_5$& $Y_1$ &$ Y_2$ & $Y_3$ & $Y_1$ &$ Y_2$ & $Y_3$ & $Y_1$ &$ Y_2$ & $Y_3$ \\
	\hline
	107.1 & 25  & Blue  & 60 & 1 & 1  & 0  & * & 1  & 0$\rightarrow$1  & 0 & 1  & 1$\wedge_\beta$0  & 0 \\
	-50   & 10  & Red   & 40 & 0 & 1  & *  & 1 & 1  & 0  & 1$\rightarrow$0 & 1  & 0  & 1$\wedge_\beta$0\\
	200.6 & 30  & Blue  & 58 & 1 & *  & 1  & 0 & 0$\rightarrow$1  & 0  & 0 &  0 & 0  & 0\\
	107.1 & 5   & Green & 33 & 0 & *  & 1  & 0 & 1  & 1$\rightarrow$0  & 0 & 1$\wedge_\beta$0 & 1$\wedge_\beta$0  & 0\\
	\dots & \dots & \dots & \dots & \dots & \dots  & \dots  & \dots & \dots  & \dots  & \dots & \dots  & \dots  & \dots\\
	\end{tabular}
	\caption{Missing and Noise representation of labels}
	\label{tab:noisemissing}
\end{table}

In Figure~\ref{fig:expmissing}, we provide the results of the incorrectness\footnote{The results of the completeness measure had been placed in the~\ref{app:missinglabels} in order to simplify the narrative.} measure obtained by fitting the NCC model on different percentages of missing labels and data sets of Table~\ref{tab:datasets}. 

The results show that as the percentage of missing labels increases the incorrectness and the completeness both decrease, especially on Emotions and Scene data sets. This means that the more imprecise we get, the more accurate are those predictions we retain. The effect is less significant on the Yeast data set, where one needs a high amount of imprecision to witness a gain in correctness. One quite noticeable result is that for the Emotions data set, even with $80\%$ of missing label, a light imprecision ($s=0.5$ ) allows us to reach a reasonable completeness of about $80\%$ with a gain of $5\%$ in terms of correct predictions. Besides, as the confidence intervals shown in Figure~\ref{fig:expmissing} are very small, and will remain so for the other settings, we thus prefer not to display them in order not to overcharge future figures.

\begin{figure}[!th]
	\centering
	\resizebox{0.48\textwidth}{!}{%
	  \includegraphics{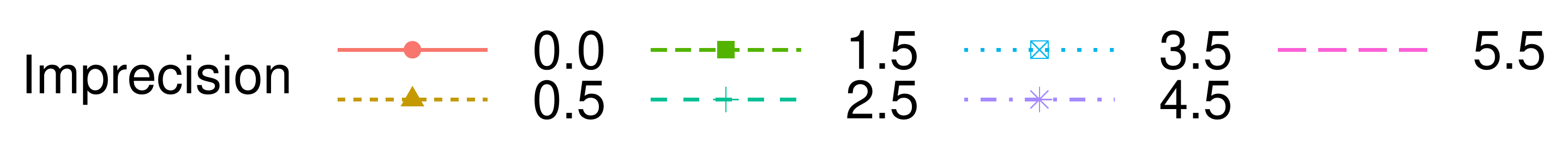}
	}\qquad%
	\subfigure[\sc Emotions]{
	   \includegraphics[width=0.32\linewidth]{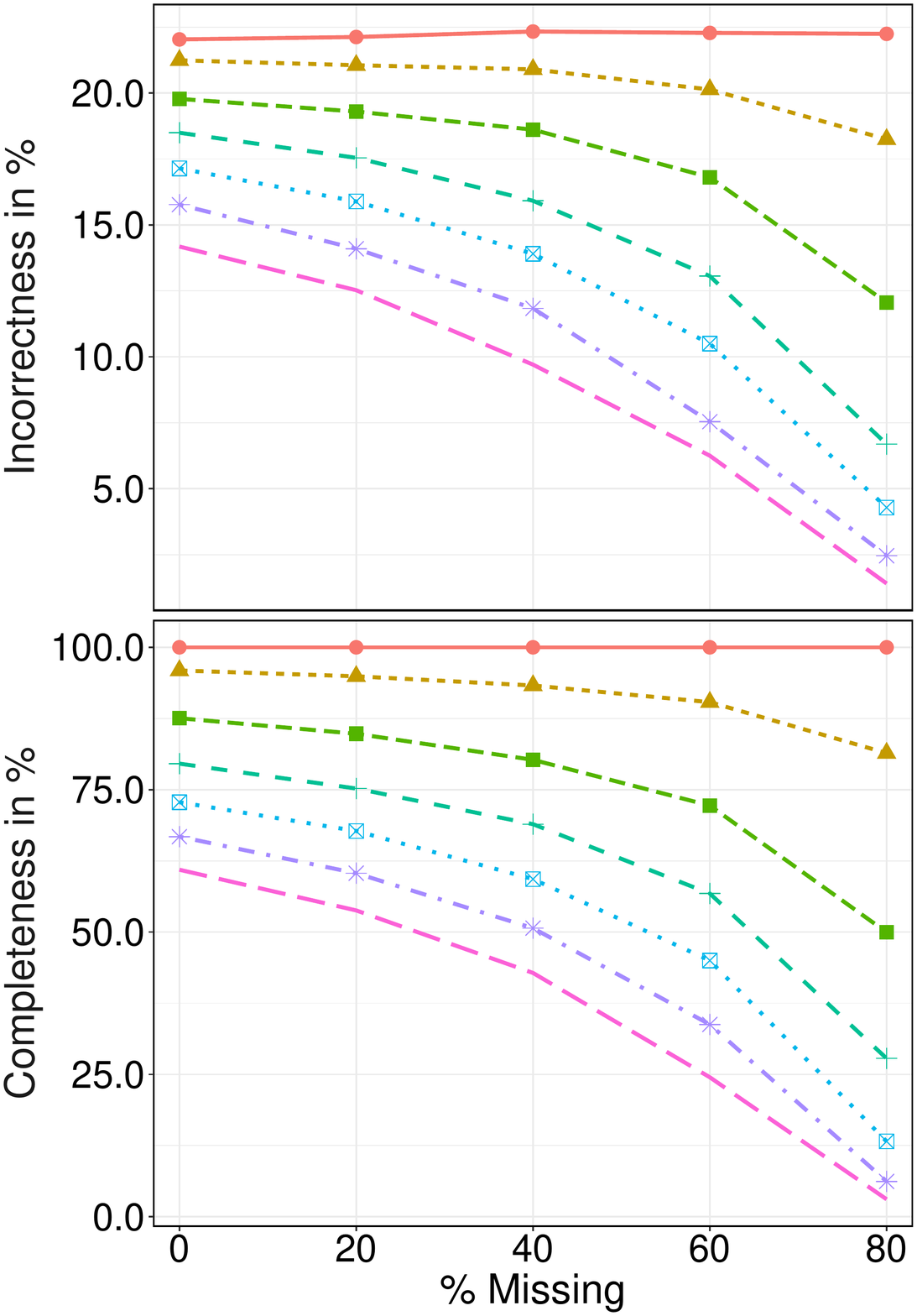}
	}%
	\subfigure[\sc Scene]{
		\includegraphics[width=0.32\linewidth]{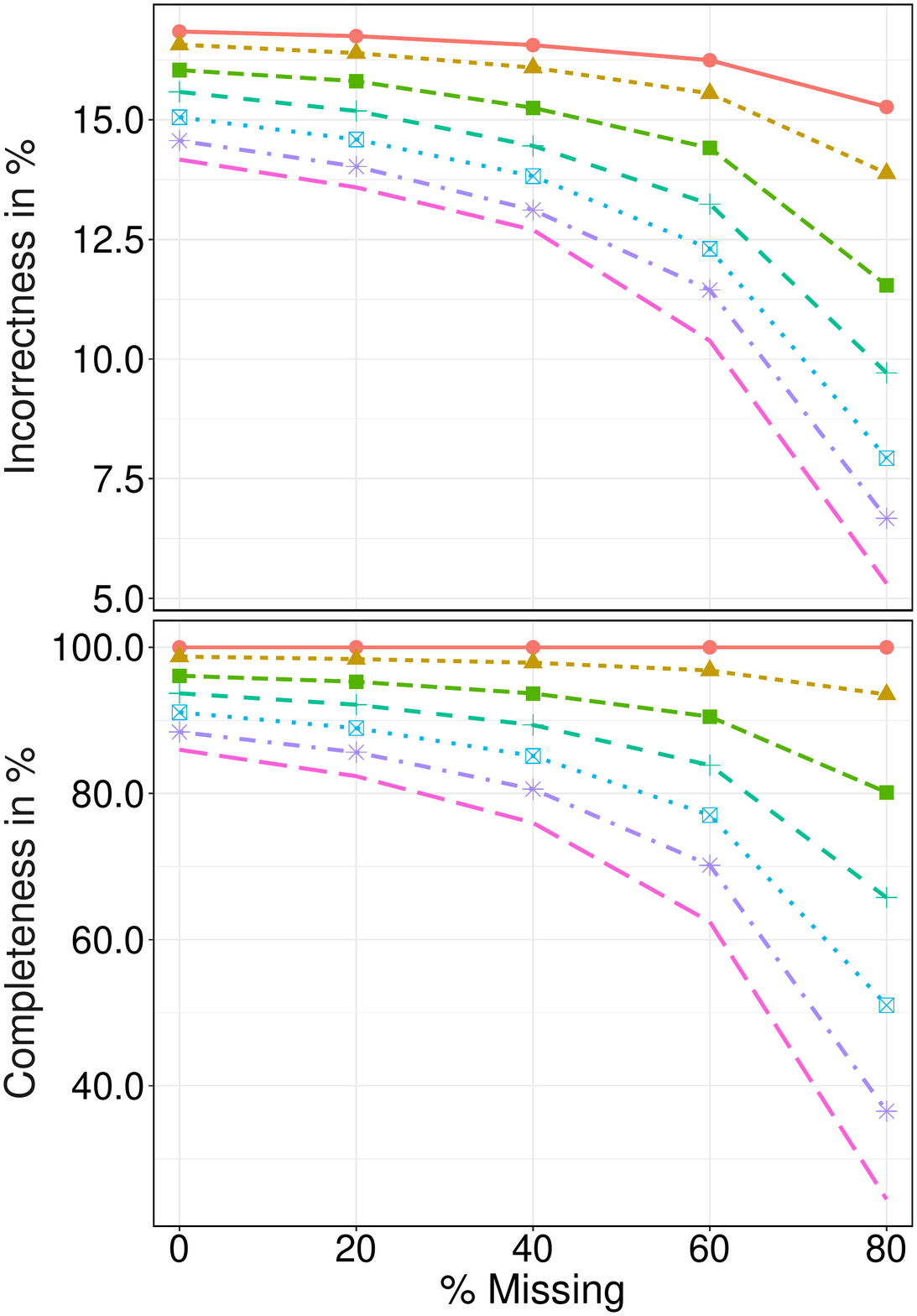}
	}%
	\subfigure[\sc Yeast]{
		\includegraphics[width=0.32\linewidth]{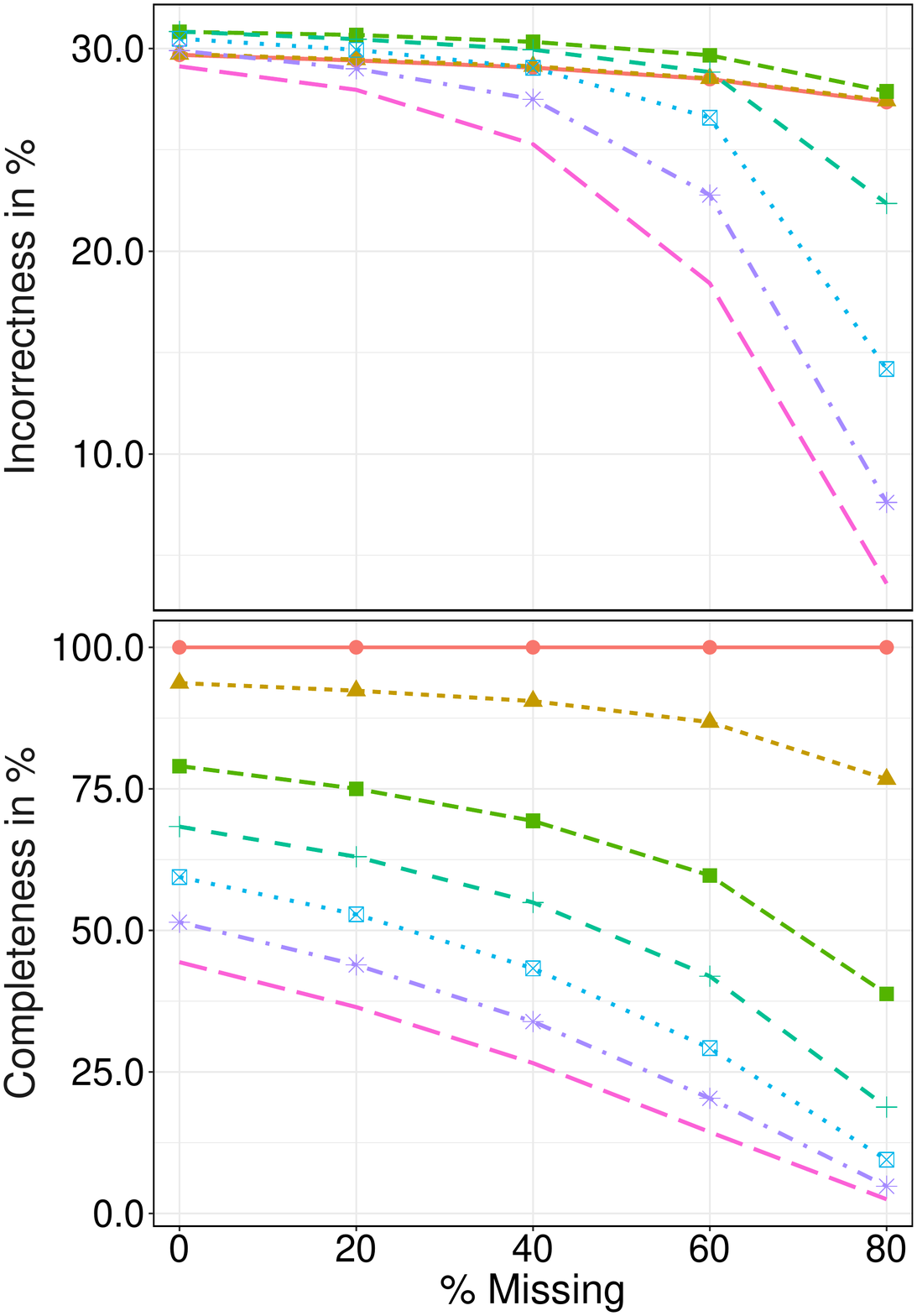}
	}%
	\caption{\textbf{Missing labels}. Evolution of the average incorrectness (top) and completeness (bottom) for each level of imprecision (a curve for each one) and a discretization $z\!=\!5$, and with respect to different percentages of missing labels (x-axis).}\label{fig:expmissing}%
\end{figure}

Results obtained are sufficient to show that skeptic inferences with probability sets may provide additional benefits when dealing with missing labels. Those results could, of course, be improved by picking other classifiers, such as the NCC2~\cite{corani2008learning}, an extension of the NCC tailored for missing values, or imprecise classifiers able to cope with continuous attributes~\cite{alarcon2019imprecise}. 

\paragraph{Noisy labels} In this setting, we proceed in the same way as with the missing setting (with different percentages of noise depending on how it will be set up), except that the value of selected labels are not assigned to $*$ (i.e., as a missing value), but are modified according to some noise scheme. We consider two different ways to modify the assignments:
\begin{enumerate}
	\item {\bf Reversing:} in this case, we reverse the current value of the selected label. In other words, if $Y_{j,i}=1$, the label $j$ of the instance $i$ becomes $Y_{j,i}=0$ (similar for the case of $Y_{j,i}=0 \rightarrow Y_{j,i}=1$), with six different noise levels $\{10, 20, 30, 40, 50, 60\}$,
	\item {\bf Flipping:} in contrast to the previous case, for each selected label $Y_{j,i}$, we replace it by the result of a Bernoulli trial with probability $\beta:=P(Y_{j,i}=1)$, i.e. $Y_{j,i} \sim \mathcal{B}er(\beta)$, with $\beta\in\{0.2, 0.5, 0.8\}$, and with three noise levels $\{40, 60, 80\}$.
\end{enumerate}
Table~\ref{tab:noisemissing} provides an illustration of these two noise settings, in the columns ``Noise-Reversing'' ($\rightarrow$) and ``Noise-Flipping'' ($\wedge_\beta$), respectively. 

\begin{figure}[!th]
	\centering
	\resizebox{0.48\textwidth}{!}{%
	  \includegraphics{images/missing/legends}
	}\qquad%
	\subfigure[\sc Emotions]{
	   \includegraphics[width=0.32\linewidth]{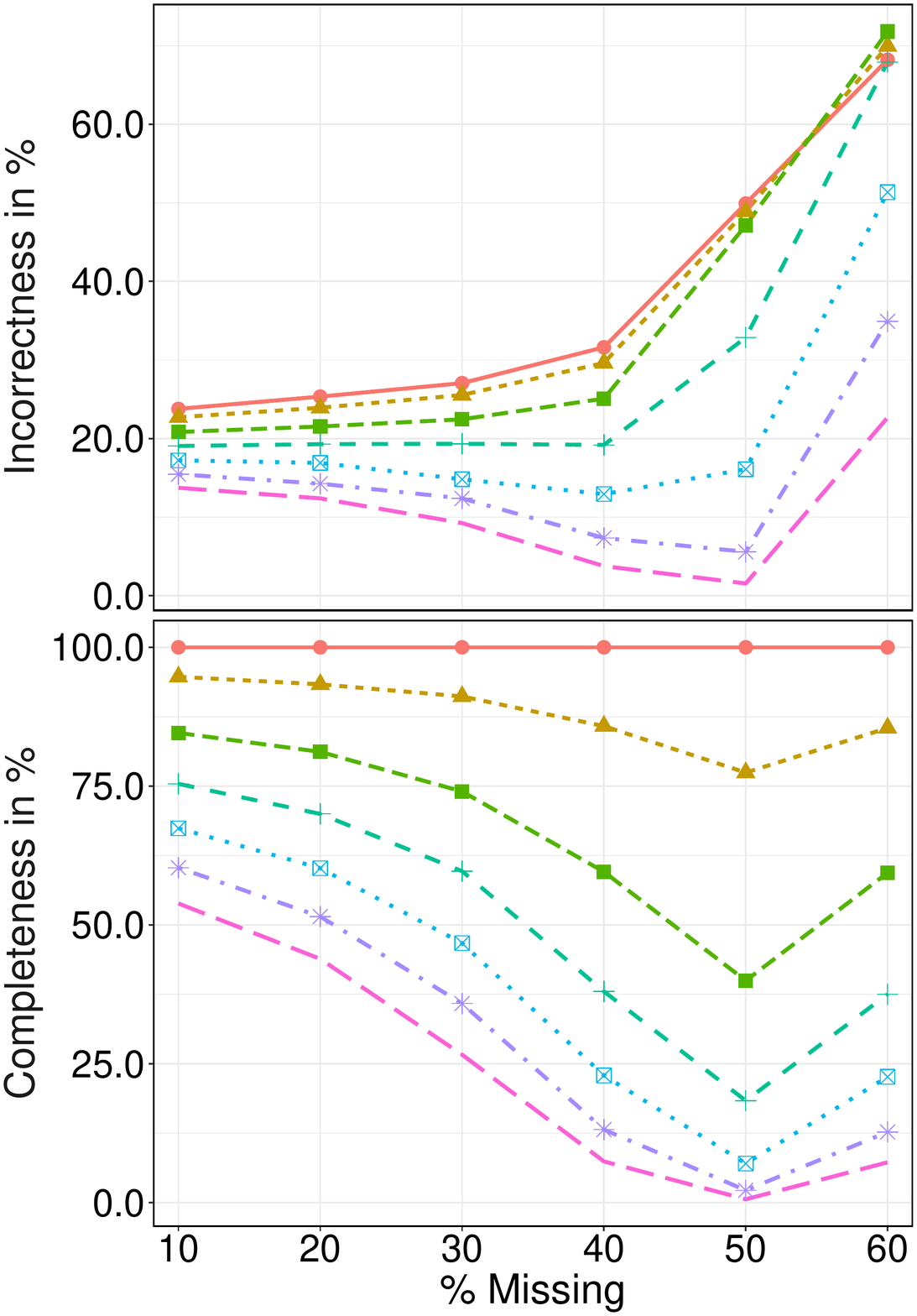}
	}%
	\subfigure[\sc Scene]{
		\includegraphics[width=0.32\linewidth]{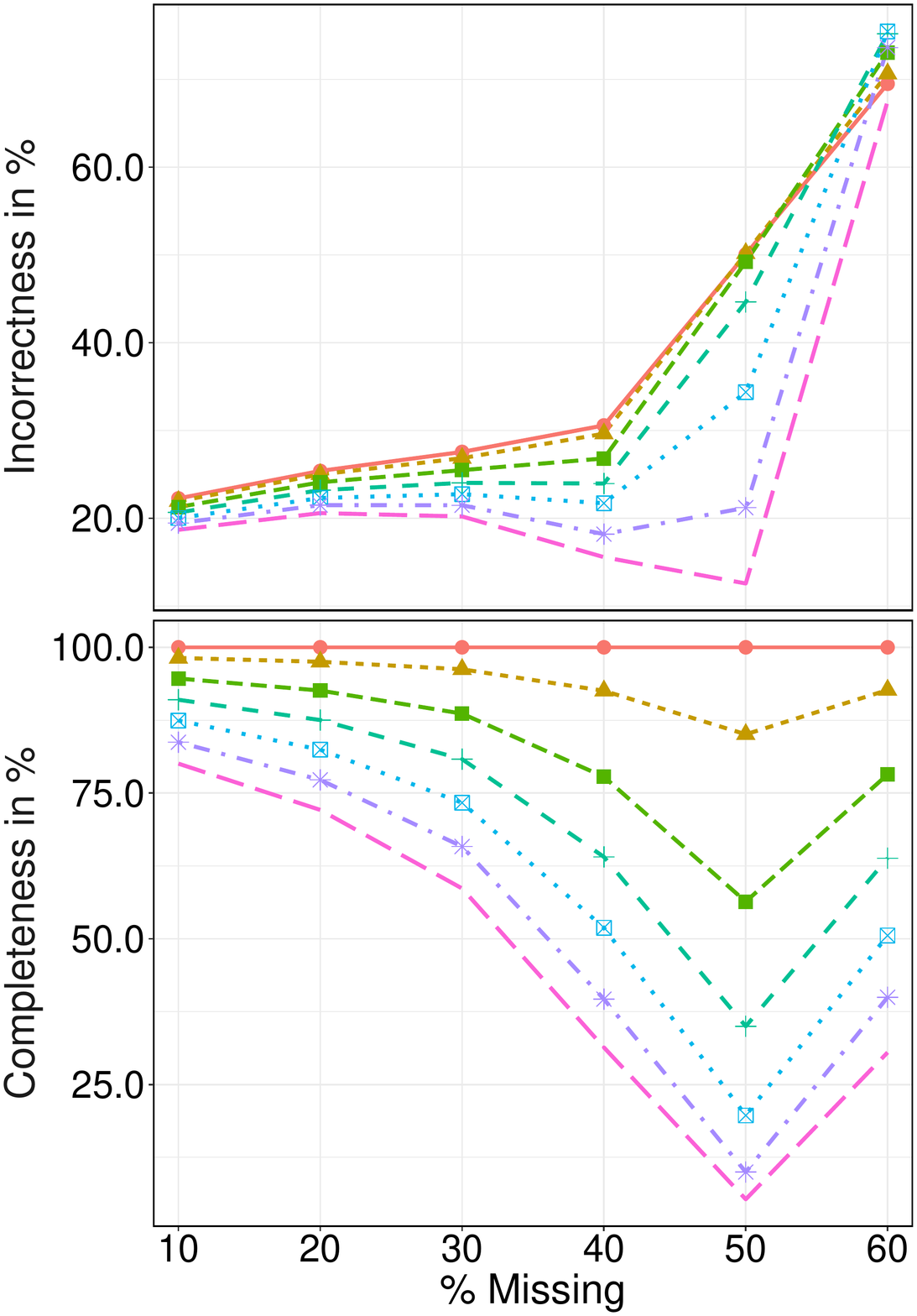}
	}%
	\subfigure[\sc Yeast]{
		\includegraphics[width=0.32\linewidth]{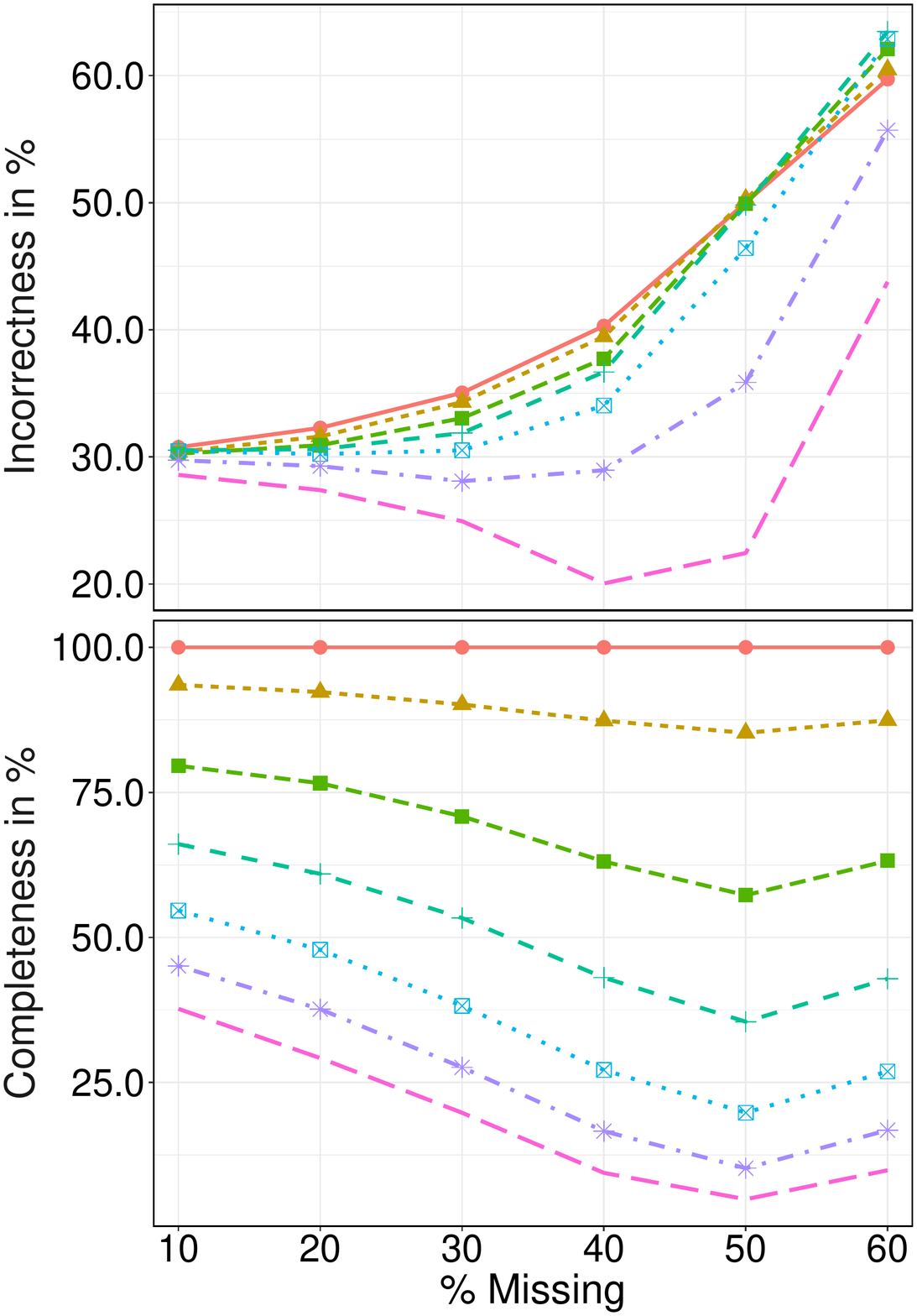}
	}%
	\caption{\textbf{Noise-Reversing}. Evolution of the average incorrectness (top) and completeness (bottom) for each level of imprecision (a curve for each one) and a discretization $z\!=\!5$, and with respect to different percentages of noisy labels (x-axis).}\label{fig:expnoisereverse}%
\end{figure}

In Figures~\ref{fig:expnoisereverse} and \ref{fig:expnoiseflipping} respectively, we provide the results of the incorrectness\footnote{The results of the completeness measure had been placed in the~\ref{app:noisereversing} and~\ref{app:noiseflipping}, for Reversing and Flipping settings, respectively.} obtained by fitting the NCC model on different percentages of \textbf{Reversing} and \textbf{Flipping} settings applied to the data sets of Table~\ref{tab:datasets}.

Concerning the {\bf Reversing}, or adversarial setting, it is clear from the graphs that allowing for imprecision and skeptical inferences provides some level of protection, which can be witnessed by the fact that at a given level of noise, including some imprecision limits the increase in incorrectness, and sometimes even improves the quality of the made predictions by abstaining on those instances where adversarial noise was introduced. Of course, this goes hand-in-hand with a corresponding decrease of completeness, but this seems a fair price to pay to protect against adversarial noise. 

\begin{figure}[!th]
	\centering
	\resizebox{0.48\textwidth}{!}{%
	  \includegraphics{images/missing/legends}
	}\qquad%
	\subfigure[\sc Emotions]{
	   \hspace{-3mm}
	   \includegraphics[width=0.33\linewidth]{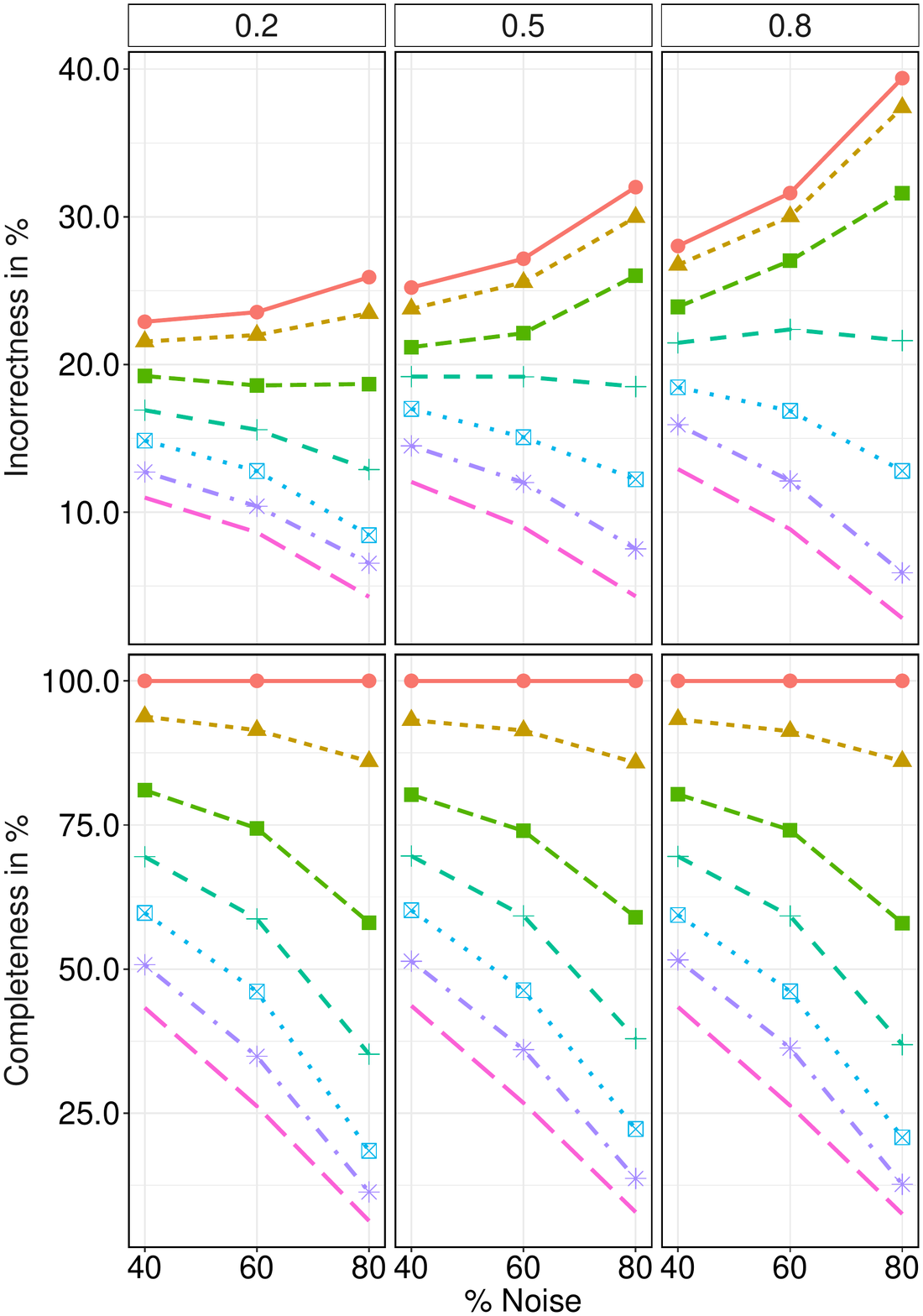}
	}%
	\subfigure[\sc Scene]{
		\includegraphics[width=0.33\linewidth]{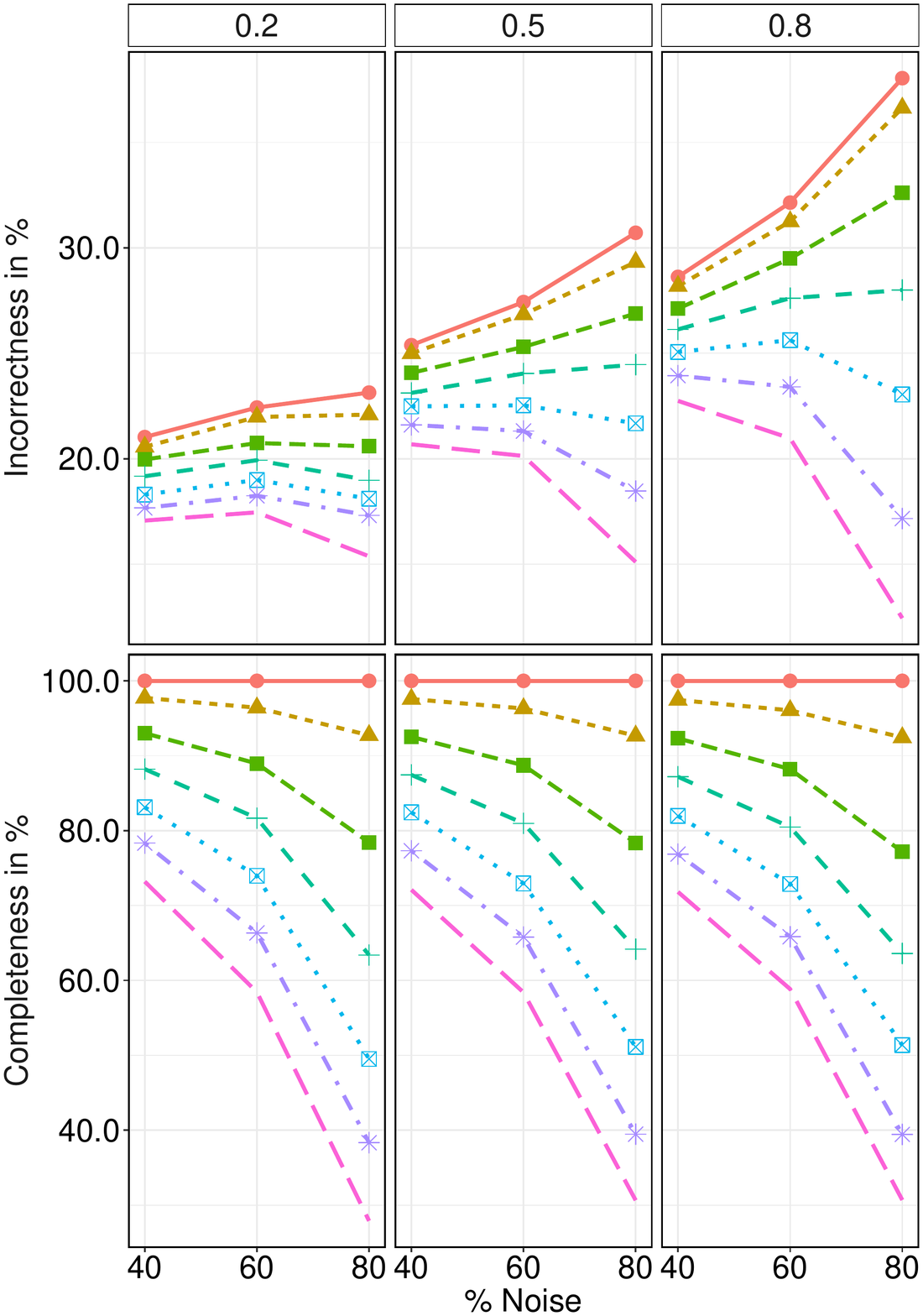}
	}%
	\subfigure[\sc Yeast]{
		\includegraphics[width=0.33\linewidth]{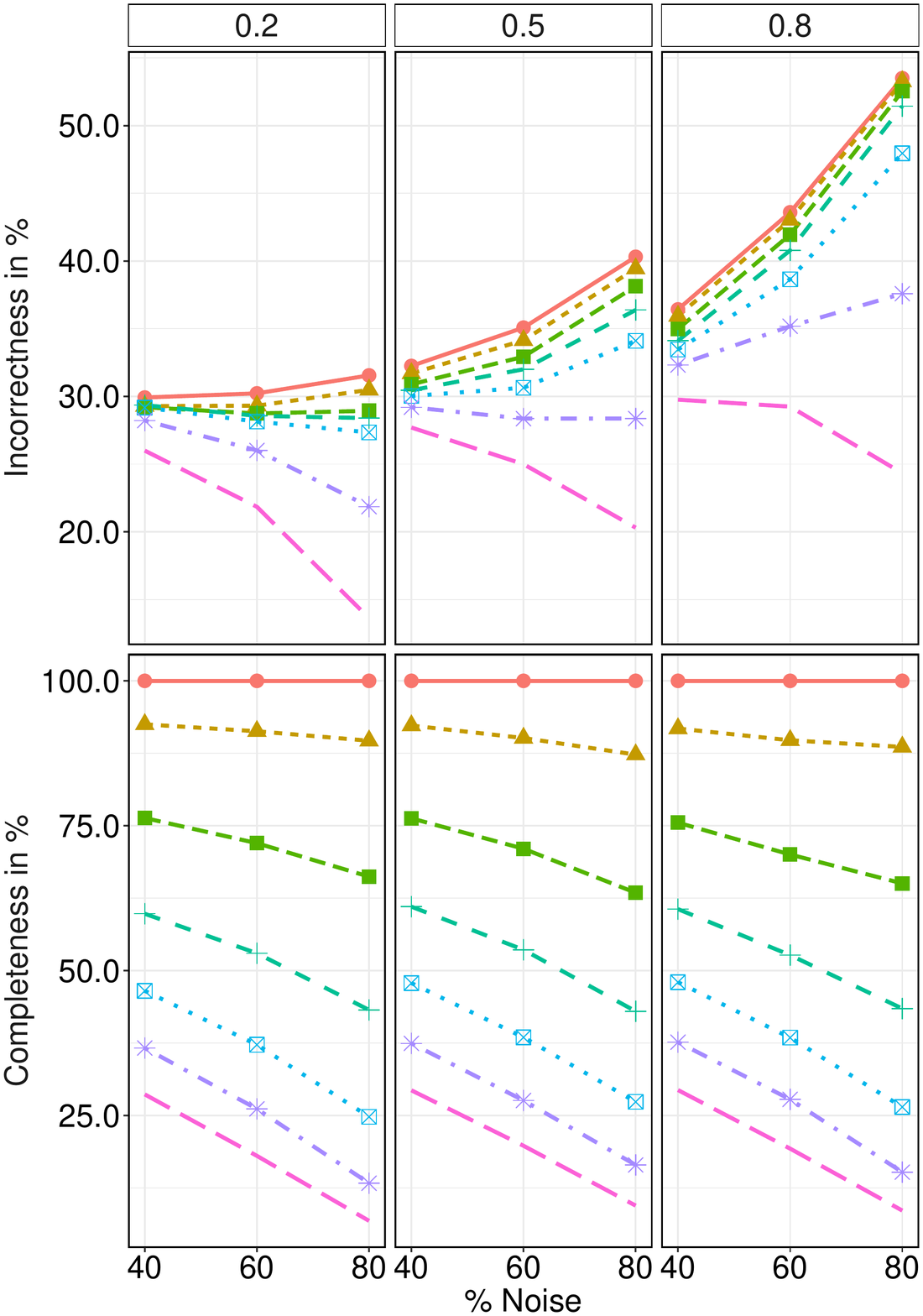}
	}%
	\caption{\textbf{Noise-Flipping}. Evolution of the average incorrectness (top) and completeness (bottom) for each level of imprecision (a curve for each one), a level of discretization $z\!=\!5$, and three different probabilities $\beta\!=\!0.2$ (left), $\beta\!=\!0.5$ (middle) and $\beta\!=\!0.8$ (right) of replacing the selected label with a $1$, and with respect to different percentages of noisy labels (x-axis).}\label{fig:expnoiseflipping}%
\end{figure}

Results obtained of the {\bf Flipping} setting are overall similar to those found in the missing label and the \textbf{Reversing} label. Notable small differences are that (1) skeptical inferences are uniformly more robust (provide more accurate predictions) than their precise counter-part, whatever the level and nature of noise, and (2) the evenness of the noise ($\beta$ value) obviously has an impact on performances, but has little impact on the overall trends.  

Similarly to the case of the missing label, it would be interesting to experiment with other imprecise classifier, as well as with other different noise settings (e.g. using other probability distributions as $\beta \sim \mathcal{U}([0,1])$ then $Y_{i,j} = 1$ if $\beta > \rho$ otherwise $Y_{i,j}=0$, where $\rho$ is a threshold parameter).

\paragraph{Rejection threshold}  With the aim of performing a preliminary experimental comparative study between the reject rule and our skeptical approach, we make use of the above-mentioned missing and noisy setups (with $\beta=\{0.2, 0.8\}$ instead).  

For the sake of simplicity, we relax Equation~\eqref{eq:optiHamPrec} by adding a threshold parameter $\gamma \in \{0.00, 0.15, 0.25, 0.35, \allowbreak 0.45\}$ (with $\gamma=0.00$ bringing back to Equation~\eqref{eq:optiHamPrec}) so that it can reproduce a rejection region\footnote{The rejection region will be delimited by two parallel lines, corresponding to the iso-density points $x$ for which $P_\newinstance(Y_{\{i\}}=1)=\frac{1}{2} - \gamma$ and $P_\newinstance(Y_{\{i\}}=1)=\frac{1}{2} + \gamma$.}, as follows 
\begin{equation}\label{eq:optiHamPrecThreshold}
\hat{y}_{i,\ell_H,\gamma}=
\begin{cases} 
	1 & \text{ if } P_\newinstance (Y_{\{i\}}=1) > \frac{1}{2} + \gamma, \\ 
	0 & \text{ if } P_\newinstance(Y_{\{i\}}=1) \leq \frac{1}{2} - \gamma, \\ 
	* & \text{ if } P_\newinstance(Y_{\{i\}}=1) \in \left]\frac{1}{2} - \gamma, \frac{1}{2} + \gamma\right[
\end{cases}.
\end{equation}
In order to employ Equation \eqref{eq:optiHamPrecThreshold}, it is necessary first to compute the estimated conditional probabilities $P_\newinstance(\cdot)$ generated from NCB model (i.e., when the hyper-parameter value of NCC model is $s=0.0$), and also keep in mind that there is no connection between the hyper-parameter $s$ of NCC model and the threshold parameter $\gamma$.

The results of the incorrectness, for the missing and noisy setups, obtained by fitting the NCC model on data sets of Table~\ref{tab:datasets}, with discretization $z=6$ and a single level of imprecise $s$ per setup, are presented in Figure~\ref{fig:rejectionmissnoise} \footnote{The rest of experiments, i.e. other levels of imprecision and the completeness measure, have been placed in~\ref{app:rejection}}.
\begin{figure}[!ht]
	\centering%
	\resizebox{0.8\textwidth}{!}{%
	  \input{images/reject/legends_reject}
	}\vspace{-2mm}%
	\renewcommand{\thesubfigure}{(a)}
	\subfigure[{\bf Missing} setting with a hyper-parameter value of NCC $s=1.5$]{
		\hspace{-2mm}
		\subfigure[{\scshape Emotions}]{
			\includegraphics[width=0.324\linewidth]
				{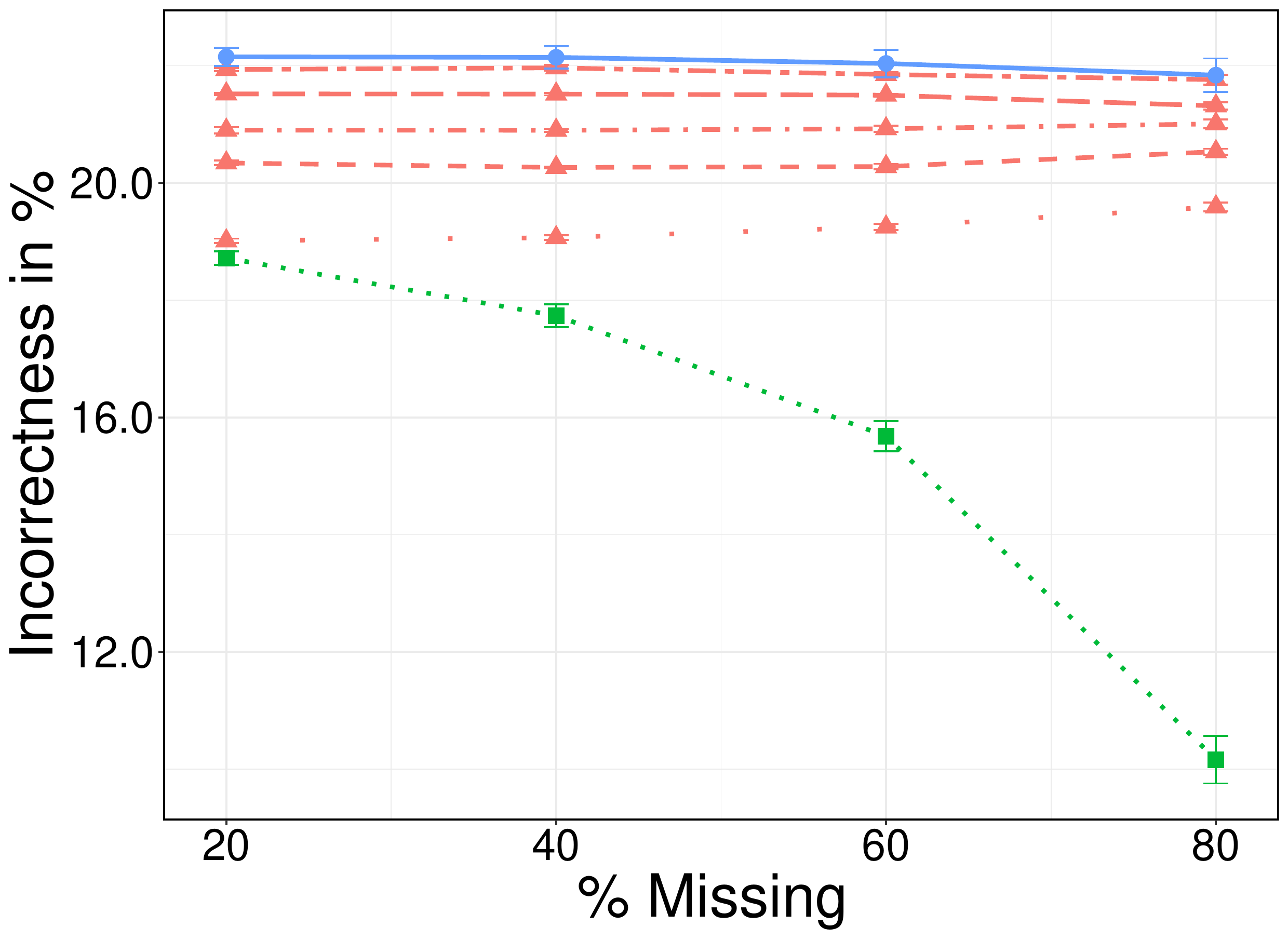}} 
		\subfigure[\scshape Scene]{
			\includegraphics[width=0.324\linewidth]
				{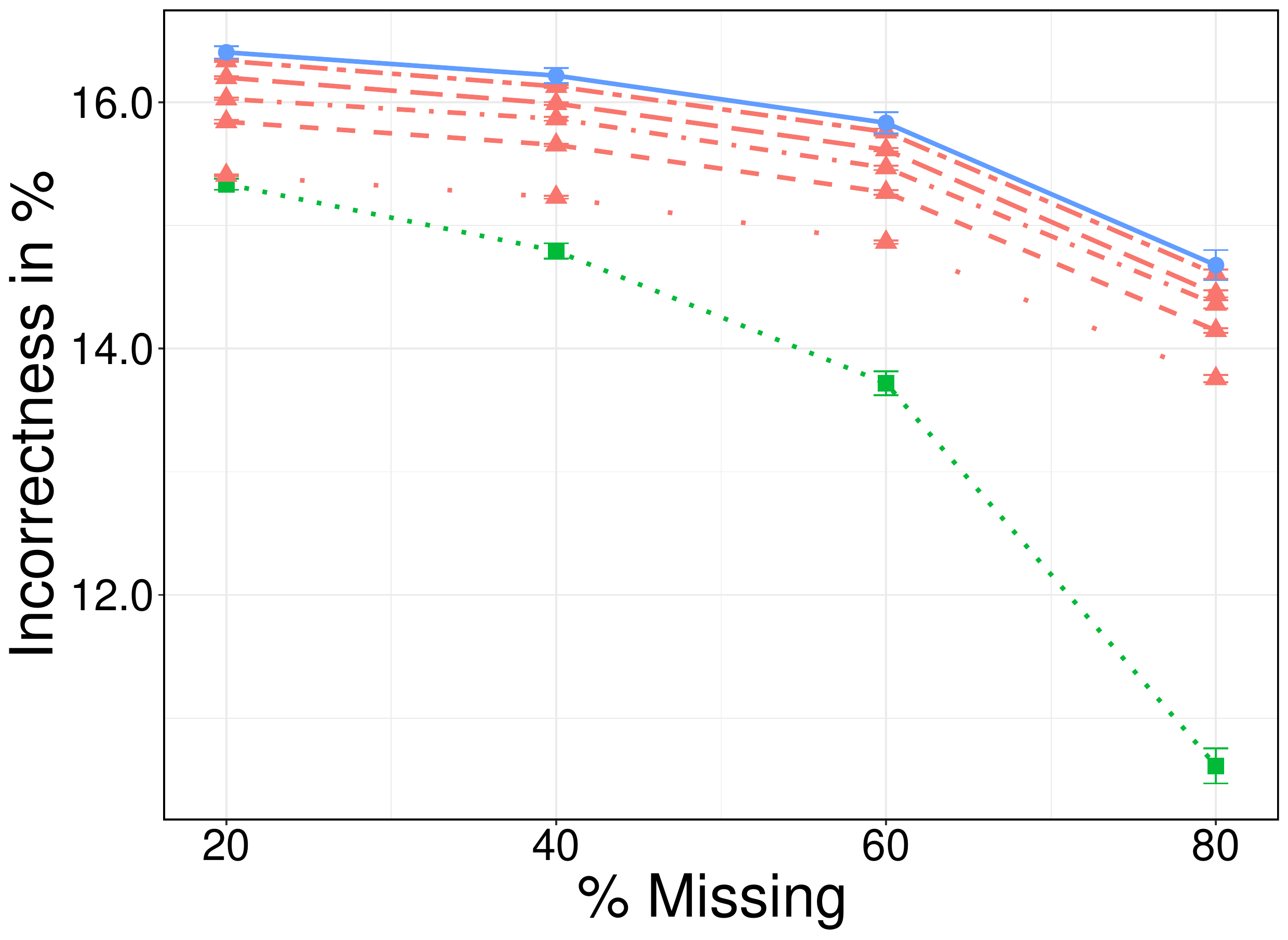}} 
		\subfigure[\scshape Yeast]{
			\includegraphics[width=0.324\linewidth]
				{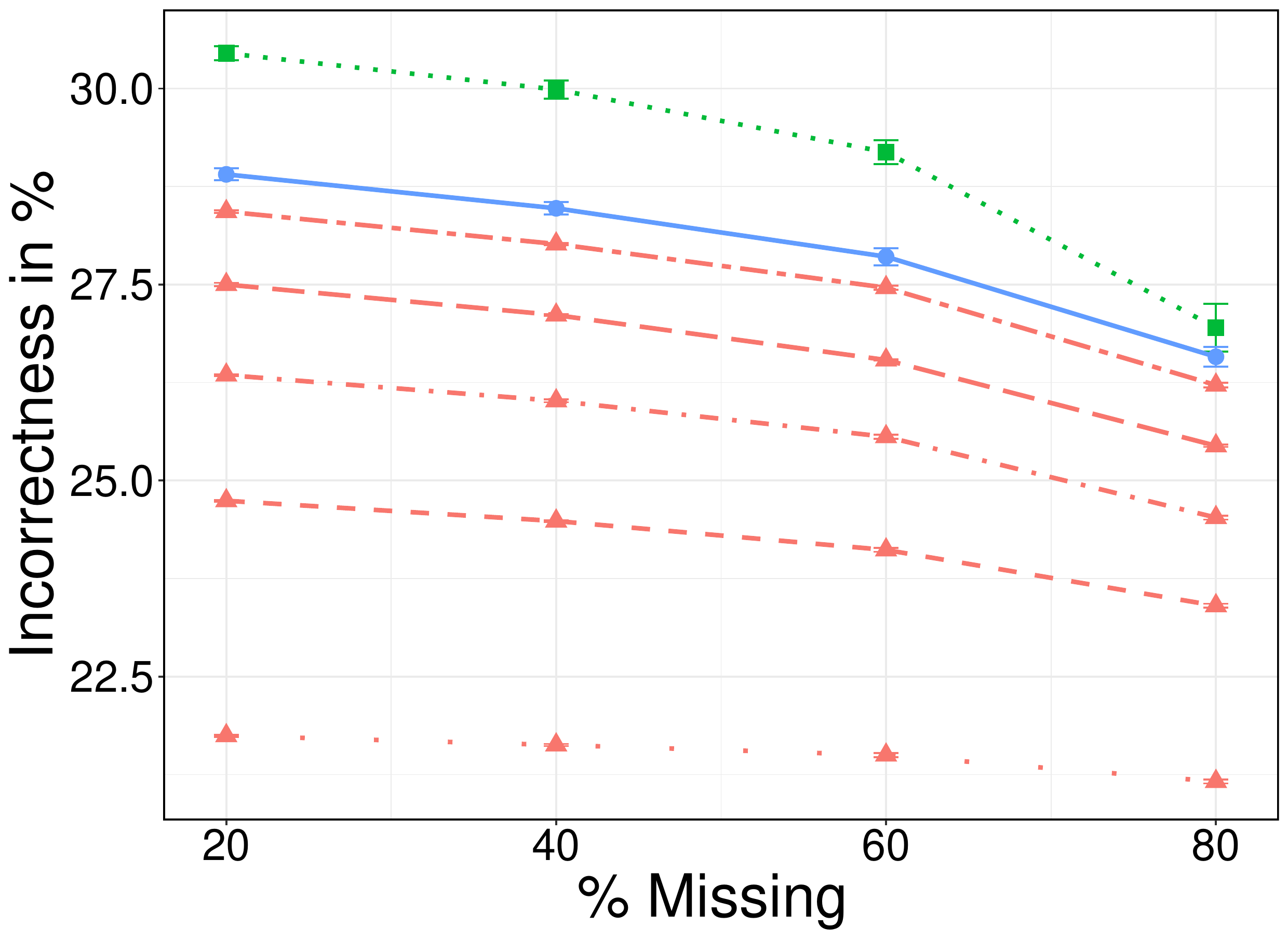}} 
	}\vspace{-3mm}
    \renewcommand{\thesubfigure}{(c)}
	\subfigure[{\bf Noise-Reversing} with a hyper-parameter value of NCC $s=2.5$]{
		\hspace{-2mm}
	  	\subfigure[\scshape Emotions]{
			\includegraphics[width=0.324\linewidth]
				{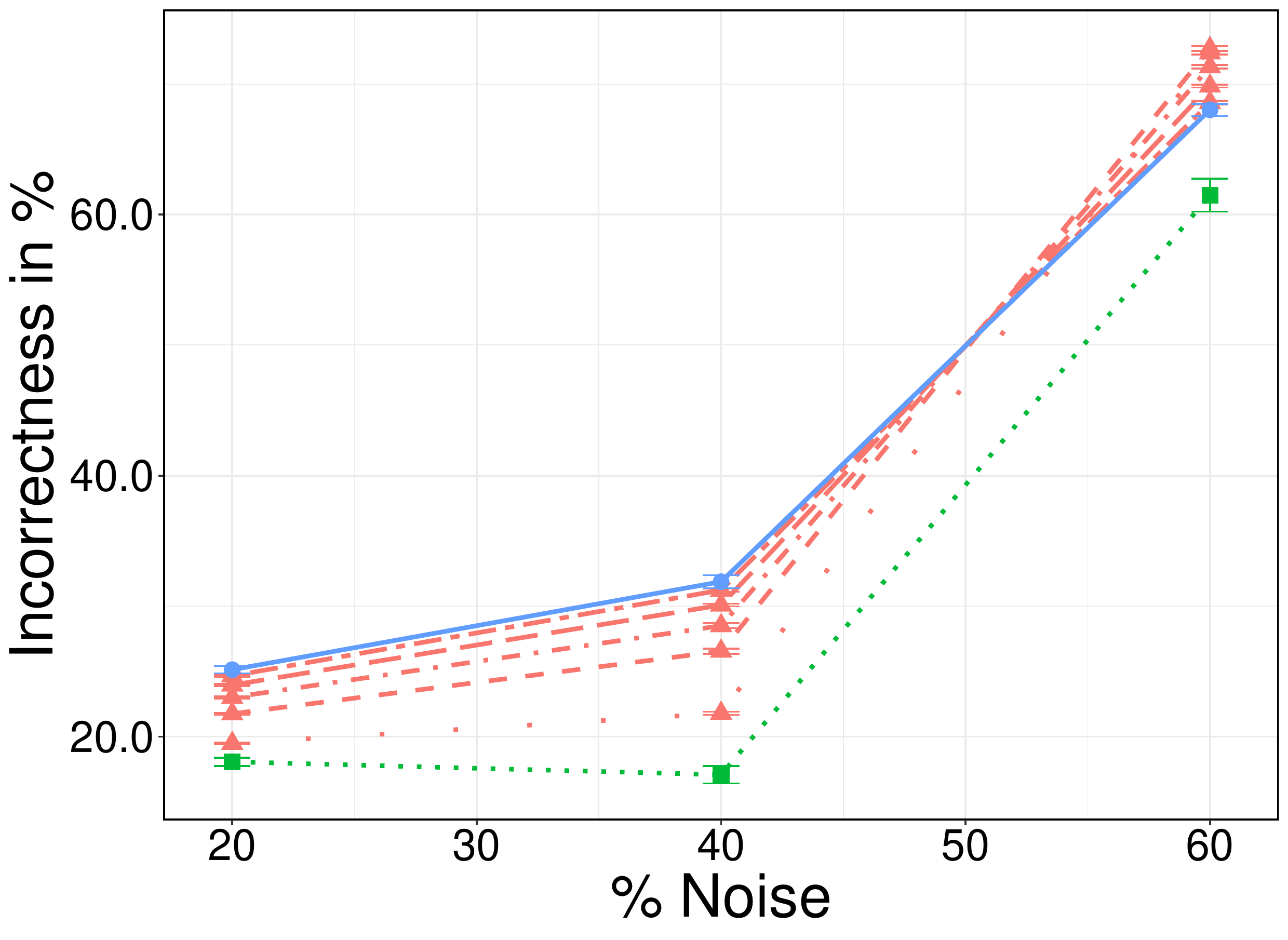}} 
		\subfigure[\scshape Scene]{
			\includegraphics[width=0.324\linewidth]
				{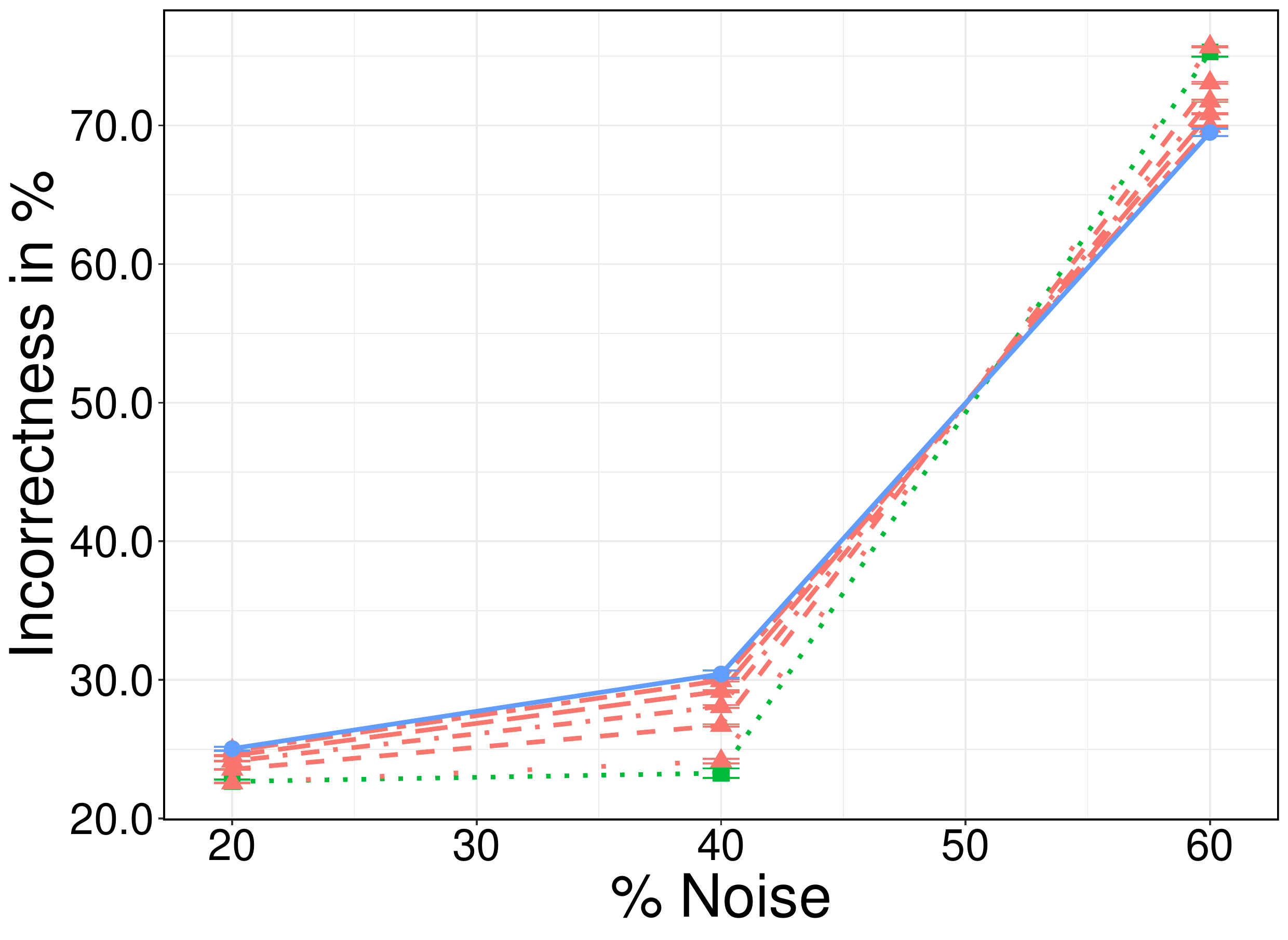}} 
		\subfigure[\scshape Yeast]{
			\includegraphics[width=0.324\linewidth]
				{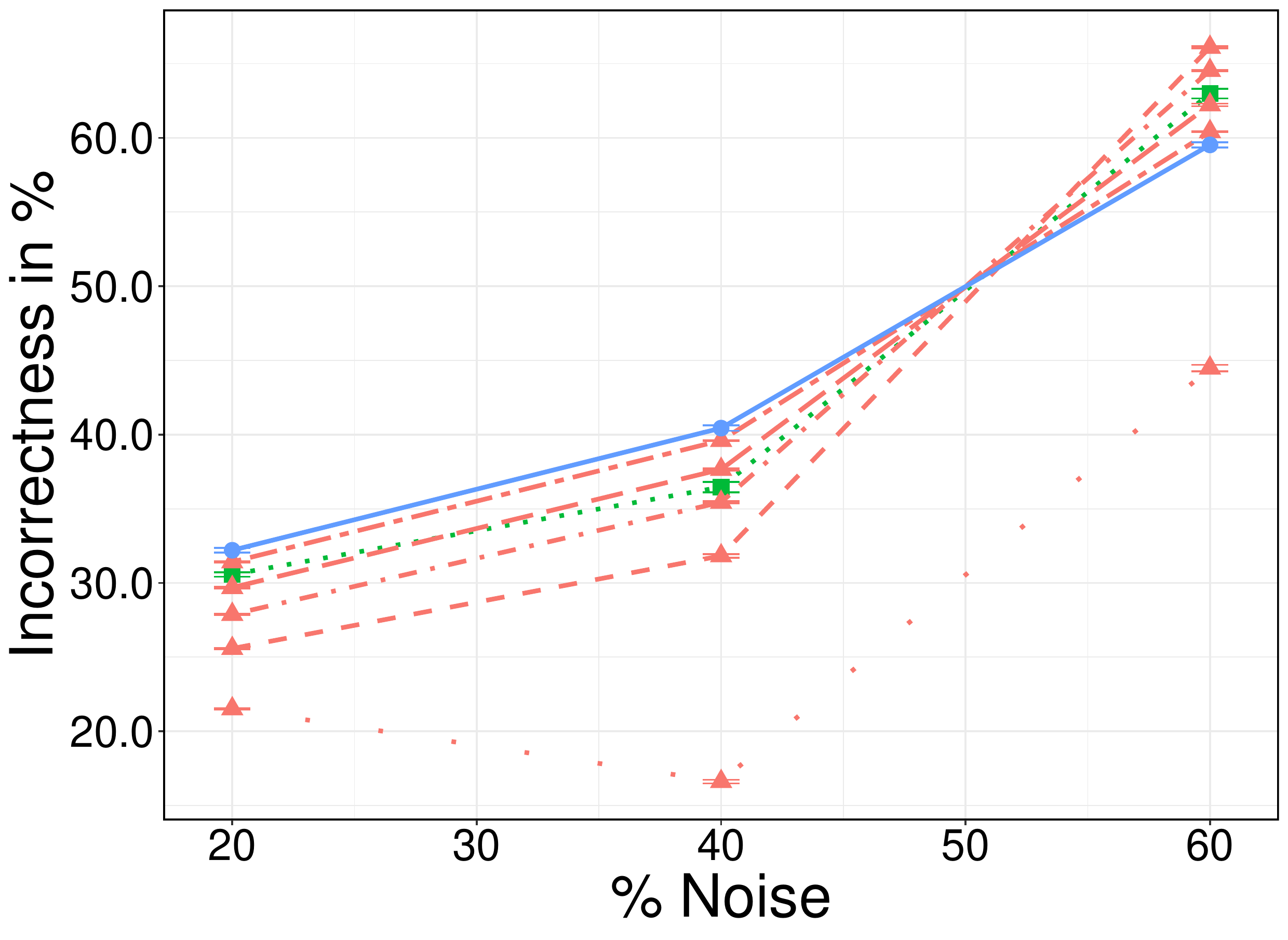}} 
	}\vspace{-3mm}
	\renewcommand{\thesubfigure}{(b)}
	\subfigure[{\bf Noise-Flipping} with a hyper-parameter value of NCC $s=3.5$]{
		\hspace{-2mm}
		\subfigure[\scshape Emotions]{
			\includegraphics[width=0.324\linewidth]
				{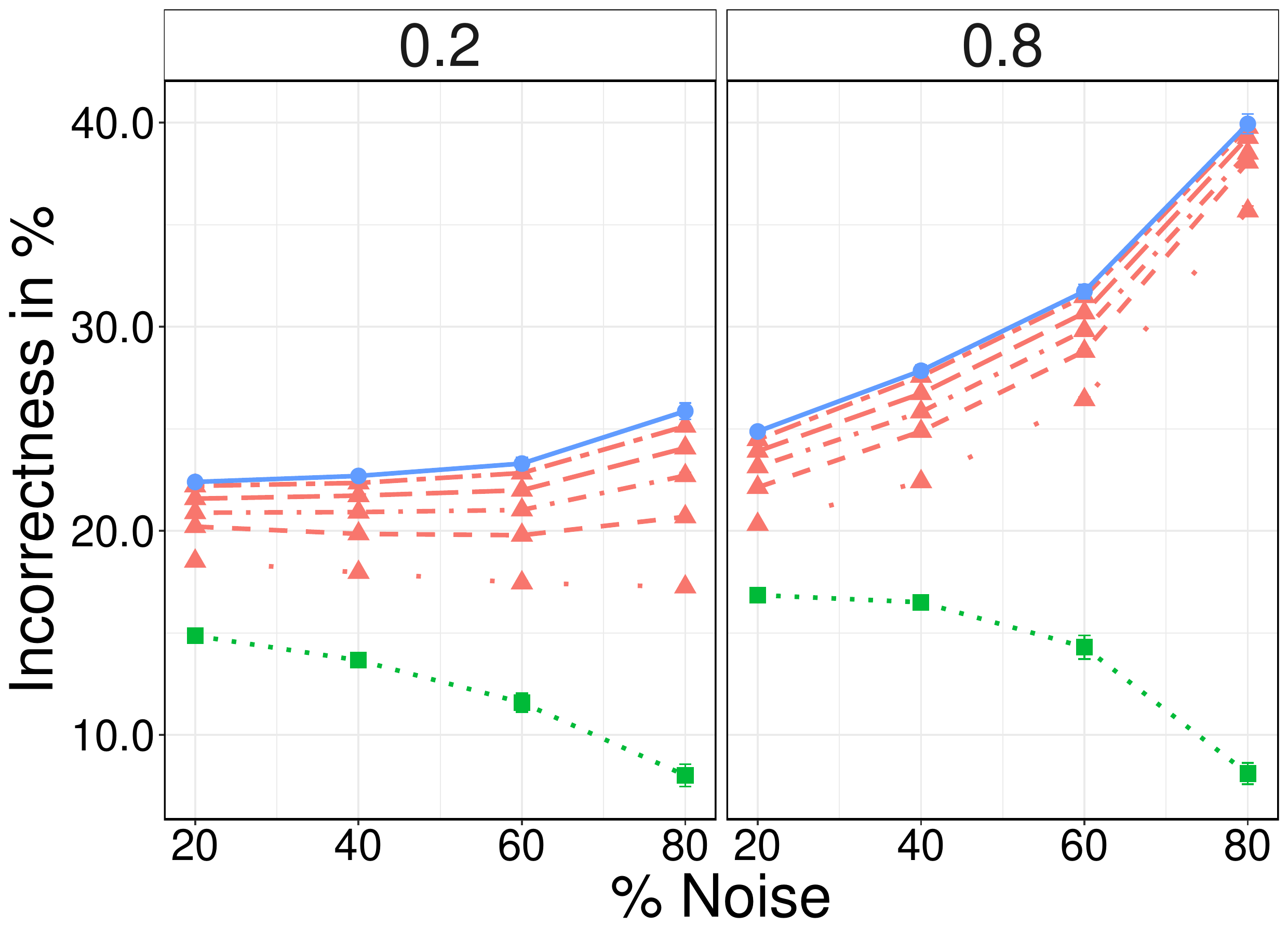}} 
		\subfigure[\scshape Scene]{
			\includegraphics[width=0.324\linewidth]
				{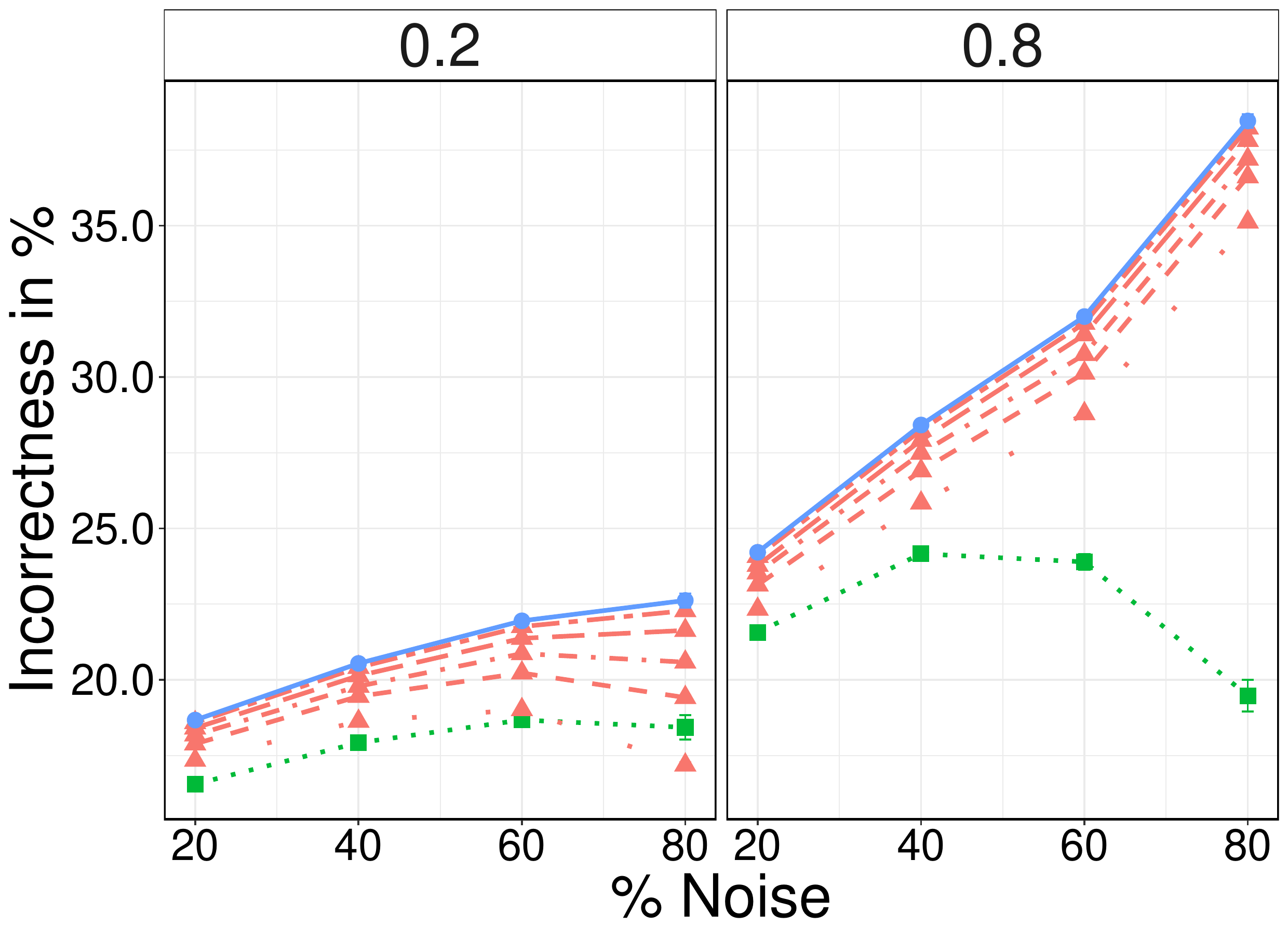}} 
		\subfigure[\scshape Yeast]{
			\includegraphics[width=0.324\linewidth]
				{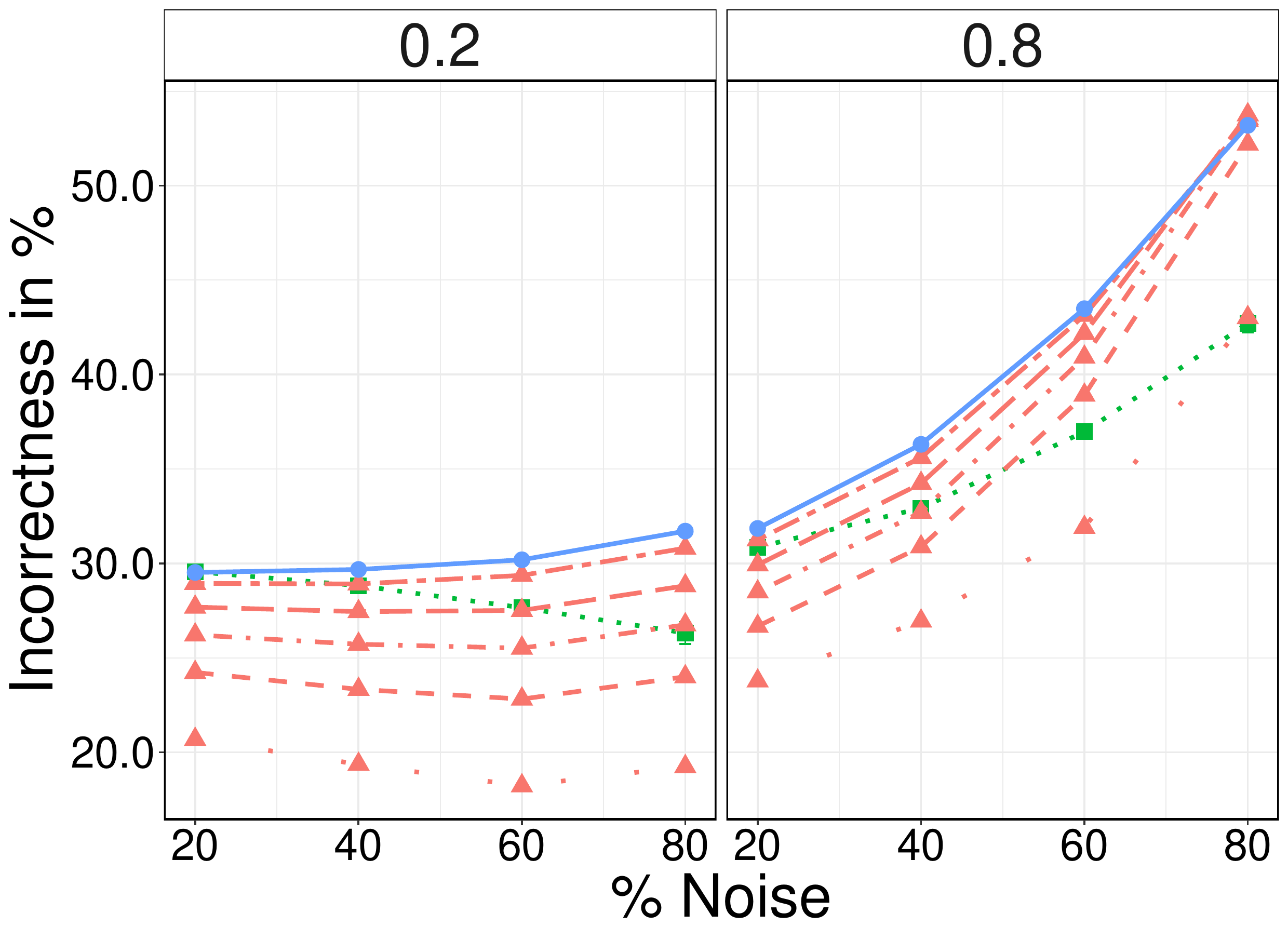}} 
	}\vspace{-2mm}
	\caption{{\bf Rejection vs. Skeptic.} Incorrectness evolution for a different level of imprecision in each setting and discretization $z=6$, and with respect to different percentages of noisy or missing labels (x-axis).}
	\label{fig:rejectionmissnoise}
\end{figure}

The results obtained of the {\bf Missing} and {\bf Flipping} settings show that the rejection rule does not react appropriately to the injection of imperfect information\footnote{Imperfect information is here used as a synonym for limited information or corrupted knowledge.} ($\geq 40\%$), be it because of the value of label is missing or corrupted, and even with a higher threshold $\gamma=0.45$ (since with $\gamma=0.50$, no label is predicted at all) it fails to abstain correctly on those hard instances where the uncertainty is higher. In contrast, we can see that our approach is quite robust in identifying this imperfection (or high uncertainty). This is proven by the fact that the incorrectness decreases significantly faster in our approach than in the rejection rule when the percentage of imperfect information increases.

On the other hand, the results based on the {\bf Reversing} setting shows that the rejection and skeptical approach have similar behaviours (i.e., slightly comparable curves). Nevertheless, we can identify that our approach has a better abstention behaviour on Emotions dataset (an in-depth study, with more labels et features, will be carried out in the next section).

Comparative experiments on the rejection and skeptical approaches provide us a first glance and an intriguing insight into how these approaches behave when confronted with imperfect information (with higher uncertainty). That is why in the next section, we perform a complementary comparative experiment between our skeptical approach, the rejection rule, and the partial abstention~\cite{nguyen2019}, all these subject to a downsampling procedure. 

\subsection{Experiments comparing skeptical approaches under downsampling}
\label{sec:downsampling}
The lack of knowledge (or limited information) is often present in different scientific and professional problems, such as unreported adverse symptoms of patients after getting the new COVID-19 vaccine, or doubtful information from a candidate applying for a job position or a housing credit. In such cases, the usage of skeptical approaches is essential, but what is the difference between the existing approaches, and how do the ones behave with the lack of knowledge?. Such questions are answered in this section in which we compare three different skeptical approaches subject to a downsampling procedure and different types of data sets (e.g. with features of type numeric, categorical, or mixed, and as well as, with few and several labels).

\paragraph{Data sets} Intending to have a diversity of data sets, we use some supplementary data sets from the MULAN repository (cf. Table~\ref{tab:datasetsdownsampling}), additionally to those presented in Table~\ref{tab:datasets}. 
\begin{table}[!ht]
	\centering
	\resizebox{!}{!}{%
	\begin{tabular}{ccccccc}
		 Data set & Domain & \#Features & \#Labels & \#Instances & \#Cardinality & \#Density \\
		\hline
		CAL500 &  music & 68 & 174 &  502 & 26.04 & 0.15\\
		Medical & text & 1449 & 45 &  978 & 1.25 & 0.03\\
		Flags & images & 19 & 7 & 194 & 3.39 & 0.49
	\end{tabular}}
	\caption{Multi-label data sets summary}\label{tab:datasetsdownsampling}
\end{table}

The data sets of Table~\ref{tab:datasetsdownsampling} complements those of Table~\ref{tab:datasets} in different aspects: (1) more labels, (2) more features, (3) different domains (text, music, ...) , and (4) different types of attributes: binary, numeric and mixed.

\paragraph{Downsampling} 
In order to show how the lack of knowledge can impact the inference step of the different skeptical approaches presented below, we downsampled the data sets of Table~\ref{tab:datasetsdownsampling} and \ref{tab:datasets} in sets of different percentages of training data sets $\{10, 20, \dots, 90\}\%$. In other words, $x\%$ is used as training data set and $(100-x)\%$ is used as test data set. We repeat this operation $50$ times for each percentage level in order to ensure that our results are close to asymptotic behaviours of classifiers.

In the case of the usage of IGDA classifier, it is necessary to ensure that the samples randomly generated have at least two samples per label to calculate the empirical covariance matrix without problems. That is why the data set CAL500 is sampled from 50\% to 90\% for our experiments. 

Metrics used for evaluating the performance of skeptical approaches on the $(1-x)\%$ test data set are: the incorrectness $IC(\hat{\mathbb{Y}}, \vect{y})$ and the incompleteness $CP(\hat{\mathbb{Y}}, \vect{y})$ already presented in Equations~\eqref{eq:incorrectness} and \eqref{eq:completeness}, where $\hat{\mathbb{Y}}$ is the skeptical inference and $\vect{y}$ is the ground-truth value.

\subsubsection{Skeptical approaches}\label{sec:absapproaches}
Among the few existing skeptical approaches already mentioned in Section~\ref{sec:introduction}, we choose to compare our proposal with the partial rejection and abstention, setting aside the indeterminate approach of Antonucci and Corani~\cite{antonucci2017multilabel}'s work which is closer to our approach but it focuses on predicting the \emph{full partial binary vector} $\hat{\vect{y}}$ by using the \textbf{zero-one loss function} whereas we perform \emph{single-label decisions} to predict the vector $\hat{\vect{y}}$ by using the \textbf{Hamming loss function}. 

\paragraph{Partial abstention} A  recent skeptical approach had been proposed by Nguyen V-L. et al in \cite{nguyen2019}, in which introduce a generalized loss function $L$ allowing to partially abstain from a prediction. This generalized loss function is based on an additive penalty and must also satisfy four properties; monotonicity, uncertainty-alignment, additive penalty for abstention, and convergence to the conventional loss function $\ell$ if no abstention (for further details, we refer to \cite{nguyen2019}). The generalized loss function therefore can be written as follows 
\begin{equation}\label{eq:partialgenloss}
	L(\vect{y}, \hat{\vect{y}}) = \ell(\vect{y}_D, \hat{\vect{y}}_D) + f(A(\hat{\vect{y}})),
\end{equation}
where $\ell$ is the original loss, $D:=D(\hat{\vect{y}})$ is a set of indices for which $\hat{y}_i\in\{0, 1\}$, and $f(A(\hat{\vect{y}}))$ a penalty function for abstaining on the set of indices  $A(\hat{\vect{y}})$ for which $\hat{y}_i$ is equal to $*$. 

In this paper, we consider two different penalty functions, as Nguyen V-L. et al did it for their experiments;
\begin{align}
	f_{SEP}(a)=a \cdot c \qquad\text{and}\qquad
	f_{PAR}(a)=\frac{(a\cdot m)}{(m + a)} \cdot c,
\end{align}
where $a$ is the number of abstained labels, $f_{SEP}$ is a linear penalty function with the hyper-parameter $c\in[0.05, 0.5]$ and $f_{PAR}$ is a concave penalty function with the hyper-parameter $c\in[0.1, 1]$. These latter functions are counting measures that only \emph{depends on the number of abstentions}, meaning that it penalizes each abstention, be it linearly $f_{SEP}$ or not $f_{PAR}$, with the same constant $c$.

When we consider the $\ell_H$ Hamming loss instead of $\ell$ and the penalty $f_{SEP}$, Nguyen V-L. et al shown that the risk minimisation under the generalized loss $L(\vect{y}, \hat{\vect{y}})$ produces the following prediction rules;
\begin{align}
	A(\hat{\vect{y}}) &= \left\{i| i\in\setn{m} ~s.t.~ c < \min(p_i, 1-p_i) \right\} \label{eq:abstainlabels}, \\
	D(\hat{\vect{y}}) &= \left\{i| i\in\setn{m} ~s.t.~ c \geq \min(p_i, 1-p_i) \right\} \label{eq:preciselabels},
\end{align}
in order to get the set of indices for the precise (Equation~\eqref{eq:preciselabels}) and abstained (Equation~\eqref{eq:abstainlabels}) labels. However, if the penalty $f_{PAR}$ is considered (or any other monotonous function), Nguyen V-L. et al proposed a procedure of time complexity $\complexity{m\log{(m)}}$~\cite[Prop. 1]{nguyen2019}. 

\paragraph{Rejection threshold} It is a well-known approach already described in Section~\ref{sec:expnoisedatasets} and well-adapted to perform single-label decisions $\hat{y}_{i,\ell_H,\gamma}\in\{1,0,*\}$, akin to our proposed generalized binary relevance described in Section~\ref{sec:Binary_rev} in terms of predicted values. As in this paper, all the predictive decisions are obtained as a result to use the Hamming loss function, we cannot use the partial rejection approach of Pillai et al.~\cite{pillai2013multi} which is based on the F-measure loss, and therefore, we will use a classical rejection approach which has already been introduced in Equation~\eqref{eq:optiHamPrecThreshold} and use the same values of the threshold parameter $\gamma$.

\subsubsection{Imprecise classifiers}
To obtain the marginal probability intervals over each label, i.e. $[\underline{p}_i, \overline{p}_i], \forall i\in\setn{m}$, we use three different credal classifiers (or imprecise classifier). These classifiers are very briefly described below and we refer to \ref{app:credalclassifier} for details on how we can calculate the marginal probability intervals in each case.

Note that in the case of the partial abstention and rejection approaches, we used the precise probability $p_i$, and not the probability interval $[\underline{p}_i, \overline{p}_i]$, to obtain skeptical decisions. Of course, in such cases, we used different values of hyper-parameters $c$ or $\gamma$ to get their skeptical decisions, see Section~\ref{sec:expresultdownsampling}.

\paragraph{Naive Credal classifier (NCC)} In the same way as in Section~\ref{sec:expnoisedatasets}, we use again the imprecise classifier NCC described above, with the same set-up but with just a single discretization level $z=6$.

\paragraph{Imprecise Gaussian discriminant analysis (IGDA)} IGDA is a extending version of the Gaussian discriminant analysis to the imprecise probabilistic setting. This extension considers the imprecision as part of its basic axioms and is based on robust Bayesian analysis and near-ignorance priors. By including an imprecise component in the model, or more precisely by using a set of Gaussian distributions with the $\mu$ mean parameter living in a convex space, IGDA can predict a set-valued decision on  hard-to-predict instances having high uncertainty. 

Depending on the internal structure of the covariance matrix parameter $\hat{\Sigma}$ of the set of Gaussian distributions, the inference complexity can become harder to compute, and that is why we opted to use an approximation method~\cite{park2017general}\footnote{Implemented in Python, see   \url{https://github.com/cvxgrp/qcqp}.} to calculate the lower probability bound instead of an exact procedure as was presented in the original paper~\cite{alarcon2021imprecise}, all this without reducing the cautious predictive global results.

In the same way as the NCC model controls the imprecision, IGDA also has a hyper-parameter $\tau$ to control the imprecision level of the probability interval. On the grounds that each dataset of Table~\ref{tab:datasetsdownsampling} and \ref{tab:datasets} can depend on the $\tau$ hyper-parameter value, we chose a set of different appropriate values $\tau \in [0, 5]$ for each one, without picking-up the most optimal value. 

\paragraph{Credal Sum-Product network classifier (CSPN)}
Recently, Mauá D.D. et al in \cite{maua2017credal} proposed to extend the relatively new deep probabilistic model known as sum-product networks (SPN)~\cite{poon2011sum} to the imprecise probabilistic setting, namely credal sum-product networks (CSPN), allowing that singleton weight $\vect{w}$ of SPN can vary in a convex space with constraints $\mathscr{C}_{\vect{w}}$.

As our purpose is to obtain the marginal probability intervals $[\underline{p}_i, \overline{p}_i]$ from an imprecise binary classification, we decided to use the class-selective SPNs approach proposed by Correia A.H et al in~\cite{Correia2019Towards}. This approach implements a suited architecture for classification tasks and reduces the inference time by using memoization techniques. Yet, as the current implementation in R language does not allow cautious inferences, we implemented an imprecise binary classification version~\footnote{Implemented in R, see~\url{https://github.com/salmuz/multilabel_cspn}.} on top of its source code~\cite{Correia2019Towards}, allowing to work, additionally of $\epsilon$-contamination approach $\mathscr{C}_{\vect{w}, \epsilon}$ already implemented in~\cite{Correia2019Towards}, with a convex space defined  from an imprecise Dirichlet model $\mathscr{C}_{N_{\vect{w}}, s}$ (IDM)~\cite{walley1996inferences}.

Like previous imprecise models, CSPN also controls the imprecision level of the probability intervals through the values of $\epsilon\in[0,1]$ hyper-parameter for $\epsilon$-contamination approach and $s>0$ hyper-parameter for the imprecise Dirichlet approach. Furthermore for sake of simplicity, we focus on the case of SPNs over discrete variables so that the data sets with continue variables are discretized in  $z\!=\!6$ equal-width intervals. 

\paragraph{Setting imprecise classifiers according to the nature of data sets}
The credal classifiers presented previously are not directly applied to all data sets, firstly because some of them does not match with the assumptions made by the credal classifier, e.g. Medical data set is composed of binary features, hence IGDA method is not adapted at all. Secondly, we consider that it is unnecessary to show redundant results, but rather consistent with the subject matter of this paper. Table~\ref{tab:methodsapplied} summarizes the credal classifier used for each data set of Table~\ref{tab:datasetsdownsampling} and \ref{tab:datasets}.
\begin{table}[!ht]
	\centering
	\resizebox{!}{!}{%
	\begin{tabular}{c||c|c|c|c|c|c}
		\sc Credal & \multirow{2}*{\sc Emotions} & \multirow{2}*{\sc Scene} & 
		\multirow{2}*{\sc Yeast} & \multirow{2}*{\sc CAL500}  & 
		\multirow{2}*{\sc Medical} & \multirow{2}*{\sc Flags}\\
		\sc Classifiers & & & & &\\
		\hline
		IGDA & \xmark &  & \xmark & \xmark &  & \xmark\\
	 	CSPN & \xmark & \xmark &  & \xmark & \xmark & \xmark\\
		NCC &   & \xmark & \xmark & \xmark & \xmark & 
	\end{tabular}}
	\caption{Credal classifiers used on data sets}
	\label{tab:methodsapplied}
\end{table}

\subsubsection{Experimental results under downsampling}\label{sec:expresultdownsampling}
The experimental results are divided into three parts, each of which shows the results obtained by using one of the credal classifiers previously discussed. To avoid overloading the illustrated results with too many curves, and also to simplify the interpretation and the narrative of our results, we preferred to display a single performance  (or green curve) of our skeptical inference approach per imprecise classifier in each chart, nevertheless, we kept five levels of inference performance for the other approaches. The performance displayed in the figures is the incorrectness, besides, the full of the experimental results are placed in \ref{app:suppresampling}.

\paragraph{NCC results}
In Figure~\ref{fig:nccsamplingresults}, we show the results of using the naive credal classifier as base classifier. The results show four important things:

\input{images/images_resampling_ncc.tex}

\begin{itemize}
	\item as the number of instances increase (x-axis) the incorrectness of our skeptical approach converges to that of the precisely-valued approach. This result reaffirms what we obtained in previous experiments, in which our skeptical approach  correctly adjusts to the lack of knowledge by performing many more set-valued informative predictions, when we have 10\% of knowledge of the dataset, and much less set-valued predictions when we  have 90\% of knowledge of the dataset;
	\item in contrast to the previous finding, the partial abstention and rejection option approaches do not react to lack of knowledge, in other words, no matter whatever is the size of the training data set, these approaches produce parallel performance curves (i.e., incorrectness) according to the value of hyper-parameter sets up;
	\item in addition to the previous finding, the results also confirm that there is not much significant difference between using a generic rejection option or a sophisticated partial abstention approach at any $x\%$ size of training data set; and finally 
	\item as regards the level of imprecision $s$ that needs to be injected into NCC model for each data set, in order to get our skeptical predictions, we can see that as the number of labels increases the level of imprecision must be bigger, for instance; the CAL500 data set needs bigger values of $s$ for every $\%$ training set and the YEAST data needs just a little more imprecision, see green curves in Figure~\ref{fig:nccsamplingresults}.
\end{itemize}

\paragraph{IGDA results} As behaviour of the incorrectness measure of the partial abstention approach in Figure~\ref{fig:nccsamplingresults} displays almost similar trend when we use the penalty function $f_{SEP}$ or $f_{PAR}$ coupled with the NBC model, in Figure~\ref{fig:igdasamplingresults} we consider judicious only to illustrate the results of the penalty function $f_{SEP}$ coupled with the (I)GDA model as a base classifier (and also with the aim of not overloading the results with too many charts).

\input{images/images_resampling_igda.tex}

Therefore, in Figure~\ref{fig:igdasamplingresults}, we show the results of all different skeptical approaches by using the imprecise Gaussian discriminant analysis (IGDA) as a base classifier. These results show four important things:

\begin{itemize}
	\item in general, we can see the same trends as in Figure~\ref{fig:nccsamplingresults} on the INDA and IEDA models, since these latter can be considered as variants of the NCC model over a continuous input space;

	\item as regards with the ILDA model, it reaches the same findings found previously, but also it  highlights a noteworthy effect of the decreasing monotonic trend on the evolution of the completeness measure in the partial abstention and the rejection approaches (see in \ref{app:suppresampling}), that is; as the number of training samples increases the completeness measure decreases monotonically, or, in other words, this means that the more knowledge such approaches have, the more cautious are its predictions. This effect contradicts what is conventionally expected, i.e. a decrease in cautious predictions when more information or knowledge is obtained;
	
	\item as regards the IQDA model, in Figure~\ref{fig:iqdasamplingresults} we can see that the incorrectness evolution is not the same as the other models, but the completeness evolution has the same behaviour as the others, i.e. an increasing monotonic function with respect to the sample size (x-axis). The behaviour of incorrectness may be explained by the fact that when the sample size is small, the covariance matrix estimate can become highly unstable\footnote{Quadratic discriminant analysis (QDA) model is known to require in general larger samples than Linear discriminant analysis (LDA) model.} (i.e., it may produce a biased estimate of the eigenvalues)~\cite{srivastava2007bayesian}, and then the (I)QDA model is more likely to make more mistakes on the only precisely-valued decisions inferred in the binary vector $\vect{y}$;

	\item and finally, unlike the NCC model and its continuous variants models such as INDA and IEDA, any data set coupled with our skeptical approach does not require large or small imprecision level values ($\tau$) to provide incorrectness curves that are lower than the precise approach.
\end{itemize}

\paragraph{CSPN results} As in previous models, we illustrate the experimental results of using the credal sum-product network (CSPN) as base (imprecise) classifier, setup with an imprecise Dirichlet model (IDM) and a $\epsilon$-contamination convex space ($\epsilon$-cont) in Figure~\ref{fig:cspnidmsamplingresults} and \ref{fig:cspnepsilonsamplingresults}, respectively.

\input{images/images_resampling_cspn_idm.tex}

\subparagraph{CSPN with Imprecise Dirichlet model} Regarding to CSPN-IDM, it yields the same findings found previously as the other models, but without the need of injecting large imprecision values $s$ of the convex space $\mathscr{C}_{N_{\vect{w}}, s}$ on our skeptical approach in order always to get incorrectness curves below the precise approach.

\input{images/images_resampling_cspn_econt.tex}
\subparagraph{CSPN with $\epsilon$-contamination} Roughly speaking, we can see the same trends as the previous experiments, except by the fact that when we use the $\epsilon$-contamination space, it is necessary, in some cases, e.g. in the Medical data set, to inject much more imprecision values of $\epsilon$.

Note that, independently of the convex space $\mathscr{C}_{N_{\vect{w}}, s}$ or $\mathscr{C}_{\vect{w}, \epsilon}$ applied to the CSPN model, we can see that there is not a monotonic  convergence of the incorrectness measure, estimated by our skeptical approach (i.e. the green curve in Figures \ref{fig:cspnidmsamplingresults} and \ref{fig:cspnepsilonsamplingresults}), to the precisely-valued approach, and besides, this incorrectness performance at different levels of imprecision tends to have almost similar results to their counterparts' approaches. This behaviour is generally unexpected in a credal classifier when more information is injected into to model, yet such behaviour can be explained by the fact that when the tree structure of CSPN is built using a specific training data set, this same tree structure is generally no longer used for the other training data set. That means that the lower and upper estimation of $\vect{w}\in\mathscr{C}_{\vect{w},*}$ in any node of the tree structure of the CSPN model can not converge to a point-wise estimation of $\vect{w}$ when more information is available.  

Results obtained from applying all different imprecise classifiers on different skeptical approaches are sufficient to show how our skeptic approach with sets of probabilities can provide better benefits than those counterparts, when information is not enough. 
\paragraph{Partial Abstention vs. Skeptical inference with sets of probabilities}\label{sec:workrelated}
Based on the experimental results obtained previously, we can highlight two important points:

\begin{itemize}
	\item as regards the skeptical inference with sets of probabilities, we can conclude that it is significantly dependent on the \emph{convergence} of the sets of probability distributions towards a point-wise probability distribution as the number of observations $N$ increases, i.e.  $\credal_Y \xrightarrow{~N\rightarrow\infty~} \mathbb{P}_Y $. In order words, the marginal probability intervals $[\underline{p}_i, \overline{p}_i]$ must converge to a point-wise probability $p_i$ when $N$ is bigger, so that it is consistent with the precise approach at the limit. This behaviour ensure that the skeptic decision (Maximality or E-admissibility) can more often abstain when information is lacking or imprecise (i.e. epistemic uncertainty).

	 \item as regards the partial abstention approach, we can conclude that such approach, as well as rejection rules, does fail to take into account the supplemental information (or knowledge) provided in the training step in order to reduce the number of abstentions (or skeptical predictions). Indeed, it is because the penalty function $f_*$ does not depend of the number of observed samples, making the convergence of observed samples $N\rightarrow\infty$ inconsequential in Equation~\eqref{eq:partialgenloss}.
	 
		In order to improve the penalty function, one could suggest to define a penalty function $f_*(A(\hat{\vect{y}}), N)$ that depends on the set of indices $A(\hat{\vect{y}})$ and the number of observed samples $N$, so that the penalty value of $f_*(\cdot, \cdot)$ vanishes when $N\rightarrow\infty$. Nonetheless, we leave to the reader an open question based on the last proposal, should one still abstain after acquiring all the information (or knowledge) of a real-world data set?.
		
\end{itemize}

\section{Conclusion and discussion}
Describing our uncertainty by a set of probabilities over combinatorial domains such as binary vectors usually leads to difficult optimisation problems at the decision step. In this paper, we investigated those problems when considering the well-known Hamming loss, providing efficient inference methods and, when considering the binary relevance scheme, connecting it to the zero/one loss. 

In essence, we significantly reduced the complexity of computing exact skeptic, cautious predictions for general probability sets, and showed that in the Binary relevance scheme, those same predictions were reduced to partial vectors computable from marginal probability bounds over the labels.  

Experiments on the simulated data sets show that this last solution, when used as an outer-approximation in the general case, degrades in quality as the number of labels increase and the level of imprecision is mild. On the other hand, experiments on various real data sets show that making skeptical inferences generally provide quite satisfactory results on different scenarios, involving missing or noisy labels.

The results of experiments conducted to compare our skeptical approach against that rejection and abstaining approaches in terms of cautious (or skeptical) inference, have clearly demonstrated that our proposal reacts very well to lack of knowledge (or information) while their counterparts do not so. The last finding makes it clear that it is better to model our uncertainty by a set of probability distributions instead of a single probability distribution.


Another natural next step will be to solve the maximality criterion using other loss functions commonly used in multi-label problems, e.g. ranking loss, Jaccard loss, F-measure, and so on. As noticed in Remark~\ref{rem:imp_prec_loss}, such problems are likely to be much more intricate when considering sets of probabilities. 

Finally, let us notice that while this paper focused on the issue of multi-label learning problems, our results readily apply to any Boolean vectors of $m$ items. As Boolean vectors and structures as well as probability bounds naturally appear in a number of other applications, including occupancy grids~\cite{mouhagir2017using} or data bases~\cite{gatterbauer2014oblivious}, a future work would be to investigate how our present findings can help in such problems. 

\section*{Acknowledgements}
This work was partially carried out in the framework of the Labex MS2T, funded by the French Government, through the National Agency for Research (Reference ANR-11-IDEX-0004-02).
\bibliographystyle{abbrv}
\bibliography{Bibli}
\newpage
\appendix 
\input{appendix_classifier.tex}
\input{appendix_proofs.tex}

\input{supplementary_results.tex}

\input{supplementary_reject.tex}
\input{supplementary_resampling.tex}
\end{document}

%% file: Notations.tex
\newcommand{\condset}[2]{{ \left\{ #1 \middle| #2 \right\} }}

\newcommand{\indicator}[1]{{\mathbbm{1}_{(#1)}\;}}

\newcommand{\credal}{{\mathcal{P}}}

\newcommand{\expe}{{\mathbb{E}}}
\newcommand{\lexpe}{{\underline{\expe}}}

\newcommand{\reals}{\mathbb{R}}

\newcommand{\complexity}[1]{\mathscr{O}(#1)}
\newcommand{\losszero}{\ell_{0/1}}

%% file: eg_precise_tree.tex
\newcommand{\exptwoeg}[3]{\underline{\mathbb{E}}_{Y_{\{#1\}}|\newinstance}\left[ \ell_H(\cdot, \overline{a}_\mathcal{I}) \middle| Y_{\{#2\}}\!=\!#3\right]}
\newcommand{\expeg}[1]{\underline{\mathbb{E}}_{Y_{\{#1\}}|\newinstance}\left[ \cdot \right]}
\tikzset{
  treenode/.style = {shape=rectangle, rounded corners,
                     draw, align=center,
                     top color=white, bottom color=blue!20},
  root/.style     = {treenode, font=\Large, bottom color=red!30},
  env/.style      = {treenode, font=\ttfamily\normalsize},
  dummy/.style    = {circle,draw}
}
\tikzstyle{level 1}=[level distance=4cm, sibling distance=1.7cm]
\tikzstyle{level 2}=[level distance=3cm, sibling distance=0.8cm]
\tikzstyle{bag} = [text width=4em, text centered]
\tikzstyle{end} = [inner sep=0pt]
\begin{tikzpicture}[
	grow=right, auto,
	sloped,
	edge from parent/.style = {draw, -latex, font=\tiny\sffamily, text=blue},
	every node/.style  = {font=\footnotesize}]
\node[root] {}
    child {
        node[env] {$y_1=1$}        
            child {
                node[end, label=right:{$(y_1\!=\!1, y_2\!=\!1)$}]  {}
                edge from parent
                node[below] (11) {$0.7$}
            }
            child {
                node[end, label=right:{$(y_1\!=\!1, y_2\!=\!0)$}]  {}
                edge from parent
                node[above] (10) {$0.3$}
            }
            edge from parent 
            node[below] (1) {$0.5$} 
    }
    child {
        node[env] {$y_1=0$}       
        child {
                node[end, label=right:{$(y_1\!=\!0, y_2\!=\!1)$}] {}
                edge from parent
                node[below] (01) {$0.2$}
            }
        child {
                node[end, label=right:{$(y_1\!=\!0, y_2\!=\!0)$}] {}
                edge from parent
                node[above] (00)  {$0.8$}
            }
        edge from parent         
        node[above] (0) {$0.5$}
    };
\end{tikzpicture}

%% file: eg_prob_loss_tree.tex
\newcommand{\exptwoeg}[3]{\underline{\mathbb{E}}_{Y_{\{#1\}}|\newinstance}\left[ \ell_H(\cdot, \overline{a}_\mathcal{I}) \middle| Y_{\{#2\}}\!=\!#3\right]}
\newcommand{\expeg}[1]{\underline{\mathbb{E}}_{Y_{\{#1\}}|\newinstance}\left[ \cdot \right]}
\tikzset{
  treenode/.style = {shape=rectangle, rounded corners,
                     draw, align=center,
                     top color=white, bottom color=blue!20},
  root/.style     = {treenode, font=\Large, bottom color=red!30},
  env/.style      = {treenode, font=\ttfamily\normalsize},
  dummy/.style    = {circle,draw}
}
\tikzstyle{level 1}=[level distance=4cm, sibling distance=2.7cm]
\tikzstyle{level 2}=[level distance=3cm, sibling distance=1.2cm]
\tikzstyle{bag} = [text width=4em, text centered]
\tikzstyle{end} = [inner sep=0pt]
\begin{tikzpicture}[
	grow=right, auto,
	sloped,
	edge from parent/.style = {draw, -latex, font=\tiny\sffamily, text=blue},
	every node/.style  = {font=\footnotesize}]
\node[root] {}
    child {
        node[env] {$y_1=1$}        
            child {
                node[end, label=right:{$(y_1\!=\!1, y_2\!=\!1)$}]  {}
                edge from parent
                node[below] (11) {$0.7$}
            }
            child {
                node[end, label=right:{$(y_1\!=\!1, y_2\!=\!0)$}]  {}
                edge from parent
                node[above] (10) {$0.3$}
            }
            edge from parent 
            node[below] (1) {$0.5$} 
    }
    child {
        node[env] {$y_1=0$}       
        child {
                node[end, label=right:{$(y_1\!=\!0, y_2\!=\!1)$}] {}
                edge from parent
                node[below] (01) {$0.2$}
            }
        child {
                node[end, label=right:{$(y_1\!=\!0, y_2\!=\!0)$}] {}
                edge from parent
                node[above] (00)  {$0.8$}
            }
        edge from parent         
        node[above] (0) {$0.5$}
    };
\draw[red, transform canvas={xshift=-15pt}, bend right,-]  (11) to node [auto, yshift=-60pt, xshift=0pt] {\rotatebox{270}{$\expe=0.3 \cdot 1=0.3$}} (10);

\draw[red, transform canvas={xshift=-15pt}, bend right,-]  (01) to node [auto,  yshift=-72pt, xshift=0pt] {\rotatebox{270}{$\expe=0.2 \cdot -1=-0.2$}} (00);
 
\draw[red, transform canvas={xshift=-15pt}, bend right,-]  (1) to node [auto, yshift=-110pt, xshift=0pt] {\rotatebox{270}{$\expe=0.5 \cdot -0.2 + 0.5 \cdot 0.3=0.05$}} (0);

\node at (11,2.5) {$\ell_{0/1} \left((0,1),\cdot\right) - \ell_{0/1} \left((1,0),\cdot\right)= $};

\node at (11,1.95) {$0$};
\node at (11,0.75) {$-1$};
\node at (11,-0.75) {$1$};
\node at (11,-1.95) {$0$};
\end{tikzpicture}

%% file: eg_Iprob_loss_tree.tex
\newcommand{\exptwoeg}[3]{\underline{\mathbb{E}}_{Y_{\{#1\}}|\newinstance}\left[ \ell_H(\cdot, \overline{a}_\mathcal{I}) \middle| Y_{\{#2\}}\!=\!#3\right]}
\newcommand{\expeg}[1]{\underline{\mathbb{E}}_{Y_{\{#1\}}|\newinstance}\left[ \cdot \right]}
\tikzset{
  treenode/.style = {shape=rectangle, rounded corners,
                     draw, align=center,
                     top color=white, bottom color=blue!20},
  root/.style     = {treenode, font=\Large, bottom color=red!30},
  env/.style      = {treenode, font=\ttfamily\normalsize},
  dummy/.style    = {circle,draw}
}
\tikzstyle{level 1}=[level distance=4cm, sibling distance=2.7cm]
\tikzstyle{level 2}=[level distance=3cm, sibling distance=1.2cm]
\tikzstyle{bag} = [text width=4em, text centered]
\tikzstyle{end} = [inner sep=0pt]
\begin{tikzpicture}[
	grow=right, auto,
	sloped,
	edge from parent/.style = {draw, -latex, font=\tiny\sffamily, text=blue},
	every node/.style  = {font=\footnotesize}]
\node[root] {}
    child {
        node[env] {$y_1=1$}        
            child {
                node[end, label=right:{$(y_1\!=\!1, y_2\!=\!1)$}]  {}
                edge from parent
                node[below] (11) {$[0.613, \mathbf{0.713}]$}
            }
            child {
                node[end, label=right:{$(y_1\!=\!1, y_2\!=\!0)$}]  {}
                edge from parent
                node[above] (10) {$[\mathbf{0.287}, 0.387]$}
            }
            edge from parent 
            node[below] (1) {$[\mathbf{0.456}, 0.556]$} 
    }
    child {
        node[env] {$y_1=0$}       
        child {
                node[end, label=right:{$(y_1\!=\!0, y_2\!=\!1)$}] {}
                edge from parent
                node[below] (01) {$[0.138, \mathbf{0.238}]$}
            }
        child {
                node[end, label=right:{$(y_1\!=\!0, y_2\!=\!0)$}] {}
                edge from parent
                node[above] (00)  {$[\mathbf{0.762}, 0.862]$}
            }
        edge from parent         
        node[above] (0) {$[0.444, \mathbf{0.544}]$}
    };
\draw[red, transform canvas={xshift=-15pt}, bend right,-]  (11) to node [auto, yshift=-47pt, xshift=0pt] {\rotatebox{270}{$\lexpe=0.287 \cdot 1$}} (10);

\draw[red, transform canvas={xshift=-15pt}, bend right,-]  (01) to node [auto,  yshift=-53pt, xshift=0pt] {\rotatebox{270}{$\lexpe=0.238 \cdot -1$}} (00);

\draw[red, transform canvas={xshift=-15pt}, bend right,-]  (1) to node [auto, yshift=-130pt, xshift=0pt] {\rotatebox{270}{$\lexpe=0.544 \cdot -0.238 + 0.456 \cdot 0.287 > 0$}} (0);

\node at (11,2.5) {$\ell_{0/1} \left((0,1),\cdot\right) - \ell_{0/1} \left((1,0),\cdot\right)= $};

\node at (11,1.95) {$0$};
\node at (11,0.75) {$-1$};
\node at (11,-0.75) {$1$};
\node at (11,-1.95) {$0$};
\end{tikzpicture}

%% file: eg_infimum_expectation.tex
\newcommand{\exptwoeg}[3]{\underline{\mathbb{E}}_{Y_{\{#1\}}|\newinstance}\left[ \ell_H(\cdot, \overline{a}_\mathcal{I}) \middle| Y_{\{#2\}}\!=\!#3\right]}
\newcommand{\expeg}[1]{\underline{\mathbb{E}}_{Y_{\{#1\}}|\newinstance}\left[ \cdot \right]}
\tikzset{
  treenode/.style = {shape=rectangle, rounded corners,
                     draw, align=center,
                     top color=white, bottom color=blue!20},
  root/.style     = {treenode, font=\Large, bottom color=red!30},
  env/.style      = {treenode, font=\ttfamily\normalsize},
  dummy/.style    = {circle,draw}
}
\tikzstyle{level 1}=[level distance=4cm, sibling distance=2.7cm]
\tikzstyle{level 2}=[level distance=5cm, sibling distance=1.2cm]
\tikzstyle{bag} = [text width=4em, text centered]
\tikzstyle{end} = [inner sep=0pt]
\begin{tikzpicture}[
	grow=right, auto,
	sloped,
	edge from parent/.style = {draw, -latex, font=\tiny\sffamily, text=blue},
	every node/.style  = {font=\footnotesize}]
\node[root] {}
    child {
        node[env] {$y_1=1$}        
            child {
                node[end, label=right:{$(y_1\!=\!1, y_2\!=\!1),~ \ell_H(\cdot,\overline{a}_\mathcal{I})=0$}]  {}
                edge from parent
                node[below] (11) {$P_{\newinstance}(Y_{\{2\}}\!=\!1|Y_{\{1\}}\!=\!1)\!\in\![0.35, 0.90]$}
            }
            child {
                node[end, label=right:{$(y_1\!=\!1, y_2\!=\!0),~ \ell_H(\cdot,\overline{a}_\mathcal{I})=1$}]  {}
                edge from parent
                node[above] (10) {$P_{\newinstance}(Y_{\{2\}}\!=\!0|Y_{\{1\}}\!=\!1)\!\in\![0.10, 0.65]$}
            }
            edge from parent 
            node[below] (1) {$P_{\newinstance}(Y_{\{1\}}\!\!=\!\!1)\!\in\![0.45, 0.70]$} 
    }
    child {
        node[env] {$y_1=0$}       
        child {
                node[end, label=right:{$(y_1\!=\!0, y_2\!=\!1),~ \ell_H(\cdot,\overline{a}_\mathcal{I})=1$}] {}
                edge from parent
                node[below] (01) {$P_{\newinstance}(Y_{\{2\}}\!=\!1|Y_{\{1\}}\!=\!0)\!\in\![0.85, 0.97]$}
            }
        child {
                node[end, label=right:{$(y_1\!=\!0, y_2\!=\!0),~ \ell_H(\cdot,\overline{a}_\mathcal{I})=0$}] {}
                edge from parent
                node[above] (00)  {$P_{\newinstance}(Y_{\{2\}}\!=\!0|Y_{\{1\}}\!=\!0)\!\in\![0.03, 0.15]$}
            }
        edge from parent         
        node[above] (0) {$P_{\newinstance}(Y_{\{1\}}\!\!=\!\!0)\!\in\![0.30, 0.55]$}
    };

\draw[red, transform canvas={xshift=-15pt}, bend right,-]  (11) to node [auto, yshift=-115pt, xshift=0pt] {\rotatebox{270}{$\exptwoeg{2}{1}{1}\!=\!0.10$}} (10);

\draw[red, transform canvas={xshift=-15pt}, bend right,-]  (01) to node [auto,  yshift=-115pt, xshift=0pt] {\rotatebox{270}{$\exptwoeg{2}{1}{0}\!=\!0.85$}} (00);

\draw[red, transform canvas={xshift=-15pt}, bend right,-]  (1) to node [auto, yshift=-132pt, xshift=0pt] {\rotatebox{270}{$\expeg{1}\!=\!0.85*0.3 + 0.1*0.7=0.33$}} (0);

\end{tikzpicture}

%% file: images/reject/legends_reject.tex
\begin{tikzpicture}[x=1pt,y=1pt]
\definecolor{fillColor}{RGB}{255,255,255}
\path[use as bounding box,fill=fillColor,fill opacity=0.00] (0,0) rectangle (830.97, 59.50);
\begin{scope}
\path[clip] (  0.00,  0.00) rectangle (884.97, 59.50);
\definecolor{drawColor}{RGB}{0,0,0}
\node[text=drawColor,anchor=base west,inner sep=0pt, outer sep=0pt, scale= 1.98] at ( 63.06, 35.35) { \sc\bf Models:};
\end{scope}
\begin{scope}
\path[clip] (  0.00,  0.00) rectangle (884.97, 59.50);
\definecolor{drawColor}{RGB}{248,118,109}
\path[draw=drawColor,line width= 1.4pt,line join=round, dash pattern=on 2pt off 2pt on 6pt off 2pt] (206.19, 40.16) -- (274.48, 40.16);
\end{scope}
\begin{scope}
\path[clip] (  0.00,  0.00) rectangle (884.97, 59.50);
\definecolor{fillColor}{RGB}{248,118,109}

\path[fill=fillColor] (240.34, 40.16) circle (4.64);
\end{scope}
\begin{scope}
\path[clip] (  0.00,  0.00) rectangle (884.97, 59.50);
\definecolor{drawColor}{RGB}{0,186,56}

\path[draw=drawColor,line width= 1.4pt,line join=round,dotted] (398.10, 40.16) -- (466.39, 40.16);
\end{scope}
\begin{scope}
\path[clip] (  0.00,  0.00) rectangle (884.97, 59.50);
\definecolor{fillColor}{RGB}{0,186,56}

\path[fill=fillColor] (432.24, 47.38) --
	(438.49, 36.56) --
	(425.99, 36.56) --
	cycle;
\end{scope}
\begin{scope}
\path[clip] (  0.00,  0.00) rectangle (884.97, 59.50);
\definecolor{drawColor}{RGB}{97,156,255}

\path[draw=drawColor,line width= 1.4pt,line join=round] (590.26, 40.16) -- (658.55, 40.16);
\end{scope}
\begin{scope}
\path[clip] (  0.00,  0.00) rectangle (884.97, 59.50);
\definecolor{fillColor}{RGB}{97,156,255}

\path[fill=fillColor] (619.77, 35.52) --
	(629.05, 35.52) --
	(629.05, 44.80) --
	(619.77, 44.80) --
	cycle;
\end{scope}
\begin{scope}
\path[clip] (  0.00,  0.00) rectangle (884.97, 59.50);
\definecolor{drawColor}{RGB}{0,0,0}

\node[text=drawColor,anchor=base west,inner sep=0pt, outer sep=0pt, scale=  1.98] at (292.92, 35.35) {Rejection};
\end{scope}
\begin{scope}
\path[clip] (0.00,  0.00) rectangle (884.97, 59.50);
\definecolor{drawColor}{RGB}{0,0,0}

\node[text=drawColor,anchor=base west,inner sep=0pt, outer sep=0pt, scale=  1.98] at (484.82, 35.35) {Imprecise};
\end{scope}
\begin{scope}
\path[clip] (0.00,  0.00) rectangle (884.97, 59.50);
\definecolor{drawColor}{RGB}{0,0,0}

\node[text=drawColor,anchor=base west,inner sep=0pt, outer sep=0pt, scale=  1.98] at (676.98, 35.35) {Precise};
\end{scope}
\begin{scope}
\path[clip] (  0.00,  0.00) rectangle (884.97, 59.50);
\definecolor{drawColor}{RGB}{0,0,0}

\node[text=drawColor,anchor=base west,inner sep=0pt, outer sep=0pt, scale=  1.98] at (8.50, 8.52) {\bf Threshold $\gamma$:~~};
\end{scope}
\begin{scope}
\path[clip] (0.00,  0.00) rectangle (884.97, 59.50);
\definecolor{drawColor}{RGB}{0,0,0}

\path[draw=drawColor,line width= 1.4pt,dash pattern=on 2pt off 2pt on 6pt off 2pt ,line join=round] (140.79, 12.33) -- (209.08, 12.33);
\end{scope}

\begin{scope}
\path[clip] (0.00,  0.00) rectangle (884.97, 59.50);
\definecolor{drawColor}{RGB}{0,0,0}

\path[draw=drawColor,line width= 1.4pt,dash pattern=on 7pt off 3pt ,line join=round] (280.01, 12.33) -- (348.30, 12.33);
\end{scope}

\begin{scope}
\path[clip] (0.00,  0.00) rectangle (884.97, 59.50);
\definecolor{drawColor}{RGB}{0,0,0}

\path[draw=drawColor,line width= 1.4pt,dash pattern=on 1pt off 3pt on 4pt off 3pt ,line join=round] (424.24, 12.33) -- (492.52, 12.33);
\end{scope}
\begin{scope}
\path[clip] (  0.00,  0.00) rectangle (884.97, 59.50);
\definecolor{drawColor}{RGB}{0,0,0}

\path[draw=drawColor,line width= 1.4pt,dash pattern=on 4pt off 4pt ,line join=round] (568.46, 12.33) -- (636.75, 12.33);
\end{scope}

\begin{scope}
\path[clip] (  0.00,  0.00) rectangle (884.97, 59.50);
\definecolor{drawColor}{RGB}{0,0,0}

\path[draw=drawColor,line width= 1.4pt,dash pattern=on 1pt off 15pt ,line join=round] (712.69, 12.33) -- (780.97, 12.33);
\end{scope}

\begin{scope}
\path[clip] (  0.00,  0.00) rectangle (884.97, 59.50);
\definecolor{drawColor}{RGB}{0,0,0}
\node[text=drawColor,anchor=base west,inner sep=0pt, outer sep=0pt, scale=  1.98] at (222.51, 8.52) {0.05};
\end{scope}

\begin{scope}
\path[clip] (  0.00,  0.00) rectangle (884.97, 59.50);
\definecolor{drawColor}{RGB}{0,0,0}
\node[text=drawColor,anchor=base west,inner sep=0pt, outer sep=0pt, scale=  1.98] at (366.74, 8.52) {0.15};
\end{scope}

\begin{scope}
\path[clip] (  0.00,  0.00) rectangle (884.97, 59.50);
\definecolor{drawColor}{RGB}{0,0,0}

\node[text=drawColor,anchor=base west,inner sep=0pt, outer sep=0pt, scale=  1.98] at (510.96, 8.52) {0.25};
\end{scope}
\begin{scope}
\path[clip] (  0.00,  0.00) rectangle (884.97, 59.50);
\definecolor{drawColor}{RGB}{0,0,0}

\node[text=drawColor,anchor=base west,inner sep=0pt, outer sep=0pt, scale=  1.98] at (655.18, 8.52) {0.35};
\end{scope}
\begin{scope}
\path[clip] (  0.00,  0.00) rectangle (884.97, 59.50);
\definecolor{drawColor}{RGB}{0,0,0}

\node[text=drawColor,anchor=base west,inner sep=0pt, outer sep=0pt, scale=  1.98] at (789.41, 8.52) {0.45};
\end{scope}
\end{tikzpicture}

%% file: images/images_resampling_ncc.tex
\begin{figure}[!th]
	\centering%
	\resizebox{0.8\textwidth}{!}{%
	  \input{images/resampling/legends_resampling_all}
	}\vspace{-2mm}\qquad%
	\renewcommand{\thesubfigure}{(a)}
	\subfigure[\scshape Partial abstention $f_{SPE}$ with NCC-$s=0.50$]{
		\hspace{-3mm}
		\subfigure[{\scshape Yeast}]{
			\includegraphics[width=0.244\linewidth]
				{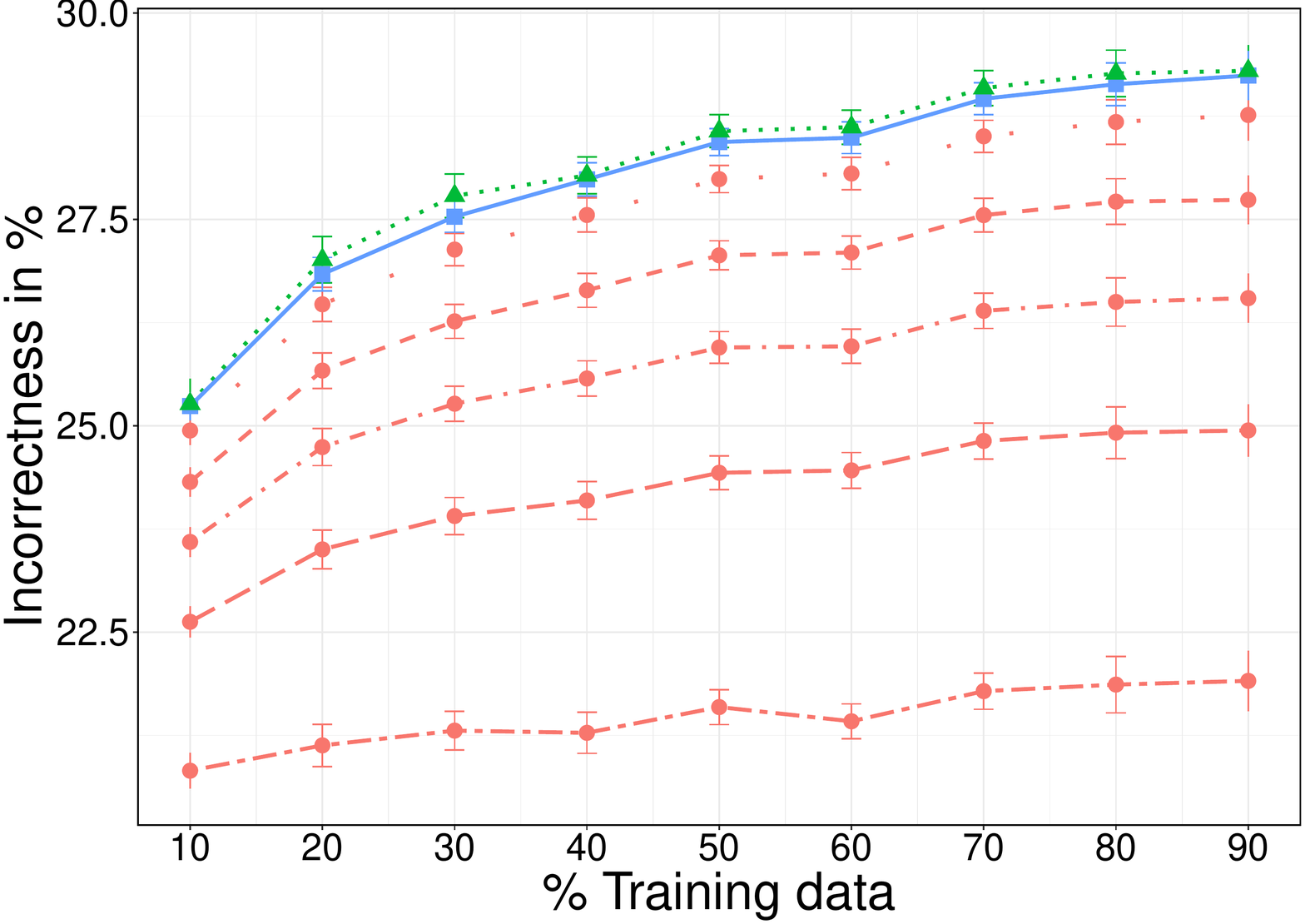}} 
		\subfigure[\scshape Scene]{
			\includegraphics[width=0.244\linewidth]
				{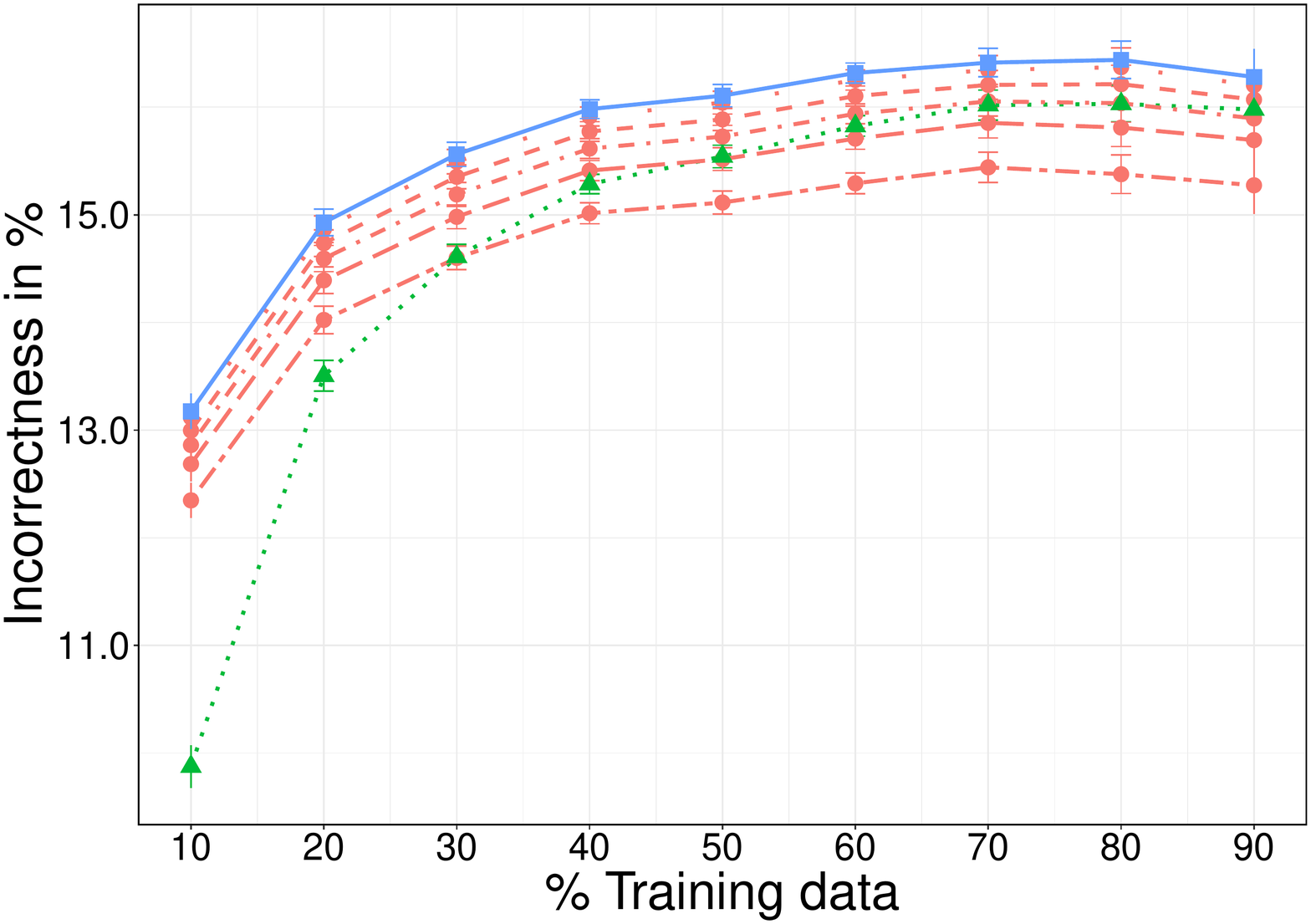}}
		\subfigure[\scshape Medical]{
			\includegraphics[width=0.244\linewidth]
				{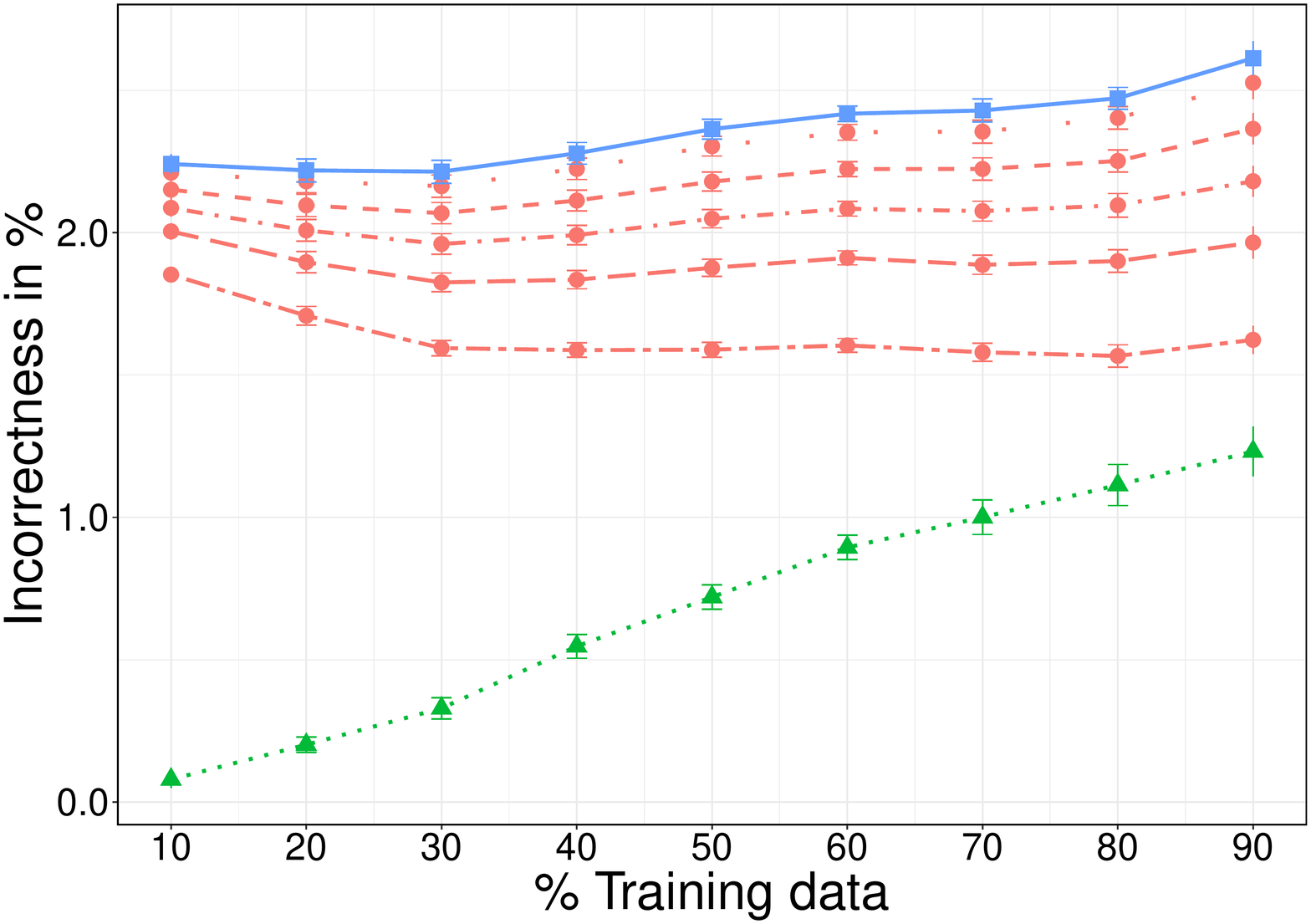}} 
		\subfigure[\scshape CAL500]{
			\includegraphics[width=0.244\linewidth]
				{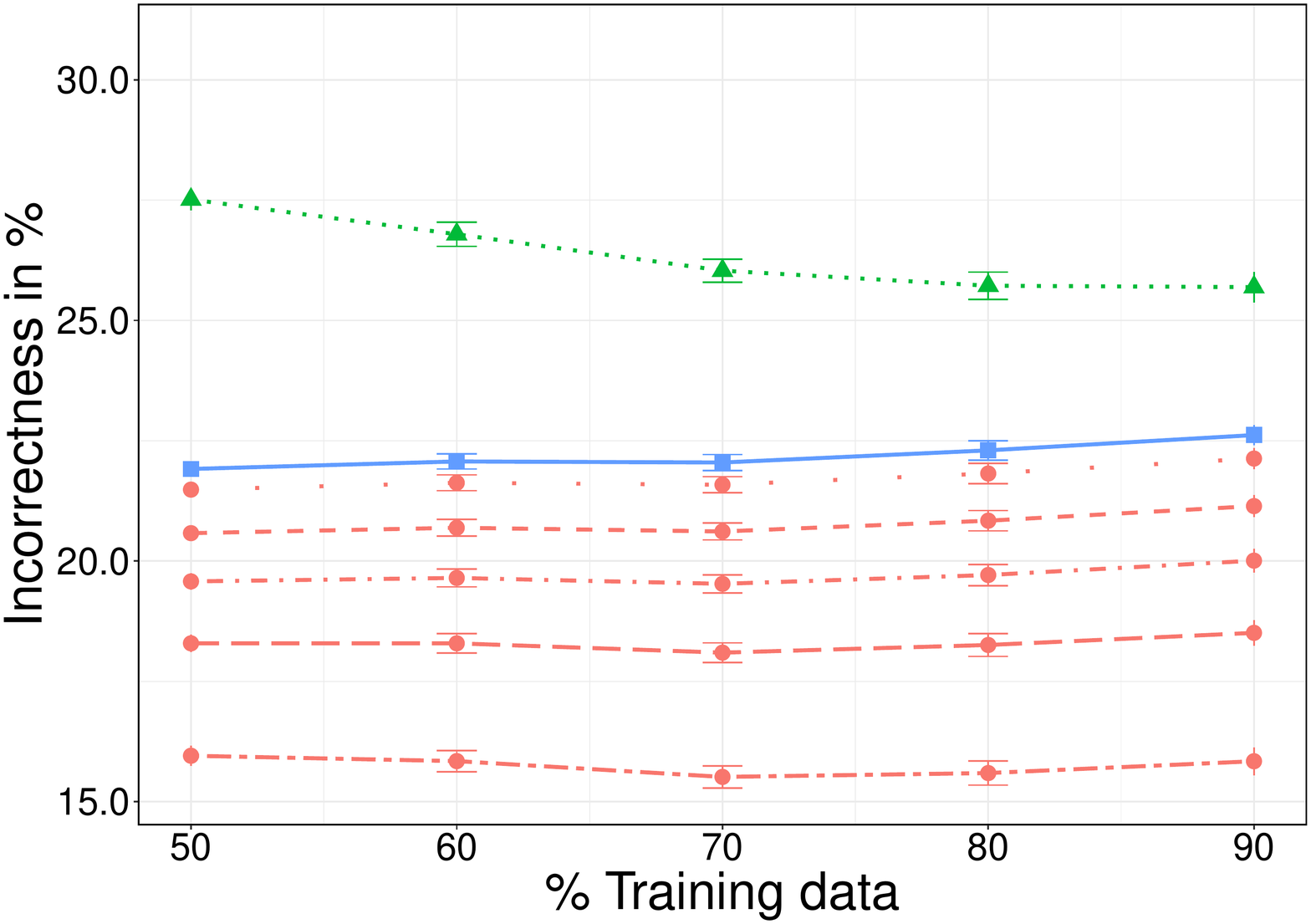}} 
	}
	\resizebox{0.8\textwidth}{!}{%
		\input{images/resampling/legends_par_params}
	}\vspace{-2mm}\qquad%
    \renewcommand{\thesubfigure}{(c)}
	\subfigure[\scshape Partial abstention $f_{PAR}$ with NCC-$s=1.50$]{
		\hspace{-3mm}
	  	\subfigure[\scshape Yeast]{
			\includegraphics[width=0.244\linewidth]
				{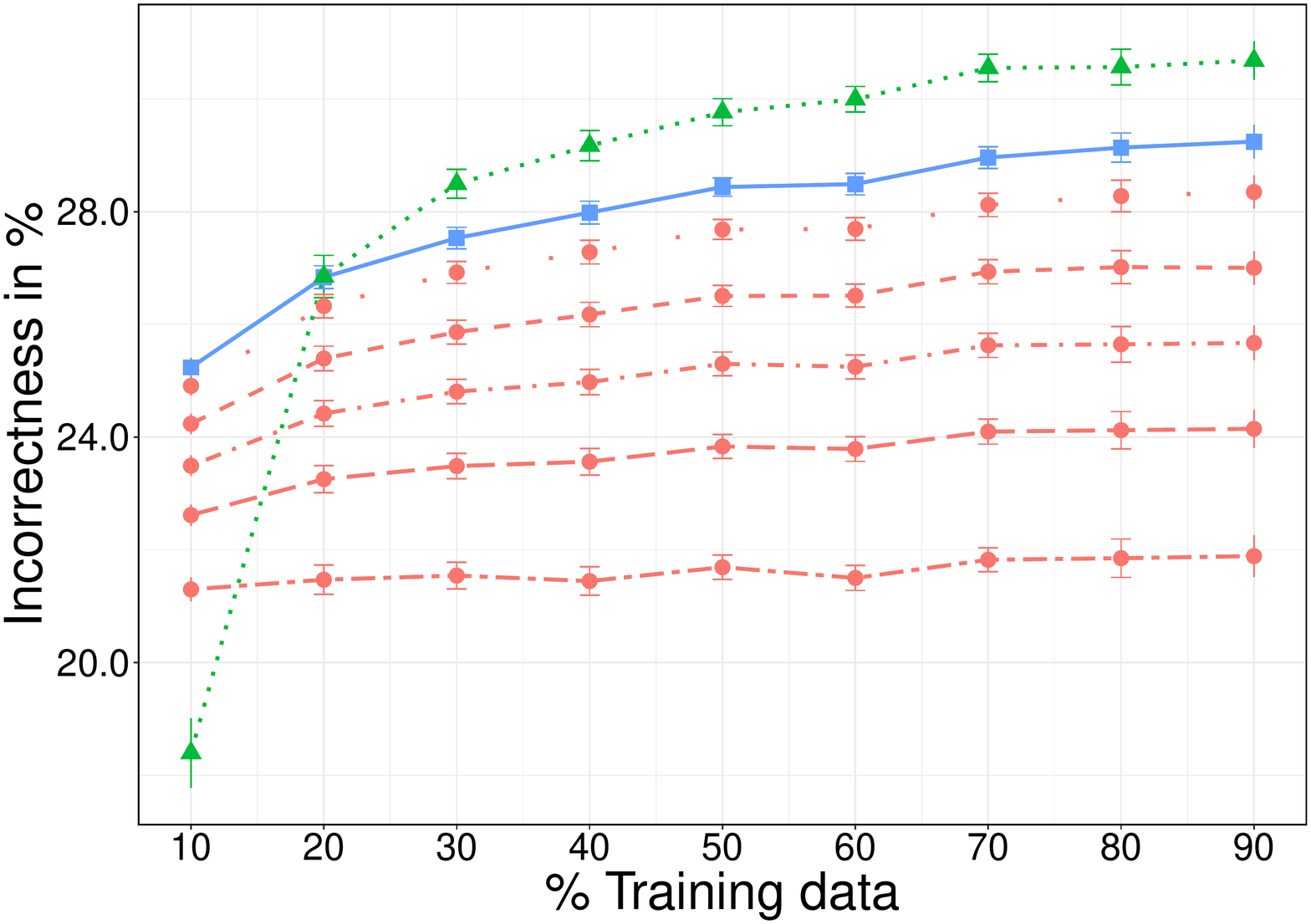}}
		\subfigure[\scshape Scene]{
			\includegraphics[width=0.244\linewidth]
				{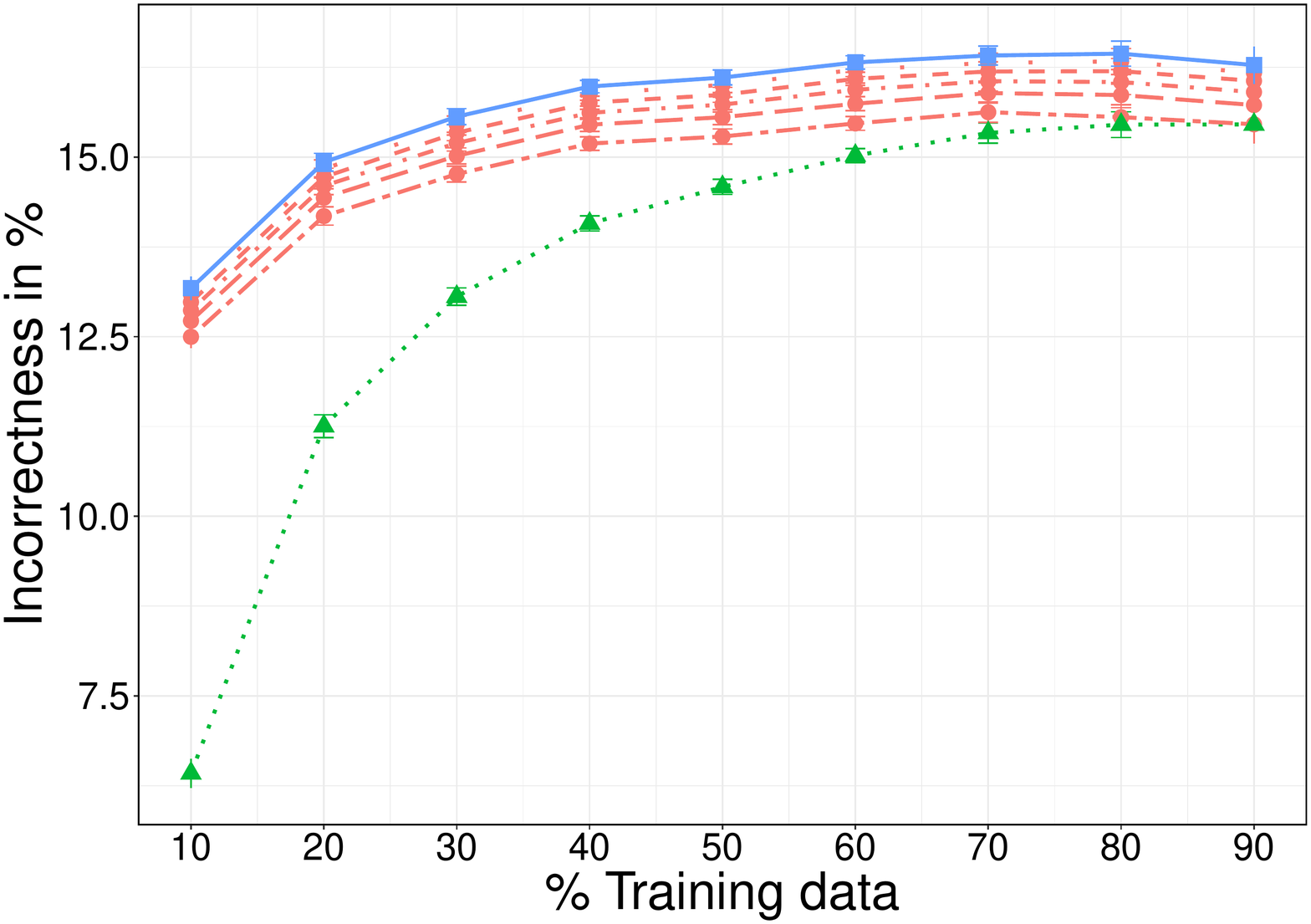}}
		\subfigure[\scshape Medical]{
			\includegraphics[width=0.244\linewidth]
				{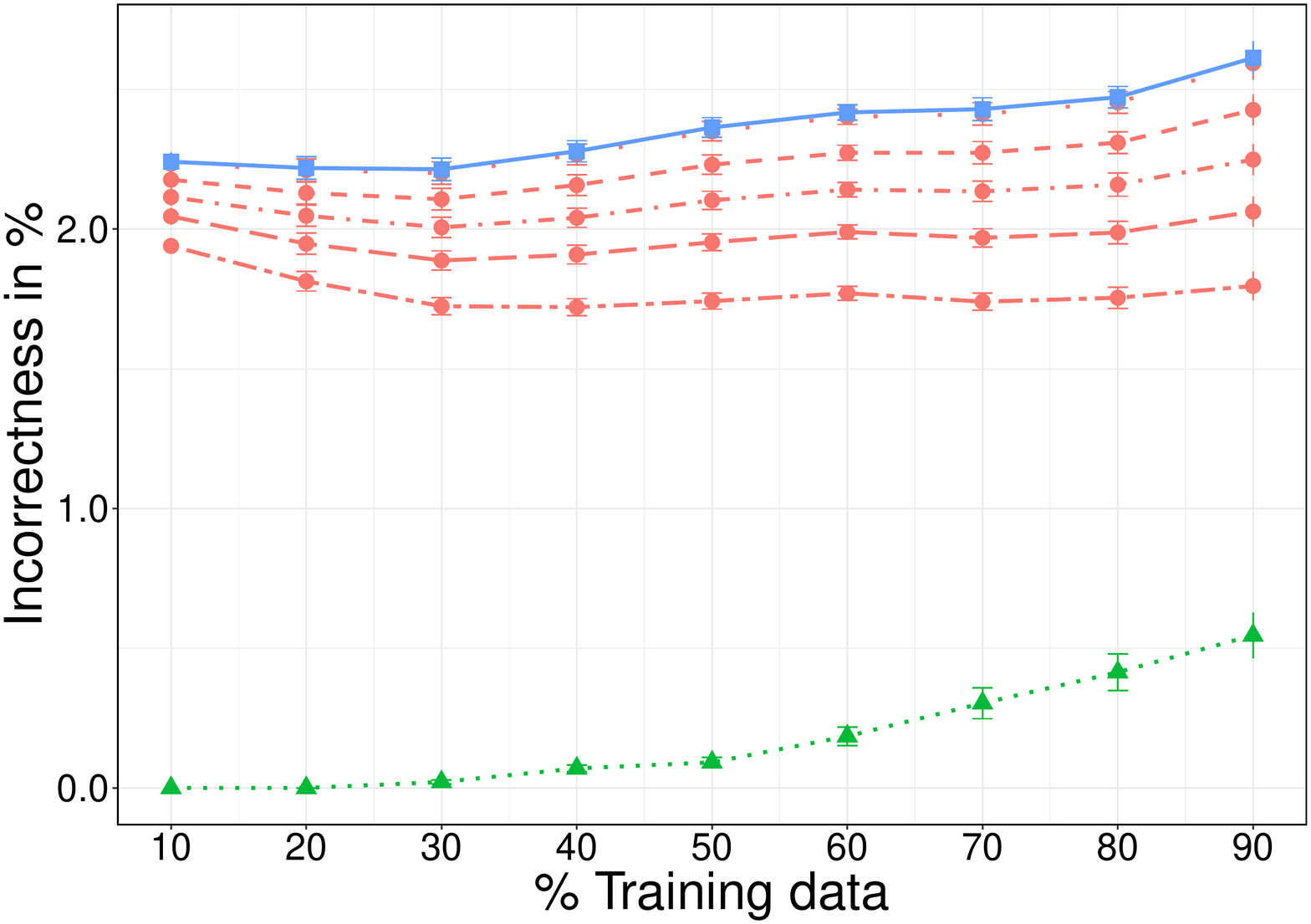}}
		\subfigure[\scshape CAL500]{
			\includegraphics[width=0.244\linewidth]
				{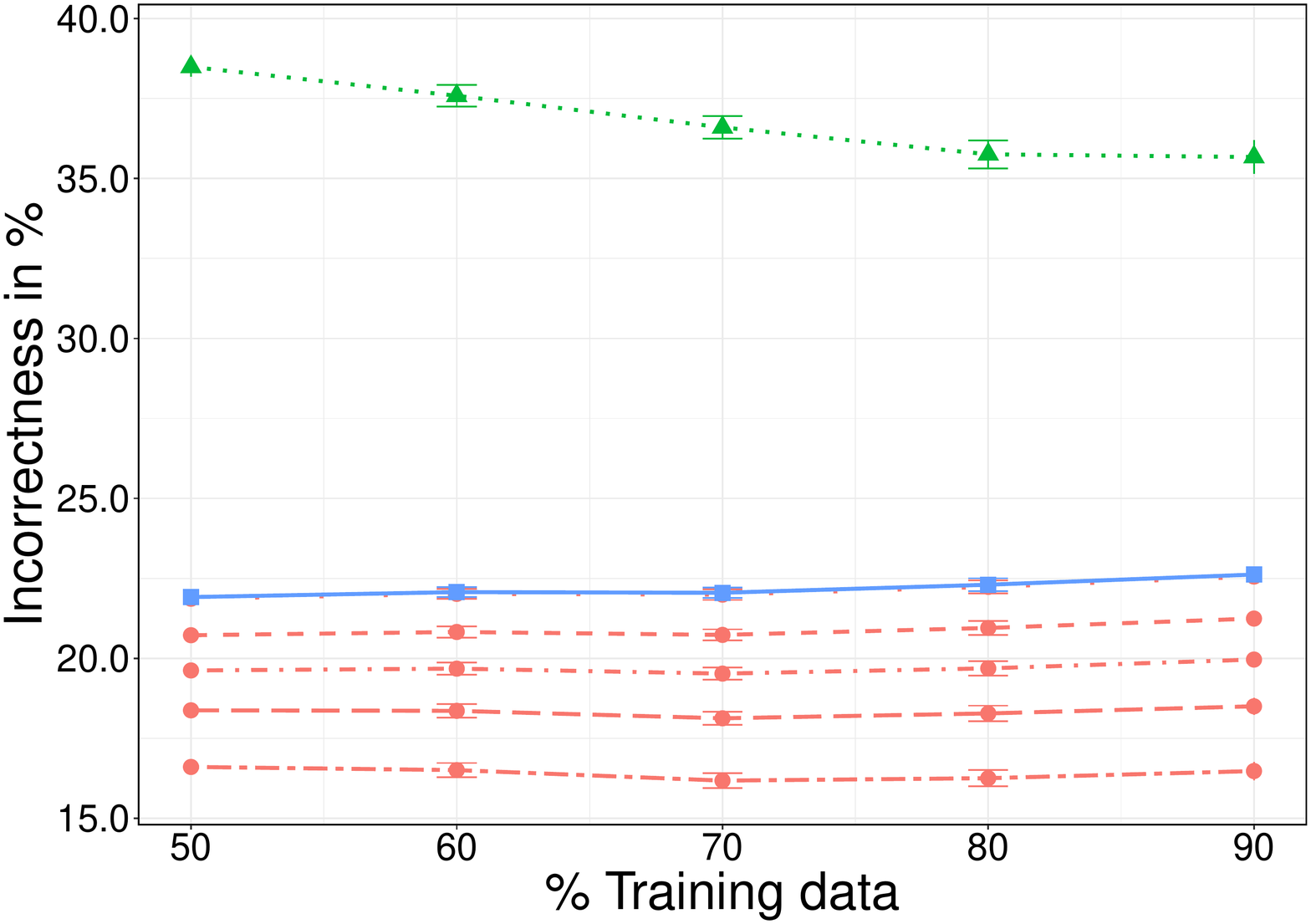}}
	}
	\resizebox{0.8\textwidth}{!}{%
		\input{images/resampling/legends_rej_params}
	}\vspace{-2mm}\qquad%
	\renewcommand{\thesubfigure}{(b)}
	\subfigure[\scshape Rejection with NCC-$s=2.50$]{ 
		\hspace{-3mm}
		\subfigure[\scshape Yeast]{
			\includegraphics[width=0.244\linewidth]
				{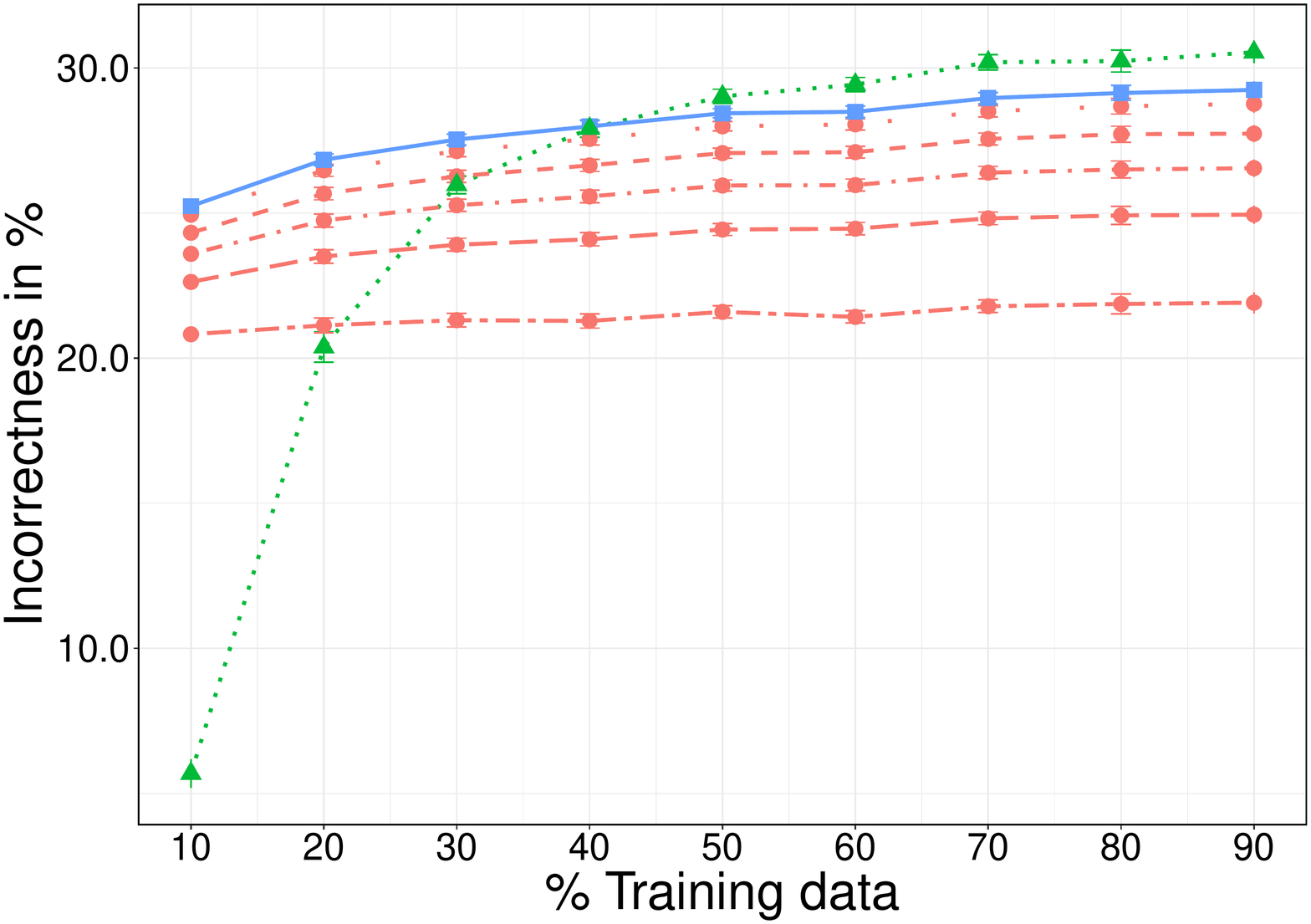}}
		\subfigure[\scshape Scene]{
			\includegraphics[width=0.244\linewidth]
				{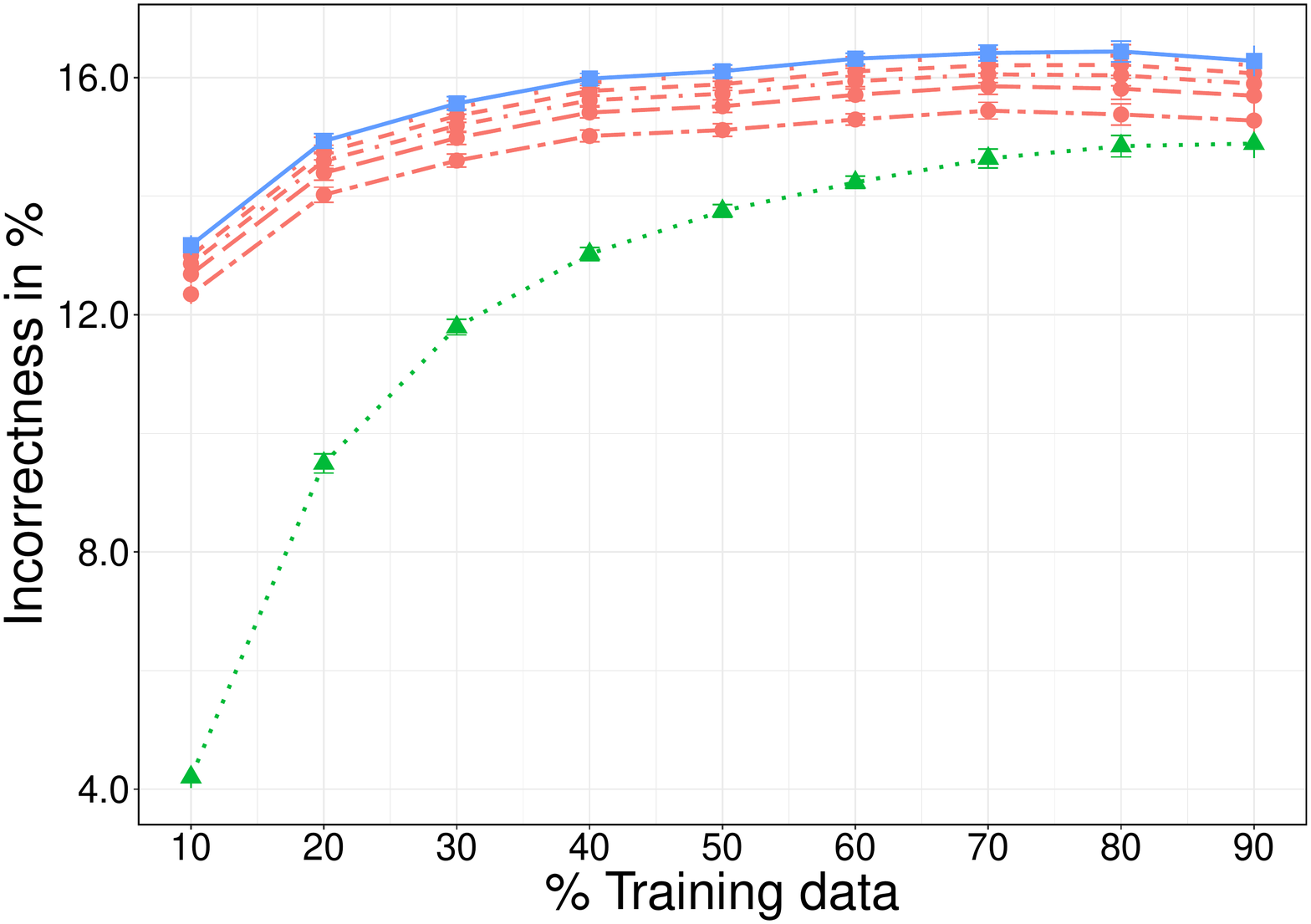}}
		\subfigure[\scshape Medical]{
			\includegraphics[width=0.244\linewidth]
				{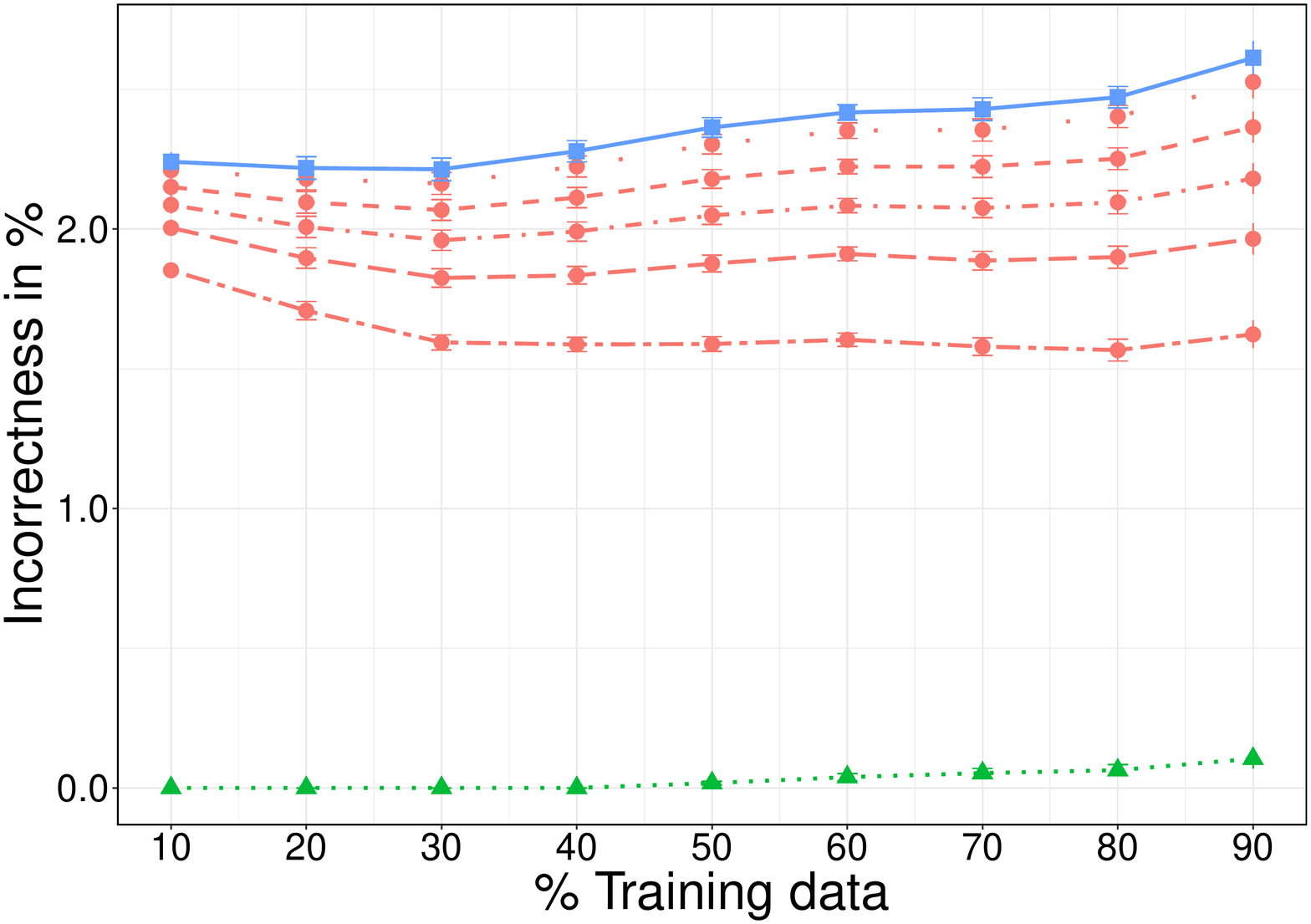}}
		\subfigure[\scshape CAL500]{
			\includegraphics[width=0.244\linewidth]
				{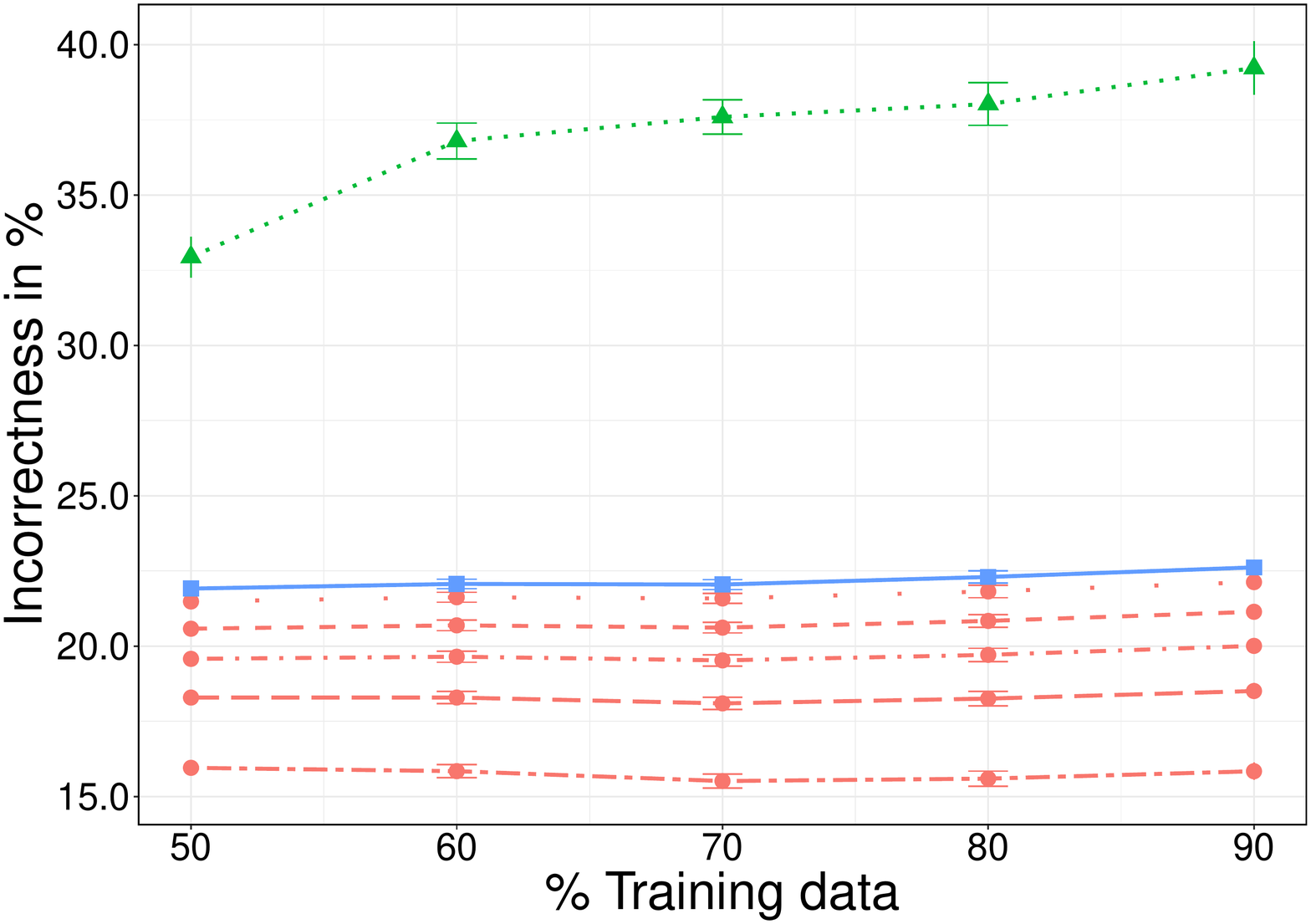}} 
	}\vspace{-2mm}
	\caption{{\bf Downsampling - Naive credal classifier.} Incorrectness (y-axis) evolution for the precise, partial abstention with $f_{SEP}$ (\nth{1} row) and $f_{PAR}$ (\nth{2} row) penalty functions, rejection (\nth{3} row), and skeptical approach, each one with different hyper-parameter levels (except to the precise approach), and with respect to different percentages of training data sets (x-axis).}
	\label{fig:nccsamplingresults}
\end{figure}

%% file: images/resampling/legends_resampling_all.tex
\begin{tikzpicture}[x=1pt,y=1pt]
\definecolor{fillColor}{RGB}{255,255,255}
\definecolor{abstention}{RGB}{248,118,109}
\path[use as bounding box,fill=fillColor,fill opacity=0.00] (0,0) rectangle (794.97, 59.50);
\begin{scope}
\path[clip] ( 0.00,  0.00) rectangle (794.97, 59.50);
\definecolor{drawColor}{RGB}{0,0,0}
\node[text=drawColor,anchor=base west,inner sep=0pt, outer sep=0pt, scale=  1.98] at (30.06, 33.35) { \sc\bf Models:};
\end{scope}
\begin{scope}
\path[clip] (  0.00,  0.00) rectangle (794.97, 59.50);
\definecolor{drawColor}{RGB}{248,118,109}

\path[draw=drawColor,line width= 1.4pt,line join=round, dash pattern=on 2pt off 2pt on 6pt off 2pt] (111.19, 40.16) -- (179.48, 40.16);
\end{scope}
\begin{scope}
\path[clip] (  0.00,  0.00) rectangle (794.97, 59.50);
\definecolor{fillColor}{RGB}{248,118,109}

\path[fill=fillColor] (145.34, 40.16) circle (  4.64);
\end{scope}
\begin{scope}
\path[clip] (  0.00,  0.00) rectangle (794.97, 59.50);
\definecolor{drawColor}{RGB}{0,186,56}

\path[draw=drawColor,line width= 1.4pt,line join=round,dotted] (398.10, 40.16) -- (466.39, 40.16);
\end{scope}
\begin{scope}
\path[clip] (  0.00,  0.00) rectangle (794.97, 59.50);
\definecolor{fillColor}{RGB}{0,186,56}

\path[fill=fillColor] (432.24, 47.38) --
	(438.49, 36.56) --
	(425.99, 36.56) --
	cycle;
\end{scope}
\begin{scope}
\path[clip] (  0.00,  0.00) rectangle (794.97, 59.50);
\definecolor{drawColor}{RGB}{97,156,255}

\path[draw=drawColor,line width= 1.4pt,line join=round] (590.26, 40.16) -- (658.55, 40.16);
\end{scope}
\begin{scope}
\path[clip] (  0.00,  0.00) rectangle (794.97, 59.50);
\definecolor{fillColor}{RGB}{97,156,255}

\path[fill=fillColor] (619.77, 35.52) --
	(629.05, 35.52) --
	(629.05, 44.80) --
	(619.77, 44.80) --
	cycle;
\end{scope}
\begin{scope}
\path[clip] (  0.00,  0.00) rectangle (794.97, 59.50);
\definecolor{drawColor}{RGB}{0,0,0}

\node[text=drawColor,anchor=base west,inner sep=0pt, outer sep=0pt, scale=  1.98] at (192.92, 33.35) {Abstention or Rejection};
\end{scope}
\begin{scope}
\path[clip] (  0.00,  0.00) rectangle (794.97, 59.50);
\definecolor{drawColor}{RGB}{0,0,0}

\node[text=drawColor,anchor=base west,inner sep=0pt, outer sep=0pt, scale=  1.98] at (484.82, 33.35) {Imprecise};
\end{scope}
\begin{scope}
\path[clip] (  0.00,  0.00) rectangle (794.97, 59.50);
\definecolor{drawColor}{RGB}{0,0,0}

\node[text=drawColor,anchor=base west,inner sep=0pt, outer sep=0pt, scale=  1.98] at (676.98, 33.35) {Precise};
\end{scope}
\begin{scope}
\path[clip] (  0.00,  0.00) rectangle (794.97, 59.50);
\definecolor{drawColor}{RGB}{0,0,0}

\node[text=drawColor,anchor=base west,inner sep=0pt, outer sep=0pt, scale=  1.98] at (  8.50, 04.52) {\bf \scshape $c$-SPE:};
\end{scope}
\begin{scope}
\path[clip] (  0.00,  0.00) rectangle (794.97, 59.50);
\definecolor{drawColor}{RGB}{0,0,0}

\path[draw=drawColor,line width= 1.4pt,dash pattern=on 2pt off 2pt on 6pt off 2pt ,line join=round] ( 95.79, 9.33) -- (164.08, 9.33);
\end{scope}
\begin{scope}
\path[clip] (  0.00,  0.00) rectangle (794.97, 59.50);
\definecolor{fillColor}{RGB}{255,255,255}

\path[fill=fillColor] (231.48,  9.92) rectangle (316.84, 28.75);
\end{scope}
\begin{scope}
\path[clip] (  0.00,  0.00) rectangle (794.97, 59.50);
\definecolor{drawColor}{RGB}{0,0,0}

\path[draw=drawColor,line width= 1.4pt,dash pattern=on 7pt off 3pt ,line join=round] (240.01, 9.33) -- (308.30, 9.33);
\end{scope}
\begin{scope}
\path[clip] (  0.00,  0.00) rectangle (794.97, 59.50);
\definecolor{fillColor}{RGB}{255,255,255}

\path[fill=fillColor] (375.70,  9.92) rectangle (461.06, 28.75);
\end{scope}
\begin{scope}
\path[clip] (  0.00,  0.00) rectangle (794.97, 59.50);
\definecolor{drawColor}{RGB}{0,0,0}

\path[draw=drawColor,line width= 1.4pt,dash pattern=on 1pt off 3pt on 4pt off 3pt ,line join=round] (384.24, 9.33) -- (452.52, 9.33);
\end{scope}
\begin{scope}
\path[clip] (  0.00,  0.00) rectangle (794.97, 59.50);
\definecolor{fillColor}{RGB}{255,255,255}

\path[fill=fillColor] (519.93,  9.92) rectangle (605.28, 28.75);
\end{scope}
\begin{scope}
\path[clip] (  0.00,  0.00) rectangle (794.97, 59.50);
\definecolor{drawColor}{RGB}{0,0,0}

\path[draw=drawColor,line width= 1.4pt,dash pattern=on 4pt off 4pt ,line join=round] (528.46, 9.33) -- (596.75, 9.33);
\end{scope}
\begin{scope}
\path[clip] (  0.00,  0.00) rectangle (794.97, 59.50);
\definecolor{fillColor}{RGB}{255,255,255}

\path[fill=fillColor] (664.15,  9.92) rectangle (749.51, 28.75);
\end{scope}
\begin{scope}
\path[clip] (  0.00,  0.00) rectangle (794.97, 59.50);
\definecolor{drawColor}{RGB}{0,0,0}

\path[draw=drawColor,line width= 1.4pt,dash pattern=on 1pt off 15pt ,line join=round] (672.69, 9.33) -- (740.97, 9.33);
\end{scope}
\begin{scope}
\path[clip] (  0.00,  0.00) rectangle (794.97, 59.50);
\definecolor{drawColor}{RGB}{0,0,0}

\node[text=drawColor,anchor=base west,inner sep=0pt, outer sep=0pt, scale=  1.98] at (182.51, 04.52) {0.05};
\end{scope}
\begin{scope}
\path[clip] (  0.00,  0.00) rectangle (794.97, 59.50);
\definecolor{drawColor}{RGB}{0,0,0}

\node[text=drawColor,anchor=base west,inner sep=0pt, outer sep=0pt, scale=  1.98] at (326.74, 04.52) {0.15};
\end{scope}
\begin{scope}
\path[clip] (  0.00,  0.00) rectangle (794.97, 59.50);
\definecolor{drawColor}{RGB}{0,0,0}

\node[text=drawColor,anchor=base west,inner sep=0pt, outer sep=0pt, scale=  1.98] at (470.96, 04.52) {0.25};
\end{scope}
\begin{scope}
\path[clip] (  0.00,  0.00) rectangle (794.97, 59.50);
\definecolor{drawColor}{RGB}{0,0,0}

\node[text=drawColor,anchor=base west,inner sep=0pt, outer sep=0pt, scale=  1.98] at (615.18, 04.52) {0.35};
\end{scope}
\begin{scope}
\path[clip] (  0.00,  0.00) rectangle (794.97, 59.50);
\definecolor{drawColor}{RGB}{0,0,0}

\node[text=drawColor,anchor=base west,inner sep=0pt, outer sep=0pt, scale=  1.98] at (759.41, 04.52) {0.45};
\end{scope}
\end{tikzpicture}

%% file: images/resampling/legends_par_params.tex
\begin{tikzpicture}[x=1pt,y=0.4pt]
\begin{scope}
\path[clip] (  0.00,  0.00) rectangle (794.97, 55.50);
\definecolor{drawColor}{RGB}{0,0,0}

\node[text=drawColor,anchor=base west,inner sep=0pt, outer sep=0pt, scale=  1.98] at (  8.50, 12.02) {\bf \scshape $c$-PAR:};
\end{scope}
\begin{scope}
\path[clip] (  0.00,  0.00) rectangle (794.97, 55.50);
\definecolor{drawColor}{RGB}{0,0,0}

\path[draw=drawColor,line width= 1.4pt,dash pattern=on 2pt off 2pt on 6pt off 2pt ,line join=round] ( 95.79, 27.33) -- (164.08, 27.33);
\end{scope}
\begin{scope}
\path[clip] (  0.00,  0.00) rectangle (794.97, 55.50);
\definecolor{fillColor}{RGB}{255,255,255}

\path[fill=fillColor] (231.48,  9.92) rectangle (316.84, 28.75);
\end{scope}
\begin{scope}
\path[clip] (  0.00,  0.00) rectangle (794.97, 55.50);
\definecolor{drawColor}{RGB}{0,0,0}

\path[draw=drawColor,line width= 1.4pt,dash pattern=on 7pt off 3pt ,line join=round] (240.01, 27.33) -- (308.30, 27.33);
\end{scope}
\begin{scope}
\path[clip] (  0.00,  0.00) rectangle (794.97, 55.50);
\definecolor{fillColor}{RGB}{255,255,255}

\path[fill=fillColor] (375.70,  9.92) rectangle (461.06, 28.75);
\end{scope}
\begin{scope}
\path[clip] (  0.00,  0.00) rectangle (794.97, 55.50);
\definecolor{drawColor}{RGB}{0,0,0}

\path[draw=drawColor,line width= 1.4pt,dash pattern=on 1pt off 3pt on 4pt off 3pt ,line join=round] (384.24, 27.33) -- (452.52, 27.33);
\end{scope}
\begin{scope}
\path[clip] (  0.00,  0.00) rectangle (794.97, 55.50);
\definecolor{fillColor}{RGB}{255,255,255}

\path[fill=fillColor] (519.93,  9.92) rectangle (605.28, 28.75);
\end{scope}
\begin{scope}
\path[clip] (  0.00,  0.00) rectangle (794.97, 55.50);
\definecolor{drawColor}{RGB}{0,0,0}

\path[draw=drawColor,line width= 1.4pt,dash pattern=on 4pt off 4pt ,line join=round] (528.46, 27.33) -- (596.75, 27.33);
\end{scope}
\begin{scope}
\path[clip] (  0.00,  0.00) rectangle (794.97, 55.50);
\definecolor{fillColor}{RGB}{255,255,255}

\path[fill=fillColor] (664.15,  9.92) rectangle (749.51, 28.75);
\end{scope}
\begin{scope}
\path[clip] (  0.00,  0.00) rectangle (794.97, 55.50);
\definecolor{drawColor}{RGB}{0,0,0}

\path[draw=drawColor,line width= 1.4pt,dash pattern=on 1pt off 15pt ,line join=round] (672.69, 27.33) -- (740.97, 27.33);
\end{scope}
\begin{scope}
\path[clip] (  0.00,  0.00) rectangle (794.97, 55.50);
\definecolor{drawColor}{RGB}{0,0,0}

\node[text=drawColor,anchor=base west,inner sep=0pt, outer sep=0pt, scale=  1.98] at (182.51, 12.02) {0.10};
\end{scope}
\begin{scope}
\path[clip] (  0.00,  0.00) rectangle (794.97, 55.50);
\definecolor{drawColor}{RGB}{0,0,0}

\node[text=drawColor,anchor=base west,inner sep=0pt, outer sep=0pt, scale=  1.98] at (326.74, 12.02) {0.20};
\end{scope}
\begin{scope}
\path[clip] (  0.00,  0.00) rectangle (794.97, 55.50);
\definecolor{drawColor}{RGB}{0,0,0}

\node[text=drawColor,anchor=base west,inner sep=0pt, outer sep=0pt, scale=  1.98] at (470.96, 12.02) {0.30};
\end{scope}
\begin{scope}
\path[clip] (  0.00,  0.00) rectangle (794.97, 55.50);
\definecolor{drawColor}{RGB}{0,0,0}

\node[text=drawColor,anchor=base west,inner sep=0pt, outer sep=0pt, scale=  1.98] at (615.18, 12.02) {0.40};
\end{scope}
\begin{scope}
\path[clip] (  0.00,  0.00) rectangle (794.97, 55.50);
\definecolor{drawColor}{RGB}{0,0,0}

\node[text=drawColor,anchor=base west,inner sep=0pt, outer sep=0pt, scale=  1.98] at (759.41, 12.02) {0.50};
\end{scope}
\end{tikzpicture}

%% file: images/resampling/legends_rej_params.tex
\begin{tikzpicture}[x=1pt,y=0.4pt]
\begin{scope}
\path[clip] (  0.00,  0.00) rectangle (794.97, 59.50);
\definecolor{drawColor}{RGB}{0,0,0}

\node[text=drawColor, anchor=base west,inner sep=0pt, outer sep=0pt, scale=1.98] at ( 8.40, 10.52) {\!\!\bf{\small Threshold} $\gamma$:};
\end{scope}
\begin{scope}
\path[clip] (  0.00,  0.00) rectangle (794.97, 59.50);
\definecolor{drawColor}{RGB}{0,0,0}

\path[draw=drawColor,line width= 1.4pt,dash pattern=on 2pt off 2pt on 6pt off 2pt ,line join=round] ( 105.79, 23.33) -- (169.08, 23.33);
\end{scope}

\begin{scope}
\path[clip] (  0.00,  0.00) rectangle (794.97, 59.50);
\definecolor{drawColor}{RGB}{0,0,0}

\path[draw=drawColor,line width= 1.4pt,dash pattern=on 7pt off 3pt ,line join=round] (240.01, 23.33) -- (308.30, 23.33);
\end{scope}
\begin{scope}
\path[clip] (  0.00,  0.00) rectangle (794.97, 59.50);
\definecolor{fillColor}{RGB}{255,255,255}

\path[fill=fillColor] (375.70,  9.92) rectangle (461.06, 28.75);
\end{scope}
\begin{scope}
\path[clip] (  0.00,  0.00) rectangle (794.97, 59.50);
\definecolor{drawColor}{RGB}{0,0,0}

\path[draw=drawColor,line width= 1.4pt,dash pattern=on 1pt off 3pt on 4pt off 3pt ,line join=round] (384.24, 23.33) -- (452.52, 23.33);
\end{scope}
\begin{scope}
\path[clip] (  0.00,  0.00) rectangle (794.97, 59.50);
\definecolor{fillColor}{RGB}{255,255,255}

\path[fill=fillColor] (519.93,  9.92) rectangle (605.28, 28.75);
\end{scope}
\begin{scope}
\path[clip] (  0.00,  0.00) rectangle (794.97, 59.50);
\definecolor{drawColor}{RGB}{0,0,0}

\path[draw=drawColor,line width= 1.4pt,dash pattern=on 4pt off 4pt ,line join=round] (528.46, 23.33) -- (596.75, 23.33);
\end{scope}
\begin{scope}
\path[clip] (  0.00,  0.00) rectangle (794.97, 59.50);
\definecolor{fillColor}{RGB}{255,255,255}

\path[fill=fillColor] (664.15,  9.92) rectangle (749.51, 28.75);
\end{scope}
\begin{scope}
\path[clip] (  0.00,  0.00) rectangle (794.97, 59.50);
\definecolor{drawColor}{RGB}{0,0,0}

\path[draw=drawColor,line width= 1.4pt,dash pattern=on 1pt off 15pt ,line join=round] (672.69, 23.33) -- (740.97, 23.33);
\end{scope}
\begin{scope}
\path[clip] (  0.00,  0.00) rectangle (794.97, 59.50);
\definecolor{drawColor}{RGB}{0,0,0}

\node[text=drawColor,anchor=base west,inner sep=0pt, outer sep=0pt, scale=  1.98] at (182.51, 10.52) {0.05};
\end{scope}
\begin{scope}
\path[clip] (  0.00,  0.00) rectangle (794.97, 59.50);
\definecolor{drawColor}{RGB}{0,0,0}

\node[text=drawColor,anchor=base west,inner sep=0pt, outer sep=0pt, scale=  1.98] at (326.74, 10.52) {0.15};
\end{scope}
\begin{scope}
\path[clip] (  0.00,  0.00) rectangle (794.97, 59.50);
\definecolor{drawColor}{RGB}{0,0,0}

\node[text=drawColor,anchor=base west,inner sep=0pt, outer sep=0pt, scale=  1.98] at (470.96, 10.52) {0.25};
\end{scope}
\begin{scope}
\path[clip] (  0.00,  0.00) rectangle (794.97, 59.50);
\definecolor{drawColor}{RGB}{0,0,0}

\node[text=drawColor,anchor=base west,inner sep=0pt, outer sep=0pt, scale=  1.98] at (615.18, 10.52) {0.35};
\end{scope}
\begin{scope}
\path[clip] (  0.00,  0.00) rectangle (794.97, 59.50);
\definecolor{drawColor}{RGB}{0,0,0}

\node[text=drawColor,anchor=base west,inner sep=0pt, outer sep=0pt, scale=  1.98] at (759.41, 10.52) {0.45};
\end{scope}
\end{tikzpicture}

%% file: images/images_resampling_igda.tex
\begin{figure}[!th]
	\centering%
	\resizebox{0.8\textwidth}{!}{%
	  \input{images/resampling/legends_resampling}
	}\vspace{-2mm}\qquad%
	\renewcommand{\thesubfigure}{(a)}
	\subfigure[\sc (Imprecise) Euclidian discriminant analysis]{
		\hspace{-3mm}
		\subfigure[{\scshape Emotions} ($\tau=0.41$)]{
			\includegraphics[width=0.244\linewidth]
				{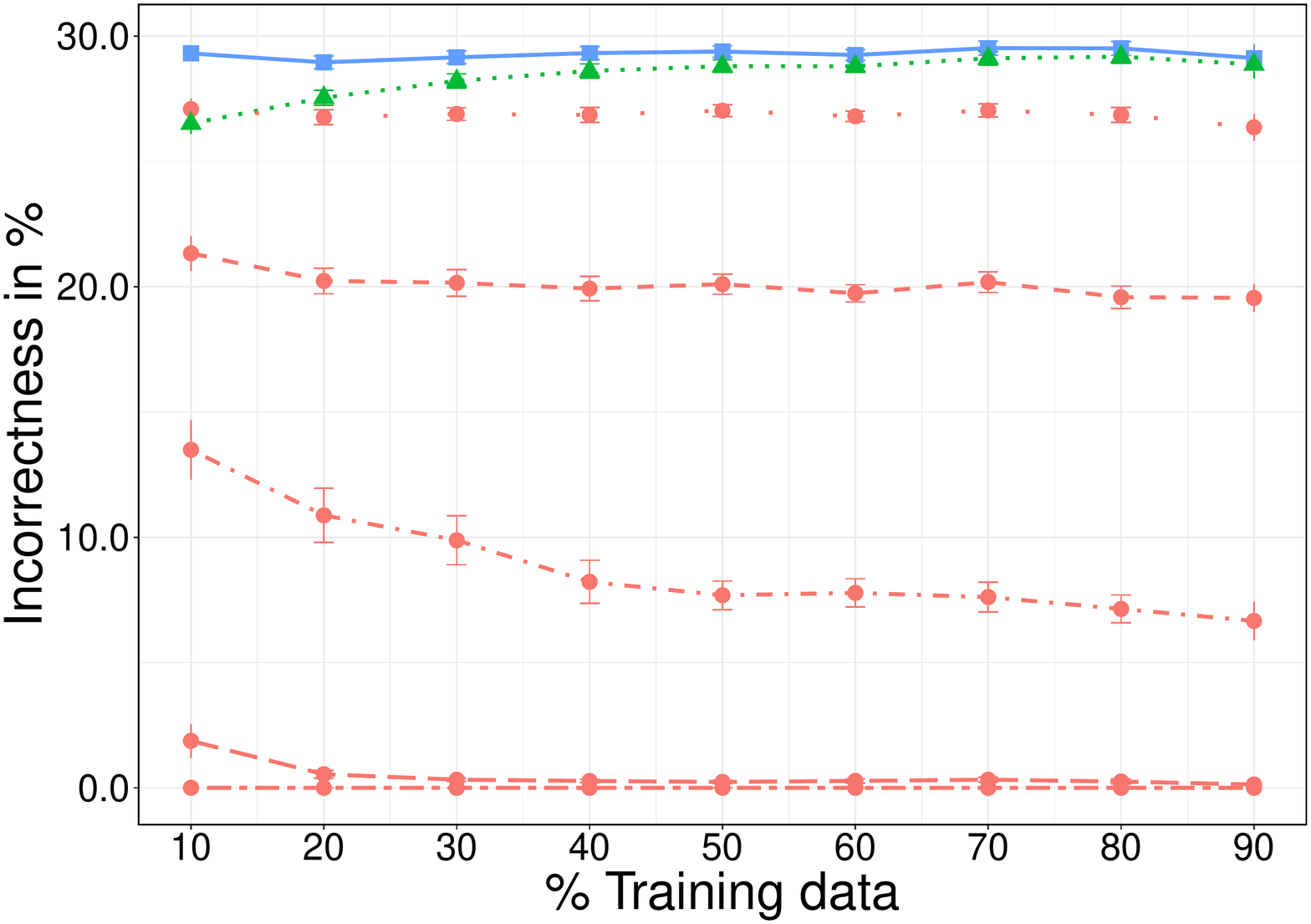}} 
		\subfigure[\scshape Flags ($\tau=0.41$)]{
			\includegraphics[width=0.244\linewidth]
				{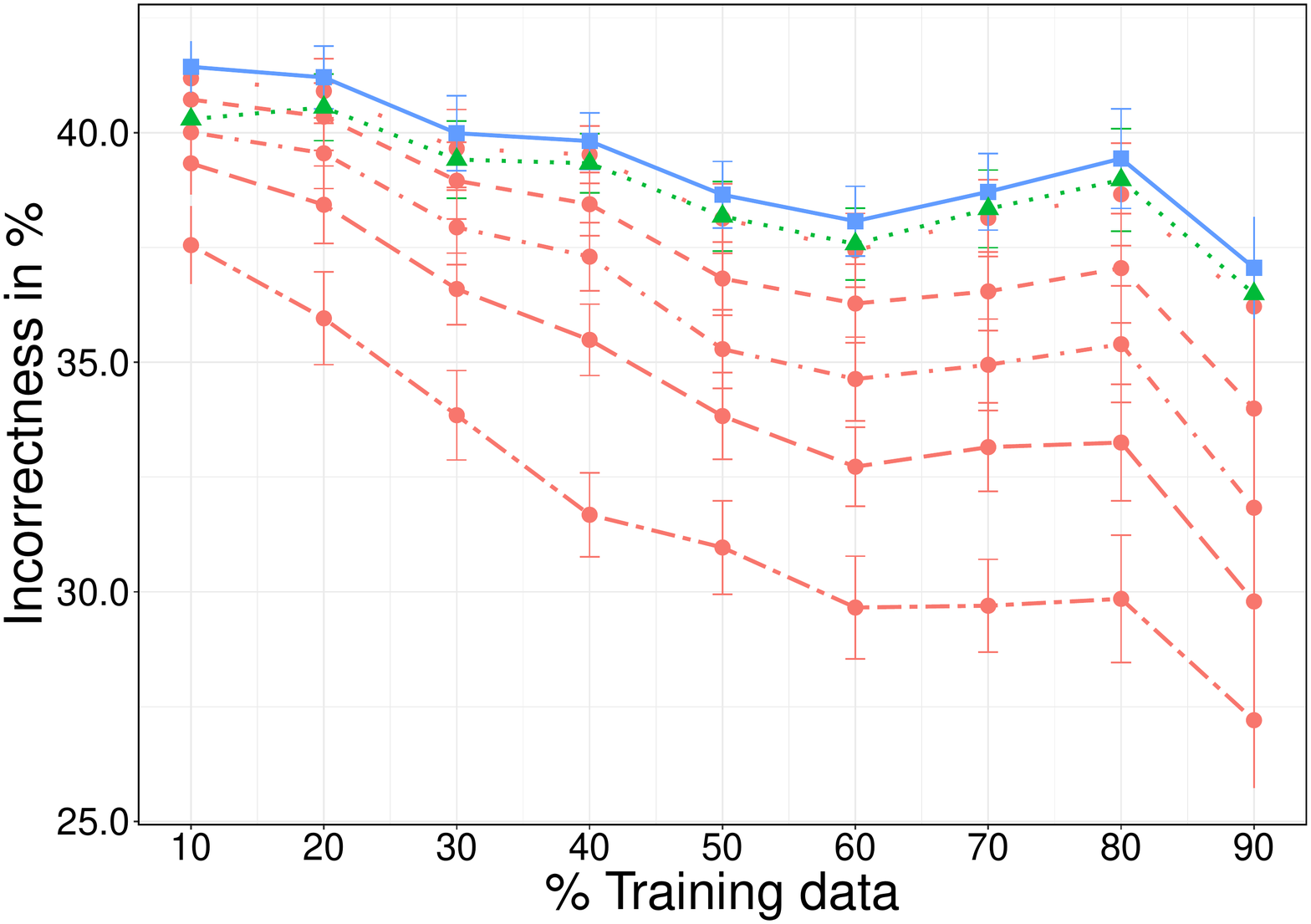}}
		\subfigure[\scshape Yeast ($\tau=4.50$)]{
			\includegraphics[width=0.244\linewidth]
				{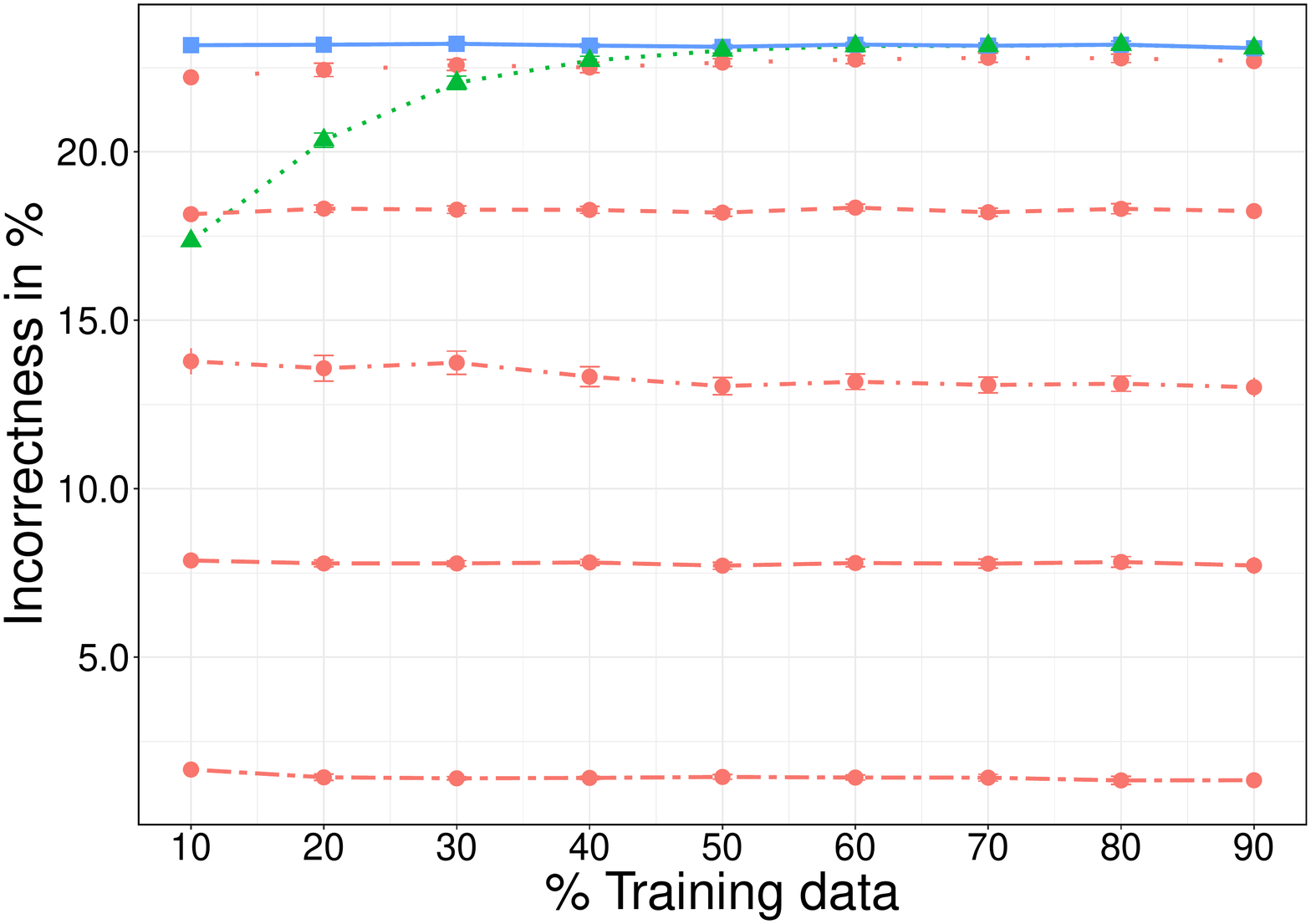}} 
		\subfigure[\scshape CAL500 ($\tau=4.50$)]{
			\includegraphics[width=0.244\linewidth]
				{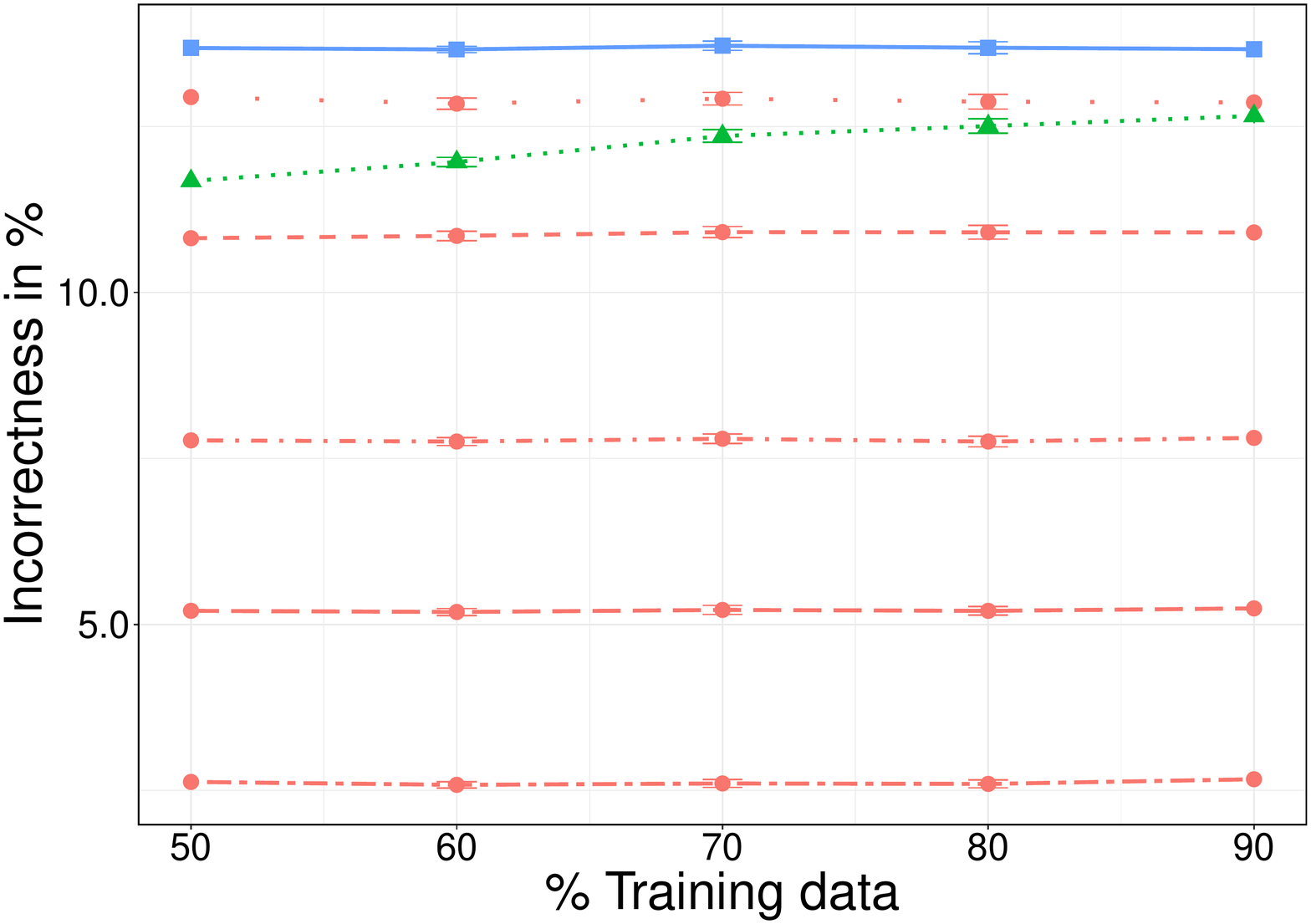}} 
	}\vspace{-2mm}\qquad%
    \renewcommand{\thesubfigure}{(c)}
	\subfigure[\sc (Imprecise) Linear discriminant analysis]{
		\hspace{-3mm}
	  	\subfigure[\scshape Emotions ($\tau=0.10$)]{
			\includegraphics[width=0.244\linewidth]
				{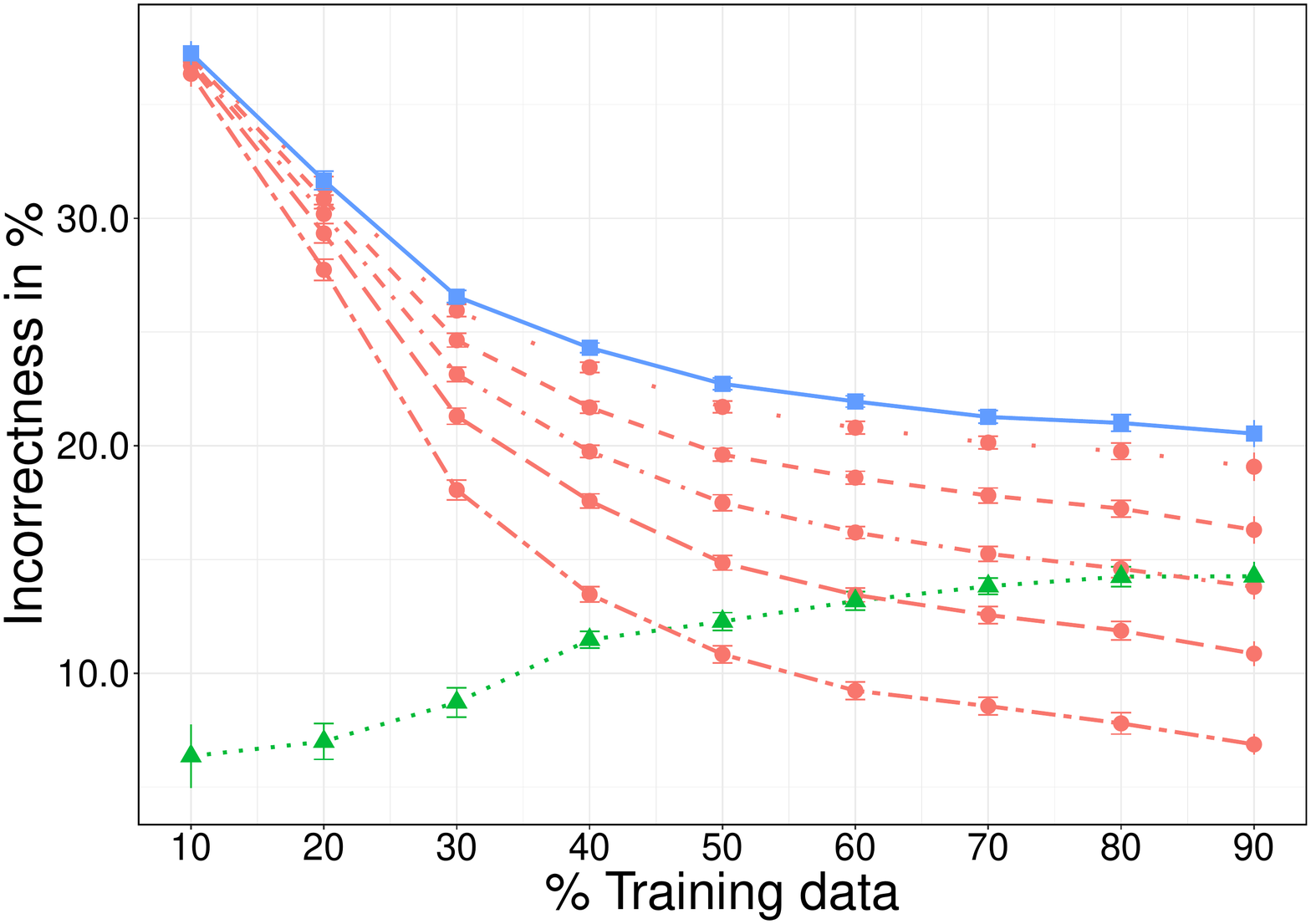}}
		\subfigure[\scshape Flags ($\tau=0.41$)]{
			\includegraphics[width=0.244\linewidth]
				{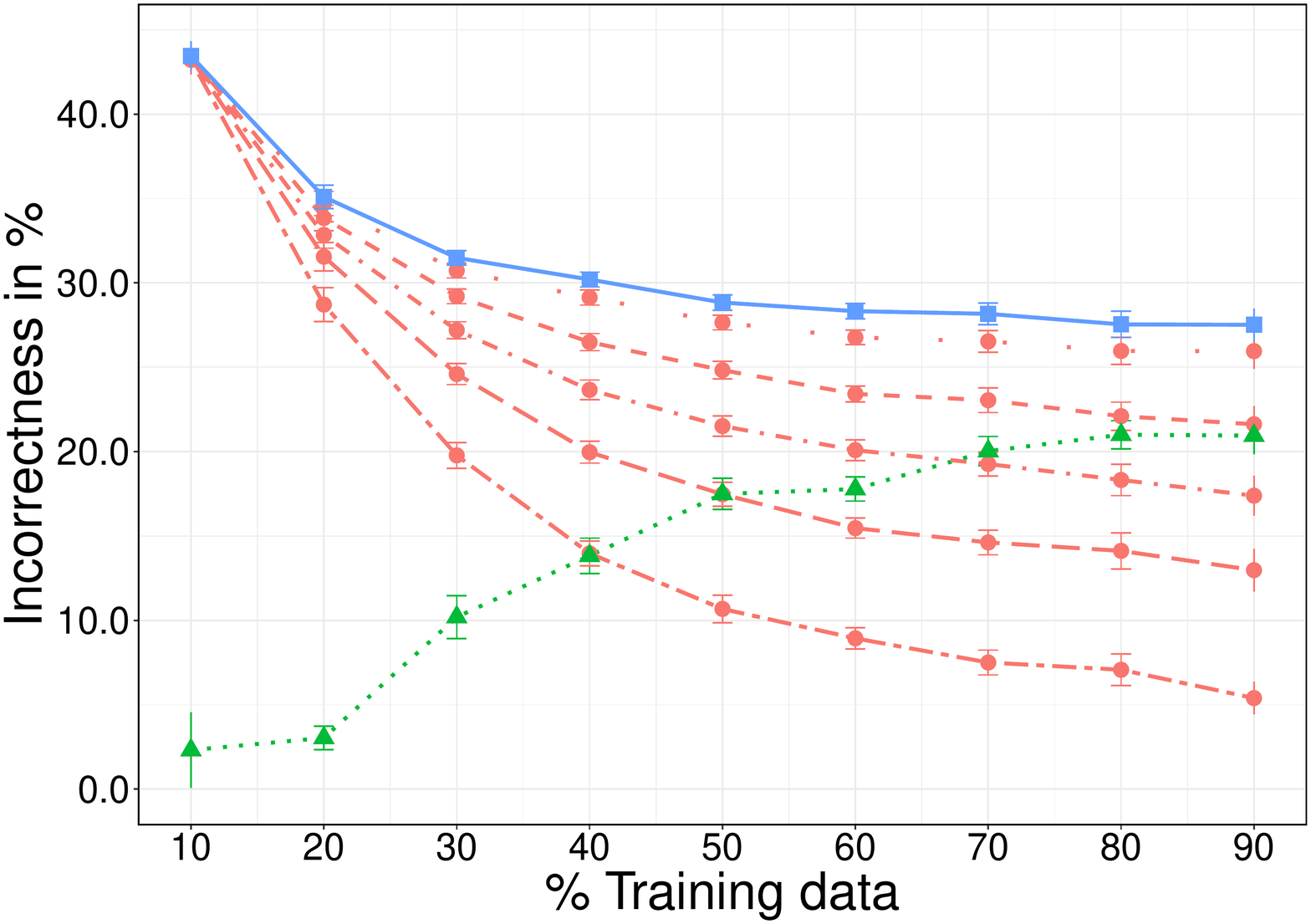}}
		\subfigure[\scshape Yeast ($\tau=0.04$)]{
			\includegraphics[width=0.244\linewidth]
				{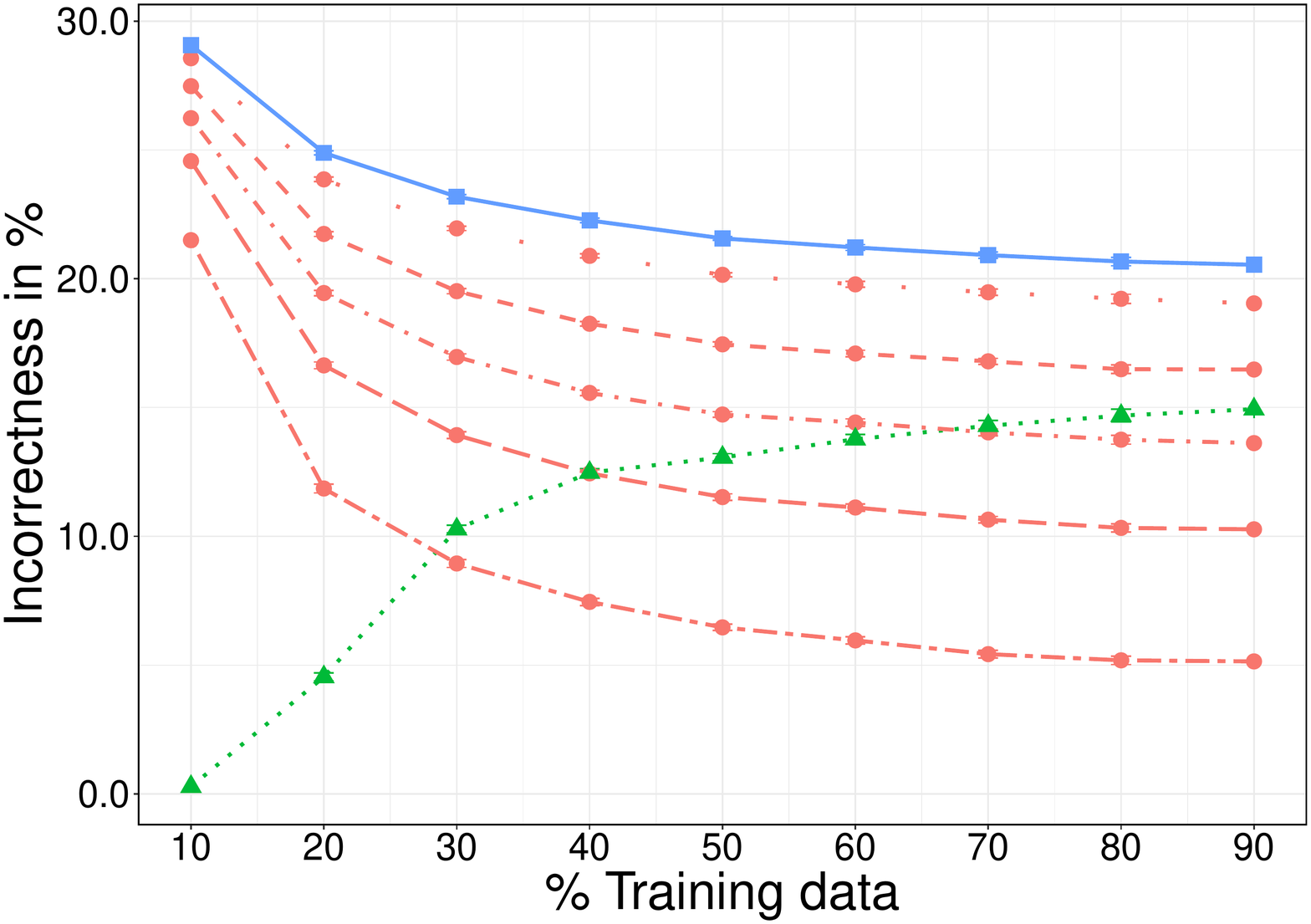}} 
		\subfigure[\scshape CAL500 ($\tau=0.15$)]{
			\includegraphics[width=0.244\linewidth]
				{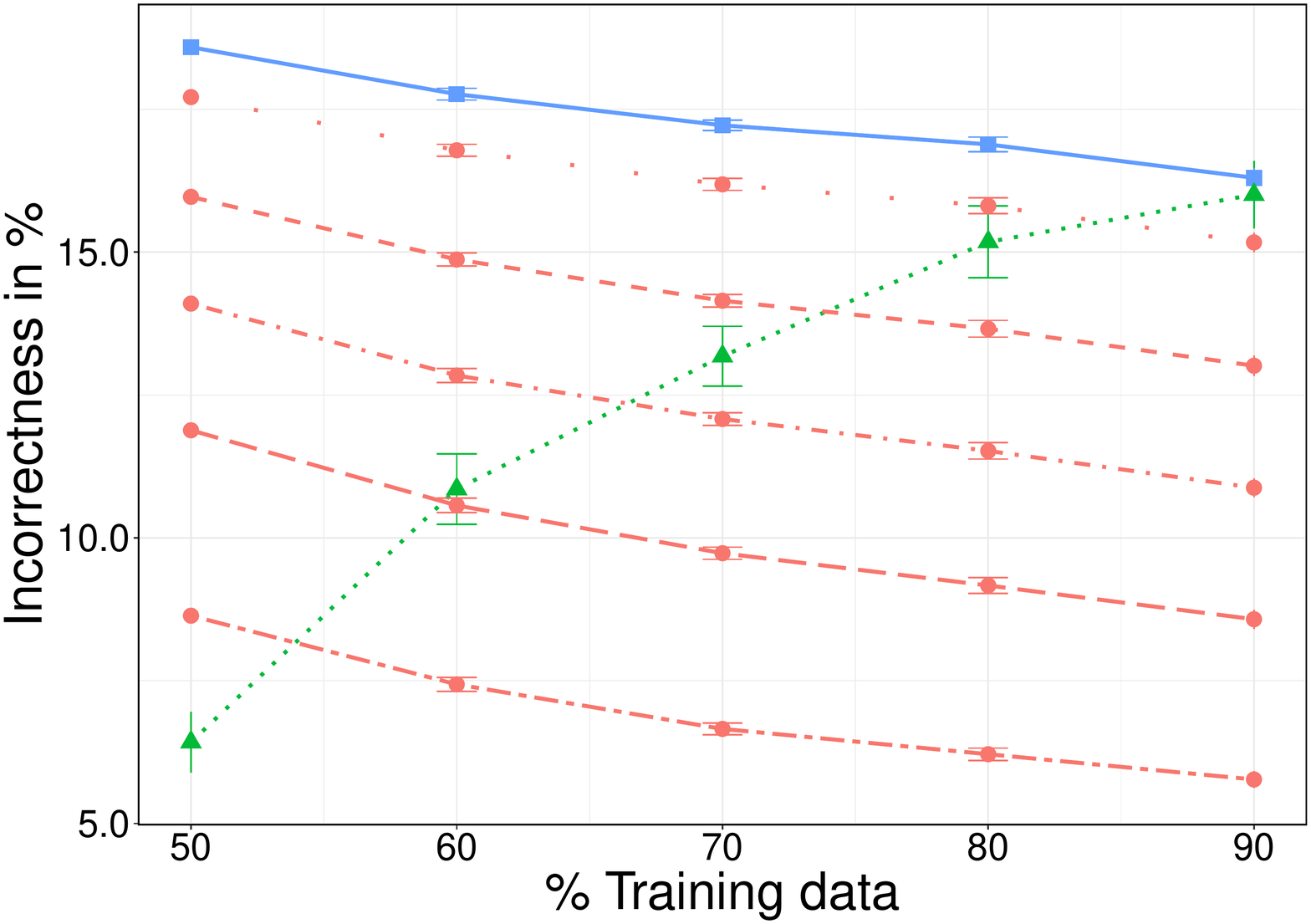}} 
	}\vspace{-2mm}\qquad%
	\renewcommand{\thesubfigure}{(b)}
	\subfigure[\sc (Imprecise) Naive discriminant analysis]{ 
		\hspace{-3mm}
		\subfigure[\scshape Emotions ($\tau=0.41$)]{
			\includegraphics[width=0.244\linewidth]
				{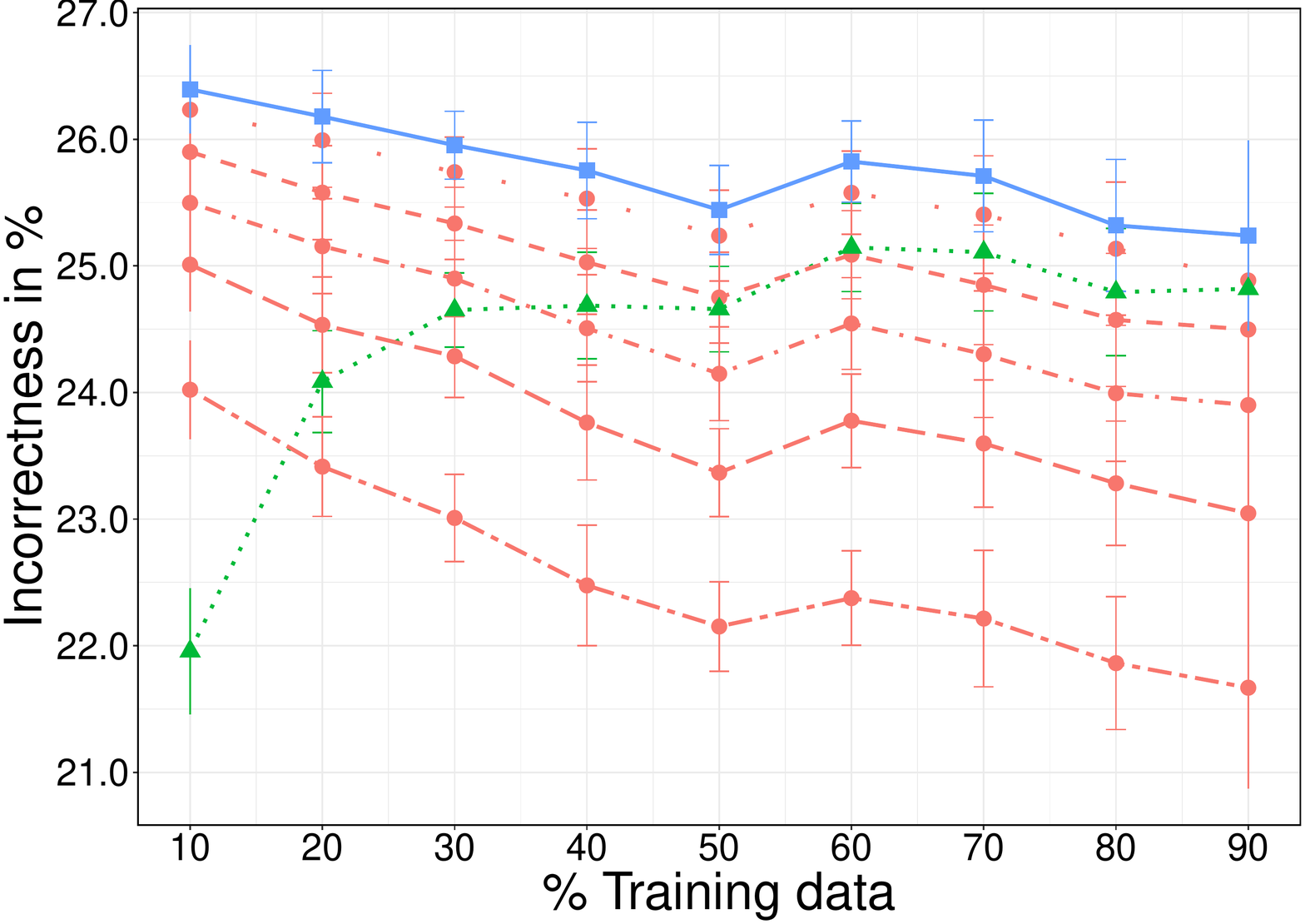}}
		\subfigure[\scshape Flags ($\tau=0.41$)]{
			\includegraphics[width=0.244\linewidth]
				{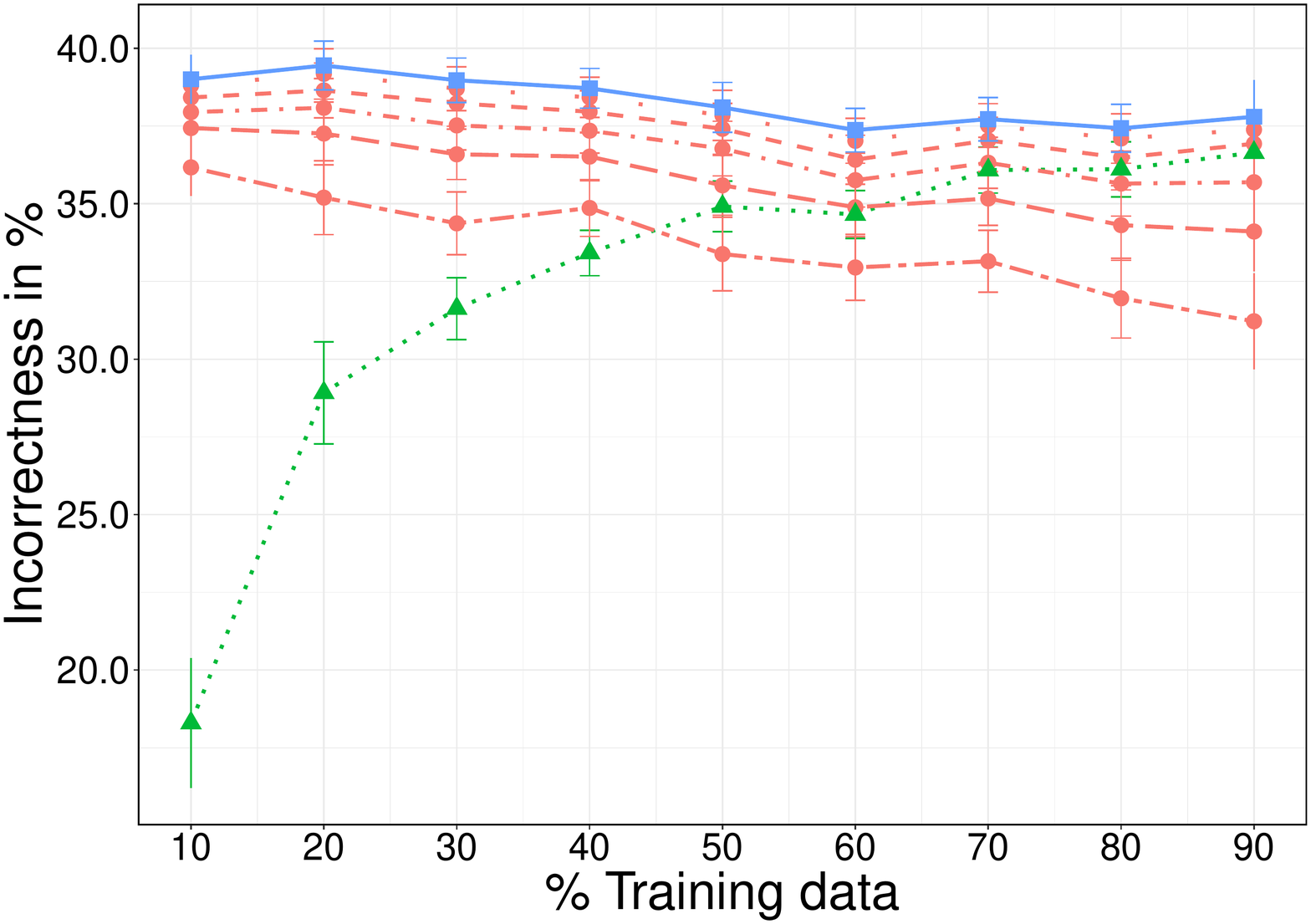}}
		\subfigure[\scshape Yeast ($\tau=4.50$)]{
			\includegraphics[width=0.244\linewidth]
				{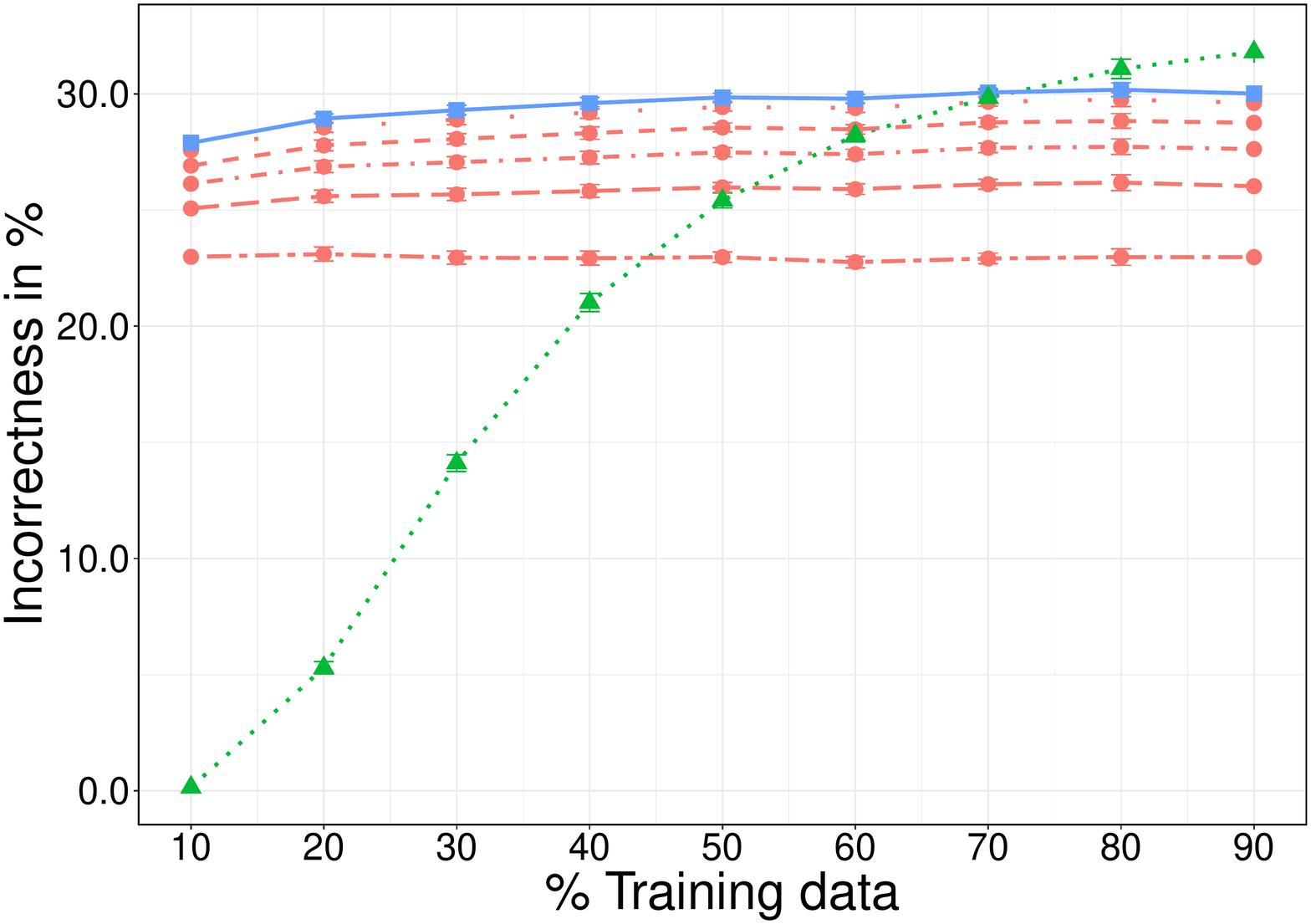}}
		\subfigure[\scshape CAL500 ($\tau=4.50$)]{
			\includegraphics[width=0.244\linewidth]
				{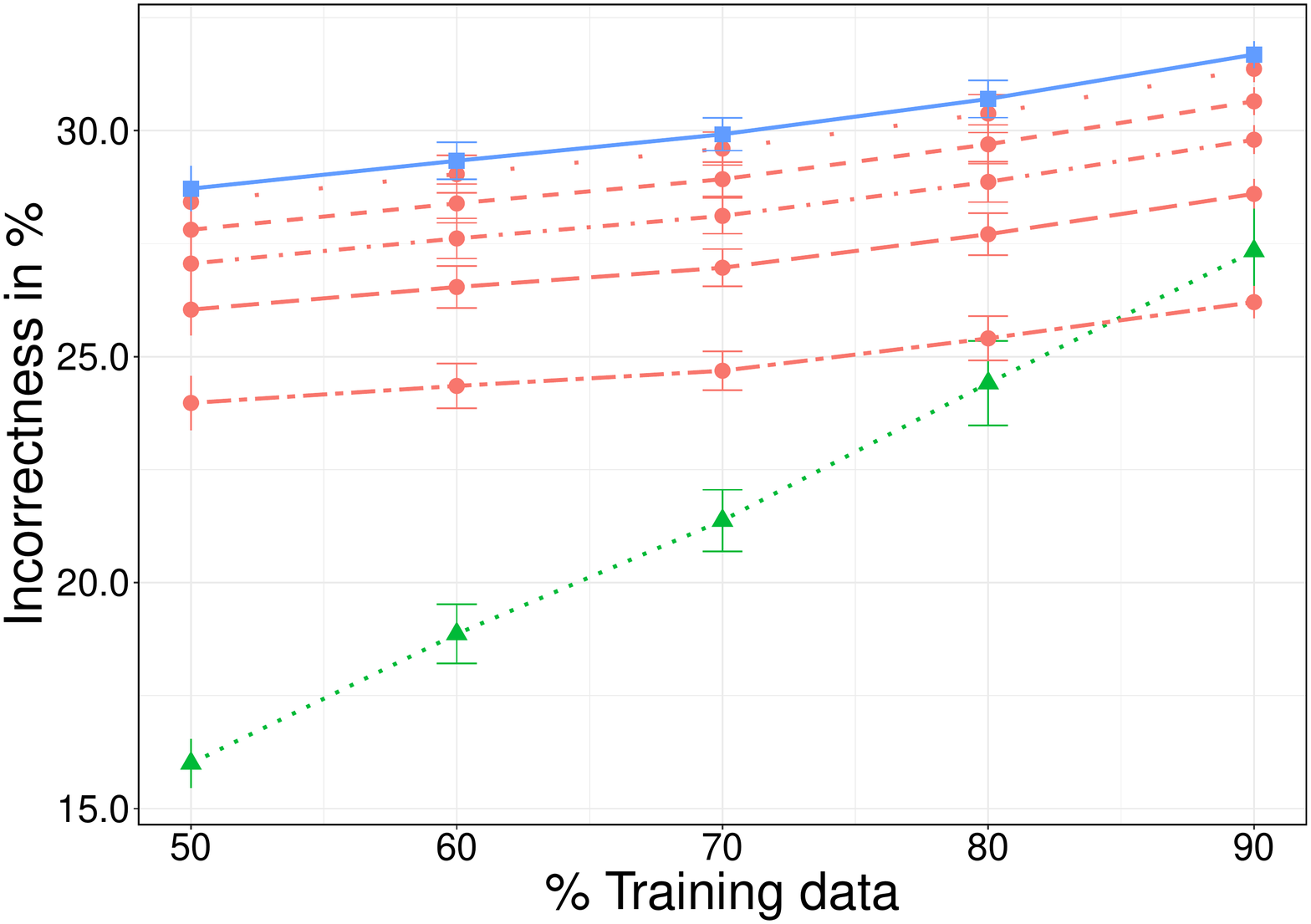}} 
	}\vspace{-2mm}\qquad%
	\renewcommand{\thesubfigure}{(d)}
	\subfigure[\sc (Imprecise) Quadratic discriminant analysis]{
		\hspace{-3mm}\label{fig:iqdasamplingresults}
		\subfigure[\scshape Emotions ($\tau=0.41$)]{
			\includegraphics[width=0.244\linewidth]
				{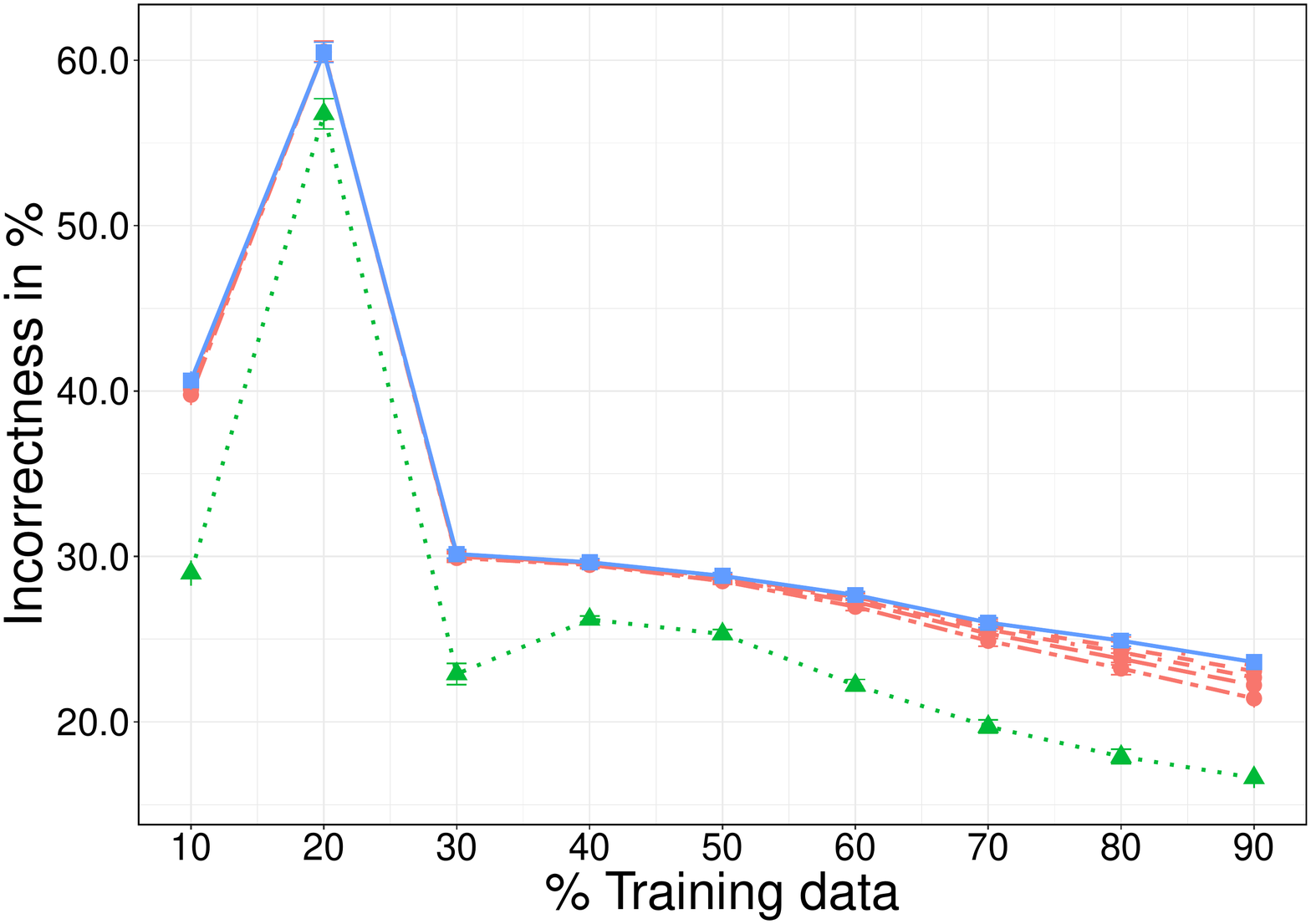}}
		\subfigure[\scshape Flags ($\tau=0.41$)]{
			\includegraphics[width=0.244\linewidth]
				{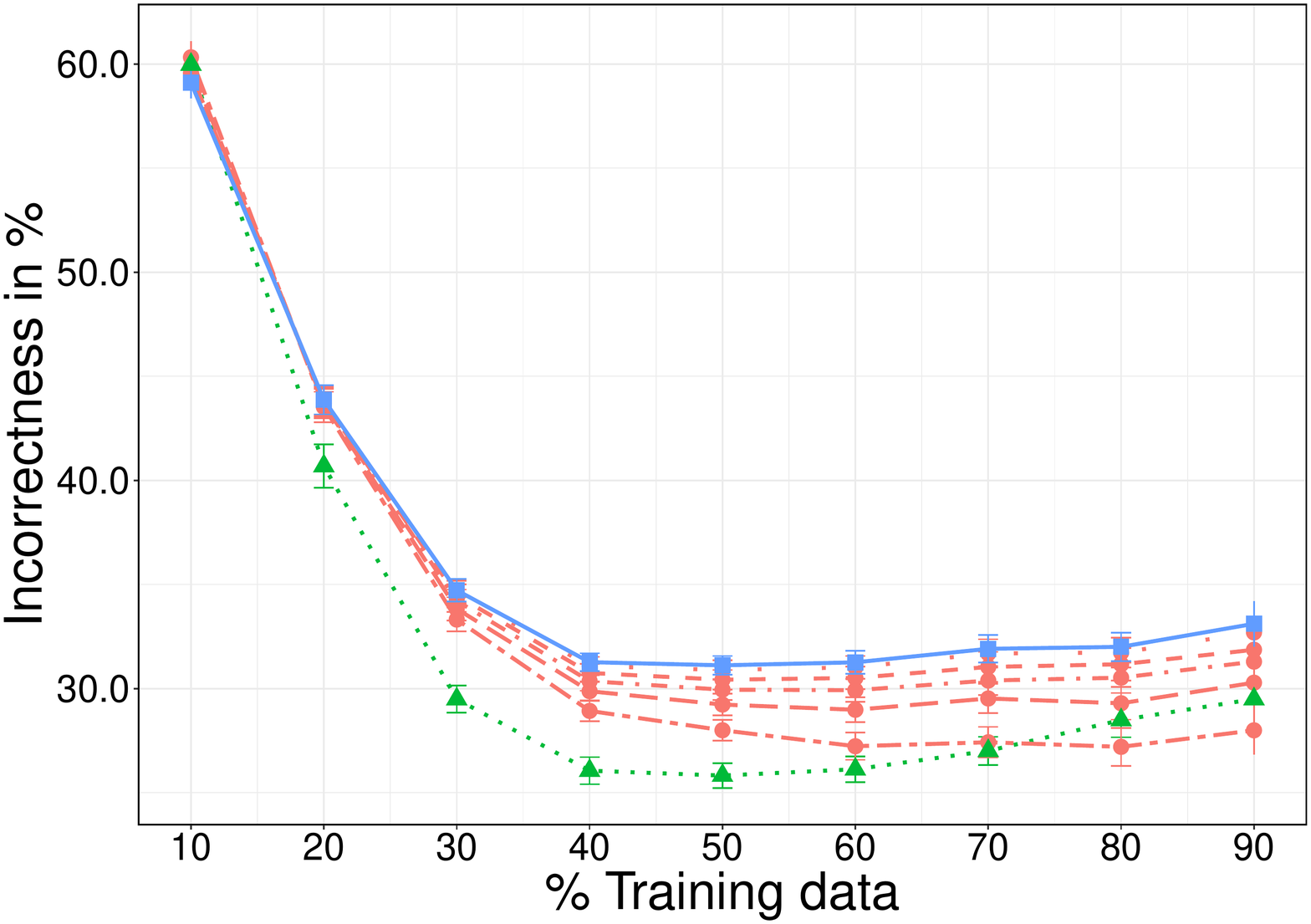}}
		\subfigure[\scshape Yeast ($\tau=0.41$)]{
			\includegraphics[width=0.244\linewidth]
				{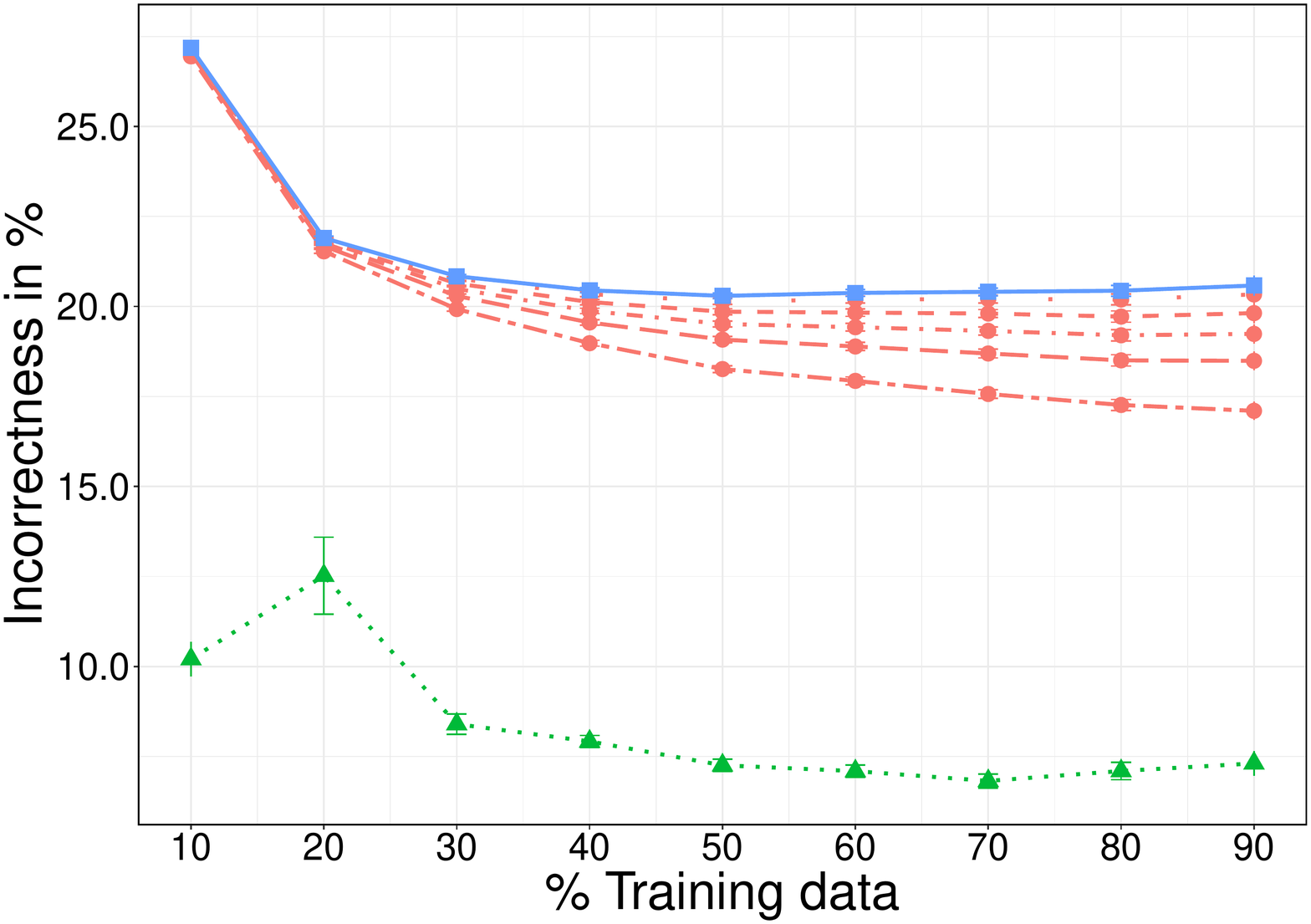}}
		\subfigure[\scshape CAL500 ($\tau=0.41$)]{
			\includegraphics[width=0.244\linewidth]
				{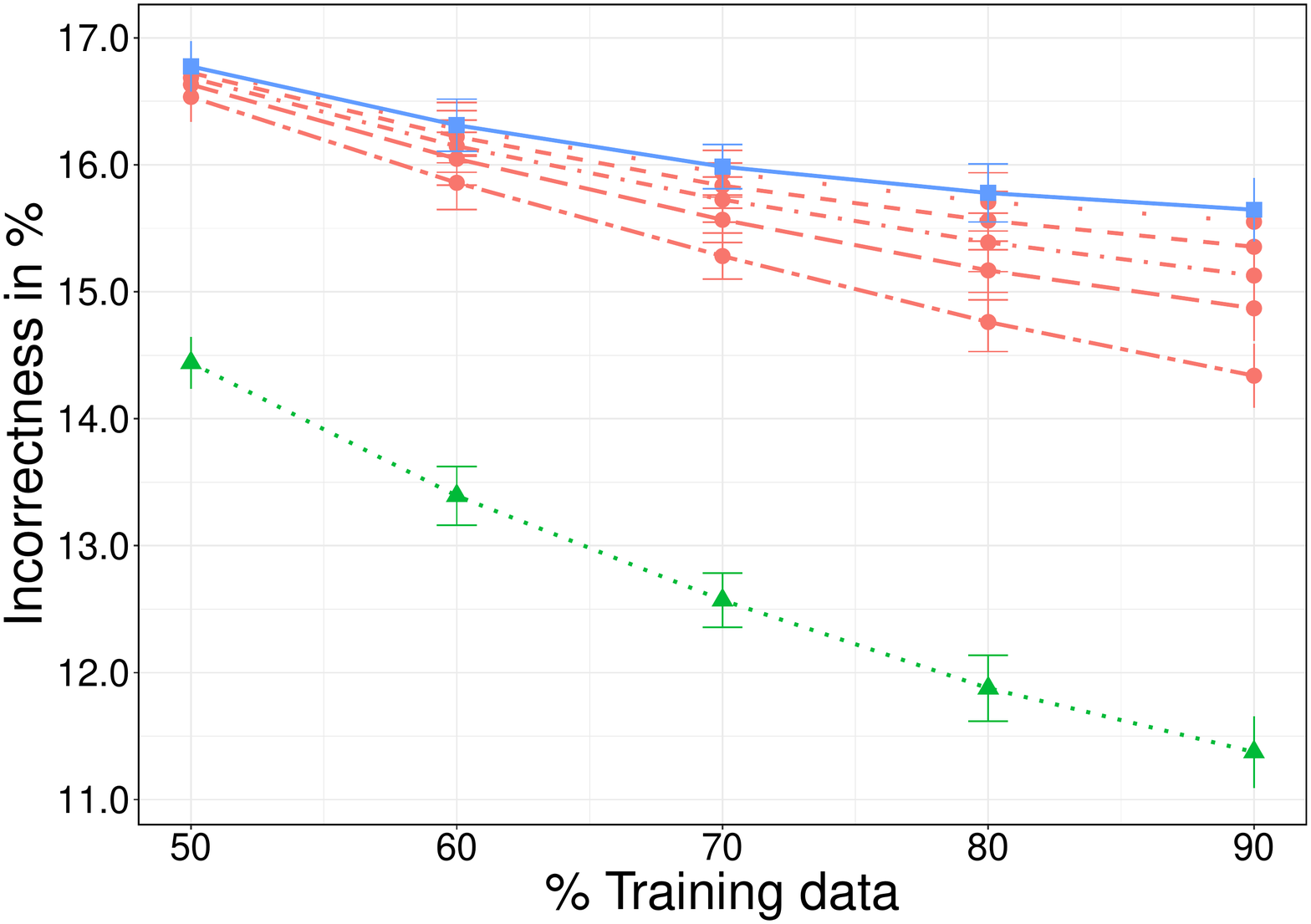}}
	}\vspace{-2mm}
	\caption{{\bf Downsampling - IGDA - Partial abstention SPE.} Incorrectness (y-axis) evolution for the precise, partial abstention with $f_{SPE}$ penalty function, and skeptical approach, each one with different hyper-parameter levels (except to the precise approach), and with respect to different percentages of training data sets (x-axis). In each row a different Gaussian classifier model is fitted.}
	\label{fig:igdasamplingresults}
\end{figure}

%% file: images/resampling/legends_resampling.tex
\begin{tikzpicture}[x=1pt,y=1pt]
\definecolor{fillColor}{RGB}{255,255,255}
\path[use as bounding box,fill=fillColor,fill opacity=0.00] (0,0) rectangle (794.97, 59.50);
\begin{scope}
\path[clip] ( 0.00,  0.00) rectangle (794.97, 59.50);
\definecolor{drawColor}{RGB}{0,0,0}
\node[text=drawColor,anchor=base west,inner sep=0pt, outer sep=0pt, scale=  1.98] at (.06, 33.35) { \sc\bf Models:};
\end{scope}
\begin{scope}
\path[clip] (  0.00,  0.00) rectangle (794.97, 59.50);
\definecolor{drawColor}{RGB}{248,118,109}

\path[draw=drawColor,line width= 1.4pt,line join=round, dash pattern=on 2pt off 2pt on 6pt off 2pt] (86.19, 40.16) -- (154.48, 40.16);
\end{scope}
\begin{scope}
\path[clip] (  0.00,  0.00) rectangle (794.97, 59.50);
\definecolor{fillColor}{RGB}{248,118,109}

\path[fill=fillColor] (120.34, 40.16) circle (  4.64);
\end{scope}
\begin{scope}
\path[clip] (  0.00,  0.00) rectangle (794.97, 59.50);
\definecolor{drawColor}{RGB}{0,186,56}

\path[draw=drawColor,line width= 1.4pt,line join=round,dotted] (398.10, 40.16) -- (466.39, 40.16);
\end{scope}
\begin{scope}
\path[clip] (  0.00,  0.00) rectangle (794.97, 59.50);
\definecolor{fillColor}{RGB}{0,186,56}

\path[fill=fillColor] (432.24, 47.38) --
	(438.49, 36.56) --
	(425.99, 36.56) --
	cycle;
\end{scope}
\begin{scope}
\path[clip] (  0.00,  0.00) rectangle (794.97, 59.50);
\definecolor{drawColor}{RGB}{97,156,255}

\path[draw=drawColor,line width= 1.4pt,line join=round] (590.26, 40.16) -- (658.55, 40.16);
\end{scope}
\begin{scope}
\path[clip] (  0.00,  0.00) rectangle (794.97, 59.50);
\definecolor{fillColor}{RGB}{97,156,255}

\path[fill=fillColor] (619.77, 35.52) --
	(629.05, 35.52) --
	(629.05, 44.80) --
	(619.77, 44.80) --
	cycle;
\end{scope}
\begin{scope}
\path[clip] (  0.00,  0.00) rectangle (794.97, 59.50);
\definecolor{drawColor}{RGB}{0,0,0}

\node[text=drawColor,anchor=base west,inner sep=0pt, outer sep=0pt, scale=  1.98] at (167.92, 33.35) {Partial abstention with $f_{SPE}$};
\end{scope}
\begin{scope}
\path[clip] (  0.00,  0.00) rectangle (794.97, 59.50);
\definecolor{drawColor}{RGB}{0,0,0}

\node[text=drawColor,anchor=base west,inner sep=0pt, outer sep=0pt, scale=  1.98] at (484.82, 33.35) {Imprecise};
\end{scope}
\begin{scope}
\path[clip] (  0.00,  0.00) rectangle (794.97, 59.50);
\definecolor{drawColor}{RGB}{0,0,0}

\node[text=drawColor,anchor=base west,inner sep=0pt, outer sep=0pt, scale=  1.98] at (676.98, 33.35) {Precise};
\end{scope}
\begin{scope}
\path[clip] (  0.00,  0.00) rectangle (794.97, 59.50);
\definecolor{drawColor}{RGB}{0,0,0}

\node[text=drawColor,anchor=base west,inner sep=0pt, outer sep=0pt, scale=  1.98] at (  8.50, 04.52) {\bf \scshape $c$-SPE:};
\end{scope}
\begin{scope}
\path[clip] (  0.00,  0.00) rectangle (794.97, 59.50);
\definecolor{drawColor}{RGB}{0,0,0}

\path[draw=drawColor,line width= 1.4pt,dash pattern=on 2pt off 2pt on 6pt off 2pt ,line join=round] ( 95.79, 9.33) -- (164.08, 9.33);
\end{scope}
\begin{scope}
\path[clip] (  0.00,  0.00) rectangle (794.97, 59.50);
\definecolor{fillColor}{RGB}{255,255,255}

\path[fill=fillColor] (231.48,  9.92) rectangle (316.84, 28.75);
\end{scope}
\begin{scope}
\path[clip] (  0.00,  0.00) rectangle (794.97, 59.50);
\definecolor{drawColor}{RGB}{0,0,0}

\path[draw=drawColor,line width= 1.4pt,dash pattern=on 7pt off 3pt ,line join=round] (240.01, 9.33) -- (308.30, 9.33);
\end{scope}
\begin{scope}
\path[clip] (  0.00,  0.00) rectangle (794.97, 59.50);
\definecolor{fillColor}{RGB}{255,255,255}

\path[fill=fillColor] (375.70,  9.92) rectangle (461.06, 28.75);
\end{scope}
\begin{scope}
\path[clip] (  0.00,  0.00) rectangle (794.97, 59.50);
\definecolor{drawColor}{RGB}{0,0,0}

\path[draw=drawColor,line width= 1.4pt,dash pattern=on 1pt off 3pt on 4pt off 3pt ,line join=round] (384.24, 9.33) -- (452.52, 9.33);
\end{scope}
\begin{scope}
\path[clip] (  0.00,  0.00) rectangle (794.97, 59.50);
\definecolor{fillColor}{RGB}{255,255,255}

\path[fill=fillColor] (519.93,  9.92) rectangle (605.28, 28.75);
\end{scope}
\begin{scope}
\path[clip] (  0.00,  0.00) rectangle (794.97, 59.50);
\definecolor{drawColor}{RGB}{0,0,0}

\path[draw=drawColor,line width= 1.4pt,dash pattern=on 4pt off 4pt ,line join=round] (528.46, 9.33) -- (596.75, 9.33);
\end{scope}
\begin{scope}
\path[clip] (  0.00,  0.00) rectangle (794.97, 59.50);
\definecolor{fillColor}{RGB}{255,255,255}

\path[fill=fillColor] (664.15,  9.92) rectangle (749.51, 28.75);
\end{scope}
\begin{scope}
\path[clip] (  0.00,  0.00) rectangle (794.97, 59.50);
\definecolor{drawColor}{RGB}{0,0,0}

\path[draw=drawColor,line width= 1.4pt,dash pattern=on 1pt off 15pt ,line join=round] (672.69, 9.33) -- (740.97, 9.33);
\end{scope}
\begin{scope}
\path[clip] (  0.00,  0.00) rectangle (794.97, 59.50);
\definecolor{drawColor}{RGB}{0,0,0}

\node[text=drawColor,anchor=base west,inner sep=0pt, outer sep=0pt, scale=  1.98] at (182.51, 04.52) {0.05};
\end{scope}
\begin{scope}
\path[clip] (  0.00,  0.00) rectangle (794.97, 59.50);
\definecolor{drawColor}{RGB}{0,0,0}

\node[text=drawColor,anchor=base west,inner sep=0pt, outer sep=0pt, scale=  1.98] at (326.74, 04.52) {0.15};
\end{scope}
\begin{scope}
\path[clip] (  0.00,  0.00) rectangle (794.97, 59.50);
\definecolor{drawColor}{RGB}{0,0,0}

\node[text=drawColor,anchor=base west,inner sep=0pt, outer sep=0pt, scale=  1.98] at (470.96, 04.52) {0.25};
\end{scope}
\begin{scope}
\path[clip] (  0.00,  0.00) rectangle (794.97, 59.50);
\definecolor{drawColor}{RGB}{0,0,0}

\node[text=drawColor,anchor=base west,inner sep=0pt, outer sep=0pt, scale=  1.98] at (615.18, 04.52) {0.35};
\end{scope}
\begin{scope}
\path[clip] (  0.00,  0.00) rectangle (794.97, 59.50);
\definecolor{drawColor}{RGB}{0,0,0}

\node[text=drawColor,anchor=base west,inner sep=0pt, outer sep=0pt, scale=  1.98] at (759.41, 04.52) {0.45};
\end{scope}
\end{tikzpicture}

%% file: images/images_resampling_cspn_idm.tex
\begin{figure}[!th]
	\centering%
	\resizebox{0.75\textwidth}{!}{%
	  \input{images/resampling/legends_resampling_all}
	}\vspace{-2mm}\qquad%
	\renewcommand{\thesubfigure}{(a)}
	\subfigure[\scshape Partial abstention $f_{SPE}$]{
		\hspace{-3mm}
		\subfigure[\scshape Flags ($\epsilon=0.03$)]{
			\includegraphics[width=0.244\linewidth]
				{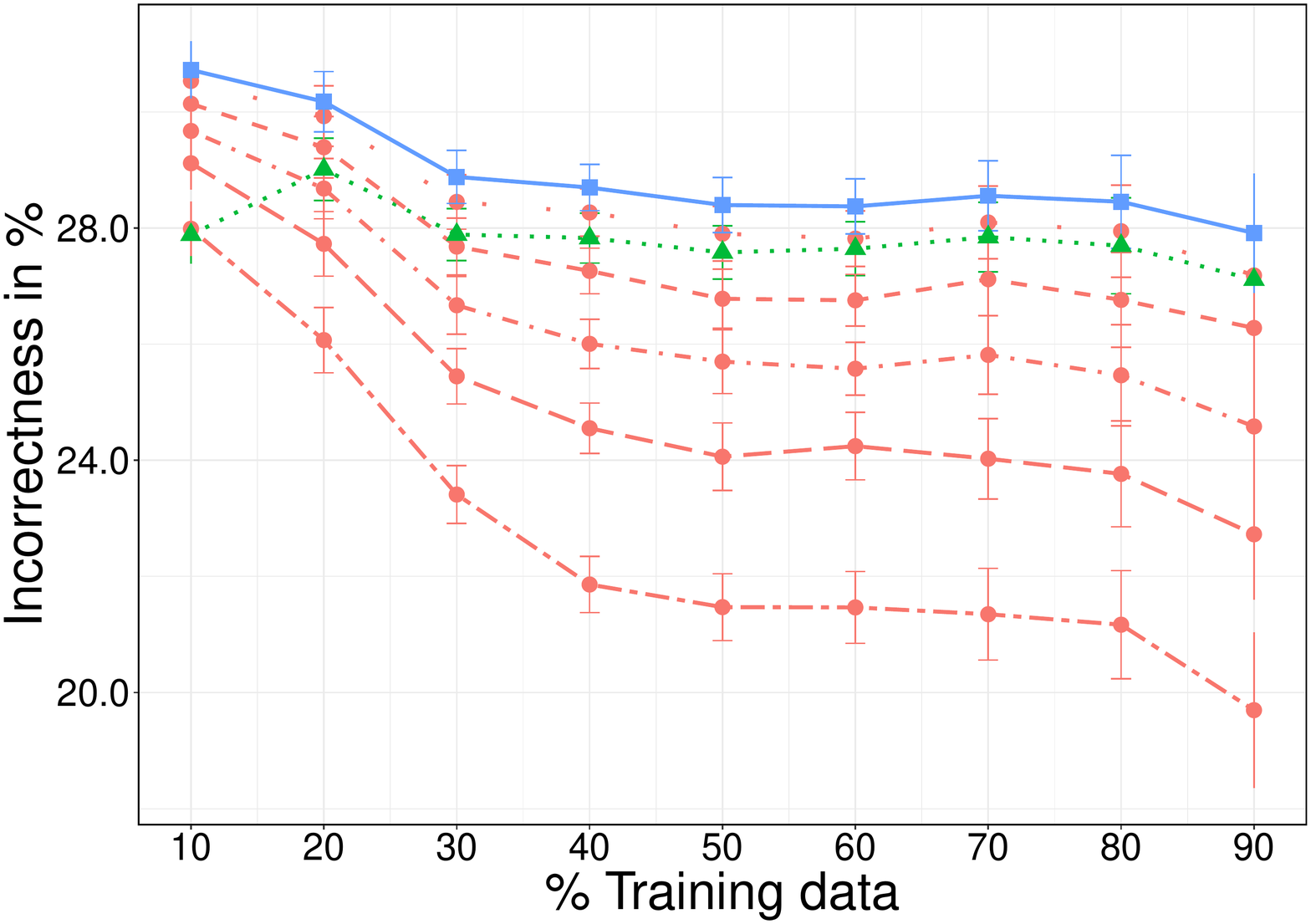}} 
		\subfigure[\scshape Scene ($\epsilon=0.05$)]{
			\includegraphics[width=0.244\linewidth]
				{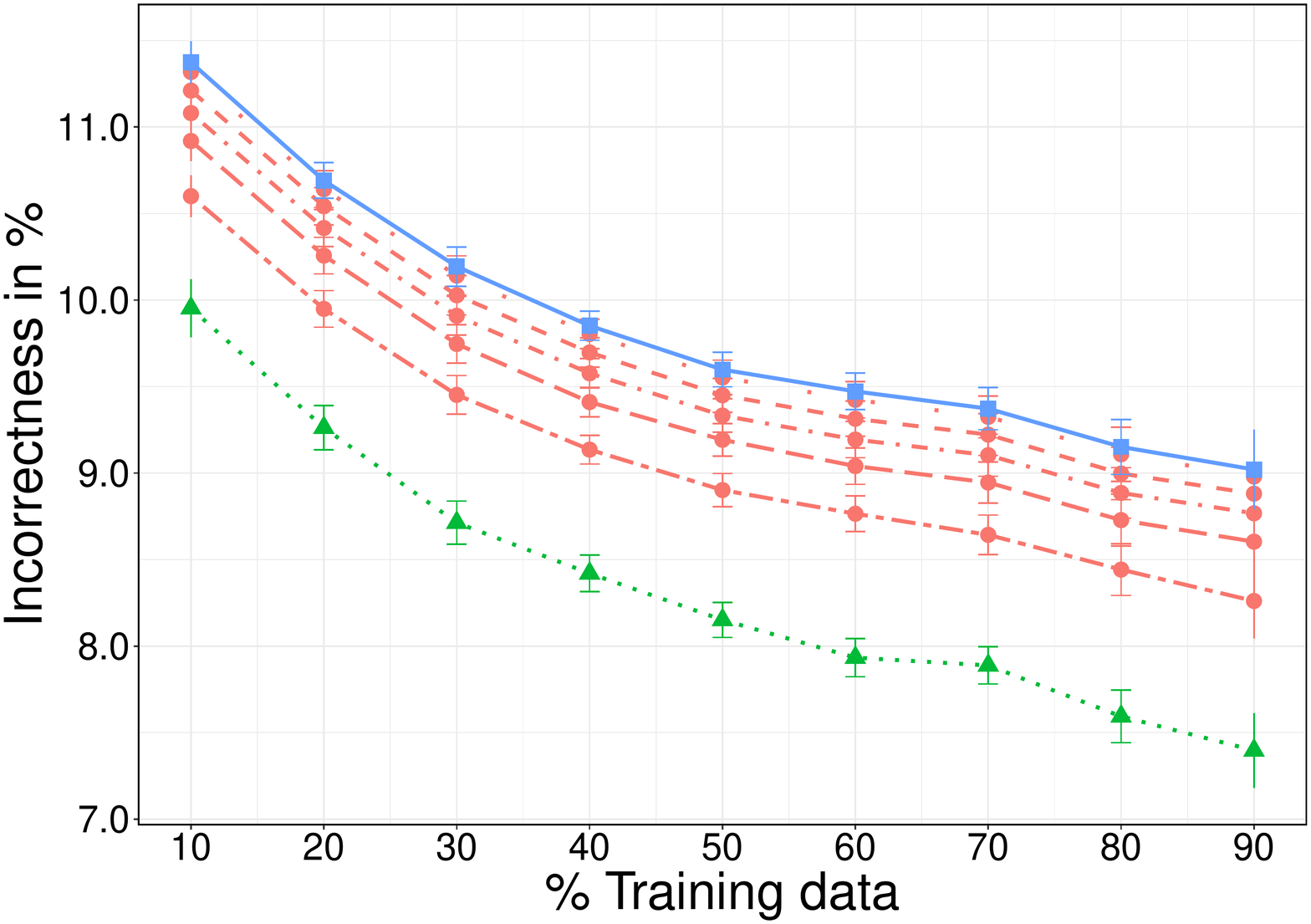}}
		\subfigure[\scshape Medical ($\epsilon=\num{1e-3}$)]{
			\includegraphics[width=0.244\linewidth]
				{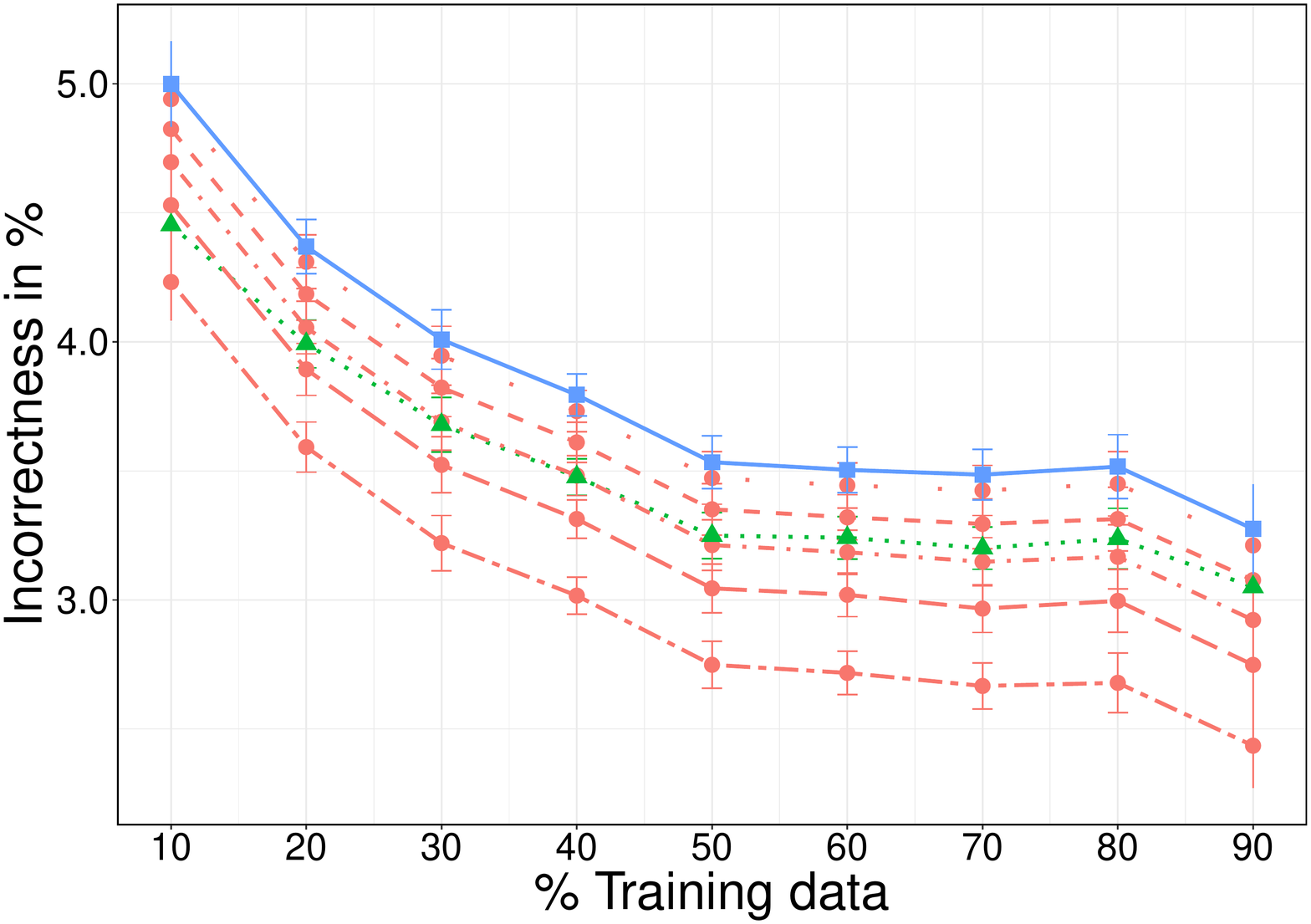}}
		\subfigure[\scshape CAL500 ($\epsilon=0.10$)]{
			\includegraphics[width=0.244\linewidth]
				{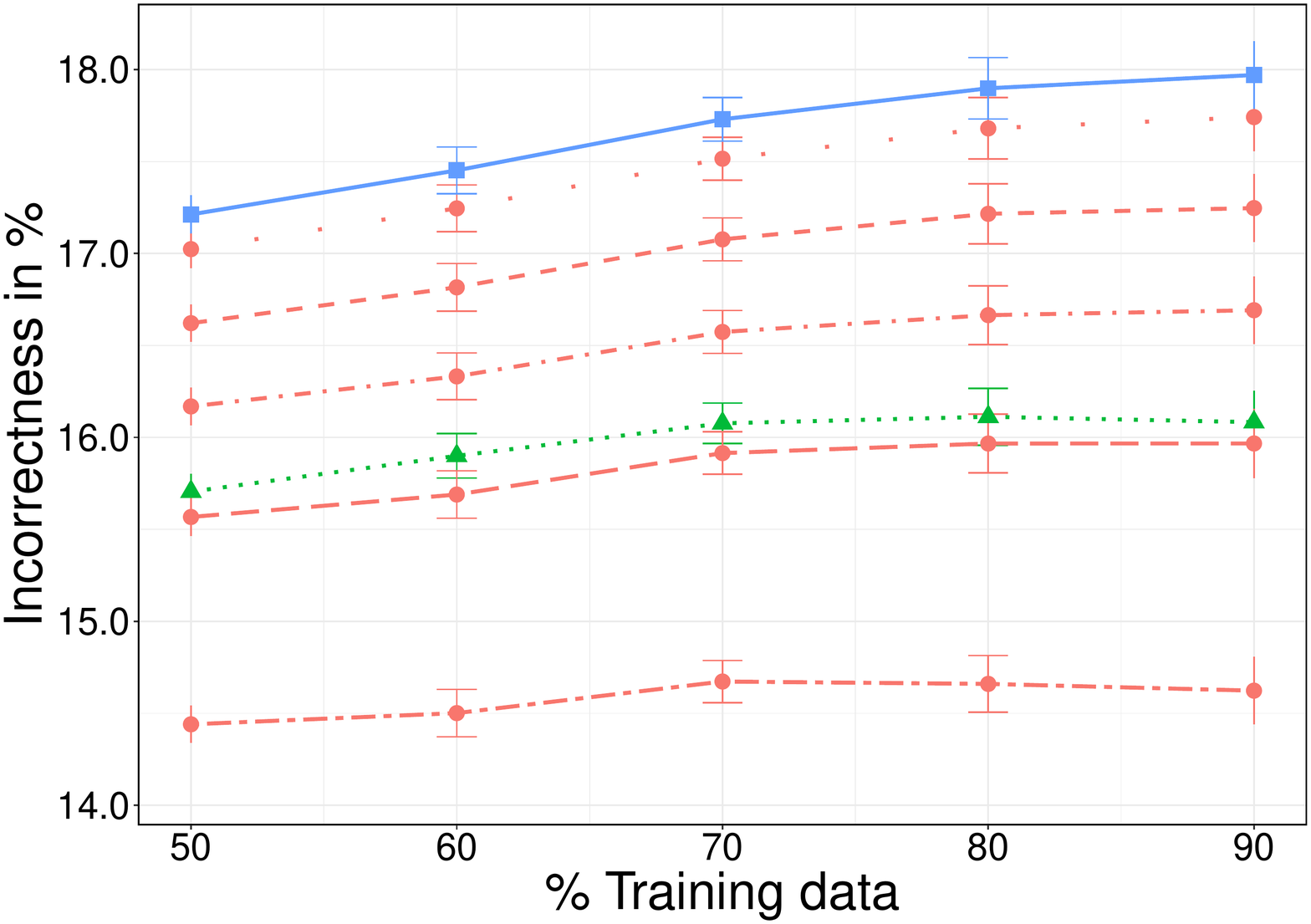}} 
	}
	\resizebox{0.75\textwidth}{!}{%
		\input{images/resampling/legends_par_params}
	}\vspace{-2mm}\qquad%
    \renewcommand{\thesubfigure}{(c)} 
	\subfigure[\scshape Partial abstention $f_{PAR}$]{
		\hspace{-3mm}
	  	\subfigure[\scshape Flags ($\epsilon=0.07$)]{
			\includegraphics[width=0.244\linewidth]
				{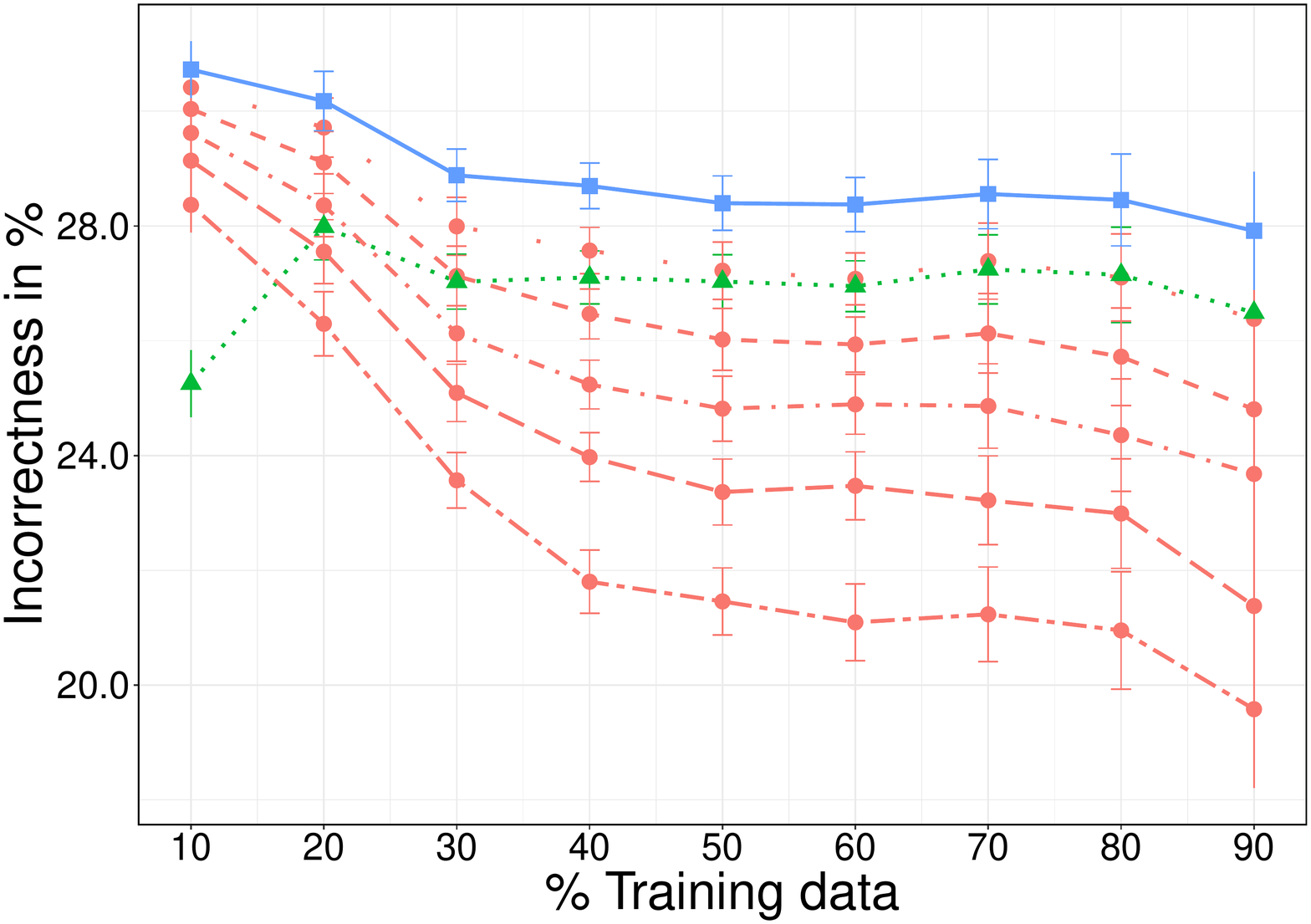}}
		\subfigure[\scshape Scene ($\epsilon=0.10$)]{
			\includegraphics[width=0.244\linewidth]
				{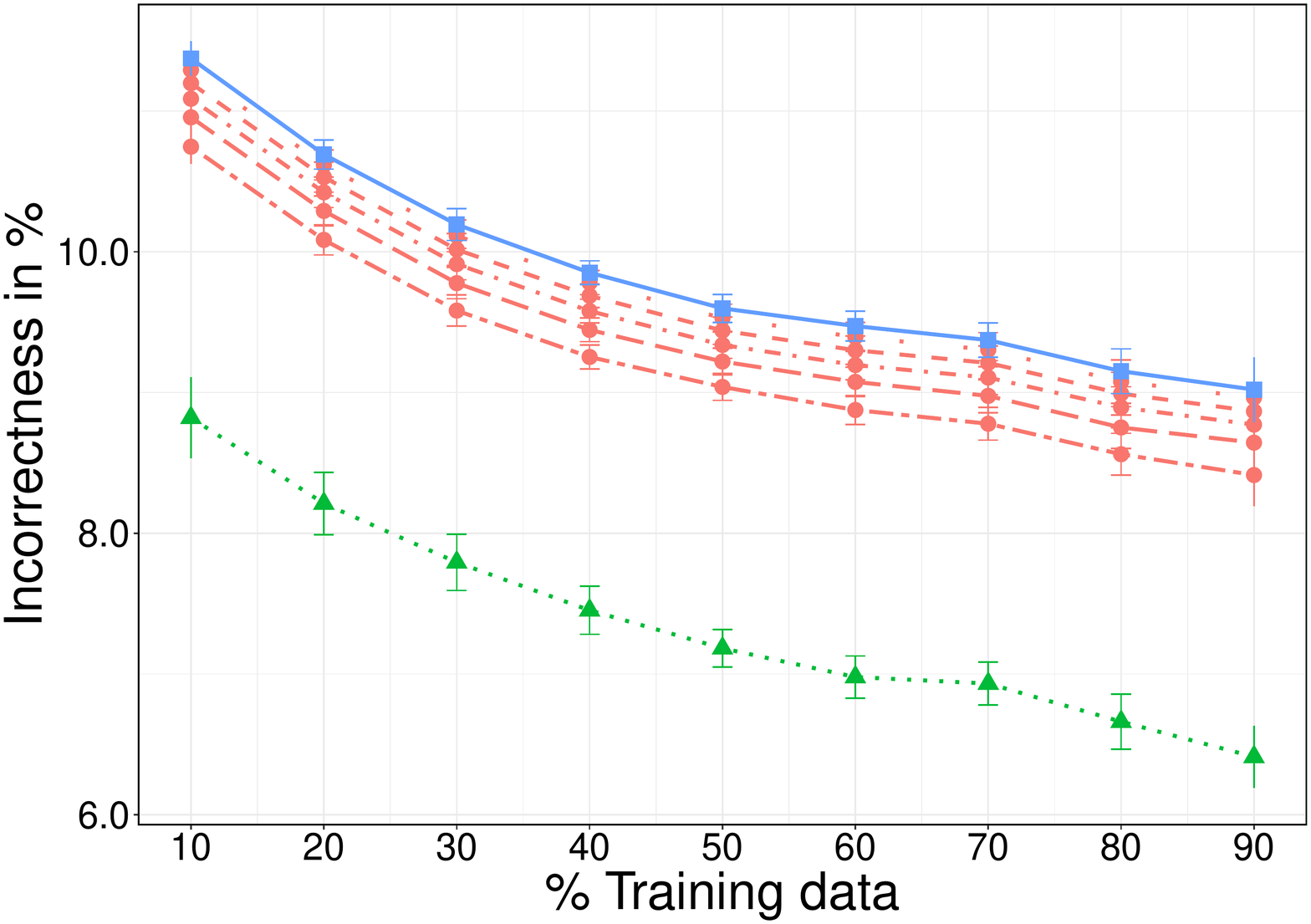}}
		\subfigure[\scshape Medical ($\epsilon=\num{3e-3}$)]{
			\includegraphics[width=0.244\linewidth]
				{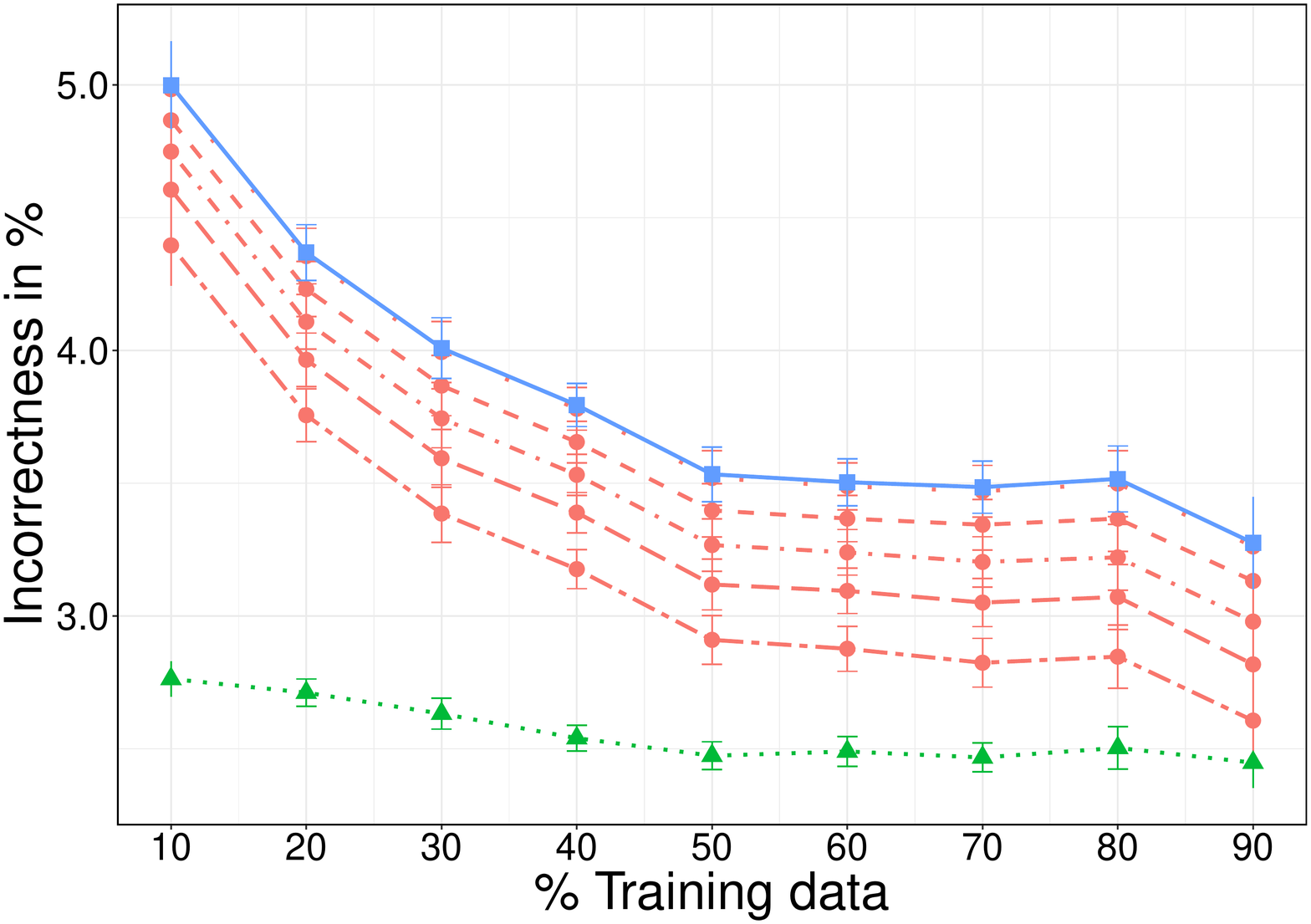}}
		\subfigure[\scshape CAL500 ($\epsilon=0.20$)]{
			\includegraphics[width=0.244\linewidth]
				{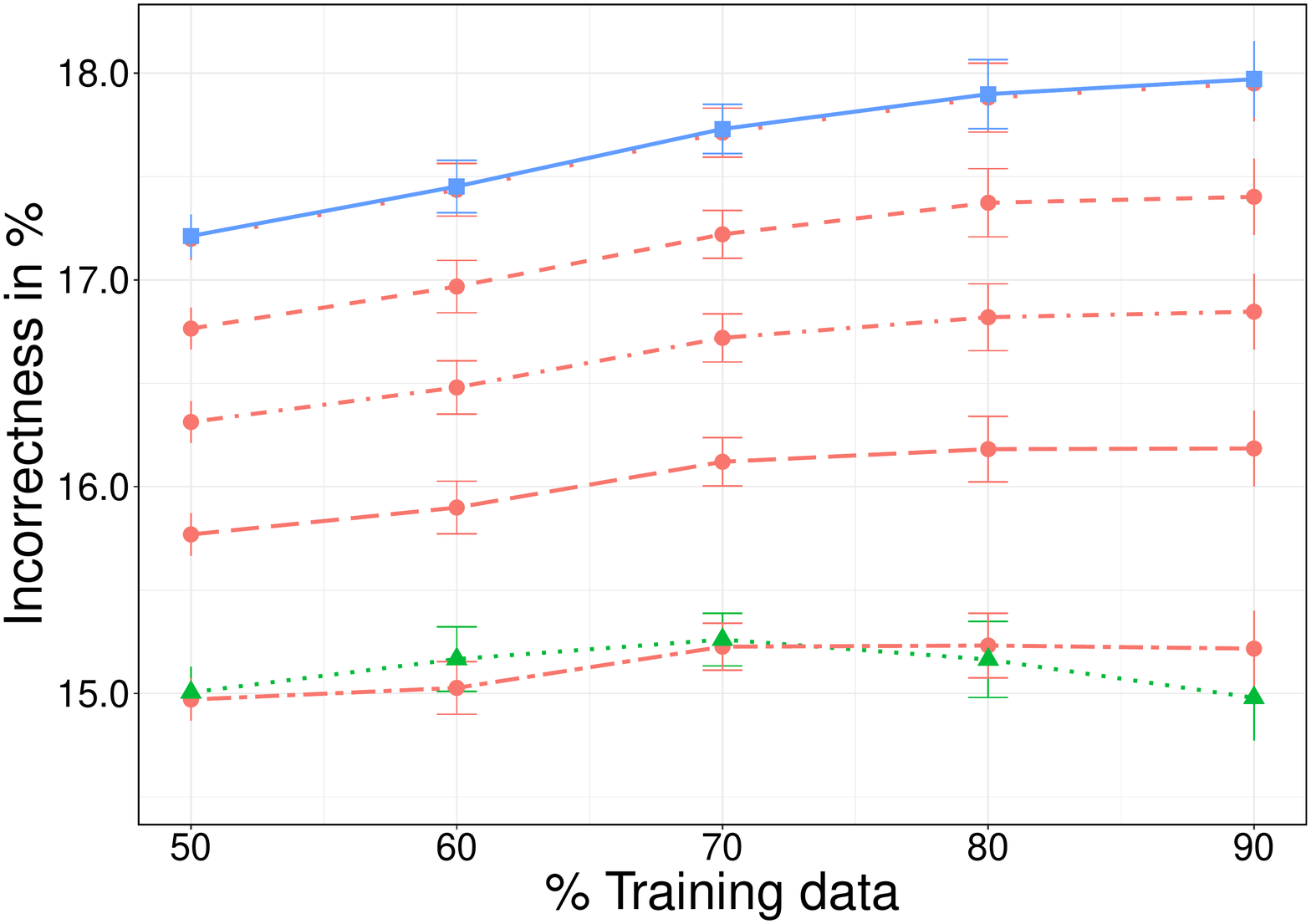}}
	}
	\resizebox{0.75\textwidth}{!}{%
		\input{images/resampling/legends_rej_params}
	}\vspace{-2mm}\qquad%
	\renewcommand{\thesubfigure}{(b)}
	\subfigure[\scshape Rejection threshold]{ 
		\hspace{-3mm}
		\subfigure[\scshape Flags ($\epsilon=0.15$)]{
			\includegraphics[width=0.244\linewidth]
				{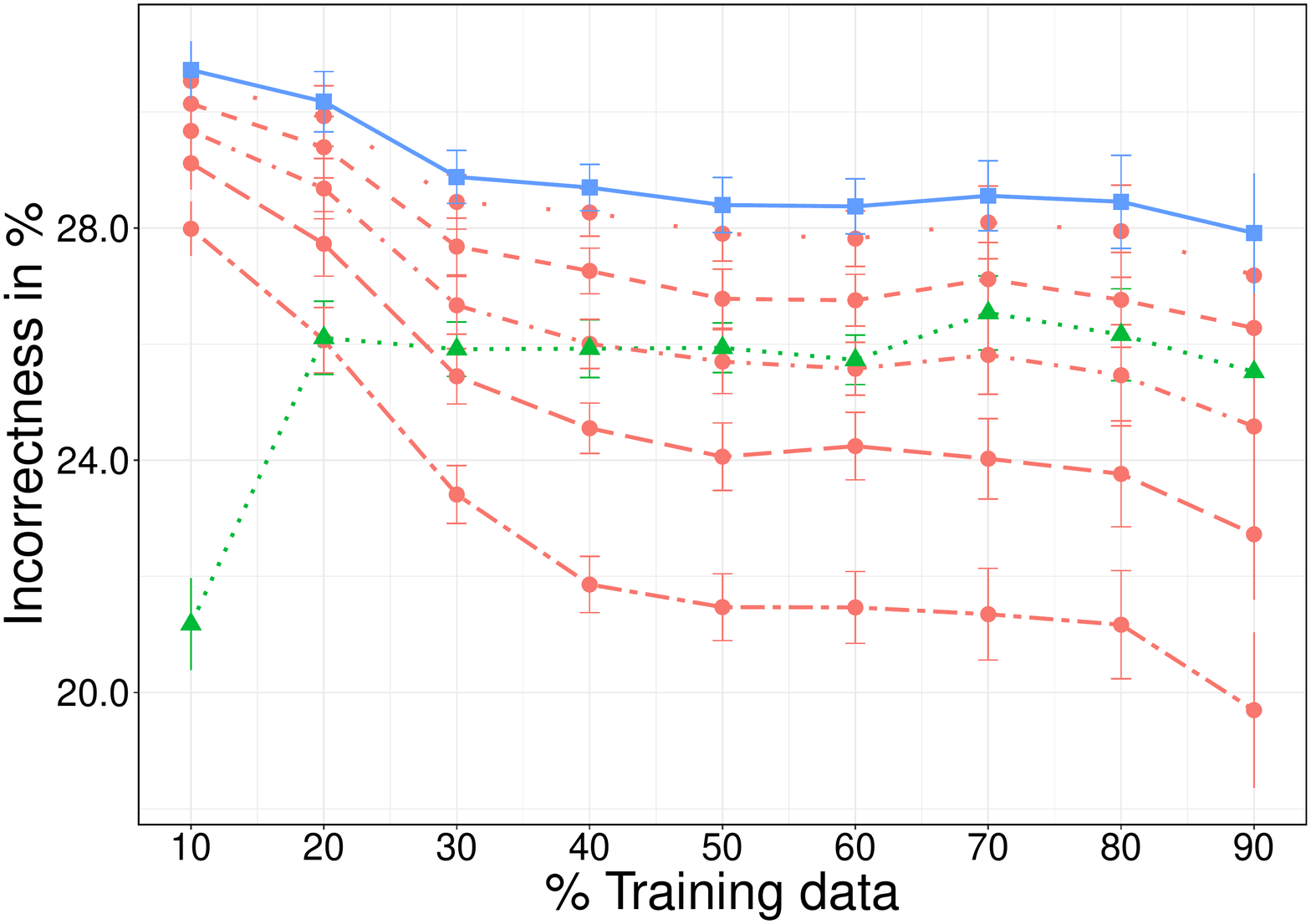}}
		\subfigure[\scshape Scene ($\epsilon=0.15$)]{
			\includegraphics[width=0.244\linewidth]
				{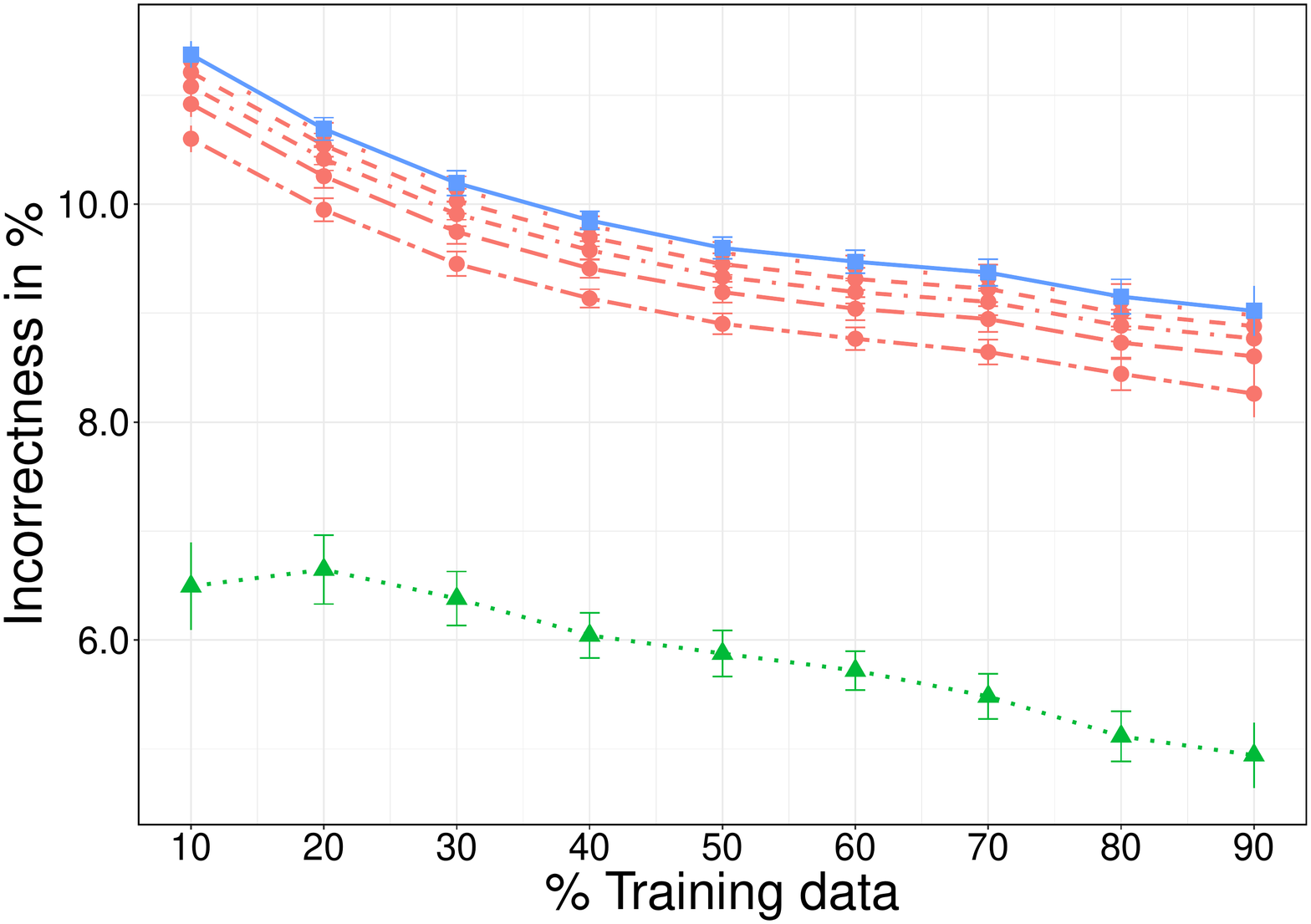}}
		\subfigure[\scshape Medical ($\epsilon=\num{5e-3}$)]{
			\includegraphics[width=0.244\linewidth]
				{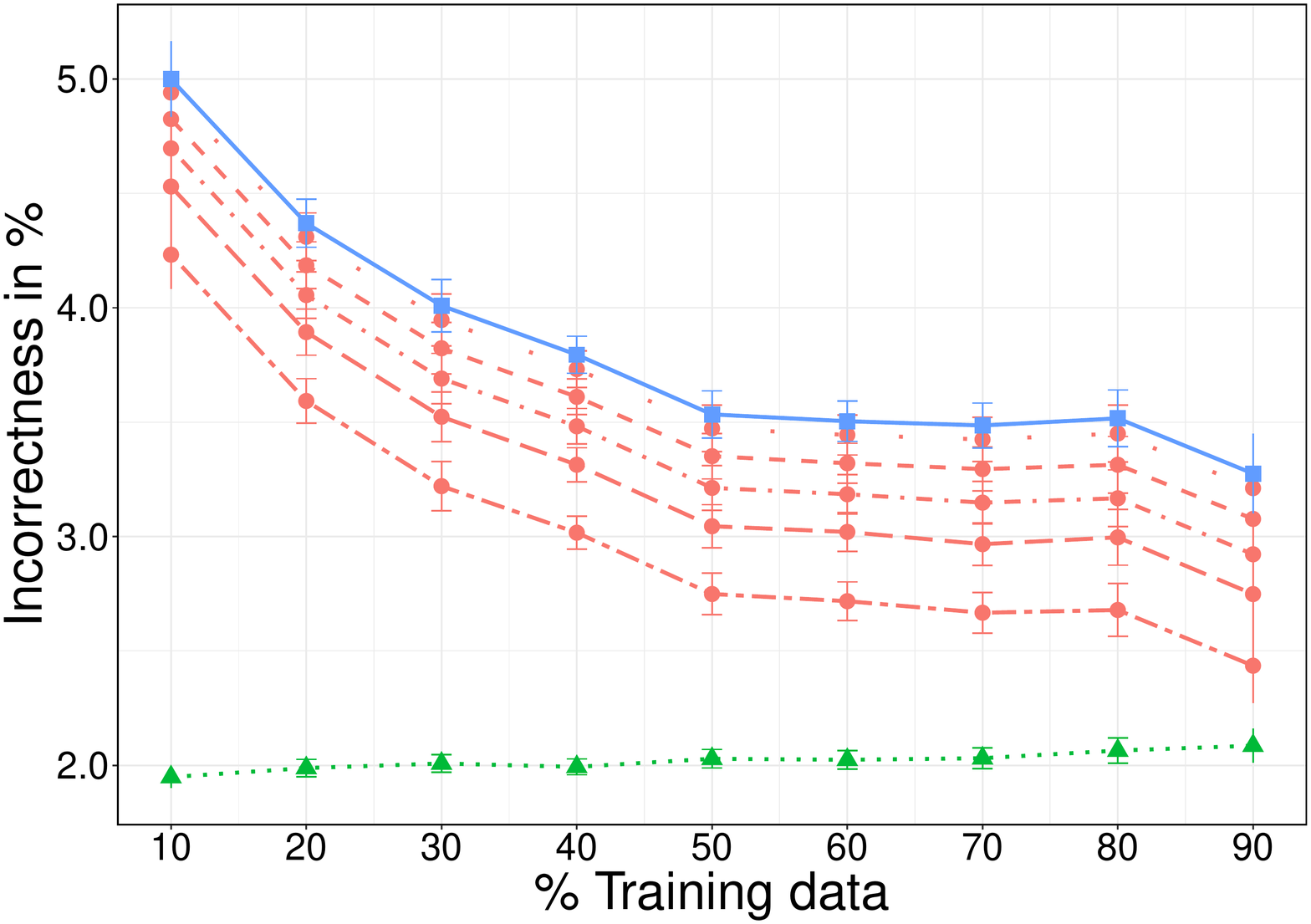}}
		\subfigure[\scshape CAL500 ($\epsilon=0.30$)]{
			\includegraphics[width=0.244\linewidth]
				{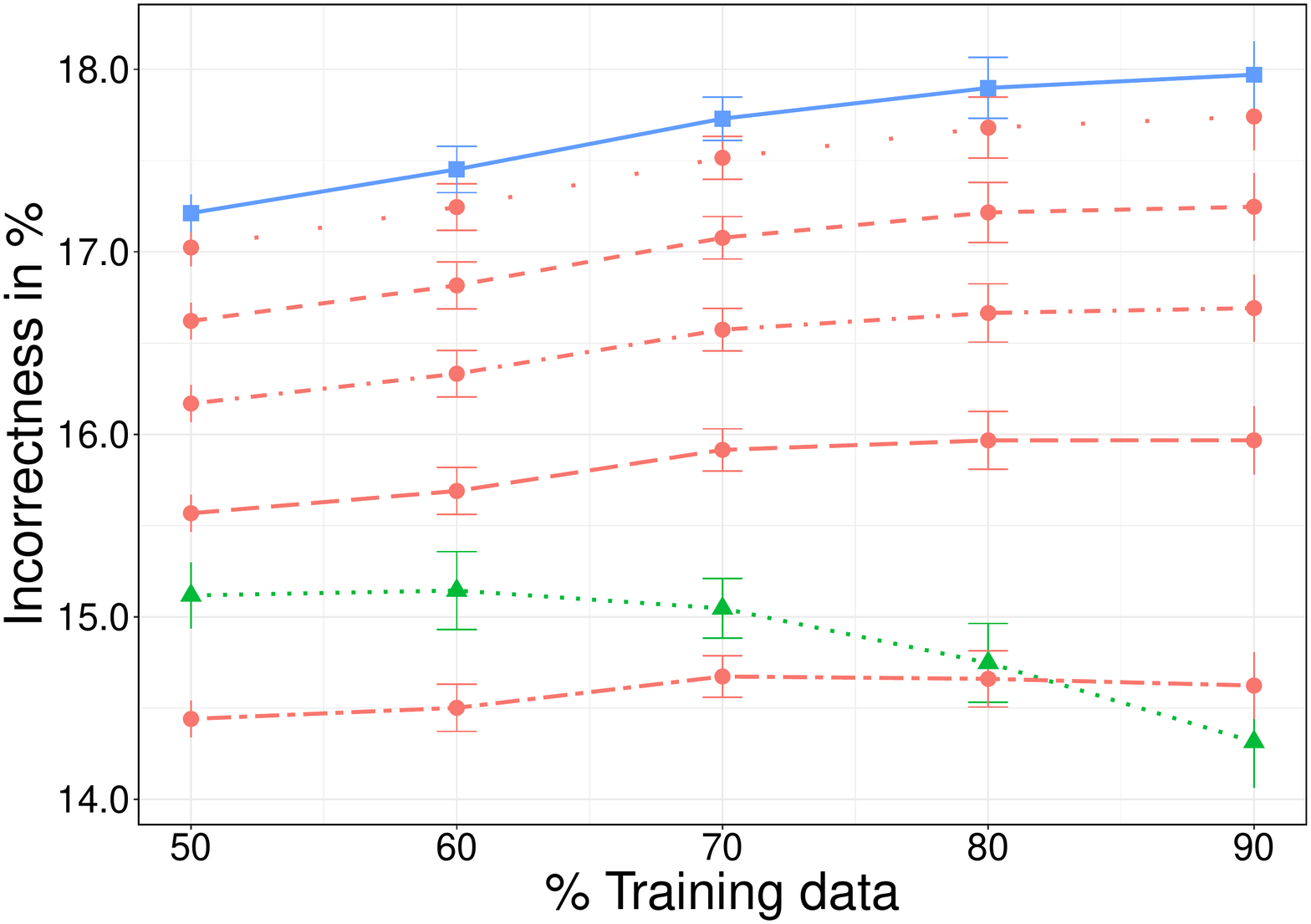}}
	}\vspace{-2mm}
	\caption{{\bf Downsampling - Credal sum product network - IDM.} Incorrectness (y-axis) evolution for the precise, partial abstention with $f_{SEP}$ (\nth{1} row) and $f_{PAR}$ (\nth{2} row) penalty functions, rejection (\nth{3} row), and skeptical approach, each one with different hyper-parameter levels (except to the precise approach), and with respect to different percentages of training data sets (x-axis).}
	\label{fig:cspnidmsamplingresults}
\end{figure}

%% file: images/images_resampling_cspn_econt.tex
\begin{figure}[!th]
	\centering%
	\resizebox{0.8\textwidth}{!}{%
	  \input{images/resampling/legends_resampling_all}
	}\vspace{-2mm}\qquad%
	\renewcommand{\thesubfigure}{(a)}
	\subfigure[\scshape Partial abstention $f_{SEP}$]{
		\hspace{-3mm}
		\subfigure[{\scshape Flags ($\epsilon=0.01$)}]{
			\includegraphics[width=0.244\linewidth]
				{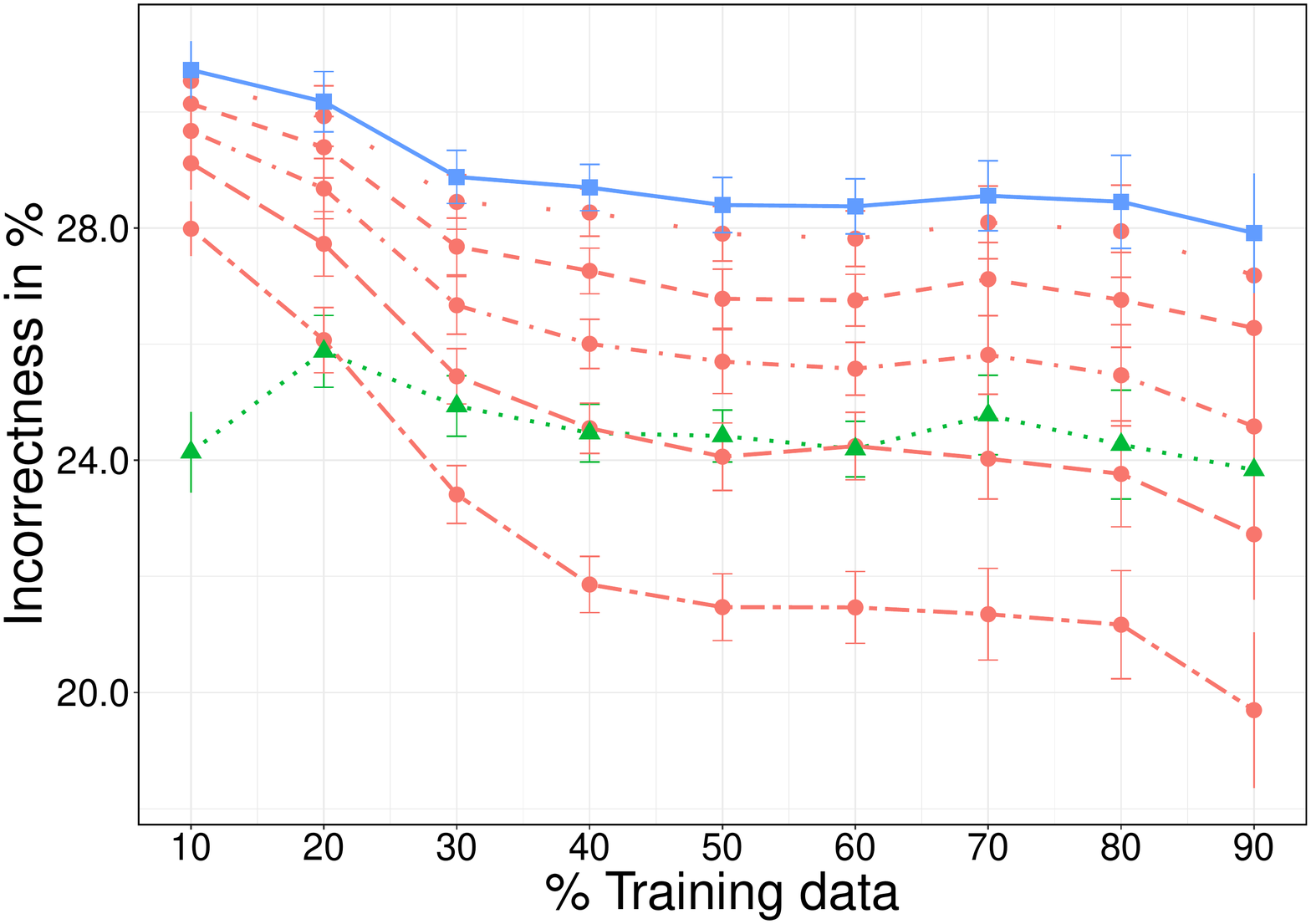}} 
		\subfigure[\scshape Scene ($\epsilon=\num{5e-4}$)]{
			\includegraphics[width=0.244\linewidth]
				{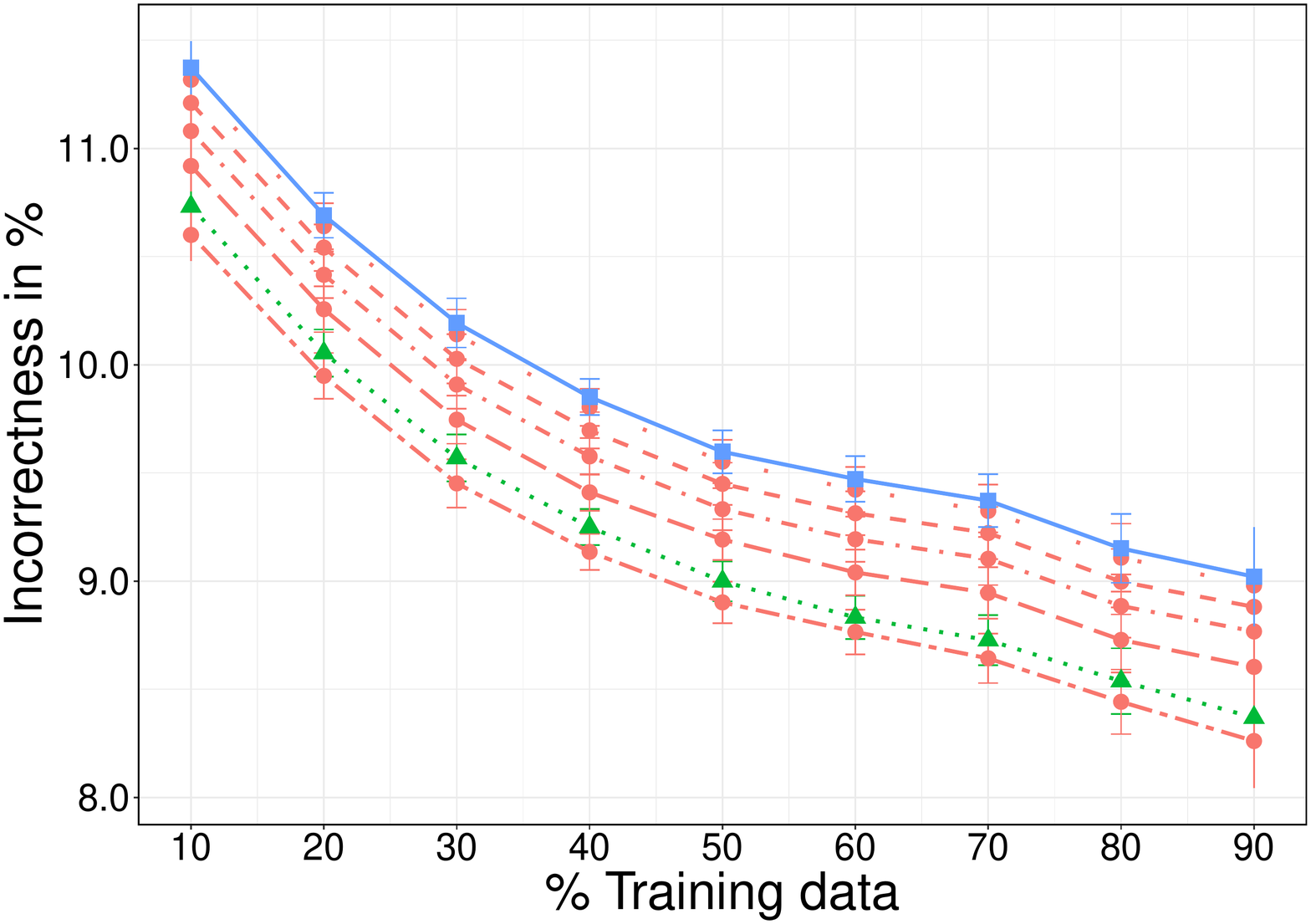}}
		\subfigure[\scshape Medical ($\epsilon=\num{5e-4}$)]{
			\includegraphics[width=0.244\linewidth]
				{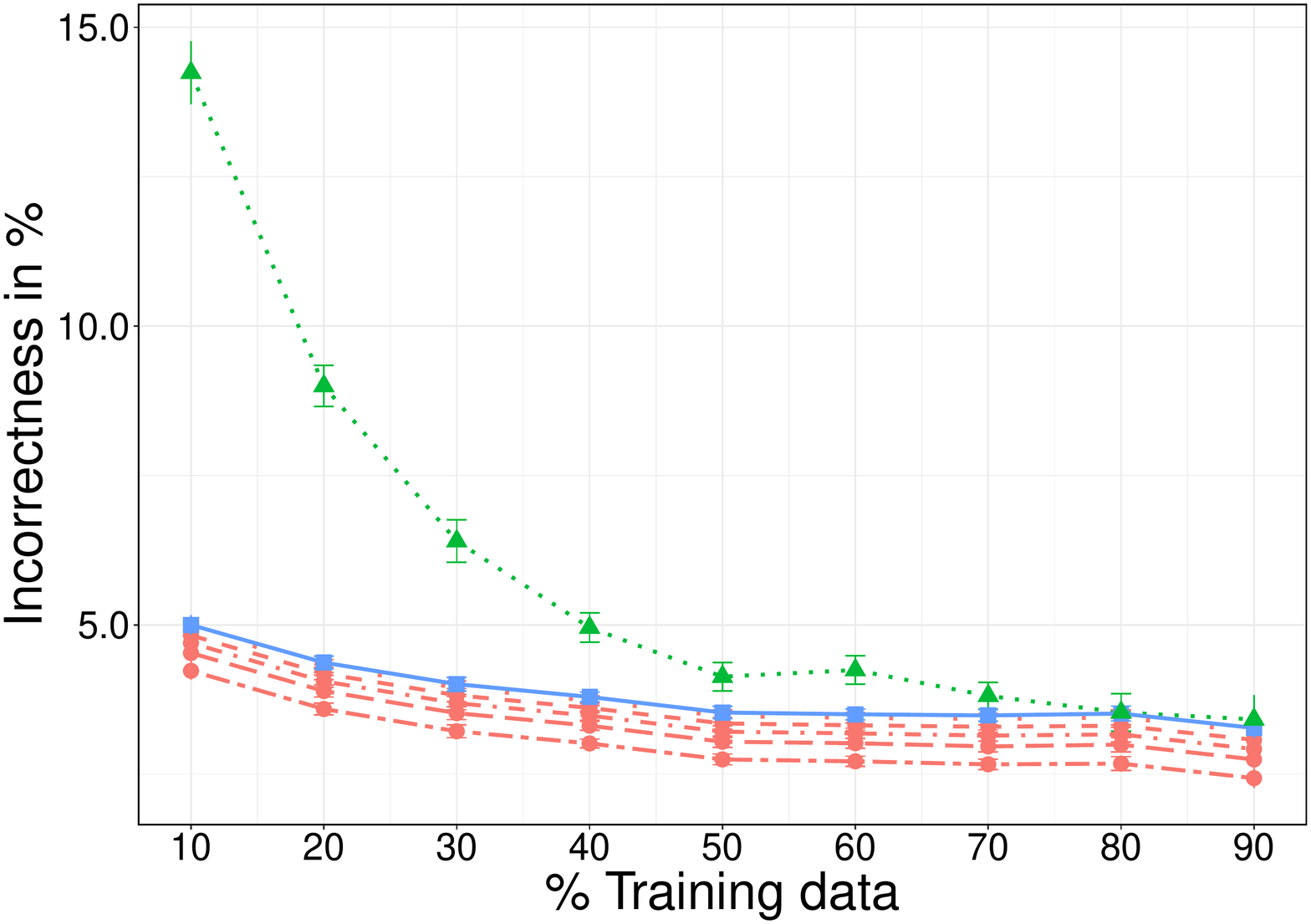}}
		\subfigure[\scshape CAL500 ($\epsilon=\num{1e-3}$)]{
			\includegraphics[width=0.244\linewidth]
				{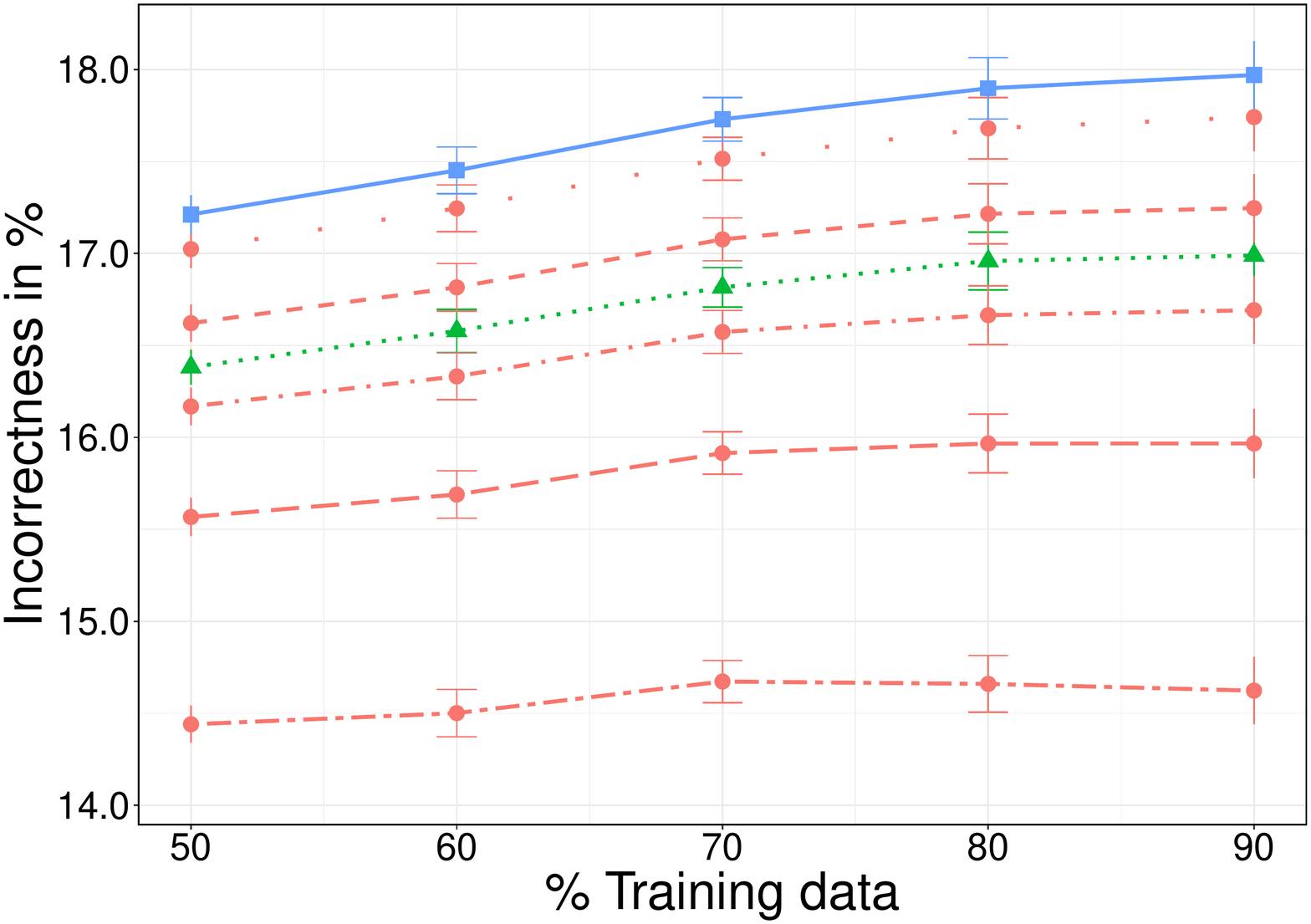}} 
	}
	\resizebox{0.75\textwidth}{!}{%
		\input{images/resampling/legends_par_params}
	}\vspace{-2mm}\qquad%
    \renewcommand{\thesubfigure}{(c)}
	\subfigure[\scshape Partial abstention $f_{PAR}$]{
		\hspace{-3mm}
	  	\subfigure[\scshape Flags ($\epsilon=0.03$)]{
			\includegraphics[width=0.244\linewidth]
				{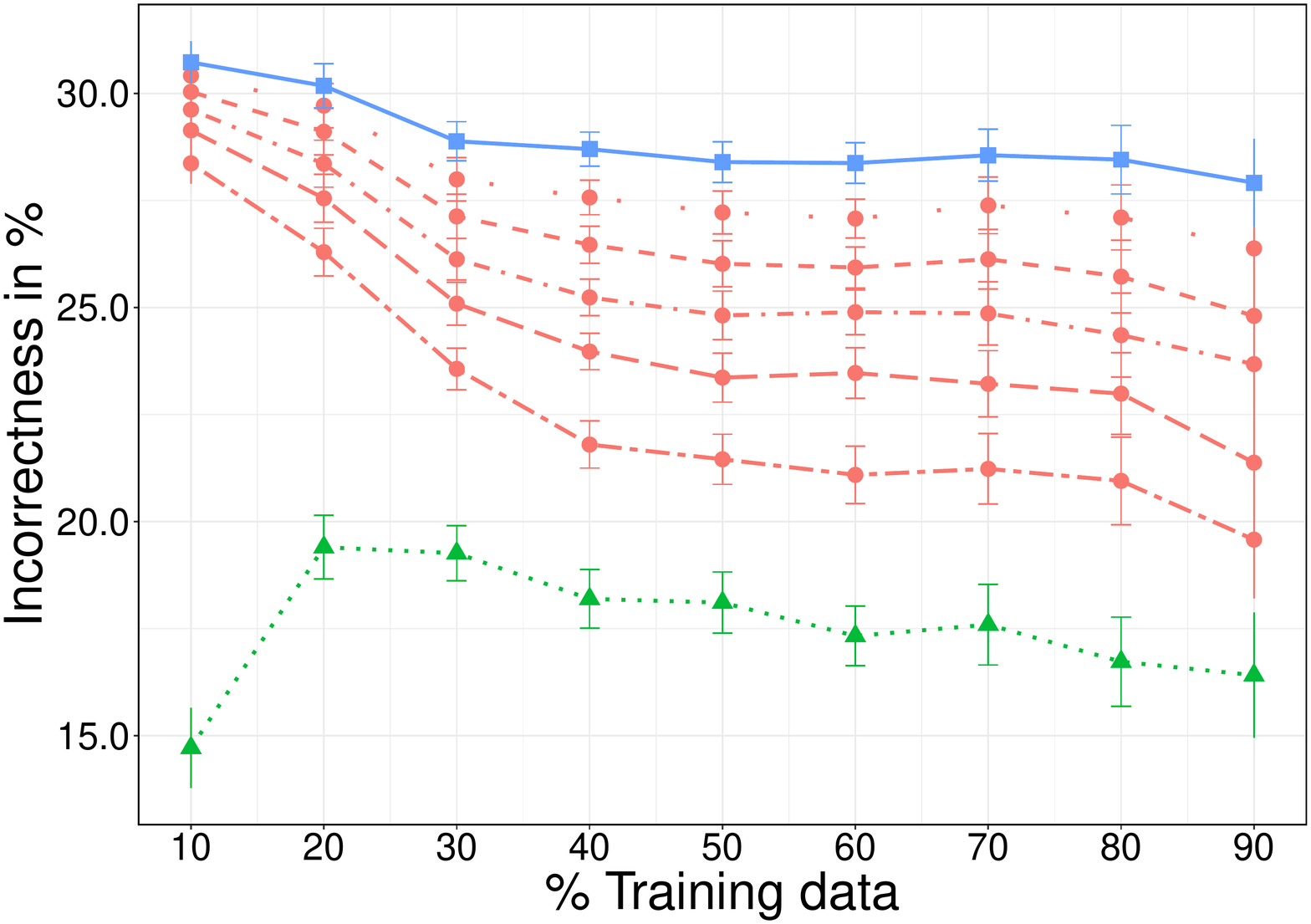}}
		\subfigure[\scshape Scene ($\epsilon=\num{1e-3}$)]{
			\includegraphics[width=0.244\linewidth]
				{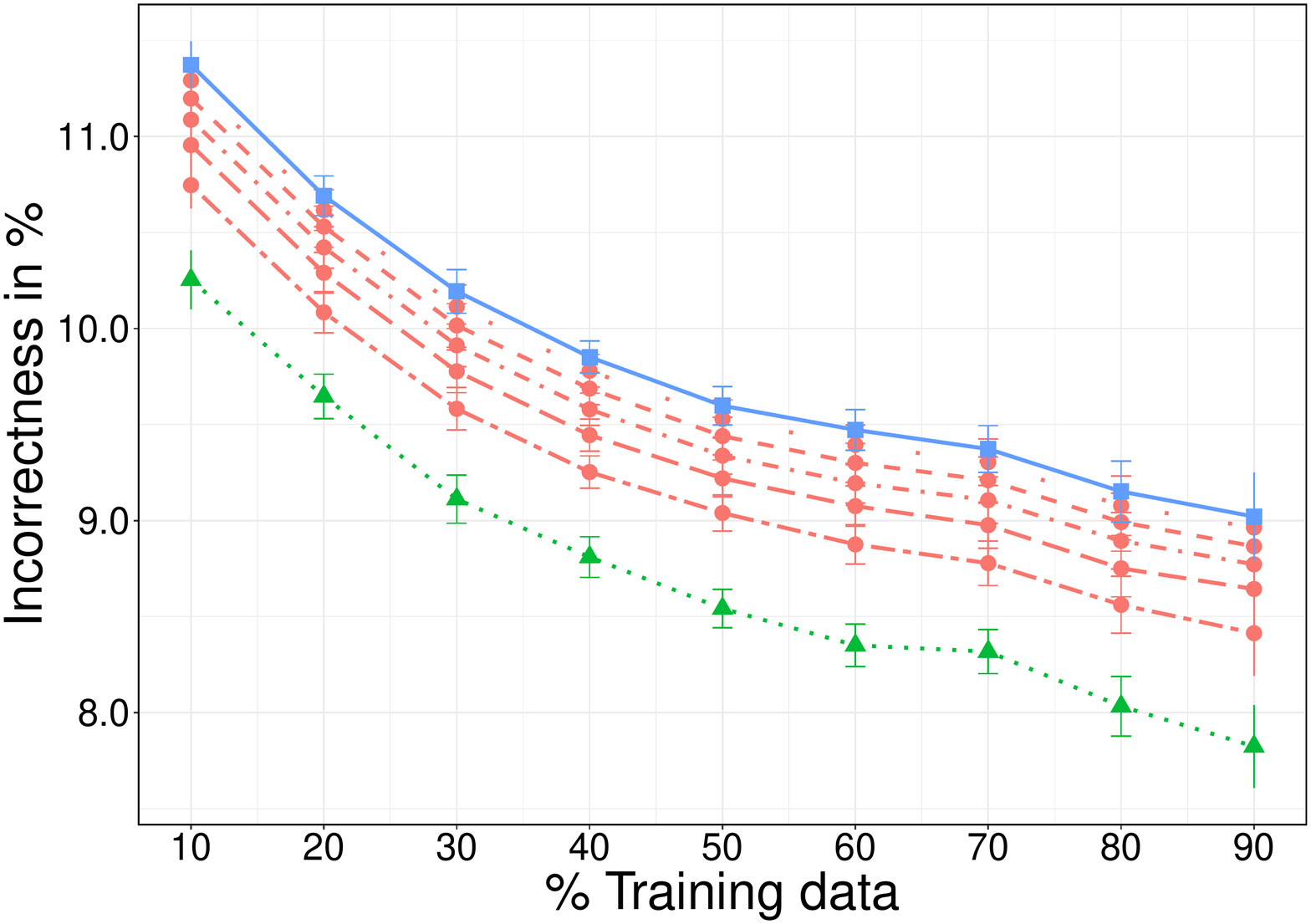}}
		\subfigure[\scshape Medical ($\epsilon=\num{1e-3}$)]{
			\includegraphics[width=0.244\linewidth]
				{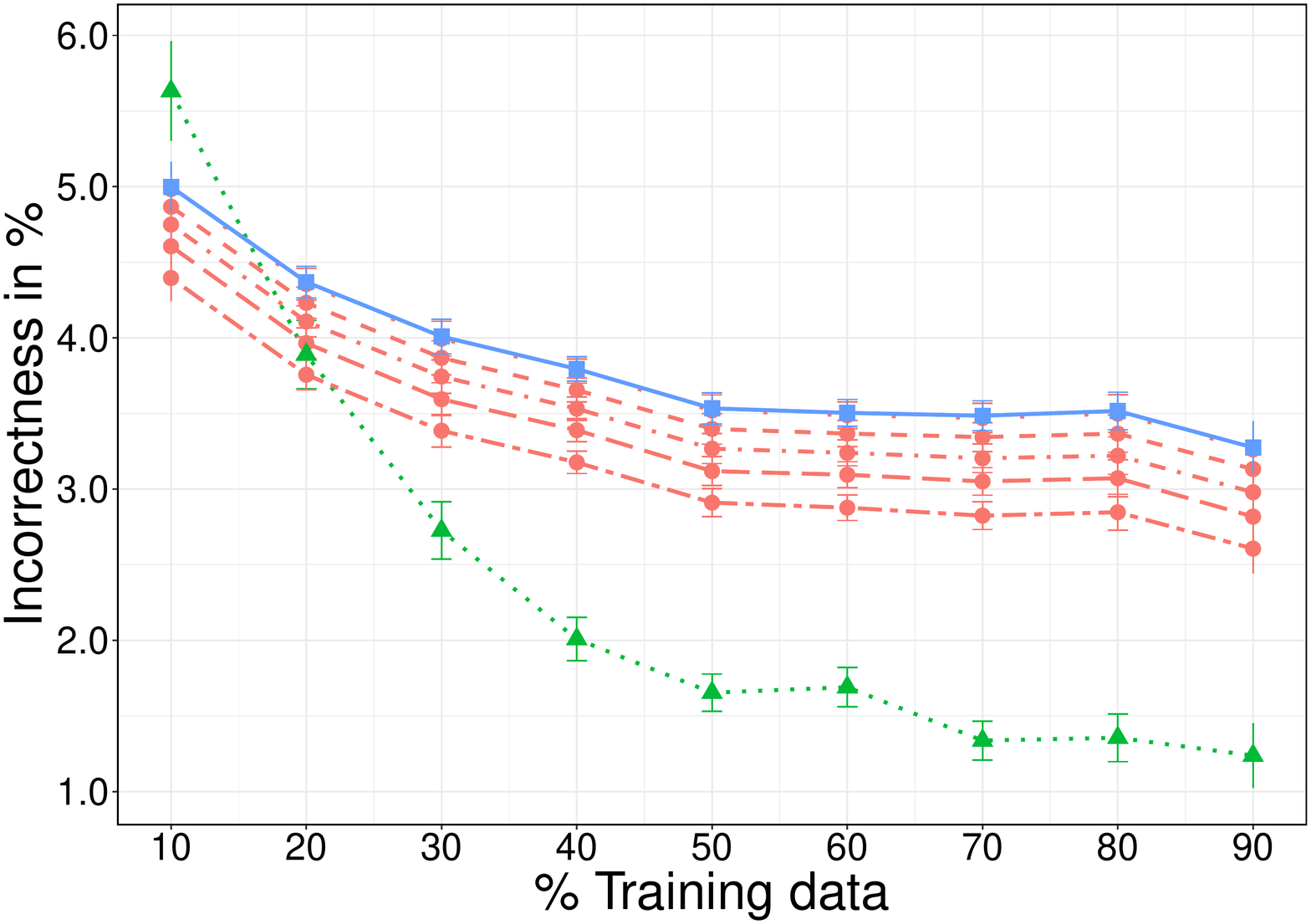}}
		\subfigure[\scshape CAL500 ($\epsilon=\num{3e-3}$)]{
			\includegraphics[width=0.244\linewidth]
				{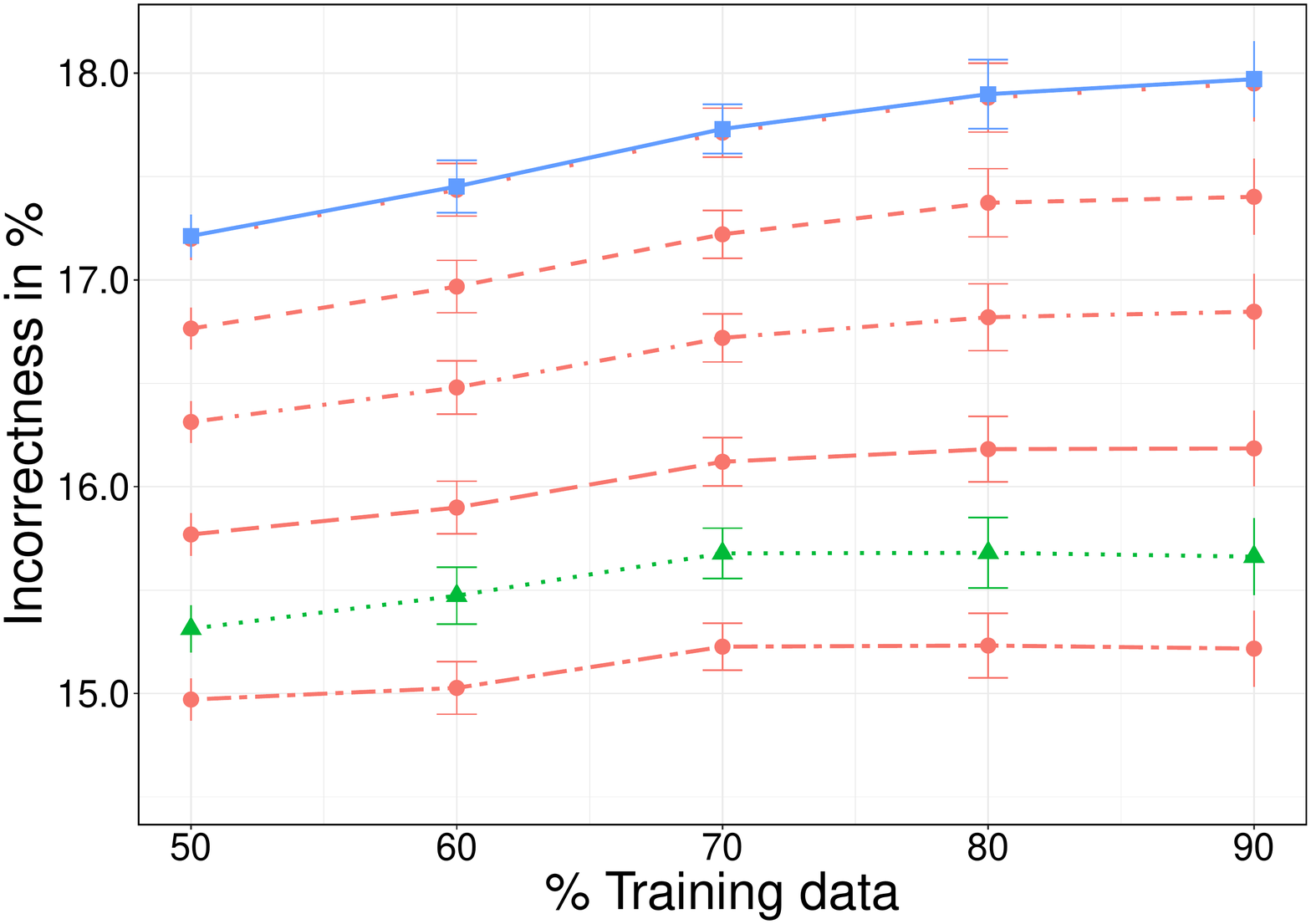}}
	}
	\resizebox{0.75\textwidth}{!}{%
		\input{images/resampling/legends_rej_params}
	}\vspace{-2mm}\qquad%
	\renewcommand{\thesubfigure}{(b)}
	\subfigure[\scshape Rejection threshold]{ 
		\hspace{-3mm}
		\subfigure[\scshape Flags ($\epsilon=0.05$)]{
			\includegraphics[width=0.244\linewidth]
				{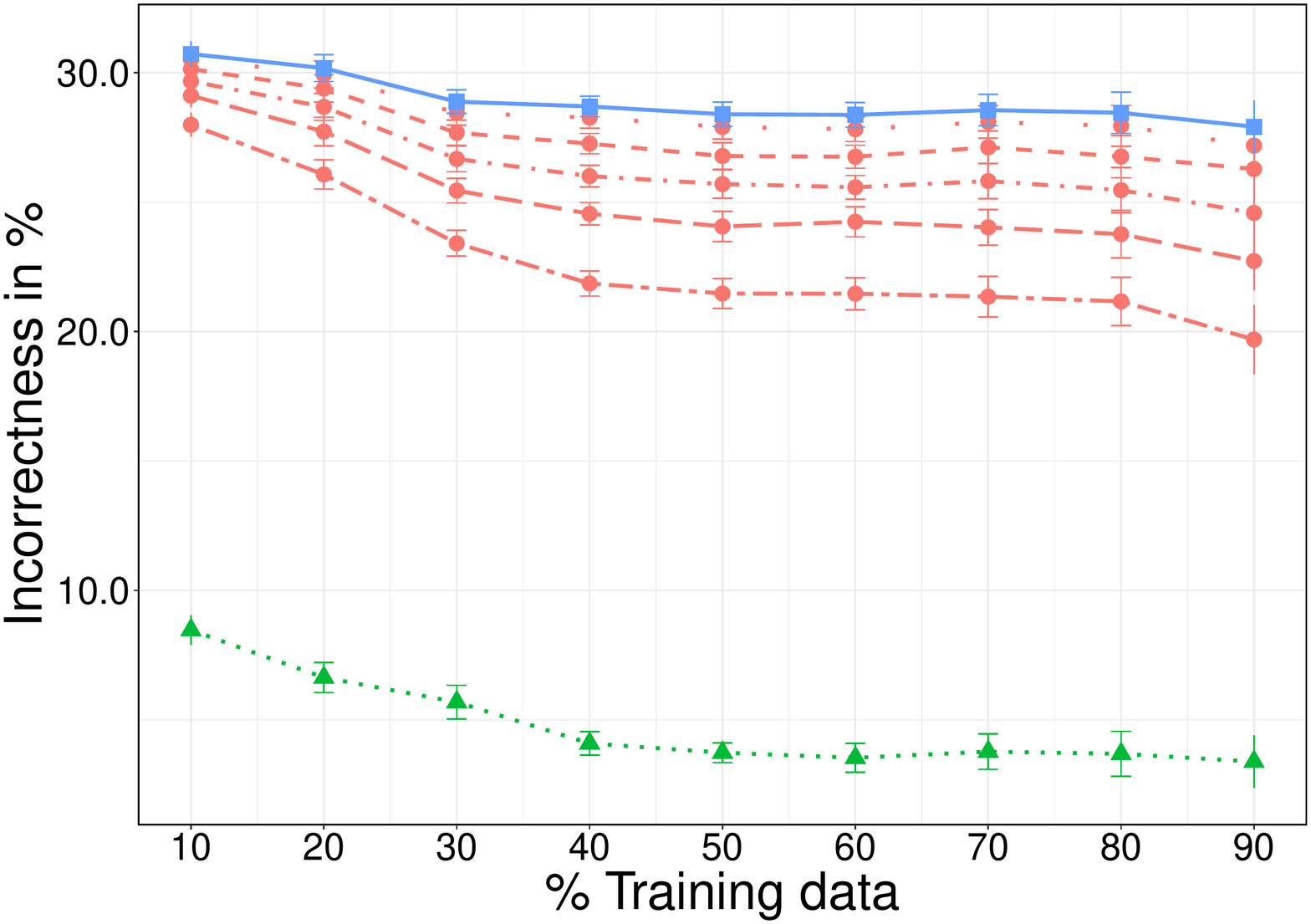}}
		\subfigure[\scshape Scene ($\epsilon=2e-3$)]{
			\includegraphics[width=0.244\linewidth]
				{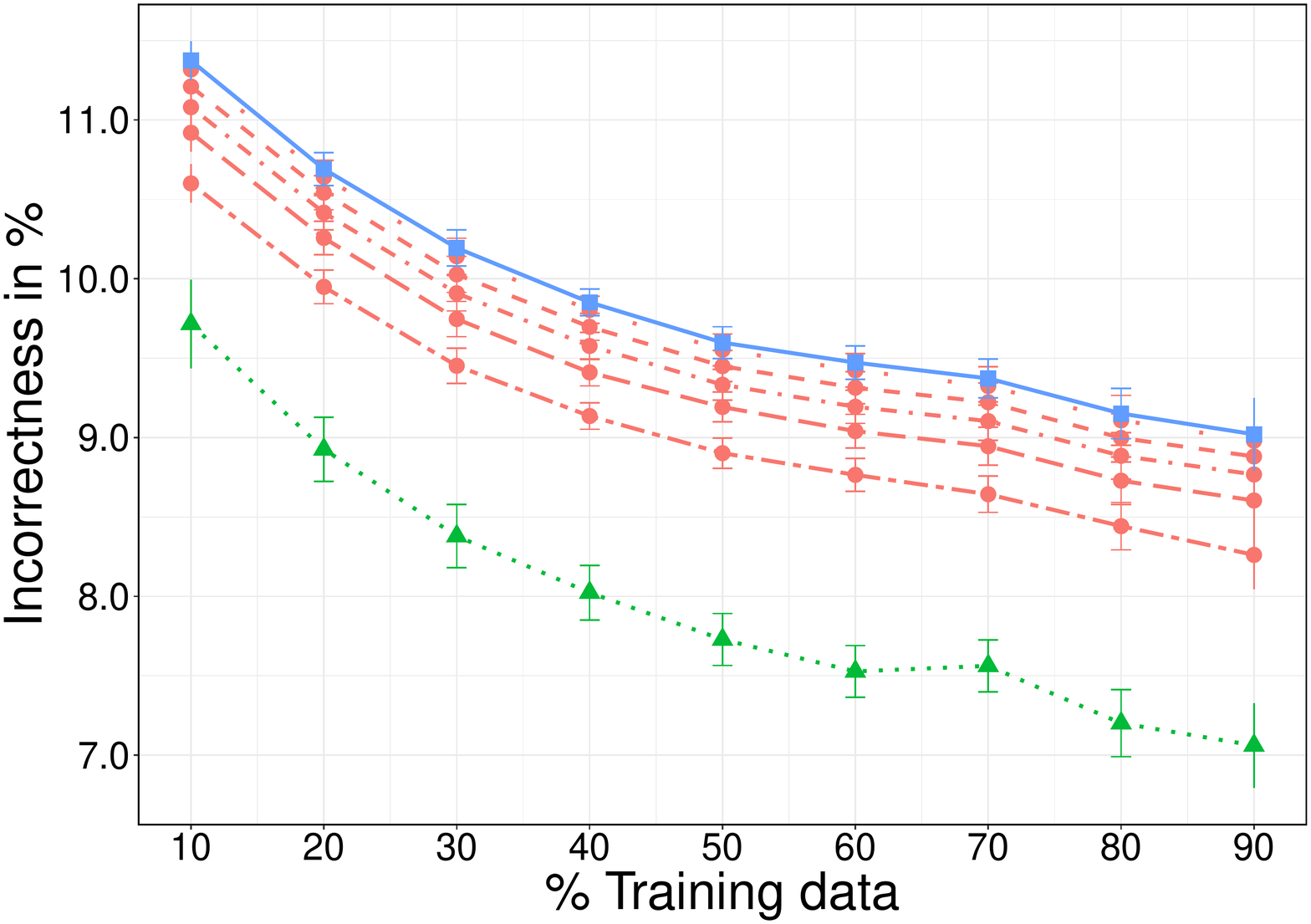}}
		\subfigure[\scshape Medical ($\epsilon=\num{3e-3}$)]{
			\includegraphics[width=0.244\linewidth]
				{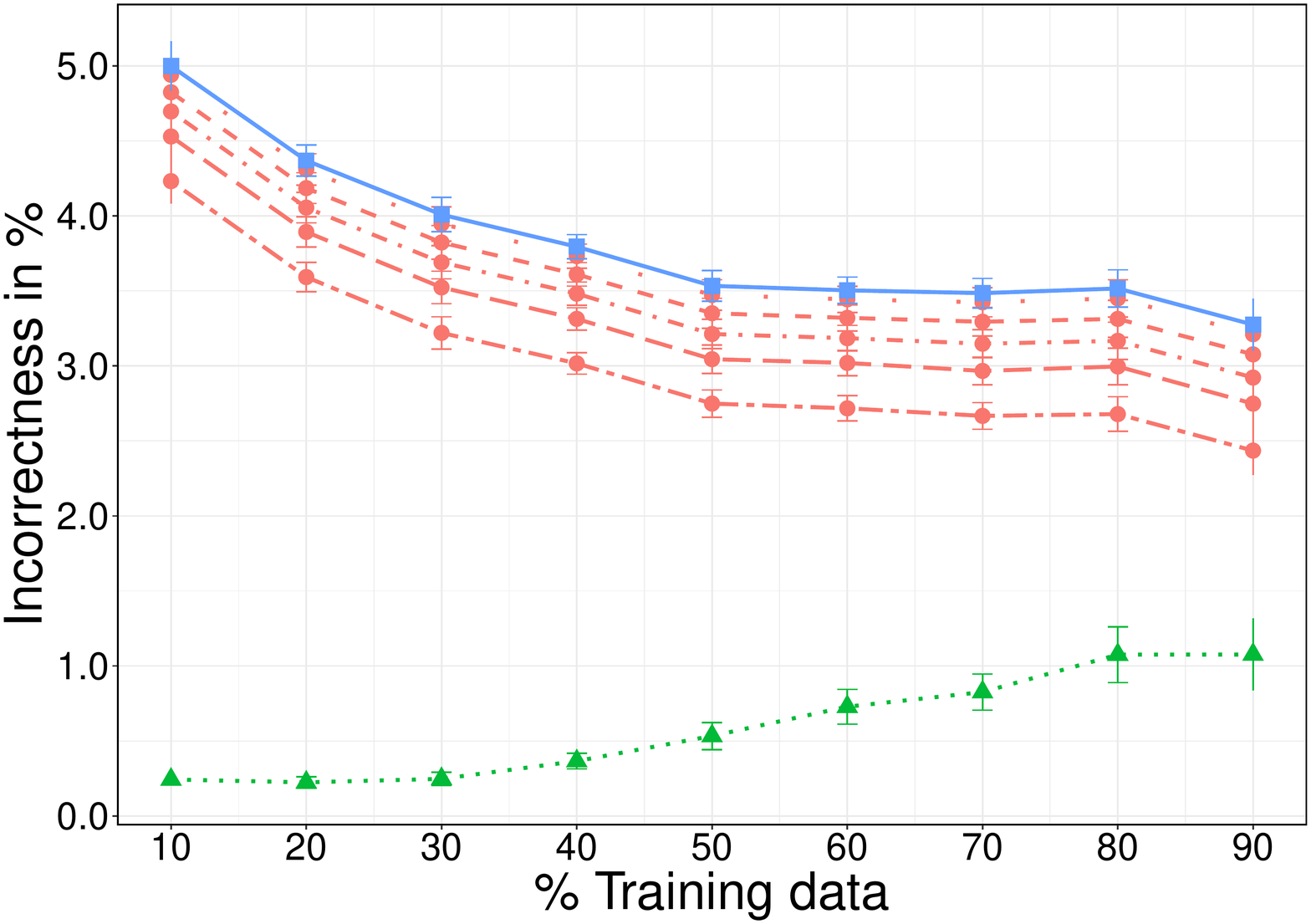}}
		\subfigure[\scshape CAL500 ($\epsilon=\num{5e-3}$)]{
			\includegraphics[width=0.244\linewidth]
				{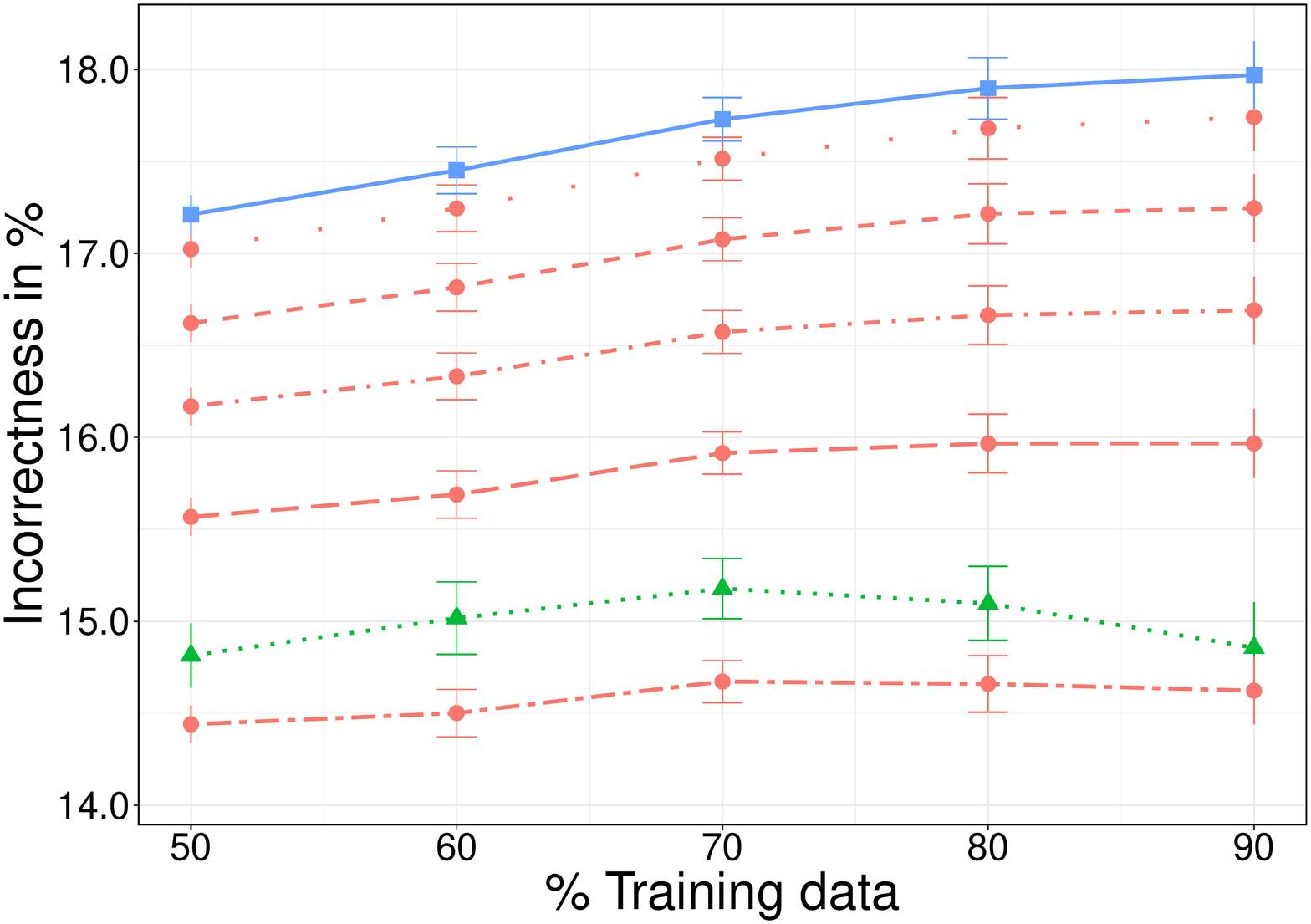}}
	}\vspace{-2mm}
	\caption{{\bf Downsampling - Credal sum product network - $\epsilon$ contamination.} Incorrectness (y-axis) evolution for the precise, partial abstention with $f_{SEP}$ (\nth{1} row) and $f_{PAR}$ (\nth{2} row) penalty functions, rejection (\nth{3} row), and skeptical approach, each one with different hyper-parameter levels (except to the precise approach), and with respect to different percentages of training data sets (x-axis).}
	\label{fig:cspnepsilonsamplingresults}
\end{figure}

%% file: appendix_classifier.tex

\section{A brief description of different Credal Classifiers used in Section~\ref{sec:expebinarybr}}\label{app:credalclassifier}
In this appendix, we introduce basic and necessary notations to understand how we can compute the lower and upper conditional probabilities of each label.

\subsection{Naive Credal classifier} \label{app:nccmethod}
In this section, we briefly remind  technical notions about how to obtain marginal probability intervals $[\underline{P}_{\vect{x}}(Y_{\{j\}}=y_j), \overline{P}_{\vect{x}}(Y_{\{j\}}=y_j)]$ over each label $j^{th}$ by using the imprecise classifier called the naïve credal classifier (NCC)\footnote{Bearing in mind that it can be replaced by any other (credal) imprecise classifiers, see~\cite[\S 10]{augustin2014introduction} or \cite{alarcon2019imprecise}.}~\cite{zaffalon2002naive}, which extends the classical naive Bayes classifier (NBC). We here remind its main features, and refer to Zaffalon~\cite{zaffalon2002naive} for further details.

NCC is based on the same assumptions as NBC, meaning that it also assumes the attribute independence given a class. In our case, we will consider each label as a simple binary classification problem, and will estimate the marginal probabilities using the model
\begin{align}\label{eq:marginalbayes}
	P(Y_{\{j\}}=y_j|X=\newinstance) =
		\frac{P(Y_{\{j\}}=y_j)
				\prod_{i=1}^d P(X_i = x_i|Y_{\{j\}}=y_j)
			}{\sum_{y_l\in\{0, 1\}}P(Y_{\{j\}}=y_l)
				\prod_{i=1}^d P(X_i = x_i|Y_{\{j\}}=y_l)
			}.
\end{align}

Computing lower and upper probability bounds  $[\underline{P},\overline{P}]$ over all possible marginals $\mathcal{P}_{Y_j}$ and conditional $\mathcal{P}_{X|Y_j}$ distributions can be performed solving the following minimization/maximization problem for Equation~\eqref{eq:marginalbayes}
\begin{align}
	\underline{P}(Y_{\{j\}}&=y_j|X=\newinstance) = \min_{P\in\mathcal{P}_{Y_j}} \min_{\substack{P\in\mathcal{P}_{X_i|Y_j}\\i=1,\dots,d}} P(Y_{\{j\}}=y_j|X=\newinstance),\label{eq:lowernccprob}\\
	\overline{P}(Y_{\{j\}}&=y_j|X=\newinstance) = \max_{P\in\mathcal{P}_{Y_j}} \max_{\substack{P\in\mathcal{P}_{X_i|Y_j}\\i=1,\dots,d}} P(Y_{\{j\}}=y_j|X=\newinstance).\label{eq:uppernccprob}
\end{align}

For practical purposes and as the number of training data is usually sufficient, we assume a precise estimation of the marginal distribution $\mathbb{P}_{Y_j}$ instead of a credal set $\mathcal{P}_{Y_j}$, so that the optimization over the credal of marginal distributions $\mathcal{P}_{Y_j}$ can be ignored. In this case, it can be shown~\cite{zaffalon2002naive} that the last equations are equivalent to
\begin{align}
	\underline{P}(Y_{\{j\}}=y_j|X=\newinstance) &=
		\left(1+
		\frac{P(Y_{\{j\}}=\overline{y}_j)
			\prod_{i=1}^d \overline{P}(X_i = x_i|Y_{\{j\}}=\overline{y}_j)}
			{P(Y_{\{j\}}=y_j)
			\prod_{i=1}^d \underline{P}(X_i = x_i|Y_{\{j\}}=y_j)}
		\right)^{-1}, \\
	\overline{P}(Y_{\{j\}}=y_j|X=\newinstance) &=
		\left(1+
		\frac{P(Y_{\{j\}}=\overline{y}_j)
			\prod_{i=1}^d  \underline{P}(X_i = x_i|Y_{\{j\}}=\overline{y})
			}
			{P(Y_{\{j\}}=y_j)
			\prod_{i=1}^d \overline{P}(X_i = x_i|Y_{\{j\}}=y_j)}
		\right)^{-1},
\end{align}
where $\overline{y}_j$ is the complement to $y_j$. To obtain the other bounds of those equations, we use the Imprecise Dirichlet model (IDM)~\cite{walley1996inferences}
\begin{align}
	\underline{P}(X_i=x_i| Y_{\{j\}}=y_j) =  \frac{n_i(y_j)}{N_{y_j}+s}
	\quad &\text{and}\quad
	\overline{P}(X_i=x_i| Y_{\{j\}}=y_j) =  \frac{n_i(y_j)+s}{N_{y_j}+s}
\end{align}
where $s\in\mathbb{R}$ is the hyper-parameter that controls the imprecision level. $n_i(y_j)$ is the number of instances in the training set where $X_i=x_i$ and the label value is $y_j$ and $N_{y_j}$ is the number of instances in the training set where the label value is $y_j$. Note that if the input space is continuous, it shall be discretized in $z$ equal-width intervals in order to get values of $n_i(y_j)$. 

Note that the higher $s$ is, the wider the intervals $[\underline{P}_{\vect{x}}(Y_{\{j\}}=y_j),\overline{P}_{\vect{x}}(Y_{\{j\}}=y_j)]$ are. For $s=0$, we retrieve the classical NBC with precise predictions, and for high enough values of $s\!>>>\!0$, the NCC model will make vacuous predictions (i.e. abstain for all labels $\forall i, Y_i=*$).



\subsection{Imprecise Gaussian discriminant classifier}
\label{app:igdadescription}
In this section, we briefly remind technical notions on how we can obtain marginal probability bounds of the $j^{th}$ label, i.e. $[\underline{P}_{\vect{x}}(Y_{\{j\}}=y_j), \overline{P}_{\vect{x}}(Y_{\{j\}}=y_j)]$, by using the imprecise Gaussian discriminant analysis (IGDA), and we refer to \cite{alarcon2021imprecise} for further details.

IGDA is an imprecise version of the classical Gaussian discriminant analysis (GDA), and as such, it is based on the same assumptions as GDA. That means, the conditional distribution $\mathbb{P}_{X|Y_{\{i\}}=y_i}$ is also modelled as a multivariate Gaussian distribution $\mathcal{N}(\mu_{Y_{\{i\}}=y_i}, \Sigma_{Y_{\{i\}}=y_i})$. However, IGDA does not just use a single conditional distribution as GDA, but a set of conditional distributions (or a set of Gaussian distributions) $\credal_{X|Y_{\{i\}}=y_i}$, and moreover, IGDA uses a robust Bayesian inference approach with a set of near-ignorance priors to estimate the set of posterior distribution of parameter $\mu_{Y_{\{i\}}=y_i}$. Thus, the set of conditional Gaussian distributions is
\begin{equation}\label{eq:setgaussianpost}
\mathscr{P}_{X|Y_{\{i\}}=y_i} = \left\{\mathbb{P}_{X|Y_{\{i\}}=y_i} ~\Big|~ \mathbb{P}_{X|Y_{\{i\}}=y_i} \sim \mathcal{N}(\mu_{Y_{\{i\}}=y_i}, \widehat{\Sigma}_{Y_{\{i\}}=y_i}),~\mu_{Y_{\{i\}}=y_i} \in \mathbb{G}_{Y_{\{i\}}=y_i} \right\},
\end{equation}
where $\widehat{\Sigma}_{Y_{\{i\}}=y_i}$ is the precise estimation of covariance matrix (i.e. the empirical covariance matrice of label $Y_{\{i\}}=y_i$) and $\mathbb{G}_{Y_{\{i\}}=y_i}$ is a convex space of estimated values for the mean $\mu_{Y_{\{i\}}=y_i}$ defined as follows
\begin{equation}
\mathbb{G}_{Y_{\{i\}}=y_i} = \left\{  \widehat{\mu}_{Y_{\{i\}}=y_i} \in \mathbb{R}^p ~\Bigg|~
  \widehat{\mu}_{i,Y_{\{i\}}=y_i} \in \left[ \frac{-\tau + n_{y_i} \overline{\boldsymbol x}_{i,n_{y_i}}}{n_{y_i}}, \frac{\tau + n_{y_i} \overline{\boldsymbol x}_{i,n_{y_i}}}{n_{y_i}} \right],
\forall i = \{ 1,...,p \}
\right\},
\end{equation}
where $\tau$ is the hyper-parameter that controls the imprecision level of the marginal probability interval (i.e. the size of the interval).

In the same way as the NCC model estimes the lower and upper probability bounds of the $j^{th}$ label, i.e. Equations \eqref{eq:lowernccprob} and \eqref{eq:uppernccprob}, IGDA does the same by applying the Bayes' theorem and the maximality criterion\footnote{The same results is obtained if we use the interval-dominance criterion.} on a binary space $y_i\in\{0, 1\}$ in order to get the lower and upper probabilities bounds over all possible conditional distributions of Equation~\eqref{eq:setgaussianpost} as follows
\begin{align}
	\underline{P}(Y_{\{j\}}=y_j|X=\newinstance) &\propto
		\min_{P\in\mathcal{P}_{Y_j}}
		\min_{P\in\mathcal{P}_{X|Y_j}}
		P(Y_{\{j\}}=y_j)
		P(X = \newinstance |Y_{\{j\}}=y_j),
		\label{eq:lowerigdaprob}\\
	\overline{P}(Y_{\{j\}}=y_j|X=\newinstance) &\propto
		\max_{P\in\mathcal{P}_{Y_j}}
		\max_{P\in\mathcal{P}_{X|Y_j}}
		P(Y_{\{j\}}=y_j)
		P(X = \newinstance |Y_{\{j\}}=y_j).
		\label{eq:upperigdaprob}
\end{align}

Once again, as in the NCC estimation case, we assume a precise estimation of the marginal distribution $\mathbb{P}_{Y_j}$ instead of a credal set $\mathcal{P}_{Y_j}$. Thus, the optimisation problems of Equations \eqref{eq:lowerigdaprob} and \eqref{eq:upperigdaprob} can be reduced by solving
\begin{align}
\underline{P}(X=\newinstance|Y_{\{j\}}=y_j) &= \underset{P \in \mathscr{P}_{X|Y_{\{i\}}=y_i}}{\inf}  ~ P(X=\newinstance|Y_{\{j\}}=y_j),
	\label{eq:infconddistrib} \\
\overline{P}(X=\newinstance|Y_{\{j\}}=y_j) &= \underset{P \in \mathscr{P}_{X|Y_{\{i\}}=y_i}}{\sup} ~ P(X=\newinstance|Y_{\{j\}}=y_j).
	\label{eq:subconddistrib}
\end{align}
As $\mathscr{P}_{X|Y_{\{i\}}=y_i}$ is a set of Gaussian distributions, the probability bounds of Equations \eqref{eq:infconddistrib} and \eqref{eq:subconddistrib} are respectively obtained by calculating the mean bounds as follows
\begin{align}
\underline{\mu}_{y_i} &= \underset{\mu_{y_i} \in \mathbb{G}_{y_i}}{\arg \inf}  -\frac{1}{2}  (\newinstance - \mu_{y_i})^T\widehat{\Sigma}_{y_i}^{-1}(\newinstance - \mu_{y_i}), \label{eq:infgaussian}\\
\overline{\mu}_{y_i} &=  \underset{\mu_{y_i}  \in \mathbb{G}_{y_i}}{\arg \sup} -\frac{1}{2}  (\newinstance - \mu_{y_i})^T\widehat{\Sigma}_{y_i}^{-1}(\newinstance - \mu_{y_i}),  \label{eq:supgaussian}
\end{align}
 where $\hat{\Sigma}_{y_i}^{-1}$ is the inverse of the covariance matrix and the subscript $y_i$ denotes $Y_{\{i\}}=y_i$. Besides, depending on the internal structure of the covariance matrix $\hat{\Sigma}_k$, the optimisation problems of Equations \eqref{eq:infgaussian} and \eqref{eq:supgaussian} can be solved easier and we can obtain four variants with different time complexities, see Table~\ref{tbl:diffvariantigda}.
\begin{table}[h]
	\centering
	\begin{tabular}{|l |l | c | c |} \hline
	 Name & Assumptions & Acronym & \shortstack{Inference\\Complexity}\\ [0.5ex] \hline\hline
	 \hline
Imprecise \textbf{Quadratic} Discriminant Analysis &
Heteroscedasticity: $\hat{\Sigma}_{Y_{\{j\}}=y_j} = \hat{\Sigma}_{y_j}$ & IQDA & $\ge\complexity{p^2}$\\
Imprecise \textbf{Linear} Discriminant Analysis &
Homoscedasticity: $\hat{\Sigma}_{Y_{\{j\}}=y_j} =\hat{\Sigma}$ & ILDA & $\ge\complexity{p^2}$ \\
Imprecise \textbf{Naïve} Discriminant Analysis &
Feature independence: $\hat{\Sigma}_{Y_{\{j\}}=y_j} = \bm \hat{\sigma}_{y_j}^T \mathbb{I}$&  INDA &   $\complexity{p}$ \\
Imprecise \textbf{Euclidean} Discriminant Analysis &
Unit-variance feature indep.: $\hat{\Sigma}_{Y_{\{j\}}=y_j} = \mathbb{I}$ &  IEDA &  $\complexity{p}$\\
	 \hline
	\end{tabular}
	\caption{Different variants of IGDA classifier}
	\label{tbl:diffvariantigda}
\end{table}

Note that we can retrieve the classical GDA model with precise predictions when the value of hyper-parameter $\tau =0$ or if the number of samples $n_{y_j}$ of sub-population $Y_{\{j\}}=y_j$ is much bigger (or tends to infinity).

\subsection{Credal Sum-Product Network classifier}
Like in the previous sections, we remind major technical features regarding how to obtain marginal probability intervals over each label $Y_{\{j\}}$, i.e. $[\underline{P}_{\vect{x}}(Y_{\{j\}}=y_j), \overline{P}_{\vect{x}}(Y_{\{j\}}=y_j)]$, by using the Credal Sum-Product Network (CSPN), which extends the \emph{precise} deep probabilistic graphical model known as Sum-Product Network (SPN)~\cite{poon2011sum} to the \emph{imprecise} probabilistic setting. For further technical details, we refer to D.D. Mau\'a et al.~\cite{maua2017credal}.

A Sum-Product Network\footnote{Poon and Domingos introduced SPN in~\cite{poon2011sum} as a new deep architecture in order to solve the intractable inference problem of Graphical models. SPN computes the inference in a time polynomial by traversing the network.} $\mathbb{S}$ represents a joint probability distribution $\mathbb{P}$ over a set of random variables by a rooted weighted directed acyclic graph with indicator variables (or univariate distributions) as leaves, and sums and products operation as internal nodes (i.e. weighted sums and products of smaller SPNs). The directed edges  $i\rightarrow j$ in $\mathbb{S}$ are associated with non-negative \textbf{weights} $w_{i,j}$, such that: (1) the directed edges from a product node to its children $j$ are labelled with a weight of one, and (2) the directed edges from a sum node to its children are labelled with a corresponding weight $w_{i,j}$ such that $\sum_j w_{ij} = 1$.

To ensure that $\mathbb{S}$ represents a valid probability distribution and computes the MAP inferences in a linear time, $\mathbb{S}$ must satisfy the following properties: \emph{Complenteness}, \emph{Decomposition}, \emph{Normalization}, and \emph{Selectivity} (for further details see~\cite{peharz2014learning}). The SPNs satisfying such properties are called selective SPNs. Becsides, as we aim to obtain the marginal probability $p_i:=P_\newinstance(Y_i=y_i)$ of the label $Y_i$ from a binary classification (or the probability interval $[\underline{p}_i, \overline{p}_i]$ for the case of the imprecise setting), we decided to use the \textbf{class-selective} SPNs approach proposed by Correia A.H et al in~\cite{Correia2019Towards}. \textbf{Class-selective} SPNs implement a suited architecture for classification tasks (see Figure~\ref{fig:csspnarchiteture}) and reduces the inference time by using memoization techniques.
\begin{figure}[!ht]
	\centering
	\includegraphics[scale=1]{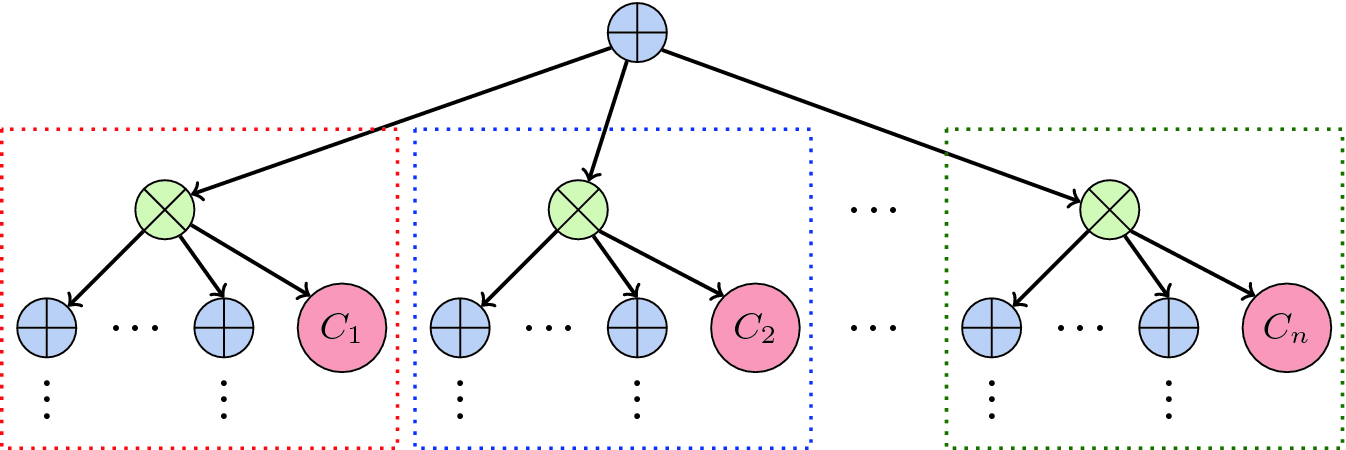}
	\caption{Class-selective SPN architecture~\cite{Correia2019Towards}, in which each $C_i$ is a leaf node denoting a different class (\{0,1\} in our case) and allowing active only one of the sub-networks $C_i$.}
	\label{fig:csspnarchiteture}
\end{figure}

A class-selective Credal SPN (CS-CSPN) is defined similarly, except that instead of using a singleton weight vector $\vect{w}$ in each sum node, it uses a set of weight vectors which lives in a convex space with constraints $\mathscr{C}_{\vect{w}, *}$. We decided to use two different constrains;
\begin{enumerate}
	\item a convex space defined from an $\epsilon$-contamination approach $\mathscr{C}_{\vect{w}, \epsilon}$
	\begin{align}
		\mathscr{C}_{\vect{w}, \epsilon} = \left\{
		(1-\epsilon)\vect{w}, + \epsilon\vect{v}: v_j \geq 0,\sum_j v_j = 1
	\right\};\label{eq:econtcspnconstrain}
	\end{align}
	\item and, a convex space defined from an imprecise Dirichlet model $\mathscr{C}_{N_{\vect{w}}, s}$ (IDM)~\cite{walley1996inferences}
	\begin{align}
	  \mathscr{C}_{N_{\vect{w}}, s} = \left\{
		\vect{w}_i: w_{ij} = \frac{N_j + s\cdot v_i}{N_i+s}, v_j\geq 0,\sum_j v_j = 1
	\right\}.\label{eq:idmcspnconstrain}
	\end{align}
\end{enumerate}

Therefore, a CS-CSPN is defined as a set of class-selective SPNs $\mathbb{C}=\{\mathbb{S}_\vect{w}: \vect{w}\in \mathscr{C}_{\vect{w}, *}\}$ over \textbf{the same network structure} of $\mathbb{S}$, and where $\mathscr{C}_{\vect{w}, *}$ can be replaced by $\mathscr{C}_{\vect{w}, \epsilon}$ or $\mathscr{C}_{N_{\vect{w}}, s}$. Of course, as $\mathbb{S}$ represents a single joint distribution $\mathbb{P}$ over its variables, $\mathbb{C}$ represents a set of joint distributions $\credal$ over its same variables, so that one can use $\mathbb{C}$ to obtain the lower and upper probability bounds.

In the same way as in the previous sections, make a decision on a binary space $y_j\in\{0,1\}$ (i.e. an (imprecise) binary classification problem) by using the set of class-selective SPNs models $\mathbb{C}$ and the maximality criterion (under to the $\ell_{1/0}$ function) amounts to calculate the lower and upper probability bounds $[\underline{p}_j, \overline{p}_j]$ of the $j^{th}$ label as follows:
\begin{align}
	\underline{P}(Y_{\{j\}}&=y_j|X=\newinstance_\varepsilon) \propto
			\min_{\vect{w}\in\mathscr{C}_{\vect{w}, *}}
				P_\vect{w}(Y_{\{j\}}=y_j, X=\newinstance_\varepsilon),\label{eq:lowercspnprob}\\
	\overline{P}(Y_{\{j\}}&=y_j|X=\newinstance_\varepsilon) \propto
		\max_{\vect{w}\in\mathscr{C}_{\vect{w}, *}}
	 		  P_\vect{w}(Y_{\{j\}}=y_j, X=\newinstance_\varepsilon).\label{eq:uppercspnprob}
\end{align}
where $\mathbb{S}_\vect{w} (y_j,\newinstance_\varepsilon) := P_\vect{w}(Y_{\{j\}}=y_j, X=\newinstance_\varepsilon)$ denotes the evaluation of the network $\mathbb{S}_\vect{w}$ subject to the evidence input value $(y_j,\newinstance_\varepsilon)$, $P_\vect{w}(X=\newinstance_\varepsilon)$ must be strictly positive so that the equations above are well-defined, $\mathscr{C}_{\vect{w}, *}$ is the constrained convex space chosen between Equation \eqref{eq:econtcspnconstrain} or \eqref{eq:idmcspnconstrain}, and finally, $\newinstance_\varepsilon$ is known as the \emph{partial} evidence which indicates that not every variable is observed of $\newinstance$, since $\varepsilon$ is a subset of indices from the total set $\{1, \dots, p\}$ of random variables $X=\{X_1, \dots, X_p\}$

As in the previous imprecise classifiers, we can retrieve the precise probability $P(Y_{\{j\}}=y_j|X=\newinstance_\varepsilon)$ by setting the hyper-parameter $\epsilon$ to $0$ for $\mathscr{C}_{\vect{w}, \epsilon}$ and  $s$ to $0$ for $\mathscr{C}_{N_{\vect{w}}, s}$.

%% file: appendix_proofs.tex
\section{On Binary relevance and other decision criteria}
\label{app:other_dec_crit}
So far, we considered only the most common skeptic decision criteria (maximality and E-admissibility) in order to get a set of predictions (either partial or not). However, there are other decision criteria using probability sets and extending the classical expected loss criterion. The most common being  (1) Interval dominance, (2) $\Gamma$-minimin, (3) $\Gamma$-maximin (at work in distributionally robust approaches). 

In this appendix, we provide some additional results regarding them, with the proofs given in \ref{app:supproof}. Let us first introduce some definitions. Given a loss function $\ell$, we will denote by 
\begin{equation}
    \overline{\mathbb{E}}_\credal\left[\ell(y, \cdot)\right]:=\max_{P \in \credal} \expe_P\left[\ell(y, \cdot)\right] \text{ and } \underline{\mathbb{E}}_\credal\left[\ell(y, \cdot)\right]:=\min_{P \in \credal} \expe_P\left[\ell(y, \cdot)\right]
\end{equation}
the upper and lower expected values of this loss under uncertainty $\credal$. They respectively provide an assessment of the worst-case and best-case situations. We can then define the decision criteria explored in this appendix.

\begin{definition}[Interval dominance] The decision $\hat{Y}^{ID}_{\ell,\credal}$ obtained by interval dominance is the set
\begin{equation}
    \hat{Y}^{ID}_{\ell,\credal}=\{y \in \mathcal{Y}: \not\exists y' \text{ s.t. } \overline{\mathbb{E}}_\credal\left[\ell(y', \cdot)\right] < \underline{\mathbb{E}}_\credal\left[\ell(y, \cdot)\right]\}
\end{equation}
\end{definition}
That is, interval dominance retain all these prediction not dominated by the worst-case expected loss situation of another prediction. It is a very conservative rule, as one has $\hat{Y}^{M}_{\ell,\credal} \subseteq \hat{Y}^{ID}_{\ell,\credal}$.

\begin{definition}[$\Gamma$-\uppercase{m}\textsc{inimax}]\label{def:minmax}
$\Gamma$-\uppercase{m}\textsc{inimax} consists in replacing the expected value of Equation~\eqref{eq:minexpe} by its upper expectation 
\begin{equation}\label{eq:gminmax}
	\hat{y}^{\Gamma_{\max}}_{\ell,\credal} = 
	\underset{y \in \mathcal{Y}}{\arg \min}~
	\overline{\mathbb{E}}_\credal\left[\ell(y, \cdot)\right].
\end{equation}
It amounts to returning the best worst-case prediction (i.e. a pessimistic attitude), since it consists in minimizing the worst possible expected loss. 
\end{definition}

\begin{definition}[ $\Gamma$-{\scshape\MakeUppercase minimin}]\label{def:minmin}
$\Gamma$-{\scshape\MakeUppercase minimin}, in contrast, consists in replacing the expected value of Equation~\eqref{eq:minexpe} by its lower expectation 
\begin{equation}\label{eq:gminmin}
	\hat{y}^{\Gamma_{\min}}_{\ell,\credal} =
	\underset{y \in \mathcal{Y}}{\arg \min}~
	\underline{\mathbb{E}}_\credal\left[\ell(y, \cdot)\right].
\end{equation} 
\end{definition}
It amounts to returning the best best-case prediction (i.e. an optimistic attitude), since it consists in choosing the prediction with the smallest lower expectation. Usually, one has that $\hat{y}^{\Gamma_{\min}}_{\ell,\credal} \neq \hat{y}^{\Gamma_{\max}}_{\ell,\credal}$.

However, when considering probability sets satisfying the hypothesis of Section~\ref{sec:Binary_rev}, we have the following results regarding these two last criteria:

\begin{proposition}\label{prop:gammaminmax}
	Given a probability set $\credal_{BR}$ and the Hamming loss $\ell_H$, we have: 
	\begin{equation}
		\hat{y}^{\Gamma_{\max}}_{\ell_H,\credal_{BR}} =
		\hat{y}^{\Gamma_{\min}}_{\ell_H,\credal_{BR}}
	\end{equation}
\end{proposition}

It is well known the set $\hat{Y}^{ID}_{\ell,\credal}$ is a superset of $\hat{\mathbb{Y}}^{M}_{\ell_H,\credal}$, due to its conservative nature. The next simple example shows that even in the case of binary relevance models, this inclusion can be strict. 

\begin{example}
    Consider the simple case where we have two labels with the following bounds: $P(Y_1=1)\in [0.6,1]$ and $P(Y_2=1)\in [0,1]$. We then have the following expectation bounds for the various predictions and a Hamming loss
    \begin{displaymath}
    \begin{array}{ccccc}
         y & (1,1) & (1,0) & (0,1) & (0,0) \\
         \underline{\mathbb{E}}\left[\ell(y, \cdot)\right] & 0 & 0 & 0.6 & 0.6 \\
         \overline{\mathbb{E}}\left[\ell(y, \cdot)\right] & 1.4 & 1.4 & 2 & 2  
    \end{array}
    \end{displaymath}
    from which we deduce that $\hat{Y}^{ID}_{\ell,\credal}=(*,*)$, while $\hat{\mathbb{Y}}^{M}_{\ell_H,\credal}=(1,*)$.
\end{example}

\begin{corollary}
Given a probability set $\credal_{BR}$ and the Hamming loss $\ell_H$, in the Figure~\ref{fig:decimplies}, we can show the following implications for the different decision criteria that are {\bf M}aximality, {\bf E}-admissibility, $\Gamma$-minimax, $\Gamma$-minimin, and {\bf I}nterval {\bf D}ominance. As usual with sets, an implication $A \to B$ means that $A \subset B$.  
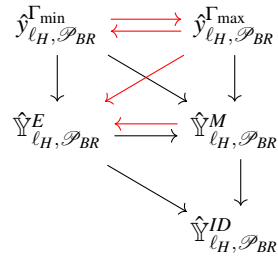
\begin{figure}[!ht]
	\centering
	\begin{tikzcd}
		\hat{y}^{\Gamma_{\min}}_{\ell_H,\credal_{BR}}
		\arrow[r,  yshift=0.5ex, red]
		\arrow[dr] \arrow[d]
		& \hat{y}^{\Gamma_{\max}}_{\ell_H,\credal_{BR}} 
		\arrow[l, yshift=-0.5ex, red]\arrow[d]
		\arrow[dl, red]\\
		\hat{\mathbb{Y}}^{E}_{\ell_H,\credal_{BR}} 
		\arrow[r, yshift=-0.5ex]\arrow[dr, yshift=-0.5ex] 
		&\hat{\mathbb{Y}}^{M}_{\ell_H,\credal_{BR}} 
		\arrow[l, yshift=0.5ex, red] 
		\arrow[d, xshift=0.5ex] \\
		&\hat{\mathbb{Y}}^{ID}_{\ell_H,\credal_{BR}}
	\end{tikzcd}	
	\caption{Decision relation under a $\credal_{BR}$ and a $\ell_H$. 
	   In red arrow, the new implications.}
	\label{fig:decimplies}
\end{figure}
\end{corollary}

\section{Result proofs}\label{app:supproof}
\begin{proof}[\bf Proof of Lemma~\ref{prop:expcond}]\label{proof:hamming}
	Let us first develop $\mathbb{E}\left[ \ell_H(\vect{y}^2, \cdot) - \ell_{H}(\vect{y}^1, \cdot) | X=x \right]$:
	\begin{align}
		\sum_{\bm y \in \spacelabel} &\left( \hams{y^2} - \hams{y^1} \right)P_x(Y=\bm y)\\
		\sum_{y_1\in \{0,1\}}\sum_{y_2\in \{0,1\}}\dots\sum_{y_m\in \{0,1\}}
 				&\left( \hams{y^2} - \hams{y^1} \right)P_x(Y=\bm y)
	\end{align}
	For a given $k \in \{1,\ldots,m\}$, let us consider the rewriting
	\begin{align}
		\overbrace{\sum_{y_1\in \{0,1\}}\sum_{y_2\in \{0,1\}}\dots\sum_{y_m\in \{0,1\}}}^{m-1}
 				&\left[\sum_{y_k\in \{0,1\}} \left( \hams{y^2} - \hams{y^1} \right) \right] P_x(Y=\bm y).
	\end{align}
	Developing the sum between brackets, we get
	\begin{align}\label{eq:bigsumproof}
		\sum_{y_k\in \{0,1\}}\hams{y^2} P_x(Y=\bm y) - \sum_{y_k\in \{0,1\}}\hams{y^1}P_x(Y=\bm y) \quad\text{(by linearity)}
	\end{align}
	Developing again the left term, we obtain
	\begin{align*}
		\sum_{y_k\in \{0,1\}}\hams{y^2} P_x(Y=\bm y) &= 
			\sum_{y_k\in \{0,1\}} \left( \hamtwo{1}+\hamtwo{2}+\dots + \hamtwo{m} \right) P_x(Y=\bm y)\\
		&= \hamtwo{1} \sum_{y_k\in \{0,1\}} P_x(Y^k=y_k) + \dots +  \sum_{y_k\in \{0,1\}} \hamtwo{k} P_x(Y^k=y_k) + \\
		& \qquad \dots + \hamtwo{m} \sum_{y_k\in \{0,1\}}  P_x(Y^k=y_k)\\
		&= \hamtwo{1} P_x(Y_{\{-k\}}) + \dots + \sum_{y_k\in \{0,1\}} \hamtwo{k} P_x(Y^k=y_k) + \\
		& \qquad \dots  + \hamtwo{m} P_x(Y_{\{-k\}})\\
		&= \sum_{y_k\in \{0,1\}} \hamtwo{k} P_x(Y^k=y_k) + \sum_{i=1, i\ne k}^m \hamtwo{i} P_x(Y_{\{-k\}}),
	\end{align*}
	where $P_x(Y^k=y_k) := P_x(Y_1, \dots, Y_k=y_k, \dots, Y_m)$ and $P_x(Y_{\{-k\}}):=P_x(Y_1, \dots, Y_{k-1}, Y_{k+1}, \dots, Y_m)$.

	Similarly, we get for the right term
	\begin{align*}
		\sum_{y_k\in \{0,1\}}\hams{y^1} P_x(Y=\bm y) &= \sum_{y_k\in \{0,1\}} \hamone{k} P_x(Y^k=y_k) + \sum_{i=1, i\ne k}^m \hamone{i} P_x(Y_{\{-k\}})
	\end{align*}
	We put back these rewritten sums in Equation~\eqref{eq:bigsumproof}
	\begin{align}
		&\overbrace{\sum_{y_1\in \{0,1\}}\sum_{y_2\in \{0,1\}}\dots\sum_{y_m\in \{0,1\}}}^{m-1}  \Bigg[
		\sum_{y_k\in \{0,1\}} \hamtwo{k} P_x(Y^k=y_k) - \sum_{y_k\in \{0,1\}} \hamone{k} P_x(Y^k=y_k)   \nonumber\\
		&\qquad + \sum_{i=1, i\ne k}^m \hamtwo{i} P_x(Y_{\{-k\}}) - \sum_{i=1, i\ne k}^m \hamone{i} P_x(Y_{\{-k\}}) \Bigg]= \nonumber\\
		&\sum_{\bm y \in \spacelabel} (\hamtwo{k} - \hamone{k}) P_x(Y=\bm y)  
		+ \overbrace{\sum_{y_1\in \{0,1\}}\dots\sum_{y_m\in \{0,1\}}}^{m-1} 
			\left[\sum_{i=1, i\ne k}^m \hamtwo{i} - \hamone{i}\right] P_x(Y_{\{-k\}}) \label{eq:summoins1}
	\end{align} 
	The left term can be reduced in the following way:
	\begin{align*}
		\sum_{\bm y \in \spacelabel} (\hamtwo{k} - \hamone{k}) P_x(Y=\bm y) &= 
		\sum_{\bm y \in \spacelabel} \hamtwo{k}P_x(Y=\bm y) - \sum_{\bm y \in \spacelabel}\hamone{k} P_x(Y=\bm y)\\
&=P_x(Y_k \neq y^2_k) -  P_x(Y_k \neq y^1_k)\\
		&=P_x(Y_k=y_{k}^1)-P_x(Y_k=y_{k}^2)
	\end{align*}
	since we have $P_x(Y_k \neq y_k)=1-P_x(Y_k = y_k)$. We can apply the same operations we just did on the right term of Equation \eqref{eq:summoins1}, and do so recursively, to finally obtain
	\begin{align}
		\sum_{i=1}^m  P(Y_i=y_i^1) - P(Y_i=y_i^2) 
	\end{align}
\end{proof}
\begin{proof}[\bf Proof of Proposition~\ref{prop:newdecision}] Using Equation~\eqref{eq:hamsimpl}, one can readily see that
\begin{align}
\vect{y}^1 \succ_M \vect{y}^2 & \iff \inf_{P \in \credal}  \sum_{i \in \mathcal{I}_{\vect{y}^1\neq \vect{y}^2}}  P(Y_i=y^1_i) - P(Y_i=y^2_i) >0& \\
& \iff \inf_{P \in \credal}  \sum_{i \in \mathcal{I}}  P(Y_i=a_i) - P(Y_i=\overline{a}_i)>0& \\
	\intertext{Accounting for the fact that  $P(Y_i=a_i) + P(Y_i=\overline{a}_i)=1$, we get } 
& \iff \inf_{P \in \credal}  \sum_{i \in \mathcal{I}}  2 P(Y_i=a_i) - 1 >0 & \\
& \iff \inf_{P \in \credal}  \sum_{i \in \mathcal{I}}   P(Y_i=a_i) > \frac{|\mathcal{I}|}{2} \label{eq:infexpectation}
\end{align}
\end{proof}

\begin{proof}[\bf Proof of Proposition~\ref{prop:Ham_partial_loss}] First, let us simply notice that $P(Y_i=a_i)=\sum_{\vect{y} \in \mathcal{Y}} \mathbbm{1}_{y_i=a_i} P(Y=\vect{y})$ and $\mathbbm{1}_{y_i=a_i}=\mathbbm{1}_{y_i\neq\overline{a}_i}$. Putting these together, we get 
\begin{align*}	
	\sum_{i \in \mathcal{I}}   P(Y_i=a_i) 
&= \sum_{i\in\mathcal{I}} \sum_{\vect{y}\in\mathcal{Y}} \mathbbm{1}_{y_i \neq \overline{a}_i} P(Y=\vect{y}) \\ 
	&= \sum_{\vect{y}\in\mathcal{Y}} \sum_{i\in\mathcal{I}} 
		\mathbbm{1}_{y_i \neq \overline{a}_i} P(Y=\vect{y}) \tag{by linearity}\\
	&= \mathbb{E}[\ell_H^*(\cdot, \overline{\vect{a}}_\mathcal{I})]
\end{align*}
where $\ell_H^*(\cdot, \overline{\vect{a}}_\mathcal{I})$ is the hamming loss
calculated in the set of indices $\mathcal{I}=\{i_1,\dots,i_q\}$ of vector $\overline{\vect{a}}_{\mathcal{I}}$, which is created in the line 5 of the Algorithm \ref{alg:HamMaxim}. Thus, we apply infimum, $\inf_{P\in\credal}$, to each side of the last equation and get what we sought. 
\end{proof}

\begin{proof}[\bf Proof of Proposition~\ref{prop:algcomplexity}]
Let us simply analyze the number of computations needed. We will need to perform $m$ times the loop of Line 2. For a given $i$, we have that $\mathcal{Z}_i=\binom{m}{i}$, meaning that this is the number of elements to check in the loop starting Line 4. Finally, there $2^i$ elements to check in the loop starting Line 5. The table below summarise the different steps. 

	\begin{displaymath}
	\begin{array}{c|ccccc}
	\text{ Index Line 2}&	i=1 &i=2  &\dots  &i=m-2 &i=m-1\\ & & & & \\
	|\mathcal{Z}_i| & \frac{m!}{1!(m-1)!} &\frac{m!}{2!(m-2)!}  &\dots  &\frac{m!}{(m-2)!2!} &\frac{m!}{(m-1)!1!}\\ & & & & \\
	|\mathcal{Y}_z| & 	\{0,1\}^{1} &\{0,1\}^{2} &\dots, &\{0,1\}^{m-2} &\{0,1\}^{m-1}
	\end{array}
	\end{displaymath}
	Overall, the number of checks to perform amounts to
	\begin{equation}
		\sum_{k=1}^{m} 2^{m-k}\frac{m!}{k!(m-k)!} = 3^m -1
	\end{equation}
\end{proof}

The proofs of the next propositions, that concern partial binary vectors, require us to first prove an intermediate result characterising partial vectors in terms of the vector set they represent. More precisely, we first express a condition for a subset $\mathbb{Y}$ of $\mathcal{Y}$ to be a partial vector, in terms of its elements. 

\begin{lemma}\label{lem:partialbincond}
A subset $\mathbb{Y}$ belongs to the space $\mathfrak{Y}$ if and only if
$$\forall \vect{y},\vect{y}' \in  \mathbb{Y}, \text{ we have that all } \vect{y}'' \in \mathcal{Y} \text{ s.t. } y_i''=y'_i \quad\forall i \in \mathcal{I}_{\vect{y}=\vect{y}'} \text{ are also in } \mathbb{Y}$$
\end{lemma}

\begin{proof}[\bf Proof of Lemma~\ref{lem:partialbincond}]
\textbf{Only if}: Immediate, since by assumption $\mathcal{I}_{\vect{y}\neq \vect{y}'} \subseteq \mathcal{I}^*$, the set of label indices on which we abstain.

\textbf{If}: Consider the set $D_{\mathbb{Y}}=\{j | \exists \vect{y},\vect{y}' \in \mathbb{Y}, y_j \neq y'_j\}$ of indices for which at least two elements of $\mathbb{Y}$ disagree. What we have to show is that under the condition of Lemma~\ref{lem:partialbincond}, any completion of $D_{\mathbb{Y}}$ is within $\mathbb{Y}$.

Without loss of generality, as we can always permute the indices, let us consider that $D_{\mathbb{Y}}$ are the $|D_{\mathbb{Y}}|$ first indices. We can then find a couple $\vect{y},\vect{y}' \in \mathbb{Y}$ such that the $k$ first elements are distinct, that is $\mathcal{I}_{\vect{y}\neq \vect{y}'}=\{1,\ldots,k\}$. It follows that the subset of vectors
\begin{equation}\label{eq:vectorsin}(\underbrace{*,\ldots,*}_{k \text{ times }},y_{k+1},\ldots,y_{|D_{\mathbb{Y}}|},y_{|D_{\mathbb{Y}}|+1}, \ldots,y_m )\end{equation}
is within $\mathbb{Y}$, by assumption. If $k <|D_{\mathbb{Y}}|$, we can find a vector $\vect{y}''$ such that its $k'$ next elements (after the $k$th first) are different from $\vect{y}$, i.e., $y_j\neq y''_j$ for $j=k+1,\ldots,k+k'$ with $k+k' \leq |D_{\mathbb{Y}}|$. Note that $k'\geq 1$ by assumption. Since the vector~\eqref{eq:vectorsin} is in $\mathbb{Y}$, we can always consider the vector $\vect{y}$ such that its $k$ first elements are different from those of $\vect{y}''$, that is in $\mathbb{Y}$. Since $\mathcal{I}_{\vect{y}\neq \vect{y}'}=\{1,\ldots,k+k'\}$, the subset of vectors
\begin{equation*}
(\underbrace{*,\ldots,*}_{k+k' \text{ times }},y_{k+k'+1},\ldots,y_{|D_{\mathbb{Y}}|},y_{|D_{\mathbb{Y}}|+1}, \ldots,y_m )
\end{equation*}
is also in $\mathbb{Y}$. Since we can repeat this construction until having two vectors with the $|D_{\mathbb{Y}}|$ first labels different, this finishes the proof. 
\end{proof}

\begin{proof}[\bf Proof of Proposition~\ref{prop:HammingPartialBR}]
   Let us first notice that, after Equation~\eqref{eq:optiHamPrec}, we have that 
   \begin{align}\label{eq:eadmissibility}
   	\vect{y} \in \hat{\mathbb{Y}}^{E}_{\ell_H,\credal} \iff 
   	\begin{cases}
   	 \underline{p}_i \leq 0.5 & \text{ for } i \in \mathcal{I}_{\vect{y}=0} \\
     \overline{p}_i \geq 0.5 & \text{ for } i \in \mathcal{I}_{\vect{y}=1}
     \end{cases}
   \end{align}
   where $\mathcal{I}_{\vect{y}=0}$, $\mathcal{I}_{\vect{y}=1}$ are the indices of labels for which $y_i=0$ and $y_i=1$. Indeed, since here we start from the marginals, $\vect{y}$ is optimal according to Hamming loss and a distribution in $\credal_{BR}$ iff we can fix $p_{i}$ to be lower than $0.5$ if $y_i=0$, and higher else. 
   
   Now, let us consider two vectors $\vect{y}^1,\vect{y}^2$ and the indices $\mathcal{I}_{\vect{y}^1\neq \vect{y}^2}$. Given the first part of this proof, if $\vect{y}^1,\vect{y}^2 \in \hat{\mathbb{Y}}^{E}_{\ell_H,\credal}$, this means that $0.5 \in [\underline{p}_i,\overline{p}_i]$ for any $i \in \mathcal{I}_{\vect{y}^1\neq \vect{y}^2}$. Therefore, given any vector $\vect{y}''$ such that $y''_i=y^1_i$ for $i \in \mathcal{I}_{\vect{y}^1= \vect{y}^2}$, for the other indices $i \in \mathcal{I}_{\vect{y}^1\neq \vect{y}^2}$, we can always fix a precise value $p_i \in [\underline{p}_i,\overline{p}_i]$ such that $\vect{y}''$ is also optimal w.r.t. $p$. More precisely, assume the assignments $p_i^1$ and $p_i^2$ result in $\vect{y}^1,\vect{y}^2$ being optimal predictions for the Hamming loss, respectively. Then $\vect{y}''$ is optimal for the assignment
   $$p''_i=\begin{cases} p_i^1 & \text{ if } y''_i=y^1_i \\
   p_i^2 & \text{ if } y''_i=y^2_i \end{cases}$$
   that is by definition within $\credal_{BR}$.
   \end{proof}

\begin{proof}[\bf Proof of Proposition~\ref{prop:EadmEqualMax}] As Proposition~\ref{prop:HammingPartialBR} shows, the E-admissible set is given by the partial vector $\hat{\vect{\mathfrak{y}}}_{\ell_H,\credal_{BR}}$. To show that it also coincides with $\hat{\mathbb{Y}}^{M}_{\ell_H,\credal_{BR}}$, we will consider the fact that $\hat{\vect{\mathfrak{y}}}_{\ell_H,\credal_{BR}} \subseteq \hat{\mathbb{Y}}^{M}_{\ell_H,\credal_{BR}}$, and will demonstrate that any vector outside $\hat{\vect{\mathfrak{y}}}_{\ell_H,\credal_{BR}}$ is dominated (in the sense of Equation~\eqref{eq:compaIP}) by a vector within $\hat{\vect{\mathfrak{y}}}_{\ell_H,\credal_{BR}}$. 

Let us consider a vector $\vect{y}' \not\in \hat{\vect{\mathfrak{y}}}_{\ell_H,\credal_{BR}}$, and the indices 
$$\mathcal{I}_{\vect{y}' \neq \hat{\vect{\mathfrak{y}}}}=\{i : \hat{\mathfrak{y}}_{i,\ell_H,\credal_{BR}}\neq *,~\hat{\mathfrak{y}}_{i,\ell_H,\credal_{BR}} \neq \vect{y}'_i\}$$
on which they necessarily differ (as we can always set the labels for which $\hat{\mathfrak{y}}_{i,\ell_H,\credal_{BR}}=*$ to be equal to $y'_i$). By Proposition~\ref{prop:expcond}, we have that 
$$\hat{\vect{\mathfrak{y}}}_{\ell_H,\credal_{BR}} \succ_M  \vect{y}' \iff \inf_{P \in \credal}  \sum_{i \in \mathcal{I}_{y' \neq \hat{y}^*}}   P(Y_i=\hat{\mathfrak{y}}_i) > \frac{|\mathcal{I}|}{2}$$
and since we have that $\hat{\mathfrak{y}}_i=1 \Rightarrow \underline{P}(Y_i=1) > 0.5$ and $\hat{\mathfrak{y}}_i=0 \Rightarrow \underline{P}(Y_i=0) > 0.5$, the right hand side inequality is satisfied. Hence, we can show that any vector outside $\hat{\vect{\mathfrak{y}}}_{\ell_H,\credal_{BR}}$ is maximally dominated by another vector in $\hat{\vect{\mathfrak{y}}}_{\ell_H,\credal_{BR}}$, meaning that $\hat{\vect{\mathfrak{y}}}_{\ell_H,\credal_{BR}} \supseteq \hat{\mathbb{Y}}^{M}_{\ell_H,\credal_{BR}}$. Combined with the fact that $\hat{\vect{\mathfrak{y}}}_{\ell_H,\credal_{BR}} \subseteq \hat{\mathbb{Y}}^{M}_{\ell_H,\credal_{BR}}$, this finishes the proof. 
\end{proof}

\begin{proof}[\bf Proof of Proposition~\ref{prop:gammaminmax}]
Let us prove first how get the prediction of each decision criterions, by harnessing the facts that the Hamming loss is decomposable and that $\mathbb{E}\left[\ell_{H}(\vect{y}, \cdot) \right]=-\sum_{i} P(Y_i=y_i)$.
\begin{enumerate}
	\item {$\Gamma$-\uppercase{m}\textsc{inimax}.---} as the Hamming
	 loss is decomposable we can easily reduce Equation \eqref{eq:gminmax}
	 as follows:
	 \begin{equation}
		\underset{\vect{y} \in \mathscr{Y}}{\arg \min}~
			\overline{\mathbb{E}}_\credal\left[ 
				\ell_{H}(\vect{y}, \cdot) \right] \iff
			\underset{\vect{y} \in \mathscr{Y}}{\arg \max}~
				\inf_{P \in \credal} \sum_{i=1}^m P(Y_i=\hat{y}_i)
	\end{equation}
	by using a probability set $\credal_{BR}$
	, we have
	\begin{equation}
		\hat{\vect{y}}^{\Gamma_{\max}}=\underset{\vect{y} \in \mathscr{Y}}{\arg \max}~
		\sum_{i=1}^m  \underline{P}(Y_i=y_i) \iff
		\hat{\vect{y}}^{\Gamma_{\max}}=\begin{cases}
			1 & if~ \underline{P}(Y_i=1) > \underline{P}(Y_i=0) \\
			0 & if~ \underline{P}(Y_i=1) < \underline{P}(Y_i=0) \\
			* & otherwise
		\end{cases}
	\end{equation}
	where $*$ can here be replaced by $0$ or $1$, as all those predictions are deemed indifferent by the \uppercase{m}\textsc{inimax} principle. 
	\item {$\Gamma$-\uppercase{m}\textsc{inimin}.---} in the same way
	as previously, we easily have :
	\begin{equation}
		\hat{\vect{y}}^{\Gamma_{\min}}=\underset{\vect{y} \in \mathscr{Y}}{\arg \max}~
		\sum_{i=1}^m  \overline{P}(Y_i=y_i) \iff
		\hat{\vect{y}}^{\Gamma_{\min}}= \begin{cases}
			1 & if~ \overline{P}(Y_i=1) > \overline{P}(Y_i=0) \\
			0 & if~ \overline{P}(Y_i=1) < \overline{P}(Y_i=0) \\
			* & otherwise
		\end{cases}
	\end{equation}
	where $*$ is to be understood as in the previous case. 
\end{enumerate}
By using the fact the the lower and upper probabilities are dual, 
$\overline{P}(Y_i=1) = 1 - \underline{P}(Y_i=0)$, it is easy to see 
that the criteria to choose the $\hat{\vect{y}}^{\Gamma_{\min}}$ is equal to 
$\hat{\vect{y}}^{\Gamma_{\max}}$, therefore, 
$\hat{\vect{y}}^{\Gamma_{\max}}_{\ell_H,\credal_{BR}} = 
\hat{\vect{y}}^{\Gamma_{\min}}_{\ell_H,\credal_{BR}}$.
\end{proof}

%% file: supplementary_results.tex
\newpage
\section{Complementary experimental results} 
\label{app:supresults}
\subsection{Missing labels}\label{app:missinglabels}
\vspace*{-6mm}
\begin{figure}[!th]
	\centering
	\resizebox{0.45\textwidth}{!}{%
	  \includegraphics{images/missing/legends}
	}\vspace{-2mm}\qquad%
	\subfigure[\sc Emotions]{
	   \includegraphics[width=0.32\linewidth]{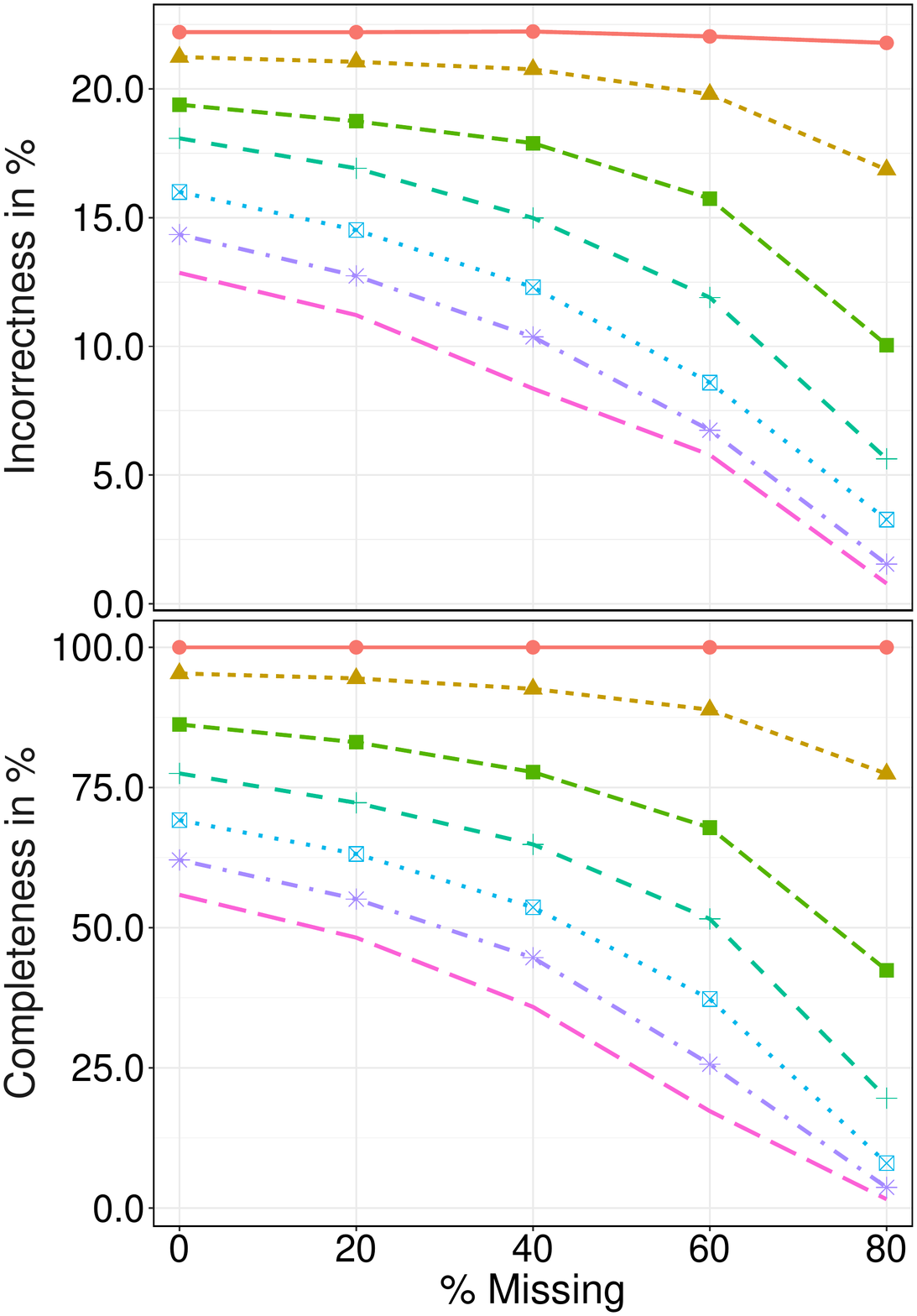}
	}%
	\subfigure[\sc Scene]{
		\includegraphics[width=0.32\linewidth]{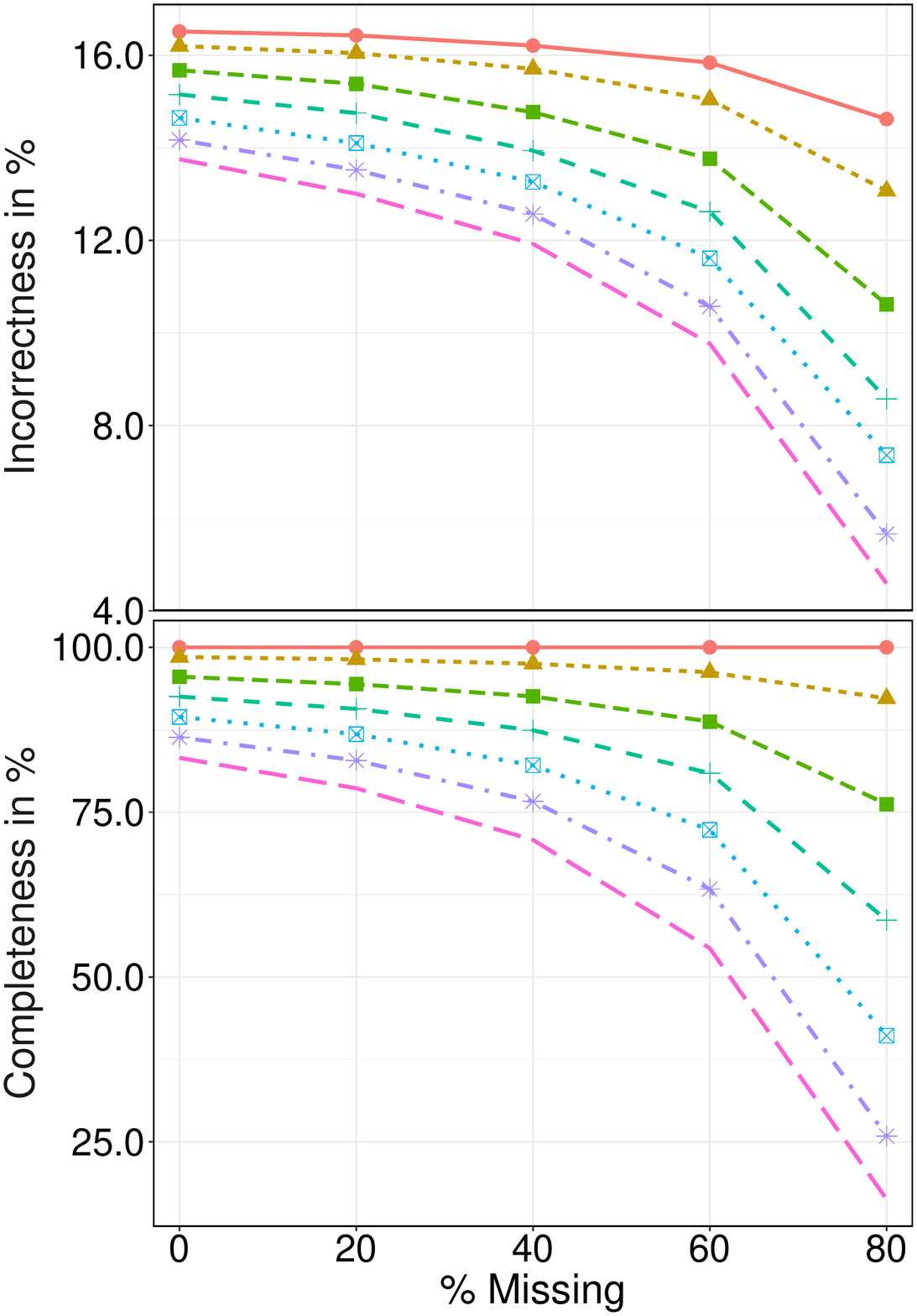}
	}%
	\subfigure[\sc Yeast]{
		\includegraphics[width=0.32\linewidth]{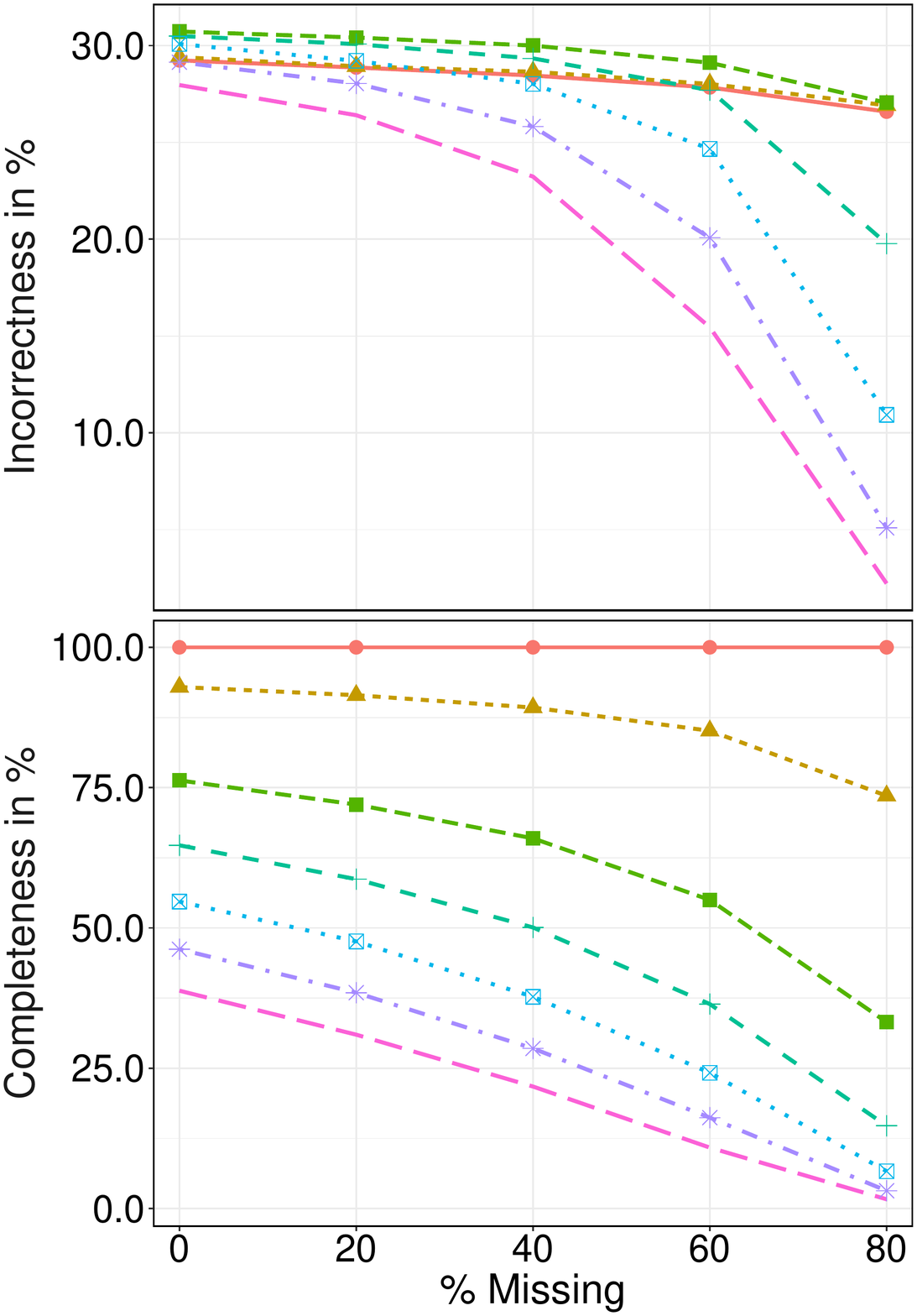}
	}%
	\vspace*{-2mm}
	\caption{\textbf{Missing labels}. Evolution of the average incorrectness (top) and completeness (bottom) for each level of imprecision (a curve for each one) and a discretization $z\!=\!6$, and with respect to different percentages of missing labels (x-axis).}\label{fig:expcphmissing}%
\end{figure}
\subsection{Noise Reversing}\label{app:noisereversing}
\begin{figure}[!th]
	\vspace*{-4mm}\centering
	\resizebox{0.45\textwidth}{!}{%
	  \includegraphics{images/missing/legends}
	}\vspace{-2mm}\qquad%
	\subfigure[\sc Emotions]{
	   \includegraphics[width=0.32\linewidth]{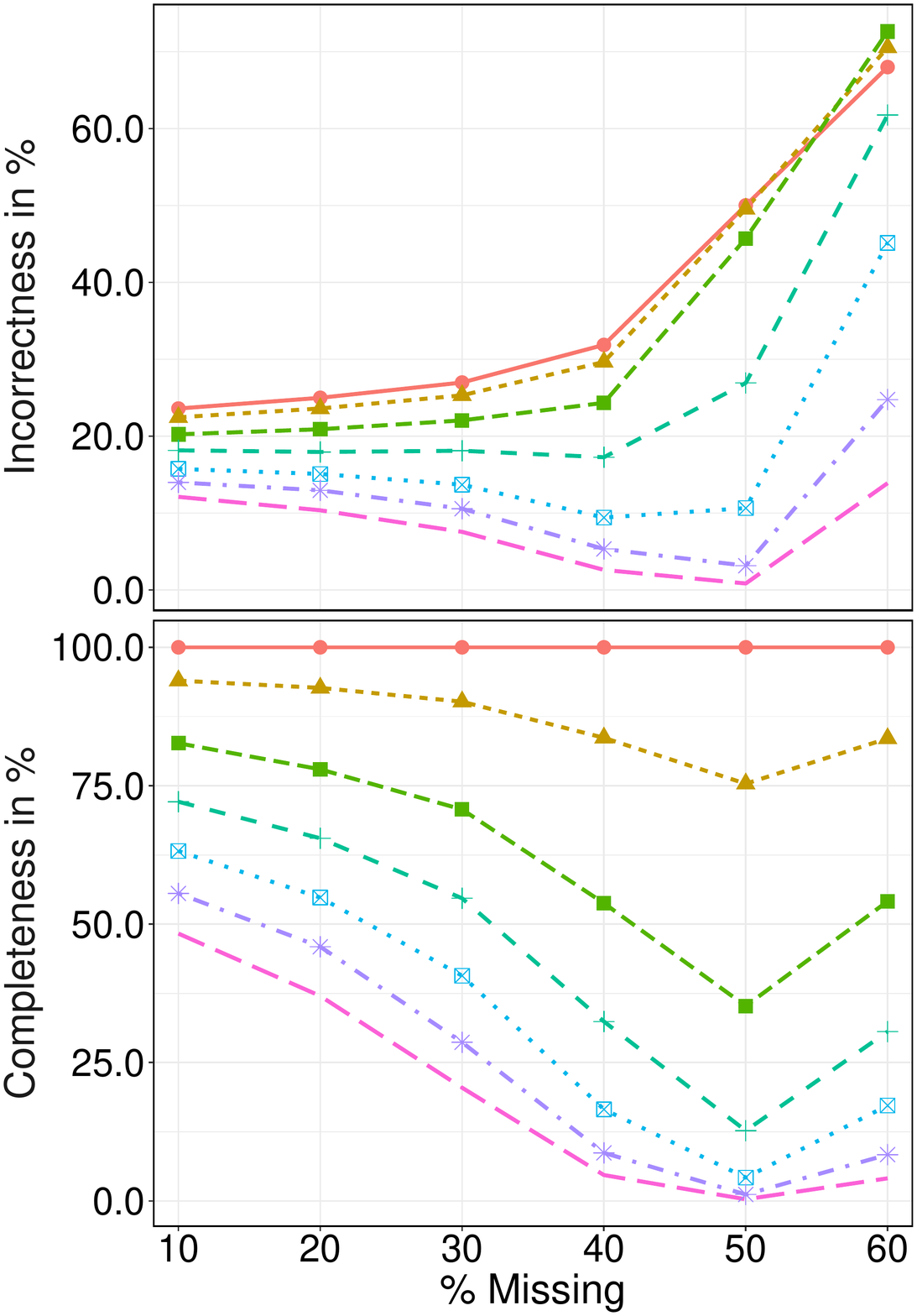}
	}%
	\subfigure[\sc Scene]{
		\includegraphics[width=0.32\linewidth]{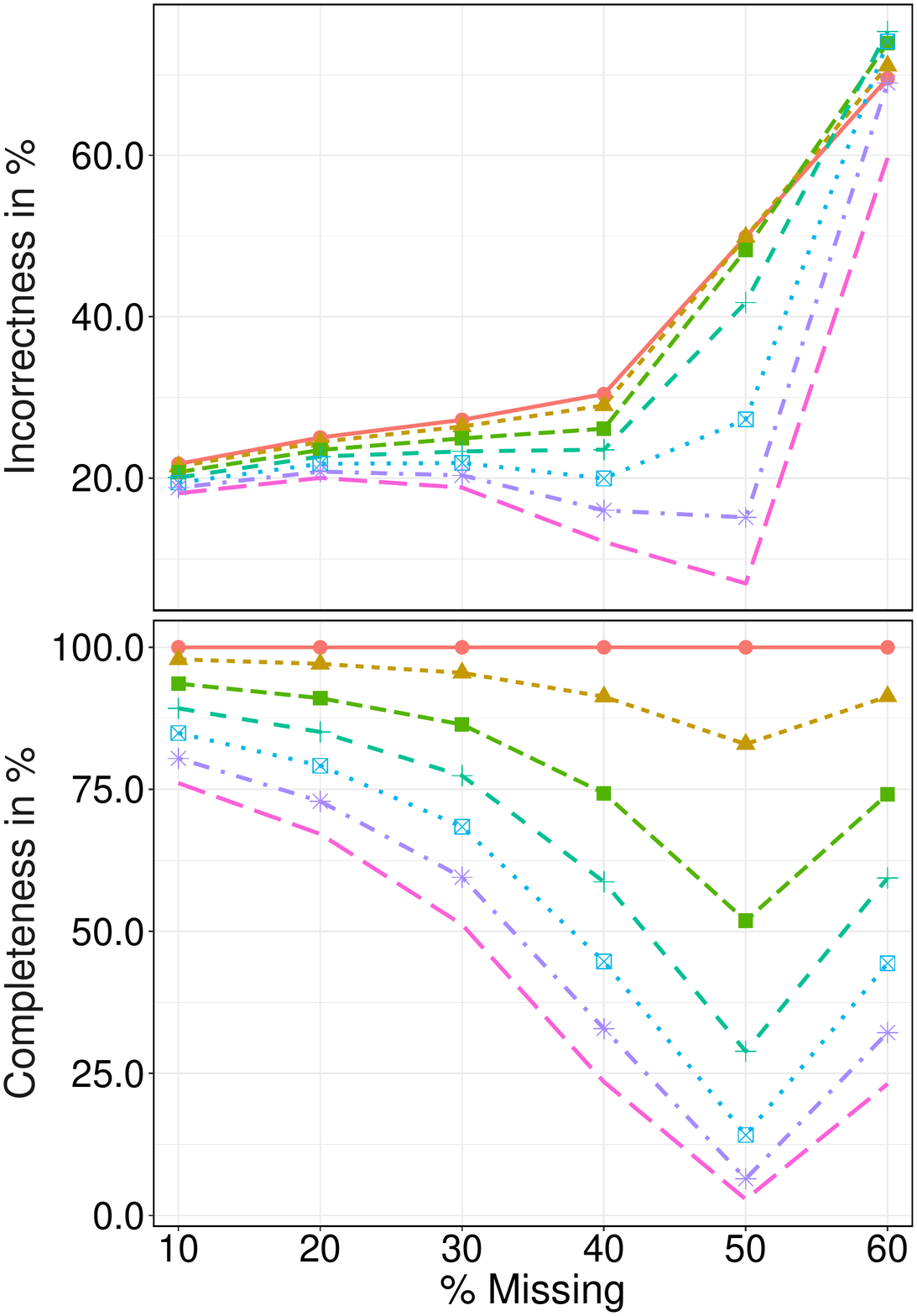}
	}%
	\subfigure[\sc Yeast]{
		\includegraphics[width=0.32\linewidth]{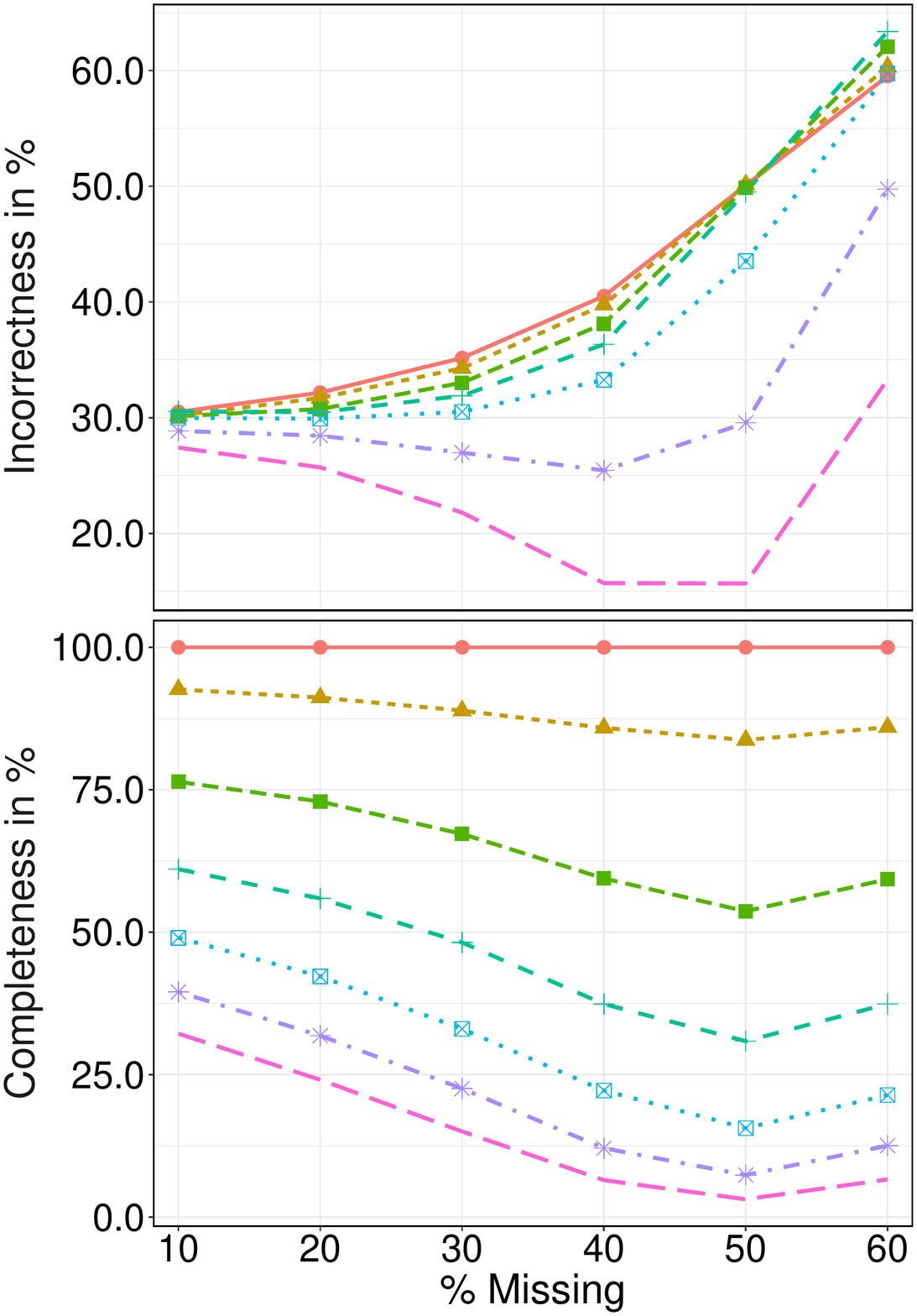}
	}%
	\vspace*{-2mm}
	\caption{\textbf{Noise-Reversing}. Evolution of the average incorrectness (top) and completeness (bottom) for each level of imprecision (a curve for each one) and a discretization $z\!=\!6$, and with respect to different percentages of noisy labels (x-axis).}\label{fig:expcphnoisereverse}%
\end{figure}
\subsection{Noise Flipping}\label{app:noiseflipping}
\begin{figure}[!th]
	\vspace*{-4mm}\centering
	\resizebox{0.45\textwidth}{!}{%
	  \includegraphics{images/missing/legends}
	}\vspace{-2mm}\qquad%
	\subfigure[\sc Emotions]{
	   \includegraphics[width=0.32\linewidth]{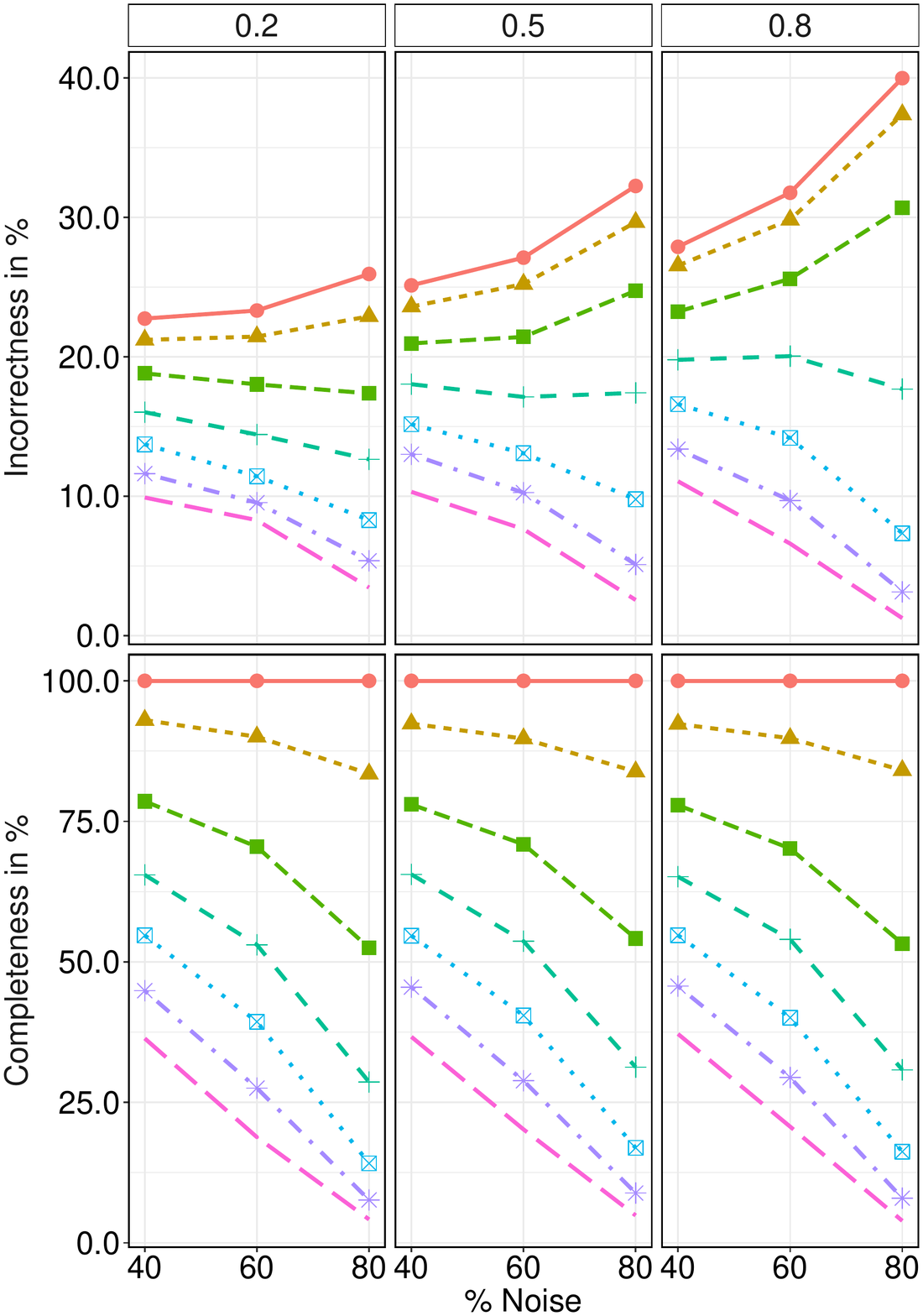}
	}%
	\subfigure[\sc Scene]{
		\includegraphics[width=0.32\linewidth]{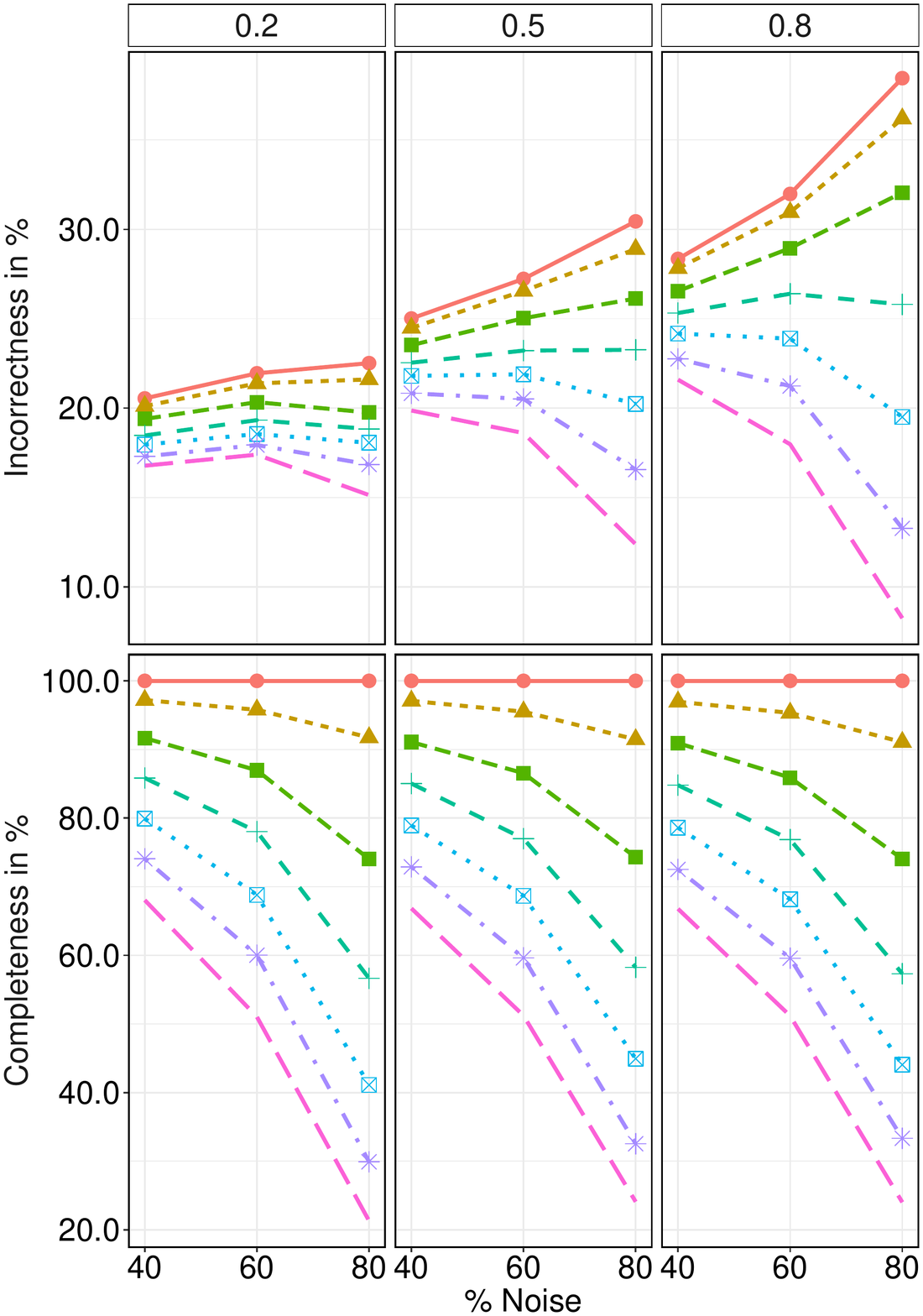}
	}%
	\subfigure[\sc Yeast]{
		\includegraphics[width=0.32\linewidth]{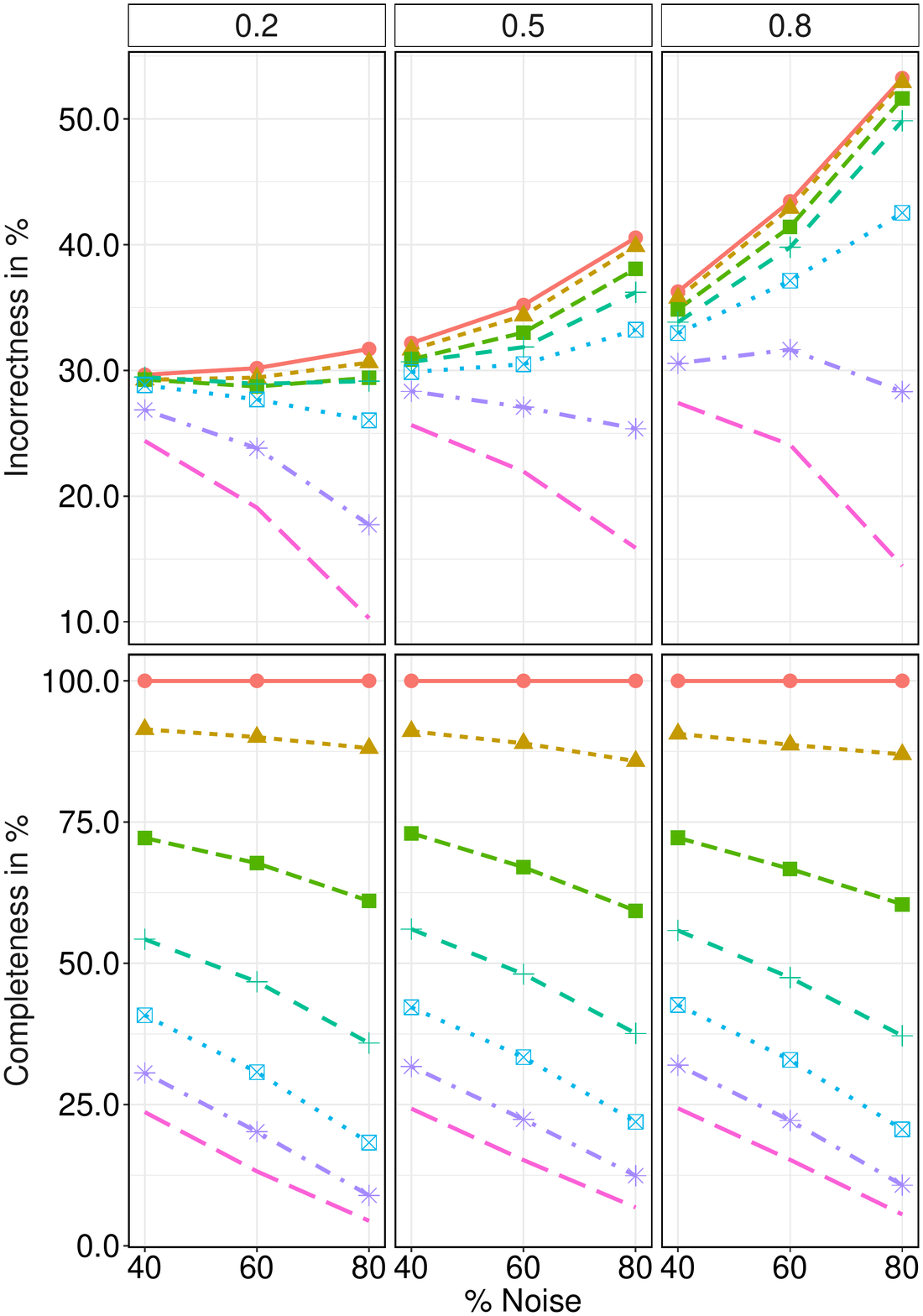}
	}%
	\vspace*{-2mm}
	\caption{\textbf{Noise-Flipping}. Evolution of the average incorrectness (top) and completeness (bottom) for each level of imprecision (a curve for each one), a level of discretization $z\!=\!6$, and three different probabilities $\beta\!=\!0.2$ (left), $\beta\!=\!0.5$ (middle) and $\beta\!=\!0.8$ (right) of replacing the selected label with a $1$, and with respect to different percentages of noisy labels (x-axis).}\label{fig:expcphnoiseflipping}%
\end{figure}

%% file: supplementary_reject.tex
\subsection{Supplementary results - Rejection Option}\label{app:rejection}
\begin{figure}[!th]
	\vspace*{-4mm}\centering
	\subfigure{
		\includegraphics[width=\linewidth]{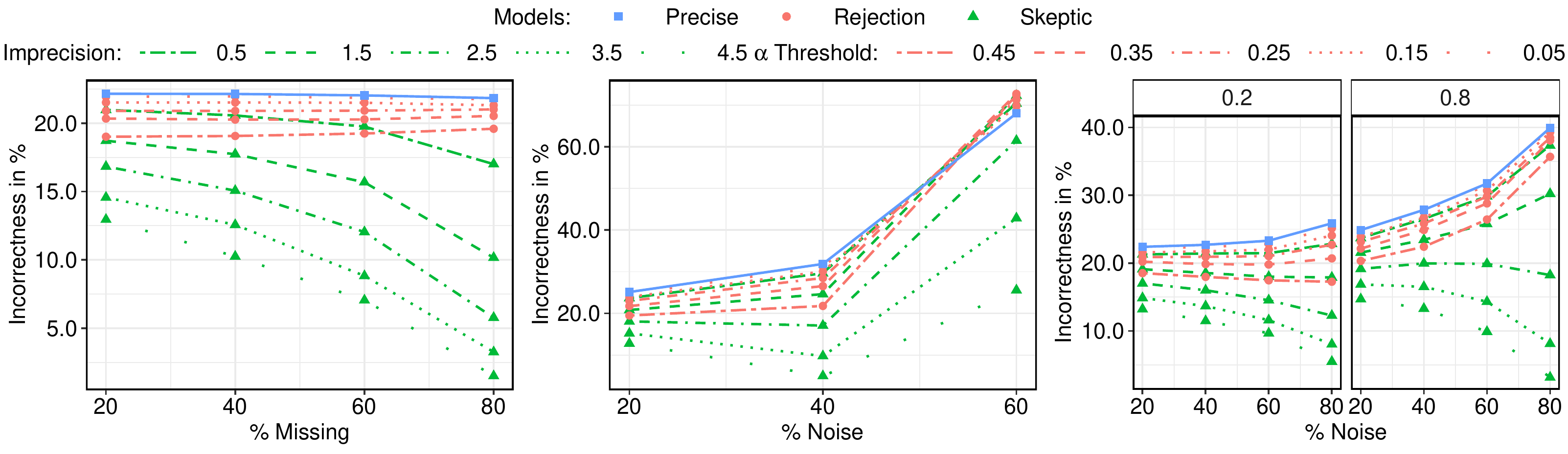}
	}\vspace{-4mm}\qquad%
	\subfigure{
		\includegraphics[width=\linewidth]{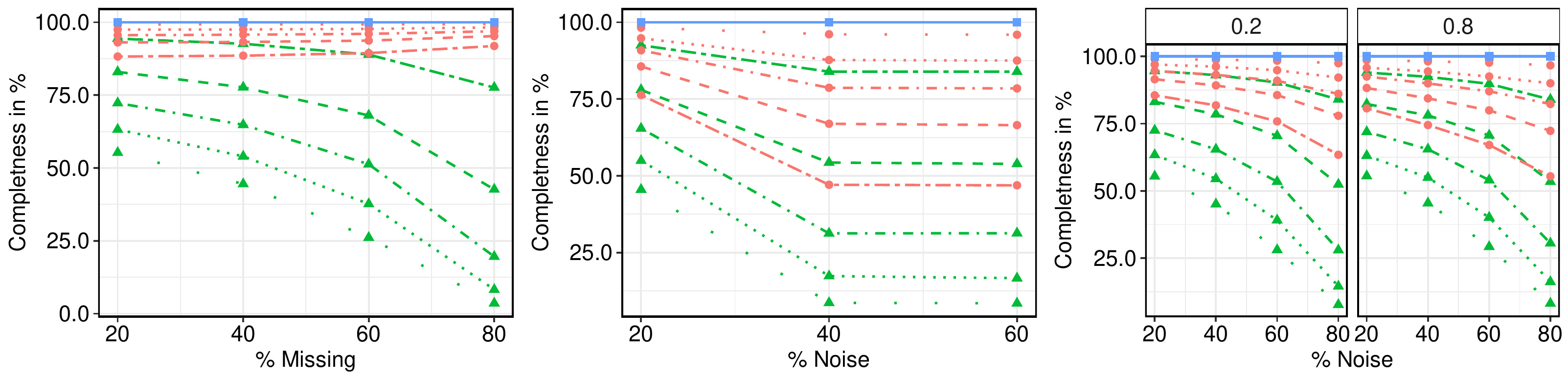}
	}\vspace{-4mm}
	\caption{{\sc Emotions:~}{\bf Rejection vs. Skeptic.} Incorrectness (top) and Completeness (bottom) evolution for different levels of imprecision and discretization $z=6$, and with respect to different percentages of noisy or missing labels (x-axis).}
\end{figure}

\begin{figure}[!th]
	\centering
	\subfigure{
		\includegraphics[width=\linewidth]{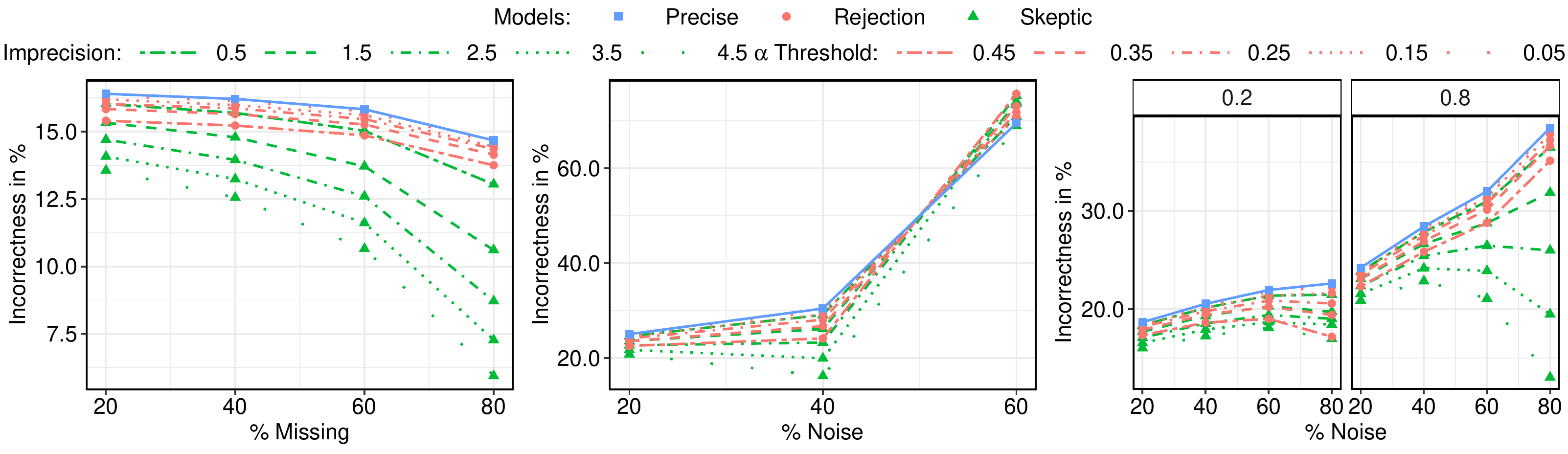}
	}\vspace{-2mm}\qquad%
	\subfigure{
		\includegraphics[width=\linewidth]{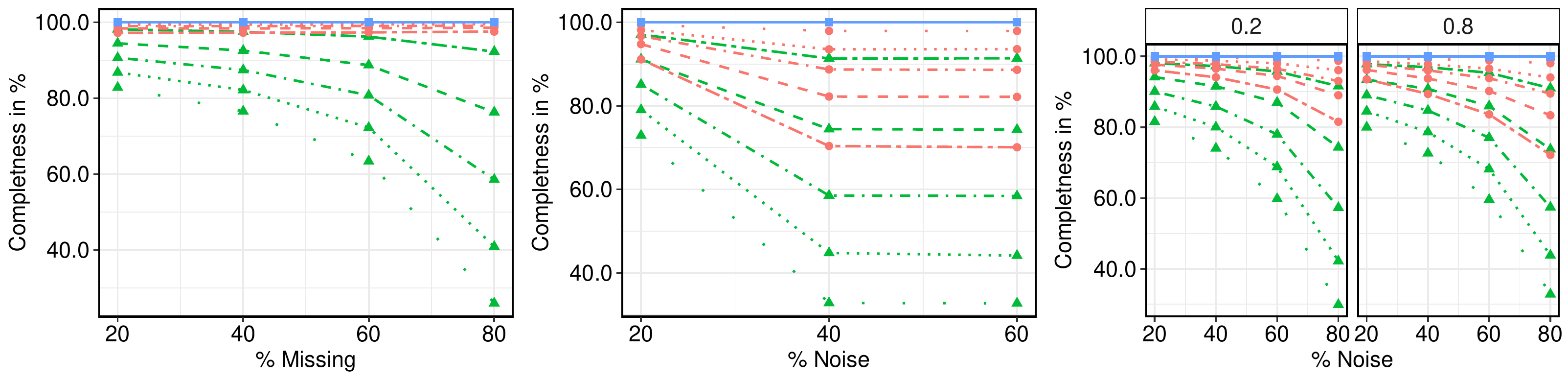}
	}\vspace{-1mm}
	\caption{{\sc Scene:~}{\bf Rejection vs. Skeptic.} Incorrectness (top) and Completeness (bottom) evolution for different levels of imprecision and discretization $z=6$, and with respect to different percentages of noisy or missing labels (x-axis).}
\end{figure}

\begin{figure}[!th]
	\centering
	\subfigure{
		\includegraphics[width=\linewidth]{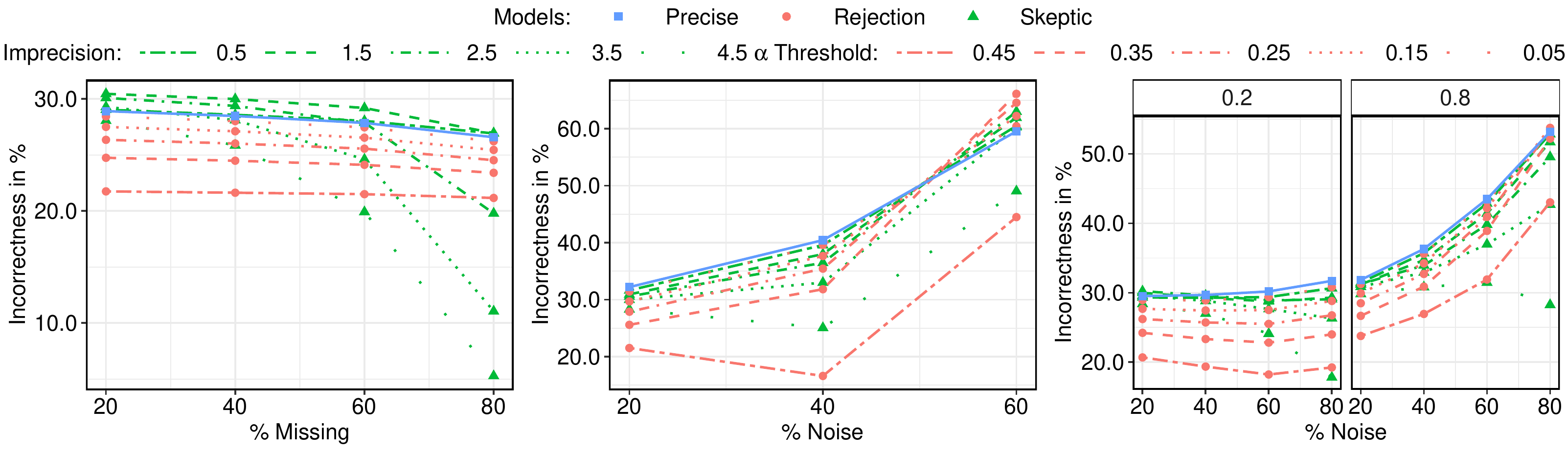}
	}\vspace{-2mm}\qquad%
	\subfigure{
		\includegraphics[width=\linewidth]{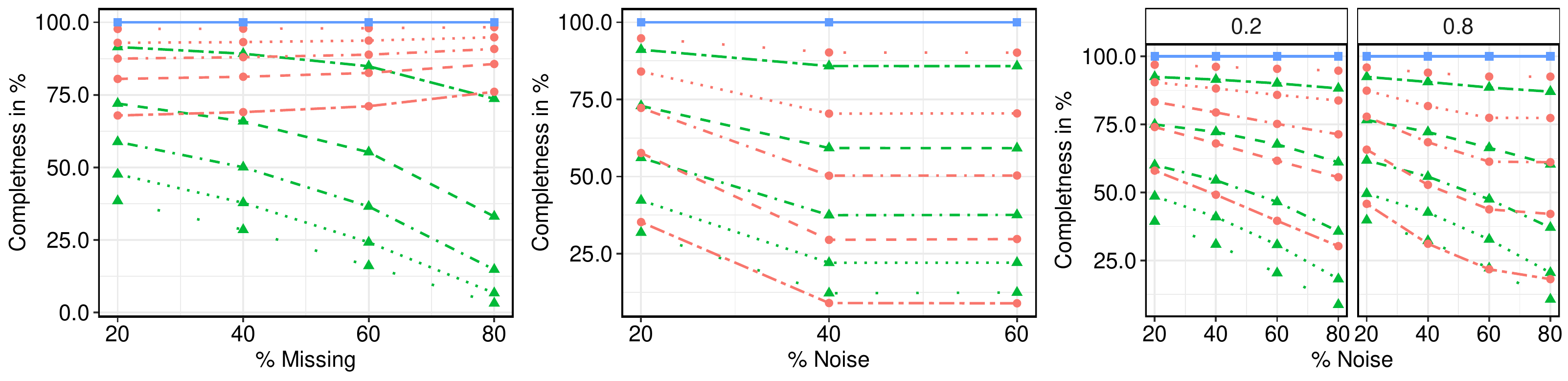}
	}\vspace{-1mm}
	\caption{{\sc Yeast:~}{\bf Rejection vs. Skeptic.} Incorrectness (top) and Completeness (bottom) evolution for different levels of imprecision and discretization $z=6$, and with respect to different percentages of noisy or missing labels (x-axis).}
\end{figure}

%% file: supplementary_resampling.tex
\clearpage
\newcommand\wimg{0.99}
\section{Supplementary Downsampling}\vspace{-5mm}
\label{app:suppresampling}
\begin{figure}[!th]
	\centering
	\subfigure{
		\includegraphics[width=\wimg\linewidth]{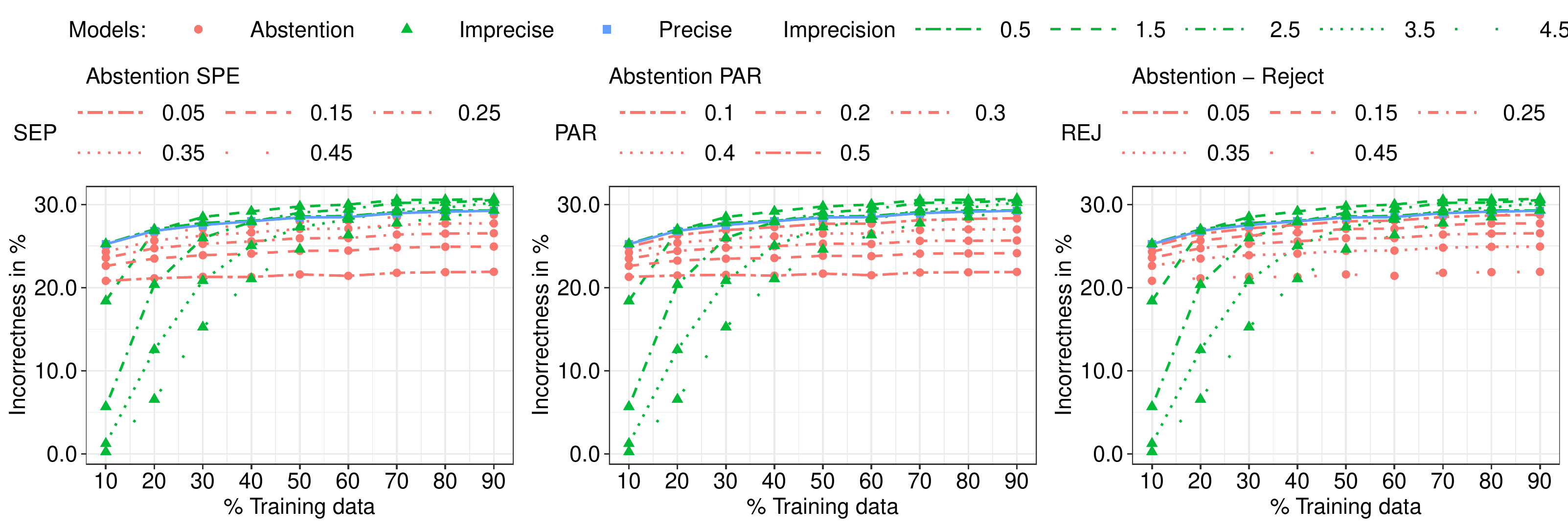}
	}\vspace{-3mm}\qquad%
	\subfigure{
		\includegraphics[width=\wimg\linewidth]{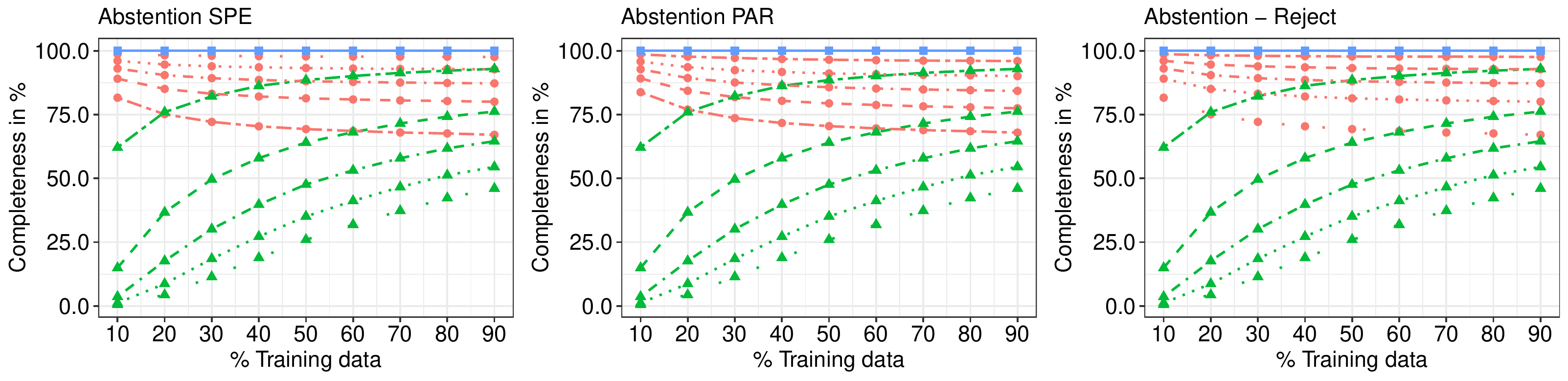}
	}\vspace{-4mm}
	\caption{{\bf Yeast - Downsampling - Naive credal classifier.} Incorrectness and Completeness (y-axis) evolution of all skeptical approachs with respect to different percentages of training data sets (x-axis).}
\end{figure}
\begin{figure}[!th]\vspace*{-4mm}
	\centering
	\subfigure{
		\includegraphics[width=\wimg\linewidth]{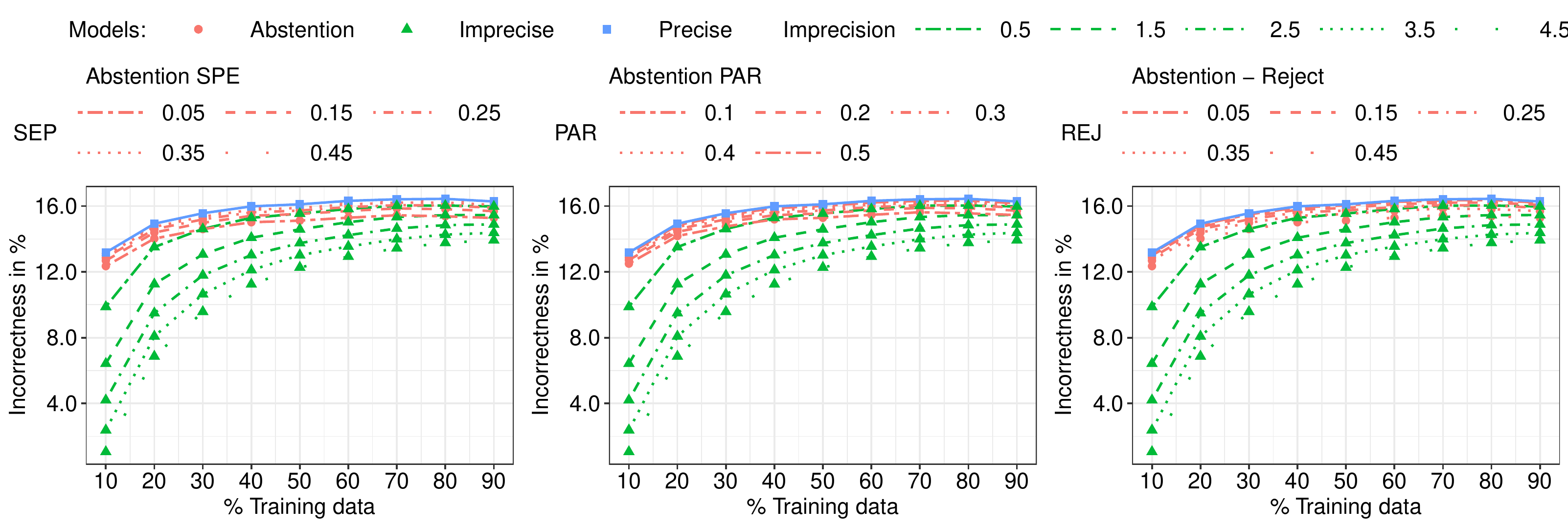}
	}\vspace{-3mm}\qquad%
	\subfigure{
		\includegraphics[width=\wimg\linewidth]{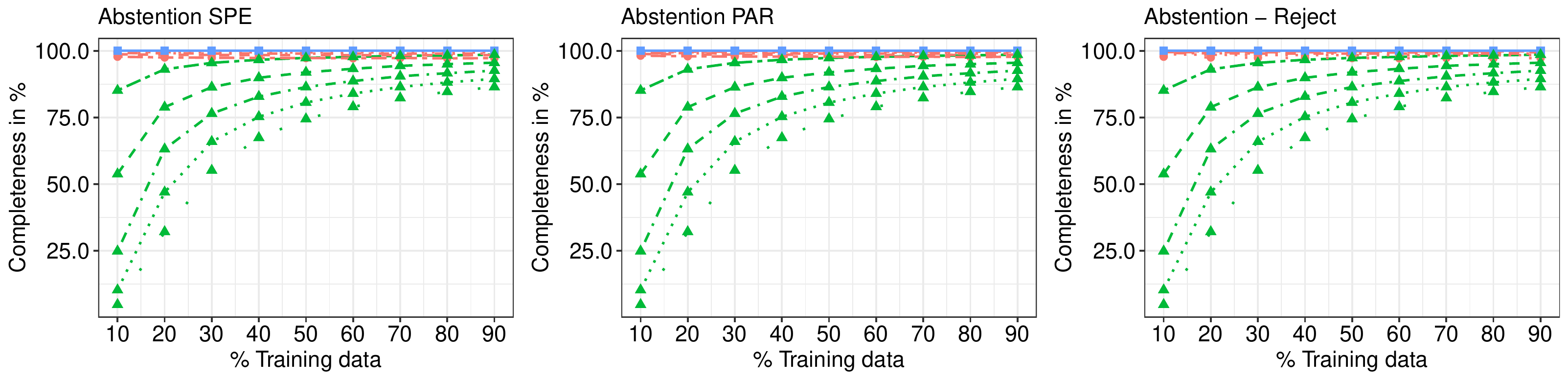}
	}\vspace{-4mm}
	\caption{{\bf Scene - Downsampling - Naive credal classifier.} Incorrectness and Completeness (y-axis) evolution of all skeptical approachs with respect to different percentages of training data sets (x-axis).}
\end{figure}
\begin{figure}[!th]
	\centering
	\subfigure{
		\includegraphics[width=\linewidth]{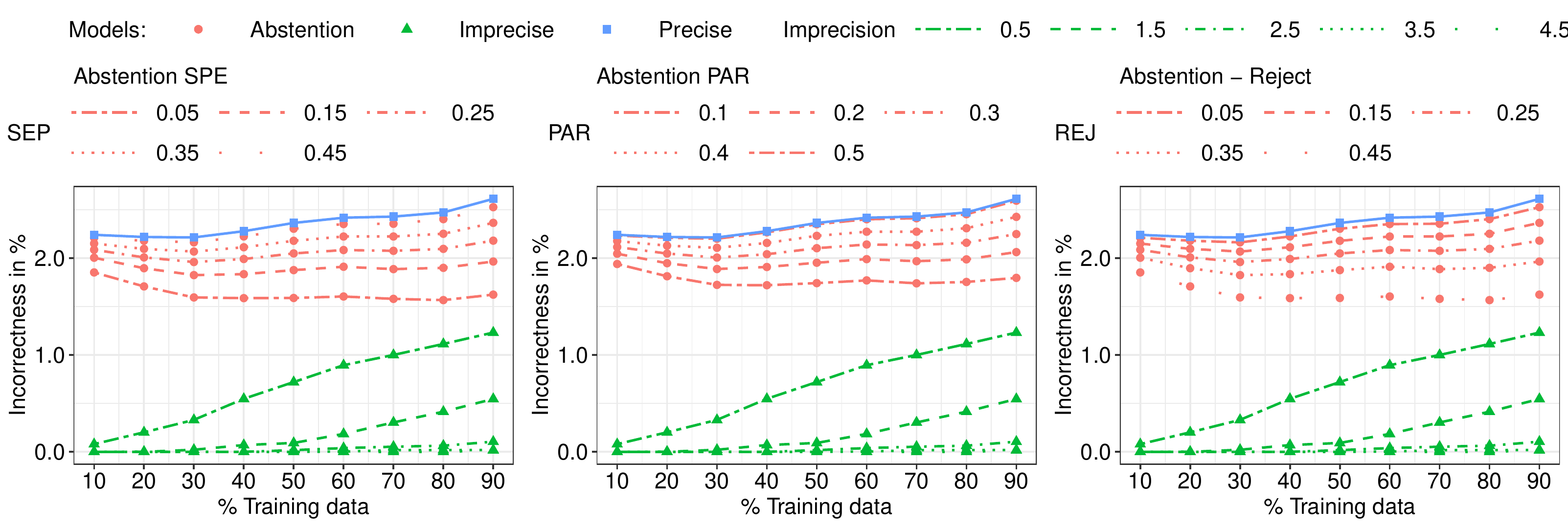}
	}\vspace{-3mm}\qquad%
	\subfigure{
		\includegraphics[width=\linewidth]{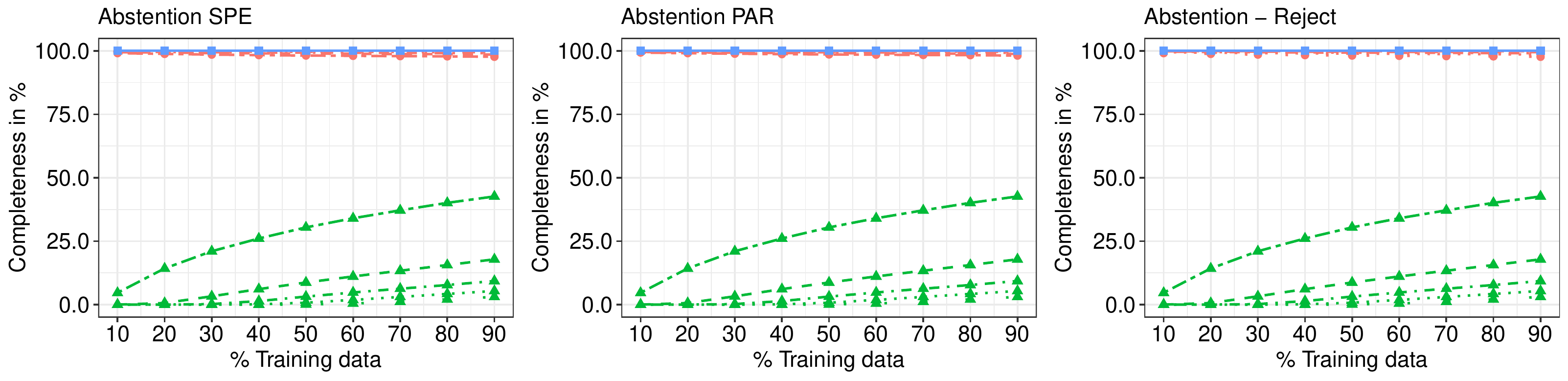}
	}\vspace{-4mm}
	\caption{{\bf Medical - Downsampling - Naive credal classifier.} Incorrectness and Completeness (y-axis) evolution of all skeptical approaches with respect to different percentages of training data sets (x-axis).}
\end{figure}
\begin{figure}[!th]
	\centering
	\subfigure{
		\includegraphics[width=\linewidth]{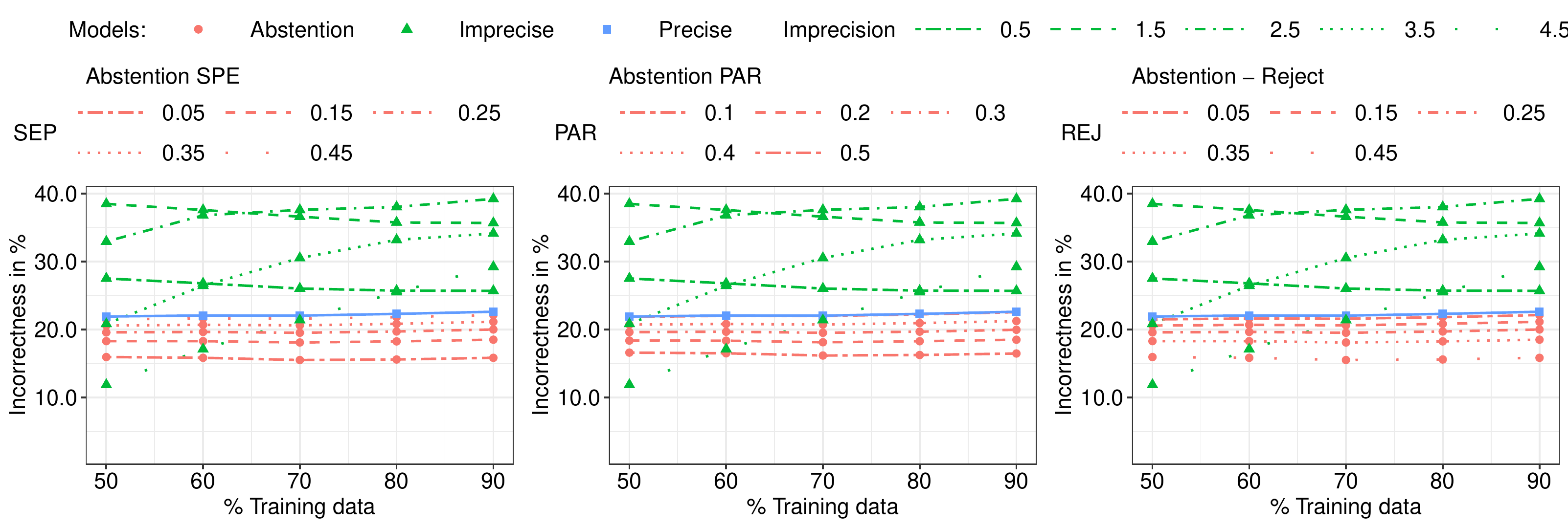}
	}\vspace{-3mm}\qquad%
	\subfigure{
		\includegraphics[width=\linewidth]{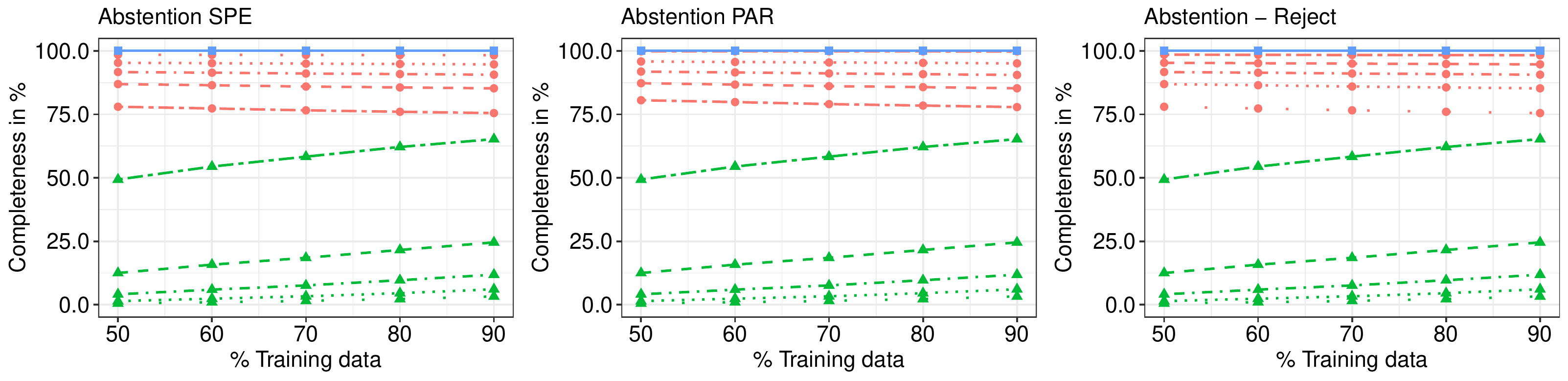}
	}\vspace{-4mm}
	\caption{{\bf CAL500 - Downsampling - Naive credal classifier.} Incorrectness and Completeness (y-axis) evolution of all skeptical approaches with respect to different percentages of training data sets (x-axis).}
\end{figure}
\begin{figure}[!th]
	\centering
	\subfigure{
		\includegraphics[width=\linewidth]{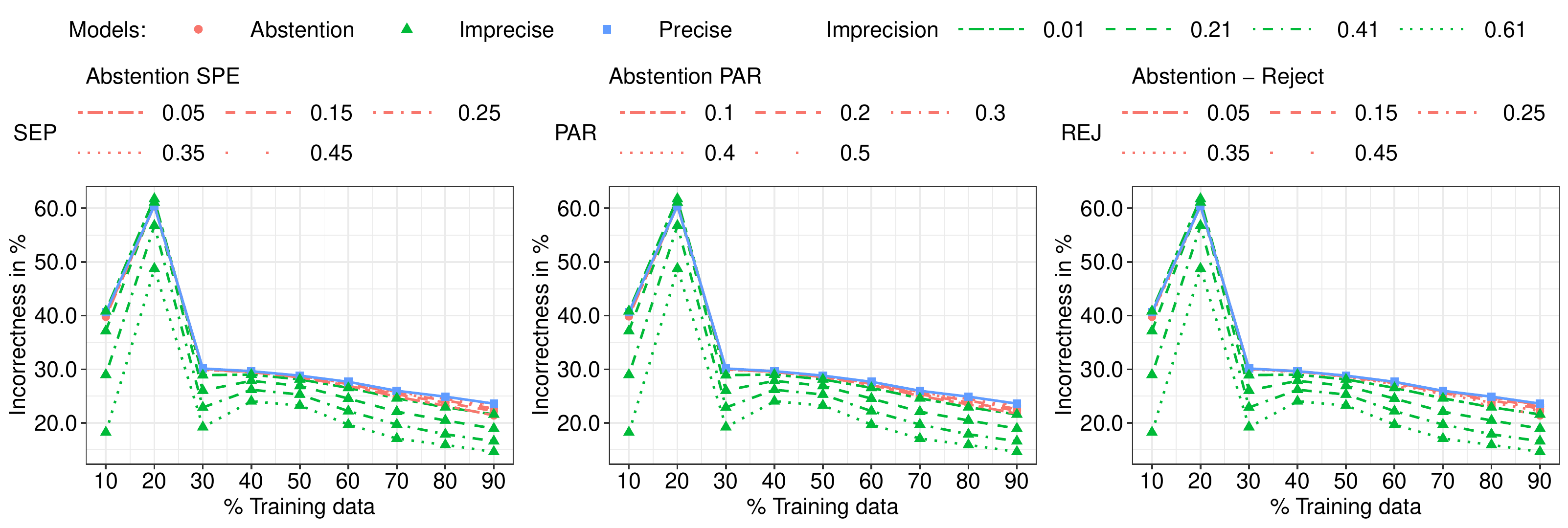}
	}\vspace{-3mm}\qquad%
	\subfigure{
		\includegraphics[width=\linewidth]{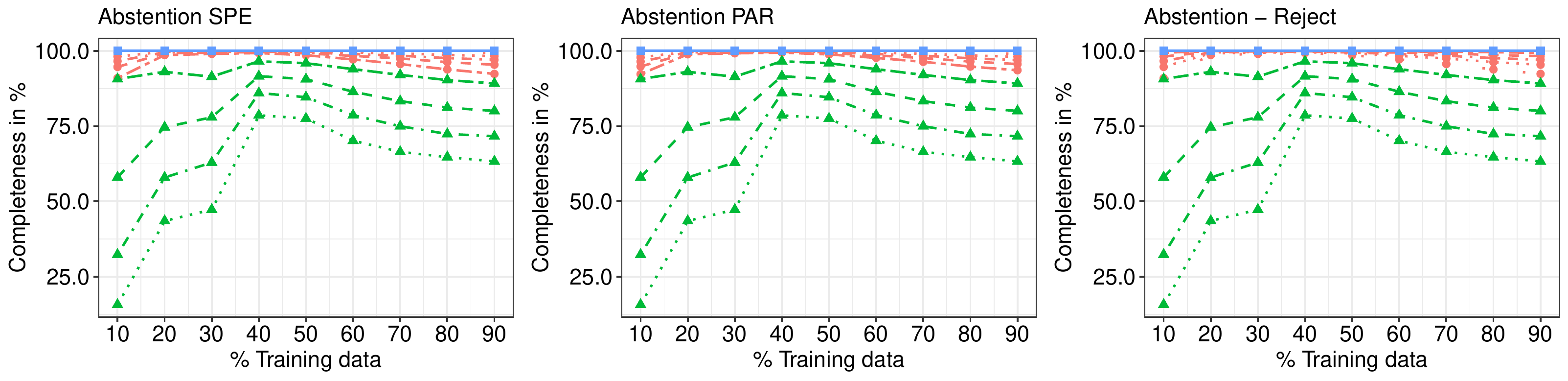}
	}\vspace{-4mm}
	\caption{{\bf Emotions - Downsampling - IQDA.} Incorrectness and Completeness (y-axis) evolution of all skeptical approaches with respect to different percentages of training data sets (x-axis).}
\end{figure}
\begin{figure}[!th]
	\centering
	\subfigure{
		\includegraphics[width=\linewidth]{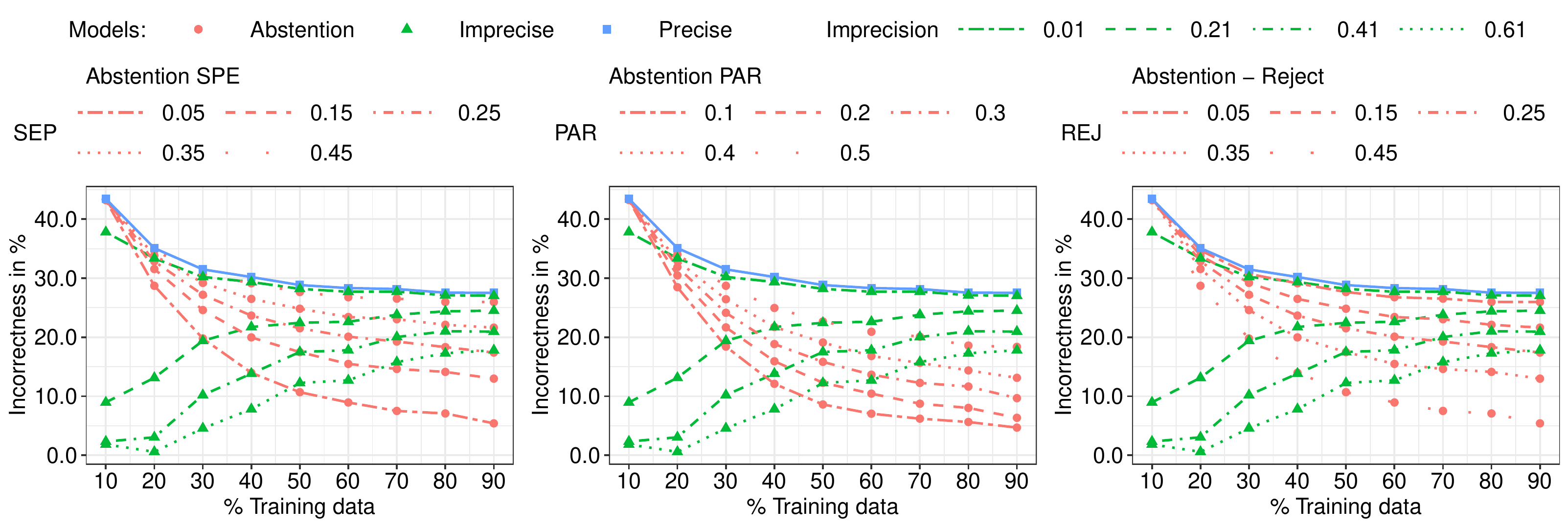}
	}\vspace{-3mm}\qquad%
	\subfigure{
		\includegraphics[width=\linewidth]{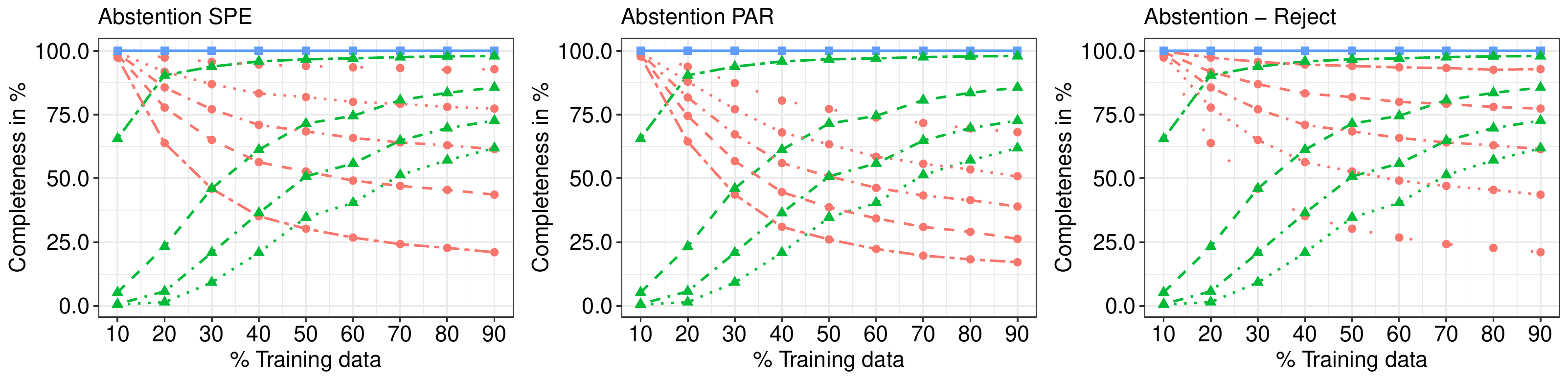}
	}\vspace{-4mm}
	\caption{{\bf Flags - Downsampling - ILDA.} Incorrectness and Completeness (y-axis) evolution of all skeptical approaches with respect to different percentages of training data sets (x-axis).}
\end{figure}
\begin{figure}[!th]
	\centering
	\subfigure{
		\includegraphics[width=\linewidth]{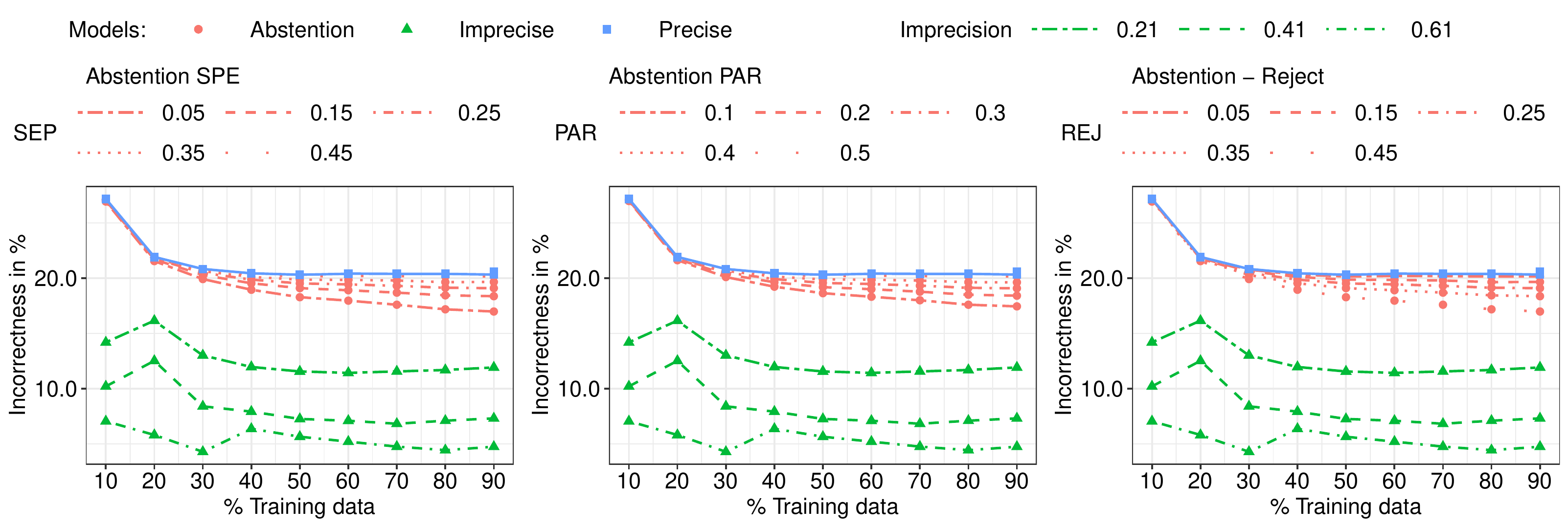}
	}\vspace{-3mm}\qquad%
	\subfigure{
		\includegraphics[width=\linewidth]{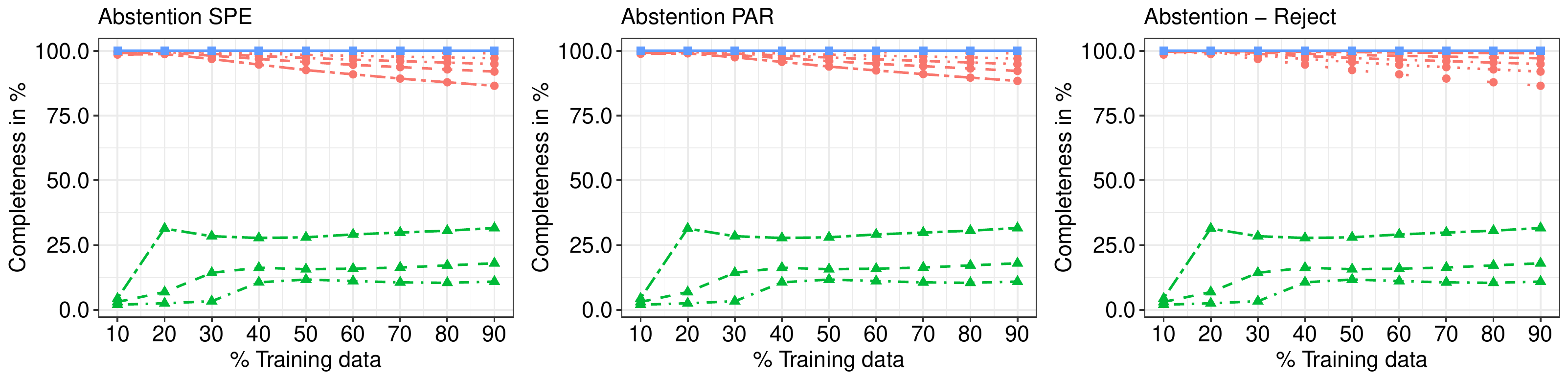}
	}\vspace{-4mm}
	\caption{{\bf Yeast - Downsampling - IQDA.} Incorrectness and Completeness (y-axis) evolution of all skeptical approaches with respect to different percentages of training data sets (x-axis).}
\end{figure}
\begin{figure}[!th]
	\centering
	\subfigure{
		\includegraphics[width=\linewidth]{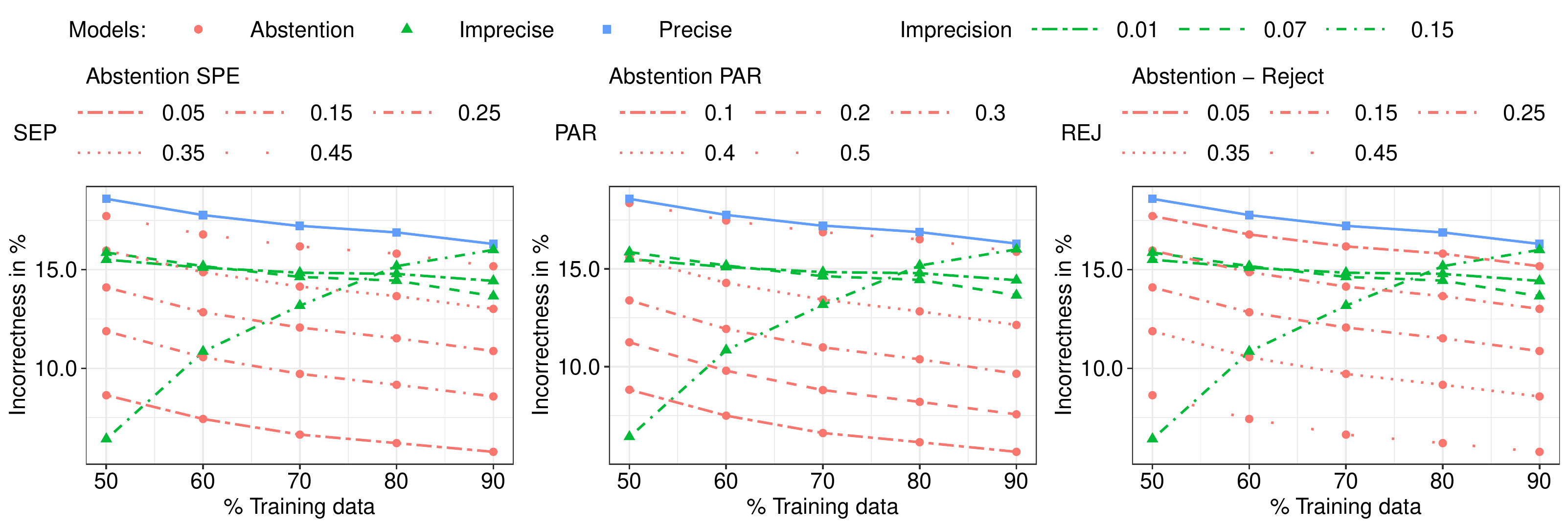}
	}\vspace{-3mm}\qquad%
	\subfigure{
		\includegraphics[width=\linewidth]{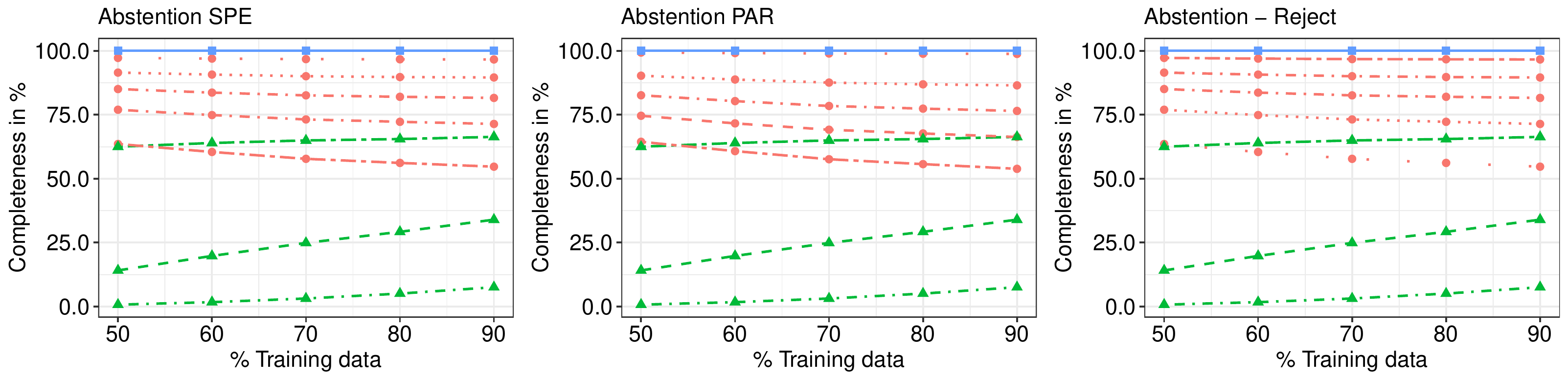}
	}\vspace{-4mm}
	\caption{{\bf CAL500 - Downsampling - ILDA.} Incorrectness and Completeness (y-axis) evolution of all skeptical approaches with respect to different percentages of training data sets (x-axis).}
\end{figure} 
\begin{figure}[!th]
	\centering
	\subfigure{
		\includegraphics[width=\linewidth]{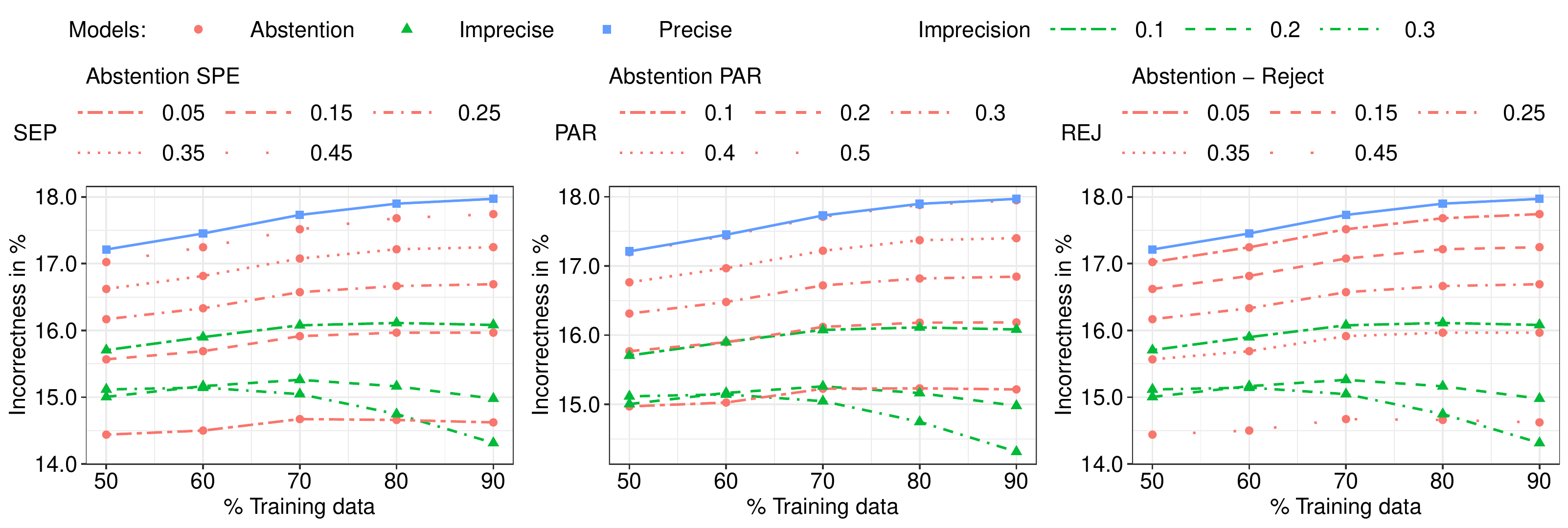}
	}\vspace{-3mm}\qquad%
	\subfigure{
		\includegraphics[width=\linewidth]{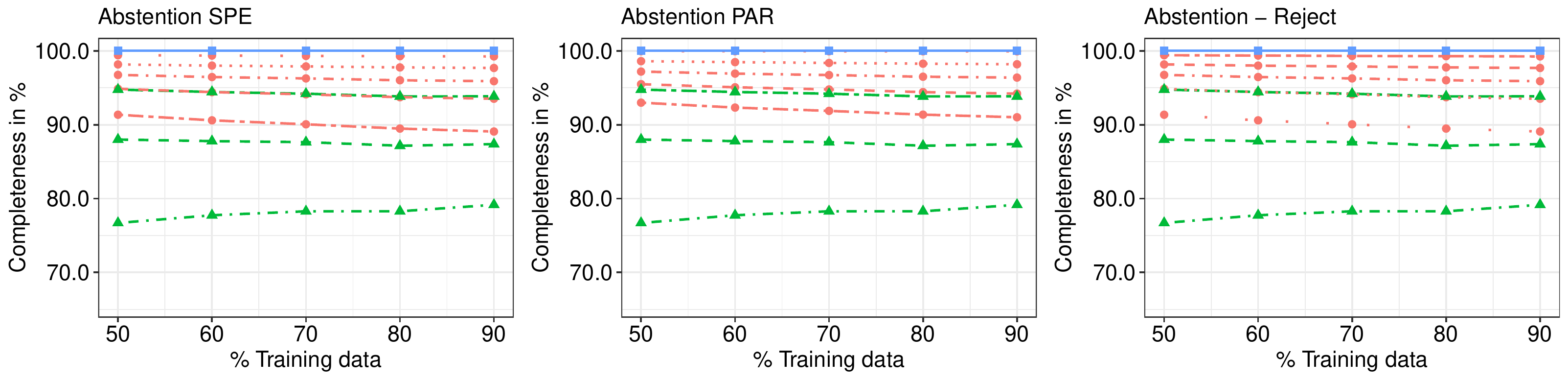}
	}\vspace{-4mm}
	\caption{{\bf CAL500 - Downsampling - CSPN IDM.} Incorrectness and Completeness (y-axis) evolution of all skeptical approaches with respect to different percentages of training data sets (x-axis).}
\end{figure} 
\begin{figure}[!th]
	\centering
	\subfigure{
		\includegraphics[width=\linewidth]{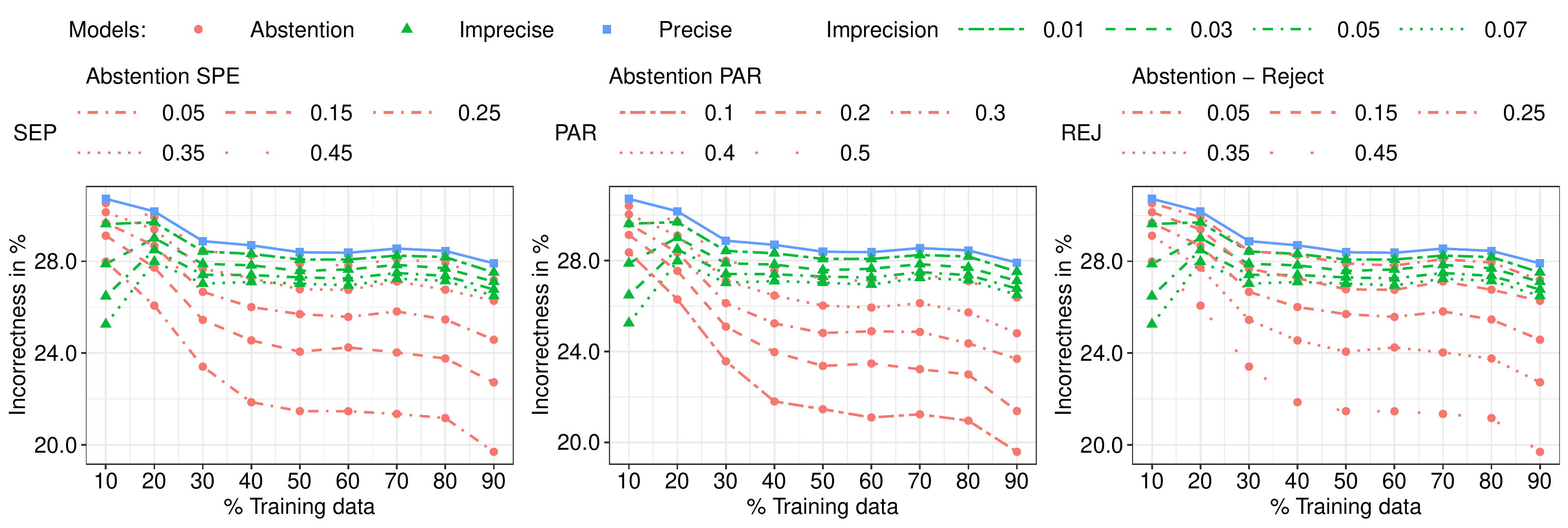}
	}\vspace{-3mm}\qquad%
	\subfigure{
		\includegraphics[width=\linewidth]{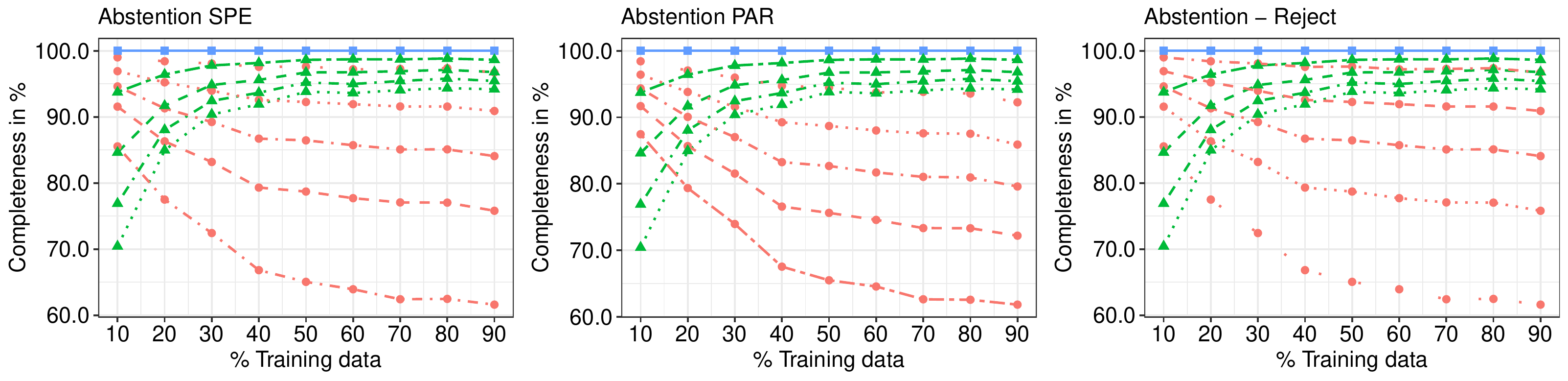}
	}\vspace{-4mm}
	\caption{{\bf Flags - Downsampling - CSPN IDM.} Incorrectness and Completeness (y-axis) evolution of all skeptical approaches with respect to different percentages of training data sets (x-axis).}
\end{figure} 
\begin{figure}[!th]
	\centering
	\subfigure{
		\includegraphics[width=\linewidth]{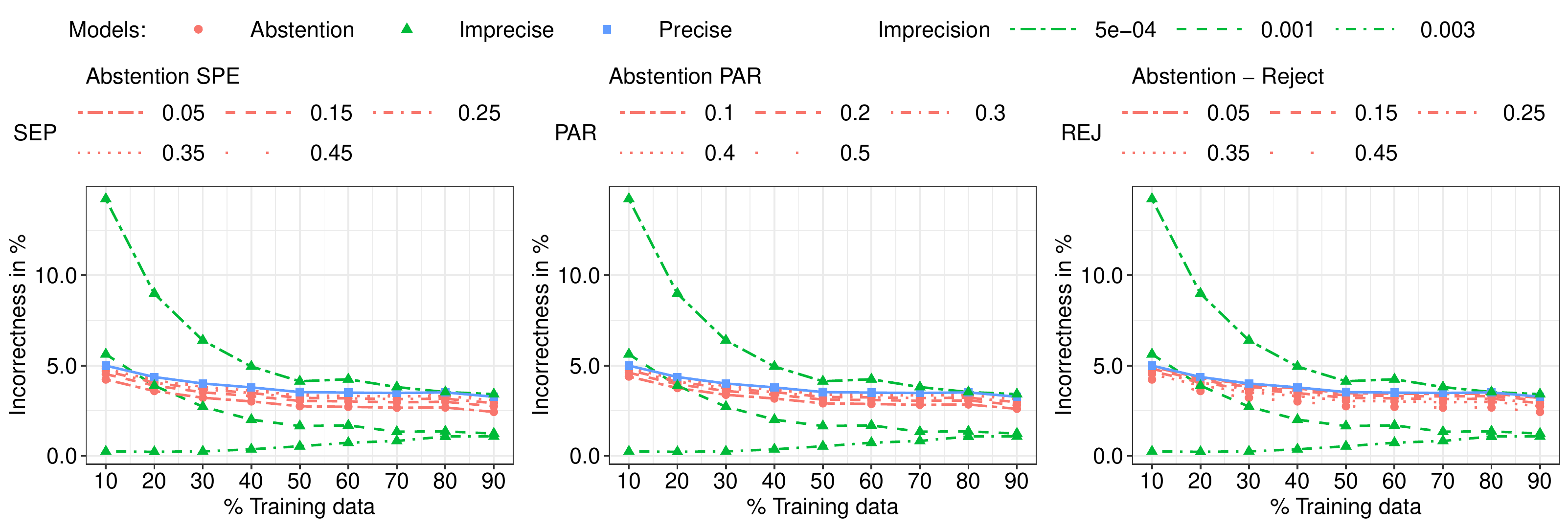}
	}\vspace{-3mm}\qquad%
	\subfigure{
		\includegraphics[width=\linewidth]{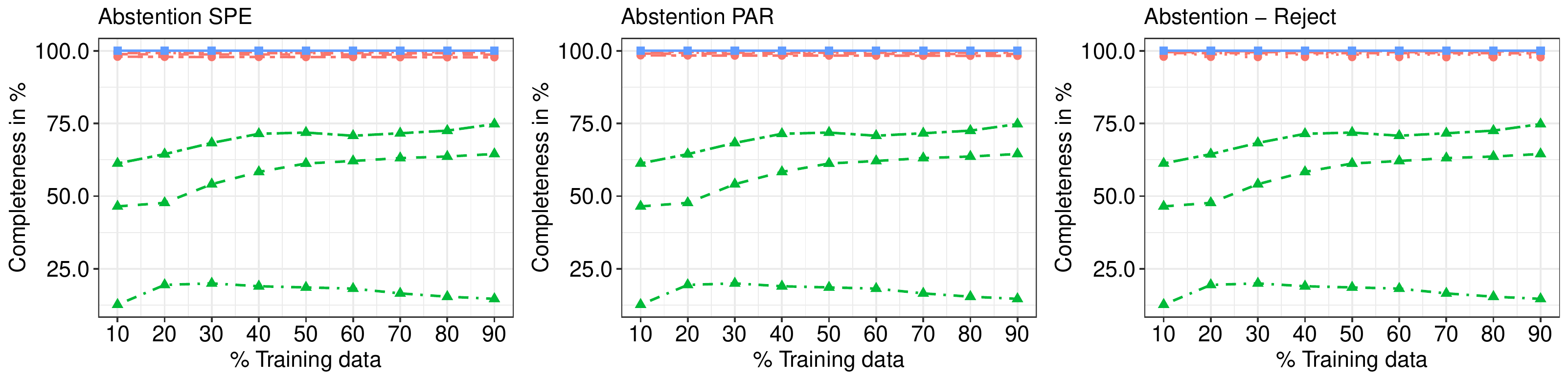}
	}\vspace{-4mm}
	\caption{{\bf Medical - Downsampling - CSPN $\epsilon$-cont.} Incorrectness and Completeness (y-axis) evolution of all skeptical approaches with respect to different percentages of training data sets (x-axis).}
\end{figure} 
\begin{figure}[!th]
	\centering
	\subfigure{
		\includegraphics[width=\linewidth]{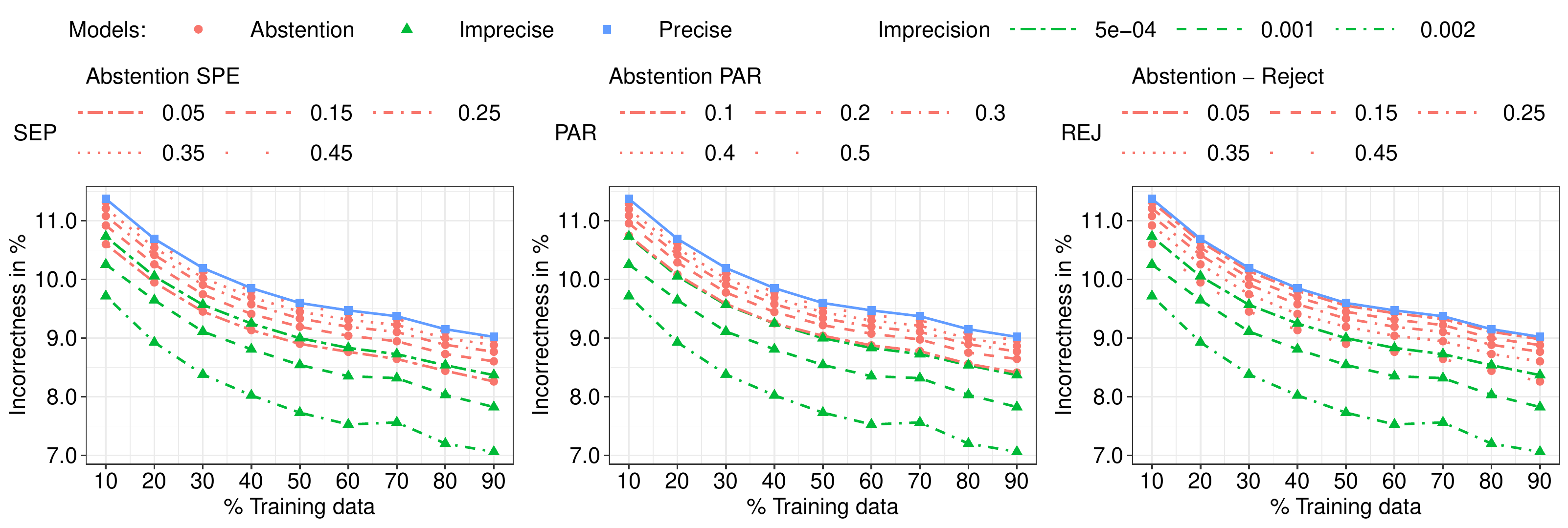}
	}\vspace{-3mm}\qquad%
	\subfigure{
		\includegraphics[width=\linewidth]{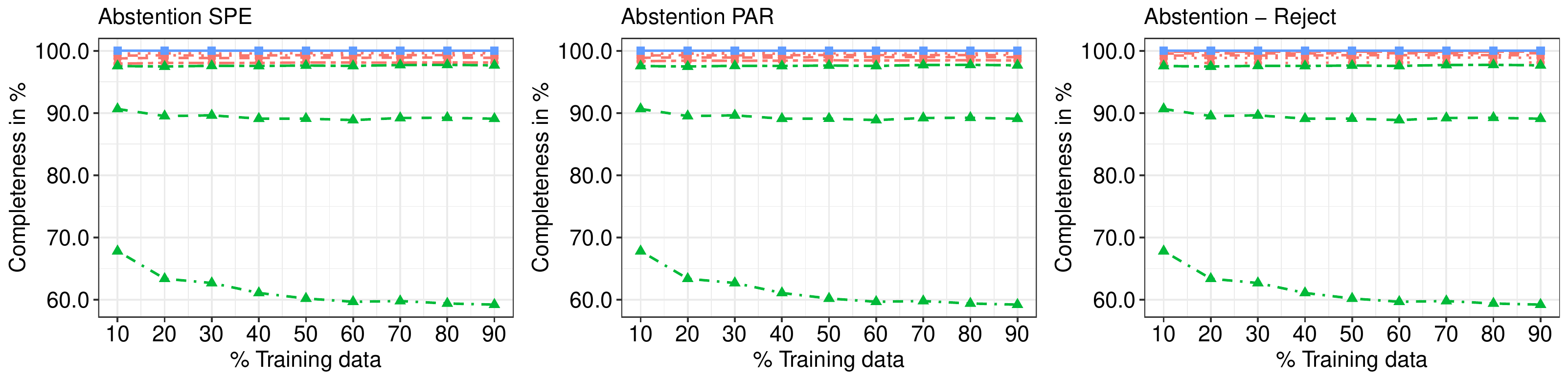}
	}\vspace{-4mm}
	\caption{{\bf Scene - Downsampling - CSPN $\epsilon$-cont.} Incorrectness and Completeness (y-axis) evolution of all skeptical approaches with respect to different percentages of training data sets (x-axis).}
\end{figure} 

%% file: ML_BR_optimal.bbl
\begin{thebibliography}{10}

\bibitem{antonucci2017multilabel}
A.~Antonucci and G.~Corani.
\newblock The multilabel naive credal classifier.
\newblock {\em International Journal of Approximate Reasoning}, 83:320--336,
  2017.

\bibitem{augustin2014introduction}
T.~Augustin, F.~P. Coolen, G.~de~Cooman, and M.~C. Troffaes.
\newblock {\em Introduction to imprecise probabilities}.
\newblock John Wiley \& Sons, 2014.

\bibitem{boutell2004learning}
M.~R. Boutell, J.~Luo, X.~Shen, and C.~M. Brown.
\newblock Learning multi-label scene classification.
\newblock {\em Pattern recognition}, 37(9):1757--1771, 2004.

\bibitem{alarcon2019imprecise}
Y.-C. Carranza-Alarcon and S.~Destercke.
\newblock Imprecise gaussian discriminant classification.
\newblock In {\em International Symposium on Imprecise Probabilities: Theories
  and Applications}, pages 59--67, 2019.

\bibitem{alarcon2021ipsita}
Y.~C. Carranza~Alarc\'on and S.~Destercke.
\newblock Distributionally robust, skeptical binary inferences in multi-label
  problems.
\newblock In A.~Cano, J.~De~Bock, E.~Miranda, and S.~Moral, editors, {\em
  Proceedings of the Twelveth International Symposium on Imprecise Probability:
  Theories and Applications}, volume 147 of {\em Proceedings of Machine
  Learning Research}, pages 51–60--51–60. PMLR, 06--09 Jul 2021.

\bibitem{alarcon2021imprecise}
Y.~C. {Carranza Alarc\'on} and S.~Destercke.
\newblock Imprecise gaussian discriminant classification.
\newblock {\em Pattern Recognition}, 112:107739, 2021.

\bibitem{chen2018robust}
R.~Chen and I.~C. Paschalidis.
\newblock A robust learning approach for regression models based on
  distributionally robust optimization.
\newblock {\em The Journal of Machine Learning Research}, 19(1):517--564, 2018.

\bibitem{corani2015credal}
G.~Corani and A.~Mignatti.
\newblock Credal model averaging for classification: representing prior
  ignorance and expert opinions.
\newblock {\em International Journal of Approximate Reasoning}, 56:264--277,
  2015.

\bibitem{corani2008learning}
G.~Corani and M.~Zaffalon.
\newblock Learning reliable classifiers from small or incomplete data sets: the
  naive credal classifier 2.
\newblock {\em Journal of Machine Learning Research}, 9(Apr):581--621, 2008.

\bibitem{Correia2019Towards}
A.~H. Correia and C.~P.~de Campos.
\newblock {\em Towards scalable and robust sum-product networks}, pages
  409--422.
\newblock Lecture Notes in Computer Science. Springer, Germany, 2019.

\bibitem{dembczynski2012label}
K.~Dembczy{\'n}ski, W.~Waegeman, W.~Cheng, and E.~H{\"u}llermeier.
\newblock On label dependence and loss minimization in multi-label
  classification.
\newblock {\em Machine Learning}, 88(1-2):5--45, 2012.

\bibitem{destercke2014multilabel}
S.~Destercke.
\newblock Multilabel prediction with probability sets: the hamming loss case.
\newblock In {\em Information Processing and Management of Uncertainty in
  Knowledge-Based Systems}, pages 496--505. Springer, 2014.

\bibitem{furnkranz2008multilabel}
J.~F{\"u}rnkranz, E.~H{\"u}llermeier, E.~L. Menc{\'\i}a, and K.~Brinker.
\newblock Multilabel classification via calibrated label ranking.
\newblock {\em Machine Learning}, 73(2):133--153, 2008.

\bibitem{gatterbauer2014oblivious}
W.~Gatterbauer and D.~Suciu.
\newblock Oblivious bounds on the probability of boolean functions.
\newblock {\em ACM Transactions on Database Systems (TODS)}, 39(1):1--34, 2014.

\bibitem{hermans2009imprecise}
F.~Hermans, E.~Quaeghebeur, et~al.
\newblock Imprecise markov chains and their limit behavior.
\newblock {\em Probability in the Engineering and Informational Sciences},
  23(4):597--635, 2009.

\bibitem{hu2018does}
W.~Hu, G.~Niu, I.~Sato, and M.~Sugiyama.
\newblock Does distributionally robust supervised learning give robust
  classifiers?
\newblock In {\em International Conference on Machine Learning}, pages
  2029--2037, 2018.

\bibitem{jain2016extreme}
H.~Jain, Y.~Prabhu, and M.~Varma.
\newblock Extreme multi-label loss functions for recommendation, tagging,
  ranking \& other missing label applications.
\newblock In {\em Proceedings of the 22nd ACM SIGKDD International Conference
  on Knowledge Discovery and Data Mining}, pages 935--944, 2016.

\bibitem{kotlowski2016surrogate}
W.~Kotlowski and K.~Dembczy{\'n}ski.
\newblock Surrogate regret bounds for generalized classification performance
  metrics.
\newblock In {\em Asian Conference on Machine Learning}, pages 301--316, 2016.

\bibitem{koyejo2015consistent}
O.~O. Koyejo, N.~Natarajan, P.~K. Ravikumar, and I.~S. Dhillon.
\newblock Consistent multilabel classification.
\newblock In {\em Advances in Neural Information Processing Systems}, pages
  3321--3329, 2015.

\bibitem{levi1980a}
I.~Levi.
\newblock {\em The Enterprise of Knowledge}.
\newblock MIT Press, London, 1980.

\bibitem{maua2017credal}
D.~D. Mau{\'a}, F.~G. Cozman, D.~Conaty, and C.~P. Campos.
\newblock Credal sum-product networks.
\newblock In {\em Proceedings of the Tenth International Symposium on Imprecise
  Probability: Theories and Applications}, pages 205--216, 2017.

\bibitem{mouhagir2017using}
H.~Mouhagir, V.~Cherfaoui, R.~Talj, F.~Aioun, and F.~Guillemard.
\newblock Using evidential occupancy grid for vehicle trajectory planning under
  uncertainty with tentacles.
\newblock In {\em 2017 IEEE 20th International Conference on Intelligent
  Transportation Systems (ITSC)}, pages 1--7. IEEE, 2017.

\bibitem{park2017general}
J.~Park and S.~Boyd.
\newblock General heuristics for nonconvex quadratically constrained quadratic
  programming.
\newblock {\em arXiv preprint arXiv:1703.07870}, 2017.

\bibitem{peharz2014learning}
R.~Peharz, R.~Gens, and P.~Domingos.
\newblock Learning selective sum-product networks.
\newblock In {\em LTPM workshop}, volume~32, 2014.

\bibitem{pillai2013multi}
I.~Pillai, G.~Fumera, and F.~Roli.
\newblock Multi-label classification with a reject option.
\newblock {\em Pattern Recognition}, 46(8):2256--2266, 2013.

\bibitem{plass2019reliable}
J.~Plass, M.~E. Cattaneo, T.~Augustin, G.~Schollmeyer, and C.~Heumann.
\newblock Reliable inference in categorical regression analysis for
  non-randomly coarsened observations.
\newblock {\em International Statistical Review}, 87(3):580--603, 2019.

\bibitem{poon2011sum}
H.~Poon and P.~Domingos.
\newblock Sum-product networks: A new deep architecture.
\newblock In {\em 2011 IEEE International Conference on Computer Vision
  Workshops (ICCV Workshops)}, pages 689--690. IEEE, 2011.

\bibitem{srivastava2007bayesian}
S.~Srivastava.
\newblock {\em Bayesian Minimum Expected Risk Estimation of Distributions for
  Statistical Learning}.
\newblock University of Washington, 2007.

\bibitem{Troffaes07}
M.~Troffaes.
\newblock Decision making under uncertainty using imprecise probabilities.
\newblock {\em Int. J. of Approximate Reasoning}, 45:17--29, 2007.

\bibitem{tsoumakas2007multi}
G.~Tsoumakas and I.~Katakis.
\newblock Multi-label classification: An overview.
\newblock {\em International Journal of Data Warehousing and Mining (IJDWM)},
  3(3):1--13, 2007.

\bibitem{nguyen2019}
E.~H. Vu-Linh~Nguyen.
\newblock Reliable multilabel classification: Prediction with partial
  abstention.
\newblock In {\em Thirty-Fourth AAAI Conference on Artificial Intelligence},
  2019.

\bibitem{walley1996inferences}
P.~Walley.
\newblock Inferences from multinomial data: learning about a bag of marbles.
\newblock {\em Journal of the Royal Statistical Society: Series B
  (Methodological)}, 58(1):3--34, 1996.

\bibitem{yang2014nested}
G.~Yang, S.~Destercke, and M.-H. Masson.
\newblock Nested dichotomies with probability sets for multi-class
  classification.
\newblock In {\em ECAI}, pages 363--368, 2014.

\bibitem{zaffalon2002naive}
M.~Zaffalon.
\newblock The naive credal classifier.
\newblock {\em Journal of statistical planning and inference}, 105(1):5--21,
  2002.

\end{thebibliography}
